\documentclass[twoside,11pt]{article}
\usepackage{jmlr2e2}

\usepackage{graphics}
\usepackage{times}
\usepackage[boxed,linesnumbered]{algorithm2e}
\usepackage{amsmath}
\usepackage{url}
\usepackage{rotating}
\usepackage{color}

\newtheorem{observation}{Observation}

\newcommand{\hide}[1]{}

\newcommand{\new}[1]{{#1}}

\newcommand{\KronFit}{{\sc KronFit}\xspace}
\newcommand{\KRG}{{Kronecker graphs}\xspace}
\newcommand{\SKRG}{{Stochastic Kronecker graphs}\xspace}
\newcommand{\dataset}[1]{{\sc #1}\xspace}
\DeclareMathOperator*{\argmax}{argmax}

\newcommand{\mat}[1]{{\bf{#1}}}       
\newcommand{\D}[2]{\frac{\partial #2}{\partial #1}} 
\newcommand{\ggraph}{{G}}             
\newcommand{\kgraph}{{K}}             
\newcommand{\nnodes}{N}               
\newcommand{\nedges}{E}               
\newcommand{\nn}[1]{{\nnodes}_{#1}}
\newcommand{\en}[1]{{\nedges}_{#1}}
\newcommand{\kn}[1]{{\kgraph}_{#1}}
\newcommand{\kzero}{\kn{1}}
\newcommand{\nzero}{\nn{1}}
\newcommand{\ezero}{\en{1}}

\newcommand{\perm}{{\sigma}}           
\newcommand{\pn}[1]{{\cal P}_{#1}}     
\newcommand{\pzero}{{\Theta}}          
\newcommand{\thij}[1]{{\theta_{{#1}}}} 
\newcommand{\pmat}{{\cal P}}           
\newcommand{\pij}[1]{{p_{{#1}}}}       

\newcommand{\diam}{D}
\newcommand{\ediam}{D^{*}}

\renewcommand{\cite}{\citep}

\ShortHeadings{Kronecker Graphs: An Approach to Modeling Networks}
{Leskovec, Chakrabarti, Kleinberg, Faloutsos, and Ghahramani}
\firstpageno{1}

\begin{document}

\title{Kronecker graphs: An Approach to Modeling Networks}

\author{\name Jure Leskovec \email jure@cs.stanford.edu \\
\addr Computer Science Department, Stanford University
\AND
\name Deepayan Chakrabarti \email deepay@yahoo-inc.com \\
\addr Yahoo! Research
\AND
\name Jon Kleinberg \email kleinber@cs.cornell.edu \\
\addr Computer Science Department, Cornell University
\AND
\name Christos Faloutsos \email christos@cs.cmu.edu \\
\addr Computer Science Department, Carnegie Mellon University
\AND
\name Zoubin Ghahramani \email zoubin@eng.cam.ac.uk \\
\addr Department of Engineering, University of Cambridge \\
Machine Learning Department, Carnegie Mellon University
}

\editor{}

\maketitle

\begin{abstract}
How can we generate realistic networks? In addition, how can we do so with
a mathematically tractable model that allows for rigorous analysis of
network properties? Real networks exhibit a long list of surprising
properties: Heavy tails for the in- and out-degree distribution, heavy
tails for the eigenvalues and eigenvectors, small diameters, and
densification and shrinking diameters over time.
Current network models and generators either fail to match several of the above
properties, are complicated to analyze mathematically, or both. Here
we propose a generative model for networks that is both
mathematically tractable and can generate networks that have all the above
mentioned structural properties. Our main idea here is to use a
non-standard matrix operation, the {\em Kronecker product}, to generate
graphs which we refer to as ``Kronecker graphs''.

First, we show that Kronecker graphs naturally obey common network
properties. In fact, we rigorously {\em prove} that they do so. We also
provide empirical evidence showing that Kronecker graphs can effectively
model the structure of real networks.

We then present \KronFit, a fast and scalable algorithm for fitting the
Kronecker graph generation model to large real networks. A naive approach
to fitting would take super-exponential time. In contrast, \KronFit\ takes
{\em linear} time, by exploiting the structure of Kronecker matrix
multiplication and by using statistical simulation techniques.

Experiments on a wide range of large real and synthetic networks show that \KronFit\ finds accurate parameters that very well mimic the properties of target
networks. In fact, using just four parameters we can accurately
model several aspects of global network structure.
Once fitted, the model parameters can be used to gain insights
about the network structure, and the resulting synthetic graphs can be
used for null-models, anonymization, extrapolations, and graph
summarization.
\end{abstract}


\section{Introduction}
\label{sec:KronIntro}
What do real graphs look like? How do they evolve over time? How can we
generate synthetic, but realistic looking, time-evolving graphs? Recently,
network analysis has been attracting much interest, with an emphasis on
finding patterns and abnormalities in social networks, computer networks,
e-mail interactions, gene regulatory networks, and many more. Most of the
work focuses on static snapshots of graphs, where fascinating ``laws''
have been discovered, including small diameters and heavy-tailed degree
distributions.

In parallel with discoveries of such structural ``laws'' there has been effort to
find mechanisms and models of network formation that generate networks with such structures. So, a
good realistic network generation model is important for at least two
reasons. The first is that it can generate graphs for extrapolations, hypothesis testing,
``what-if'' scenarios, and simulations, when real graphs are difficult or
impossible to collect. For example, how well will a given protocol run on
the Internet five years from now? Accurate network models can produce more
realistic models for the future Internet, on which simulations can be run.
The second reason is more subtle. It forces us to think about network
properties that generative models should obey to be realistic.

\new{In this paper we introduce Kronecker graphs, a generative network model
which obeys all the main static network patterns that have appeared in the
literature~\cite{faloutsos99powerlaw,barabasi99diameter,chakrabarti04rmat,farkas01spectral,mihail02random,watts98smallworld}.} Our model also obeys recently discovered temporal
evolution patterns~\cite{jure05dpl,jure07evolution}. And, contrary to
other models that match this combination of network properties (as for example,~\cite{bu02internet,klemm02clustered,vazquez03growing,jure05dpl,zheleva09evol}),
Kronecker graphs also lead to tractable analysis and rigorous proofs. Furthermore,
the Kronecker graphs generative process also has a nice natural
interpretation and justification.

Our model is based on a matrix operation, the {\em Kronecker product}.
\new{There are several known theorems on Kronecker products. They correspond
exactly to a significant portion of what we want to prove: heavy-tailed
distributions for in-degree, out-degree, eigenvalues, and eigenvectors.} We
also demonstrate how a \KRG\ can match the behavior of several real
networks (social networks, citations, web, internet, and others). While
Kronecker products have been studied by the algebraic combinatorics
community (see, \emph{e.g.},~\cite{chow97tensor,imrich98factoring,imrich00product,hammack09proof}), the present work is the
first to employ this operation in the design of network models to match
real data.

Then we also make a step further and tackle the following problem: Given a
large real network, we want to generate a synthetic graph, so that the
resulting synthetic graph matches the properties of the real network as
well as possible.

Ideally we would like: (a) A graph generation model that {\em naturally}
produces networks where many properties that are also found in real
networks naturally emerge. (b) The model parameter estimation should be fast and scalable,
so that we can handle networks with millions of nodes. (c) The resulting
set of parameters should generate realistic-looking networks that match
the statistical properties of the target, real networks.

In general the problem of modeling network structure presents several
conceptual and engineering challenges: Which generative model should we
choose, among the many in the literature? How do we measure the goodness
of the fit? (Least squares don't work well for power laws, for subtle
reasons!) If we use likelihood, how do we estimate it faster
than in time quadratic on the number of nodes? How do we solve the node
correspondence problem, i.e., which node of the real network corresponds to what
node of the synthetic one?

To answer the above questions we present \KronFit, a fast and scalable
algorithm for fitting Kronecker graphs by using the maximum likelihood
principle. When calculating the likelihood there are two challenges:
First, one needs to solve the node correspondence problem by matching the
nodes of the real and the synthetic network. Essentially, one has to
consider all mappings of nodes of the network to the rows and columns of
the graph adjacency matrix. This becomes intractable for graphs with more
than tens of nodes. Even when given the ``true'' node correspondences, just
evaluating the likelihood is still prohibitively expensive for large graphs
that we consider, as one needs to evaluate the probability of each possible edge. We present solutions to both of these
problems: We develop a Metropolis sampling algorithm for sampling node
correspondences, and approximate the likelihood to obtain a {\em linear}
time algorithm for Kronecker graph model parameter estimation
that scales to large networks with millions of nodes and
edges. \KronFit gives orders of magnitude speed-ups against older methods
(20 minutes on a commodity PC, versus 2 days on a 50-machine cluster).

Our extensive experiments on synthetic and real networks show that \KRG\
can efficiently model statistical properties of networks, like degree
distribution and diameter, while using only four parameters.

Once the model is fitted to the real network, there are several benefits
and applications:

\begin{enumerate}
\item[(a)] {\em Network structure:}
  the parameters give us insight into the global structure of the
    network itself.
\item[(b)] {\em Null-model:} when working with network data we would
    often like to assess the significance or the extent to which a
    certain network property is expressed. We can use Kronecker graph as an accurate null-model.
\item[(c)] {\em Simulations:} given an algorithm working on a graph we
    would like to evaluate how its performance depends on various
    properties of the network. Using our model one can generate graphs
    that exhibit various combinations of such properties, and then
    evaluate the algorithm.
\item[(d)] {\em Extrapolations:} we can use the model to generate a
    larger graph, to help us understand how the network will look like
    in the future.
\item[(e)] {\em Sampling:} conversely, we can also generate a smaller
    graph, which may be useful for running simulation experiments
    (\emph{e.g.}, simulating routing algorithms in computer networks,
    or virus/worm propagation algorithms), when these algorithms may
    be too slow to run on large graphs.
\item[(f)] {\em Graph similarity:} to compare the similarity of the
    structure of different networks (even of different sizes) one can
    use the differences in estimated parameters as a similarity
    measure.
\item[(g)] {\em Graph visualization and compression:} we can compress
    the graph, by storing just the model parameters, and the
    deviations between the real and the synthetic graph. Similarly,
    for visualization purposes one can use the structure of the
    parameter matrix to visualize the backbone of the network, and
    then display the edges that deviate from the backbone structure.
\item[(h)] {\em Anonymization:} suppose that the real graph cannot be
    publicized, like, \emph{e.g.}, corporate e-mail network or
    customer-product sales in a recommendation system. Yet, we would
    like to share our network. Our work gives ways to such a
    realistic, 'similar' network.
\end{enumerate}

The current paper builds on our previous work on Kronecker
graphs~\cite{jure05kronecker,jure07kronfit} and is organized as follows:
Section~\ref{sec:KronRelated} briefly surveys the related literature. In
section \ref{sec:KronProposed} we introduce the Kronecker graph model,
and give formal statements about the properties of networks it generates.
We investigate the model using simulation in Section
\ref{sec:KronSimulation} and continue by introducing \KronFit, the
Kronecker graphs parameter estimation algorithm, in Section
~\ref{sec:KronKronFit}. We present experimental results on a wide range of real and
synthetic networks in Section~\ref{sec:KronExperiments}. We close with
discussion and conclusions in sections~\ref{sec:KronDiscussion}
and~\ref{sec:KronConclusion}.

\section{Relation to previous work on network modeling}
\label{sec:KronRelated}
Networks across a wide range of domains present surprising regularities, such as power laws, small diameters, communities, and so on. We use these patterns as sanity checks, that is, our synthetic graphs should match those properties of the real target graph.

Most of the related work in this field has concentrated on two aspects: properties and patterns found in real-world networks, and then ways to find models to build understanding about the emergence of these properties. First, we will discuss the commonly found patterns in (static and temporally evolving) graphs, and finally, the state of the art in graph generation methods.

\subsection{Graph Patterns}

Here we briefly introduce the network patterns (also referred to as properties or statistics) that we will later use to compare the similarity between the real networks and their synthetic counterparts produced by the Kronecker graphs model. While many patterns have been discovered, two of the principal ones are heavy-tailed degree distributions and small diameters.

{\em Degree distribution:} The degree-distribution of a graph is a power law if the number of nodes $\nnodes_d$ with degree $d$ is given by $\nnodes_d \propto d^{-\gamma}\quad (\gamma>0)$ where $\gamma$ is called the power law exponent. Power laws have been found in the Internet~\cite{faloutsos99powerlaw}, the Web~\cite{kleinberg99web,broder00bowtie}, citation graphs~\cite{redner98citation}, online social networks~\cite{chakrabarti04rmat} and many others.

{\em Small diameter:} Most real-world graphs exhibit relatively small diameter (the ``small- world'' phenomenon, or ``six degrees of separation''~\cite{milgram67smallworld}): A graph has diameter $\diam$ if every pair of nodes can be connected by a path of length at most $\diam$ edges. The diameter $\diam$ is susceptible to outliers. Thus, a more robust measure of the pair wise distances between nodes in a graph is the {\em \new{integer} effective diameter}~\cite{tauro01topology}, which is the minimum number of links (steps/hops) in which some fraction (or quantile $q$, say $q = 0.9$) of all connected pairs of nodes can reach each other.
\new{Here we make use of {\em effective diameter} which we define as follows~\cite{jure05dpl}. For each natural number $h$, let $g(h)$ denote the fraction of connected node pairs whose shortest connecting path has length at most $h$,
\emph{i.e.}, at most $h$ hops away. We then consider a function defined over all positive real numbers $x$ by linearly interpolating between the points $(h,g(h))$ and $(h+1,g(h+1))$ for each $x$, where $h=\lfloor x \rfloor$, and we define the {\em effective diameter} of the network to be the value $x$ at which the function $g(x)$ achieves the value 0.9.} The effective diameter has been found to be small for large real-world graphs, like Internet, Web, and online social networks~\cite{albert02statistical,milgram67smallworld,jure05dpl}.

{\em Hop-plot:} It extends the notion of diameter by plotting the number of reachable pairs $g(h)$ within $h$ hops, as a function of the number of hops $h$~\cite{palmer02anf}. It gives us a sense of how quickly nodes' neighborhoods expand with the number of hops.

{\em Scree plot:} This is a plot of the eigenvalues (or singular values) of the graph adjacency matrix, versus their rank, using the logarithmic scale. The scree plot is also often found to approximately obey a power law~\cite{chakrabarti04rmat,farkas01spectral}. Moreover, this pattern was also found analytically for random power law graphs~\cite{mihail02random,chung03eigenvalues}.

{\em Network values:} The distribution of eigenvector components (indicators of ``network value'') associated to the largest eigenvalue of the graph adjacency matrix has also been found to be skewed \cite{chakrabarti04rmat}.

{\em Node triangle participation:} \new{Edges in real-world networks and especially in social networks tend to cluster~\cite{watts98smallworld} and form triads of connected nodes. Node triangle participation is a measure of transitivity in networks. It counts the number of triangles a node participates in, {\em i.e.}, the number of connections between the neighbors of a node. The plot of the number of triangles $\Delta$ versus the number of nodes that participate in $\Delta$ triangles has also been found to be skewed~\cite{tsourakakis08triangles}.}

{\em Densification power law:} The relation between the number of edges $\nedges(t)$ and the number of nodes $\nnodes(t)$ in evolving network at time $t$ obeys the {\em densification power law} (DPL), which states that $\nedges(t) \propto \nnodes(t)^a$. The {\em densification exponent} $a$ is typically greater than $1$, implying that the average degree of a node in the network is {\em increasing} over time (as the network gains more nodes and edges). This means that real networks tend to sprout many more edges than nodes, and thus densify as they grow~\cite{jure05dpl,jure07evolution}.

{\em Shrinking diameter:} The effective diameter of graphs tends to shrink or stabilize as the number of nodes in a network grows over time~\cite{jure05dpl,jure07evolution}. This is somewhat counterintuitive since from common experience as one would expect that as the volume of the object (a graph) grows, the size (\emph{i.e.}, the diameter) would also grow. But for real networks this does not hold as the diameter shrinks and then seems to stabilize as the network grows.

\subsection{Generative models of network structure}

The earliest probabilistic generative model for graphs was the Erd\H{o}s-R\'{e}nyi~\cite{erdos60random} random graph model, where each pair of nodes has an identical, independent probability of being joined by an edge. The study of this model has led to a rich mathematical theory. However, as the model was not developed to model real-world networks it produces graphs that fail to match real networks in a number of respects (for example, it does not produce heavy-tailed degree distributions).

The vast majority of recent network models involve some form of \new{{\em preferential attachment} \cite{barabasi99emergence,albert02statistical,winick02inet,kleinberg99web,kumar99extracting,flaxman07geometric} } that employs a simple rule: new node joins the graph at each time step, and then creates a connection to an existing node $u$ with the probability proportional to the degree of the node $u$. This leads to the ``rich get richer'' phenomena and to power law tails in degree distribution. However, the diameter in this model grows slowly with the number of nodes $\nnodes$, which violates the ``shrinking diameter'' property mentioned above.

There are also many variations of preferential attachment model, all somehow employing the ``rich get richer'' type mechanism, e.g., the ``copying model''~\cite{kumar00stochastic}, the ``winner does not take all'' model~\cite{pennock02winners}, the ``forest fire'' model~\cite{jure05dpl}, the ``random surfer model'' \cite{blum06surfer}, etc.

A different family of network methods strives for small diameter and local clustering in networks. Examples of such models include the {\em small-world} model~\cite{watts98smallworld} and the Waxman generator~\cite{waxman88routing}. Another family of models shows that heavy tails emerge if nodes try to optimize their connectivity under resource constraints~\cite{carlson99hot,fabrikant02hot}.

In summary, most current models focus on modeling only one (static) network property, and neglect the others. In addition, it is usually hard to analytically analyze properties of the network model. On the other hand, the Kronecker graph model we describe in the next section addresses these issues as it matches multiple properties of real networks at the same time, while being analytically tractable and lending itself to rigorous analysis.

\subsection{Parameter estimation of network models}

\new{Until recently relatively little effort was made to fit the above network models to real data. One of the difficulties is that most of the above models usually
define a mechanism or a principle by which a network is constructed, and thus parameter estimation is either trivial or almost impossible.}

Most work in estimating network models comes from the area of social sciences, statistics and social network analysis where the {\em exponential random graphs}, also known as  $p*$ model, were introduced \cite{wasserman96pstar}. The model essentially defines a log linear model over all possible graphs $G$, $p(G|\theta) \propto \exp(\theta^{T}s(G))$, where $G$ is a graph, and $s$ is a set of functions, that can be viewed as summary statistics for the structural features of the network. The $p*$ model usually focuses on ``local'' structural features of networks (like, \emph{e.g.}, characteristics of nodes that determine a presence of an edge, link reciprocity, etc.). As exponential random graphs have been very useful for modeling small networks, and individual nodes and edges, our goal here is different in a sense that we aim to accurately model the structure of the network as a whole. Moreover, we aim to model and estimate parameters of networks with millions of nodes, while even for graphs of small size ($>100$ nodes) the number of model parameters in exponential random graphs usually becomes too large, and estimation prohibitively expensive, both in terms of computational time and memory.

Regardless of a particular choice of a network model, a common theme when estimating the likelihood $P(G)$ of a graph $G$ under some model is the challenge of finding the correspondence between the nodes of the true network and its synthetic counterpart. The node correspondence problem results in the factorially many possible matchings of nodes. One can think of the correspondence problem as a test of graph isomorphism. Two isomorphic graphs $G$ and $G'$ with differently assigned node IDs should have same likelihood $P(G)=P(G')$ so we aim to find an accurate mapping between the nodes of the two graphs.

An ordering or a permutation defines the mapping of nodes in one network to nodes in the other network. For example, Butts~\cite{butts05permutation} used permutation sampling to determine similarity between two graph adjacency matrices, while Bez{\'a}kov{\'a} {\em et al.}~\cite{bezakova06mle} used permutations for graph model selection. Recently, an approach for estimating parameters of the ``copying'' model was introduced~\cite{wiuf06likelihood}, however authors also note that the class of ``copying'' models may not be rich enough to accurately model real networks. As we show later, Kronecker graph model seems to have the necessary expressive power to mimic real networks well.
\section{Kronecker graph model}
\label{sec:KronProposed}
The Kronecker graph model we propose here is based on a recursive
construction. Defining the recursion properly is somewhat subtle, as a
number of standard, related graph construction methods fail to produce
graphs that densify according to the patterns observed in real networks,
and they also produce graphs whose diameters increase. To produce
densifying graphs with constant/shrinking diameter, and thereby match the
qualitative behavior of a real network, we develop a procedure that is
best described in terms of the {\em Kronecker product} of matrices.

\begin{table}[t]
  \begin{center}
  \begin{tabular}{l||l}
    {\sc Symbol} & {\sc Description}\\
    \hline\hline
    $\ggraph$ & Real network\\
    $\nnodes $ & Number of nodes in $\ggraph$ \\
    $\nedges $ & Number of edges in $\ggraph$ \\
    $\kgraph $ & Kronecker graph (synthetic estimate of $G$) \\
    $\kzero$  & Initiator of a \KRG \\
    $\nzero$ & Number of nodes in initiator $\kzero$ \\
    $\ezero$ & Number of edges in $\kzero$ \new{(the expected number of edges in $\pn{1}$, $\ezero=\sum \thij{ij}$)} \\
    $ G \otimes H $ & Kronecker product of adjacency matrices of graphs $G$ and $H$\\
    $\kn{1}^{[k]}=\kn{k} = K$& $k^{th}$ Kronecker power of $\kn{1}$ \\
    $\kn{1}[i,j]$ & Entry at row $i$ and column $j$ of $\kn{1}$ \\
    $\pzero = \pn{1}$ & Stochastic Kronecker initiator \\
    $\pn{1}^{[k]}=\pn{k} = \pmat$ & $k^{th}$ Kronecker power of $\pn{1}$\\
    $\thij{ij}=\pn{1}[i,j]$ & Entry at row $i$ and column $j$ of $\pn{1}$\\
    $\pij{ij} = \pn{k}[i,j]$ & Probability of an edge $(i,j)$ in $\pn{k}$, \emph{i.e.}, entry at
row $i$ and column $j$ of $\pn{k}$\\
    $K = R(\pmat)$ & Realization of a Stochastic Kronecker graph $\pmat$ \\
    $l(\pzero)$ & Log-likelihood. Log-prob. that $\pzero$ generated real graph $G$, $\log
P(\ggraph | \pzero)$\\
    $\hat\pzero$ & Parameters at maximum likelihood, $\hat{\pzero}=\argmax_\pzero P(G|\pzero)$
\\
    $\perm$  & Permutation that maps node IDs of $G$ to those of $\pmat$\\
    $a$      & Densification power law exponent, $\nedges(t) \propto \nnodes(t)^a$ \\
    $\diam$  & Diameter of a graph \\
    $\nnodes_c$ & Number of nodes in the largest weakly connected component of a graph\\
    $\omega$ & Proportion of times {\tt SwapNodes} permutation proposal distribution is used\\
    \end{tabular}
    \caption{Table of symbols.}
    \label{tab:KronSym}
  \end{center}
\end{table}

\subsection{Main idea}

The main intuition behind the model is to create self-similar graphs,
recursively. We begin with an {\em initiator} graph $\kn{1}$, with
$\nzero$ nodes and $\ezero$ edges, and by recursion we produce
successively larger graphs $\kn{2}, \kn{3}, \ldots$ such that the $k^{\rm
th}$ graph $\kn{k}$ is on $\nn{k} = \nzero ^ k$ nodes. If we want these
graphs to exhibit a version of the Densification power
law~\cite{jure05dpl}, then $\kn{k}$ should have \(\en{k} = \ezero^k\)
edges. This is a property that requires some care in order to get right,
as standard recursive constructions (for example, the traditional
Cartesian product or the construction of~\cite{barabasi01deterministic})
do not yield graphs satisfying the densification power law.

It turns out that the {\em Kronecker product} of two matrices is the right
tool for this goal. The Kronecker product is defined as follows:

\begin{definition}[Kronecker product of matrices]
Given two matrices \mbox{$\mat{A}=\left[a_{i,j}\right]$} and $\mat{B}$ of
sizes $n \times m$ and $n' \times m'$ respectively, the Kronecker product
matrix $\mat{C}$ of dimensions $(n \cdot n') \times (m \cdot m')$ is given
by
\begin{equation}
   \mat{C} = \mat{A} \otimes \mat{B} \doteq
   \left(\begin{array}{cccc}
     a_{1,1} \mat{B} & a_{1,2} \mat{B} & \ldots  & a_{1,m} \mat{B}\\
     a_{2,1} \mat{B} & a_{2,2} \mat{B} & \ldots  & a_{2,m} \mat{B}\\
     \vdots  & \vdots  & \ddots  & \vdots \\
     a_{n,1} \mat{B} & a_{n,2} \mat{B} & \ldots  & a_{n,m} \mat{B}\\
   \end{array}\right)
\end{equation}
\end{definition}

We then define the Kronecker product of two graphs simply as the Kronecker
product of their corresponding adjacency matrices.

\begin{definition}[Kronecker product of graphs~\cite{weichsel62kronecker}]
If $G$ and $H$ are graphs with adjacency matrices $A(G)$ and $A(H)$
respectively, then the Kronecker product $G \otimes H$ is defined as the
graph with adjacency matrix $A(G) \otimes A(H)$.
\end{definition}

\begin{figure}[t]
\begin{center}
  \begin{tabular}{ccc}
  \raisebox{0.35in}{\includegraphics[height=0.8in]{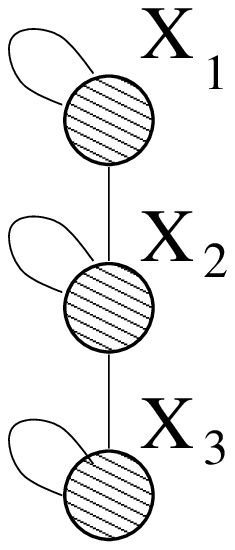}} &
  \raisebox{0.3in}{\includegraphics[height=0.8in]{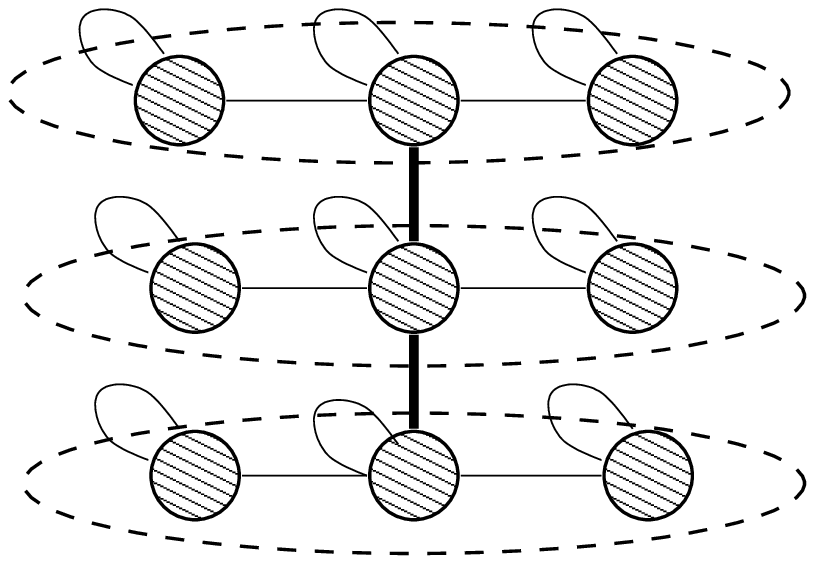}}
  & \includegraphics[height=1.2in]{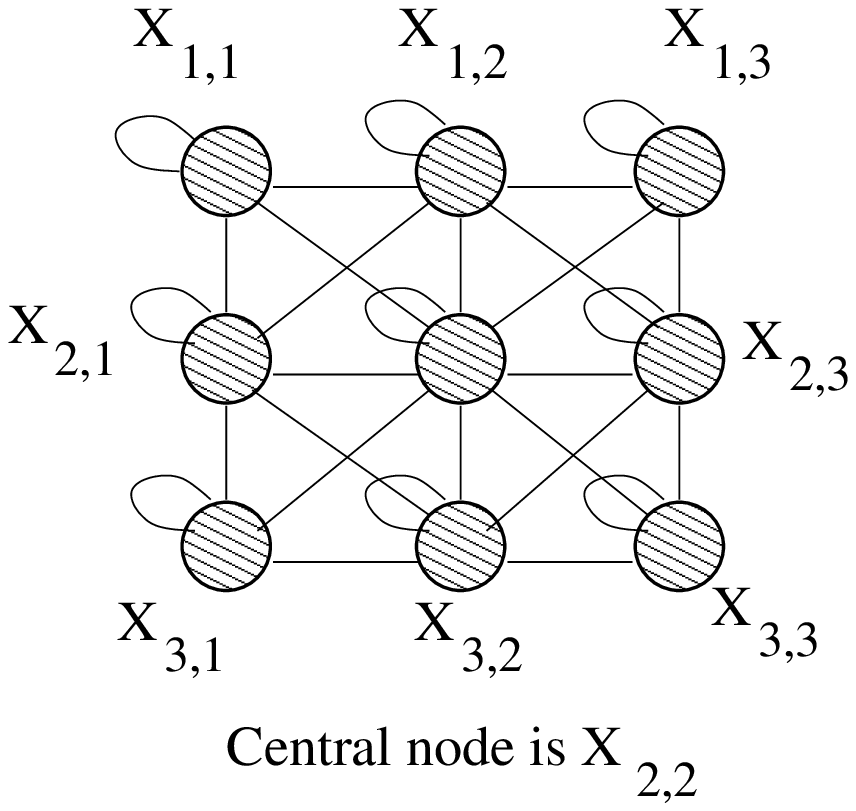} \\
  (a) Graph $\kn{1}$ & (b) Intermediate stage & (c) Graph
  $\kn{2}=\kn{1}\otimes \kn{1}$ \\
  \includegraphics[width=0.17\textwidth]{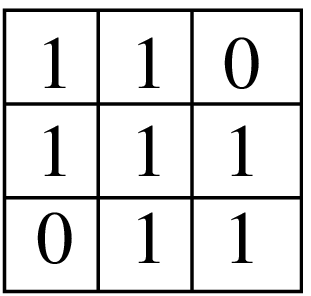} & &
  \includegraphics[width=0.17\textwidth]{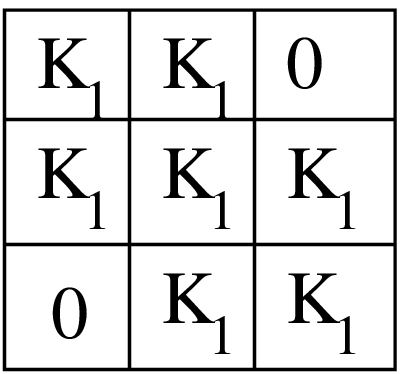} \\
  (d) Adjacency matrix & & (e) Adjacency matrix \\
  \hspace{1.5em} of $\kn{1}$ & & \hspace{2em}of $\kn{2}=\kn{1}\otimes \kn{1}$ \\
  \end{tabular}
  \caption{{\em Example of Kronecker multiplication:} Top: a ``3-chain''
  initiator graph and its Kronecker product with itself. Each of the $X_i$
  nodes gets expanded into $3$ nodes, which are then linked using
  Observation~\ref{obs:KronEdges}.  Bottom row: the corresponding
  adjacency matrices. See figure~\ref{fig:KronSpy3chain} for adjacency
  matrices of $\kn{3}$ and $\kn{4}$.}
  \label{fig:KronSpyplots}
 \end{center}
\end{figure}

\begin{figure}[t]
  \begin{center}
  \begin{tabular}{ccc}
    \includegraphics[width=0.4\textwidth]{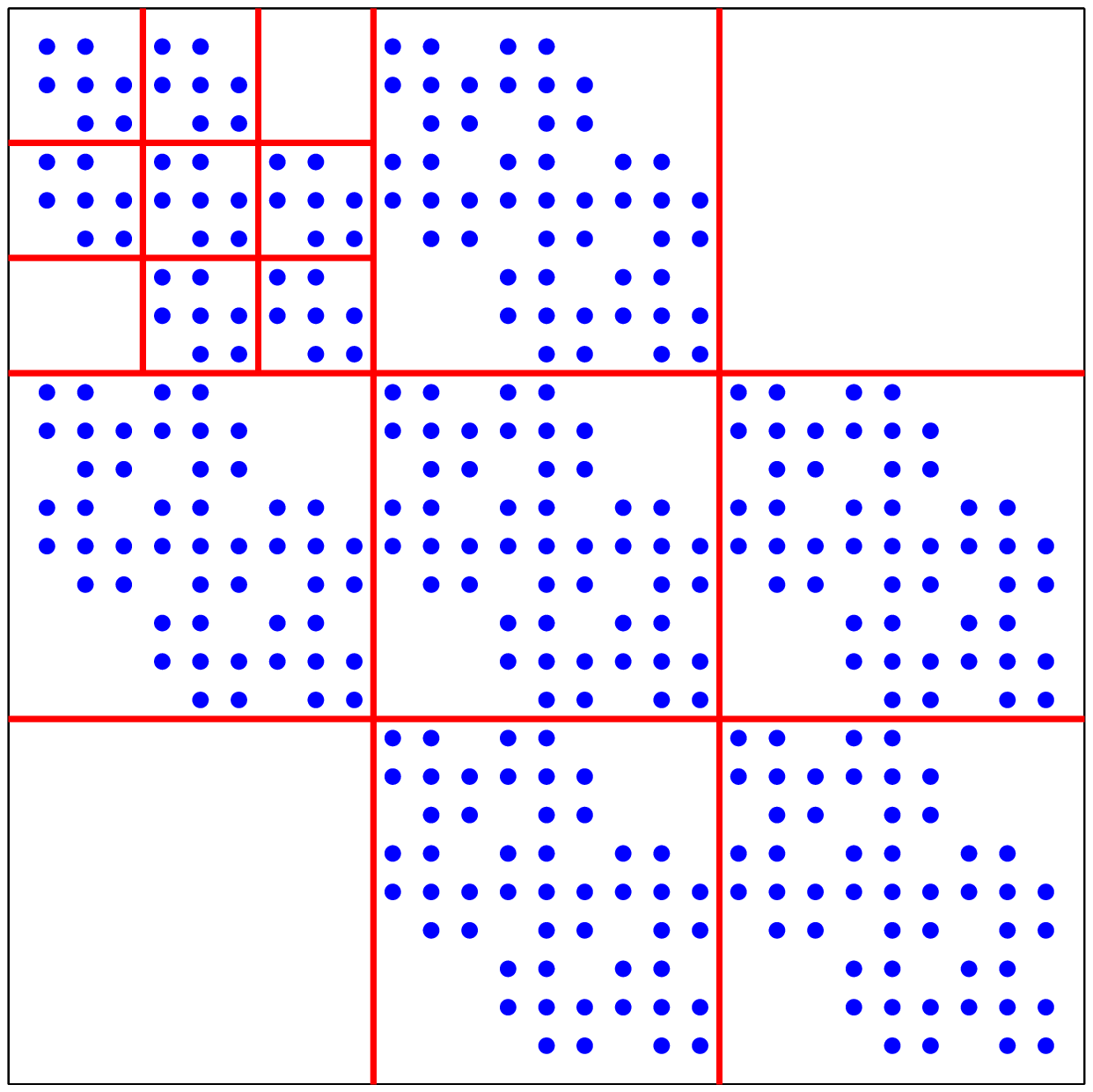} & \hspace{1cm} &
    \includegraphics[width=0.4\textwidth]{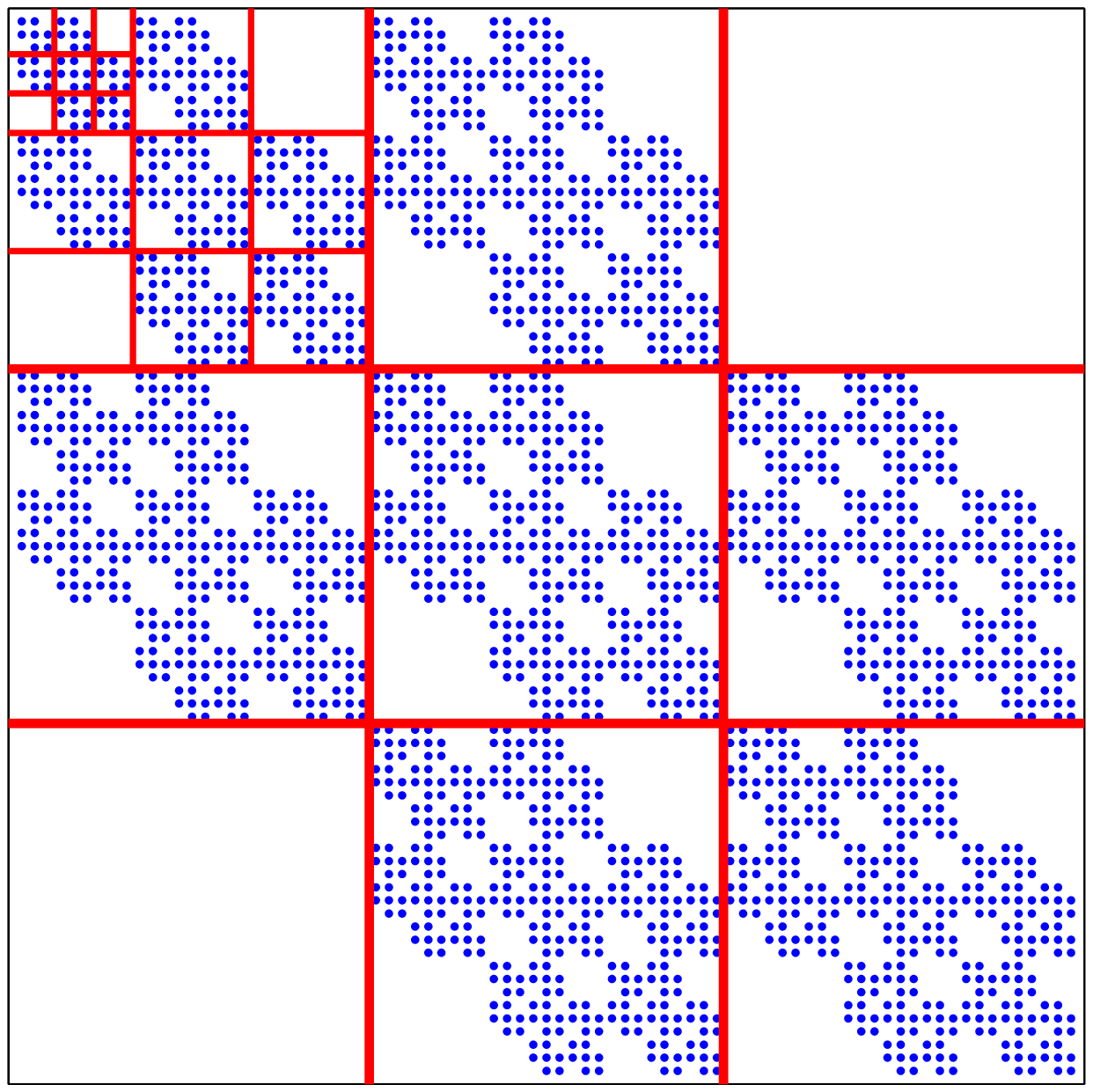} \\
    (a) $\kn{3}$ adjacency matrix ($27 \times 27$) & & (b) $\kn{4}$ adjacency matrix ($81\times
81$)\\
  \end{tabular}
  \end{center}
  \caption{Adjacency matrices of $\kn{3}$ and $\kn{4}$, the $3^{rd}$ and
  $4^{th}$ Kronecker power of $\kn{1}$ matrix as defined in
  Figure~\ref{fig:KronSpyplots}. Dots represent non-zero matrix entries,
  and white space represents zeros. Notice the recursive self-similar
  structure of the adjacency matrix.}
  \label{fig:KronSpy3chain}
\end{figure}

\begin{observation}[Edges in Kronecker-multiplied graphs]
  \begin{eqnarray}
  \mbox{Edge~}(X_{ij}, X_{kl}) \in G \otimes H \mbox{~iff~} (X_i,X_k)\in G
  \mbox{~and~} (X_j,X_l)\in H\nonumber
  \end{eqnarray}
  where $X_{ij}$ and $X_{kl}$ are nodes in $G\otimes H$, and $X_i$,
  $X_j$, $X_k$ and $X_l$ are the corresponding nodes in $G$ and $H$,
  as in Figure~\ref{fig:KronSpyplots}. \label{obs:KronEdges}
\end{observation}

The last observation is crucial, and deserves elaboration.
Basically, each node in $G \otimes H$ can be represented as an ordered
pair $X_{ij}$, with $i$ a node of $G$ and $j$ a node of $H$, and with an
edge joining $X_{ij}$ and $X_{kl}$ precisely when $(X_i, X_k)$ is an edge
of $G$ and $(X_j,X_l)$ is an edge of $H$. This is a direct consequence of
the hierarchical nature of the Kronecker product.
Figure~\ref{fig:KronSpyplots}(a--c) further illustrates this by showing
the recursive construction of $G \otimes H$, when $G = H$ is a 3-node
chain. Consider node $X_{1,2}$ in Figure~\ref{fig:KronSpyplots}(c):
It belongs to the $H$ graph that replaced node $X_1$ (see
Figure~\ref{fig:KronSpyplots}(b)), and in fact is the $X_2$ node
(\emph{i.e.}, the center) within this small $H$-graph.

We propose to produce a growing sequence of matrices by iterating the
Kronecker product:

\begin{definition}[Kronecker power]
  The $k^{th}$ power of $\kn{1}$ is defined as the matrix $\kn{1}^{[k]}$
  (abbreviated to $\kn{k}$), such that:
  \begin{eqnarray}
  \kn{1}^{[k]} = \kn{k} = \mbox{\raisebox{-0.75em}{$\begin{array}{c}
    \underbrace{\kn{1} \otimes \kn{1} \otimes \ldots \kn{1}} \\
    k \mbox{~times}
        \end{array}$}}
    ~ = ~ \kn{k-1} \otimes \kn{1}
        \nonumber
  \end{eqnarray}
\end{definition}

\begin{definition}[Kronecker graph]
  Kronecker graph of order $k$ is defined by the adjacency matrix
  $\kn{1}^{[k]}$, where $\kn{1}$ is the Kronecker initiator adjacency matrix.
\end{definition}

The self-similar nature of the Kronecker graph product is clear: To
produce $\kn{k}$ from $\kn{k-1}$, we ``expand'' (replace) each node of
$\kn{k-1}$ by converting it into a copy of $\kn{1}$, and we join these
copies together according to the adjacencies in $\kn{k-1}$ (see
Figure~\ref{fig:KronSpyplots}). This process is very natural: one can
imagine it as positing that communities within the graph grow recursively,
with nodes in the community recursively getting expanded into miniature
copies of the community. Nodes in the sub-community then link among
themselves and also to nodes from other communities.

\new{Note that there are many different names to refer to Kronecker product of graphs.
Other names for the Kronecker product are tensor product,
categorical product, direct product, cardinal product,
relational product, conjunction, weak direct product
or just product, and even Cartesian product~\cite{imrich00product}.}

\begin{figure}[t]
  \begin{center}
  \begin{tabular}{clc}
    \raisebox{1cm}{\includegraphics[width=0.2\textwidth]{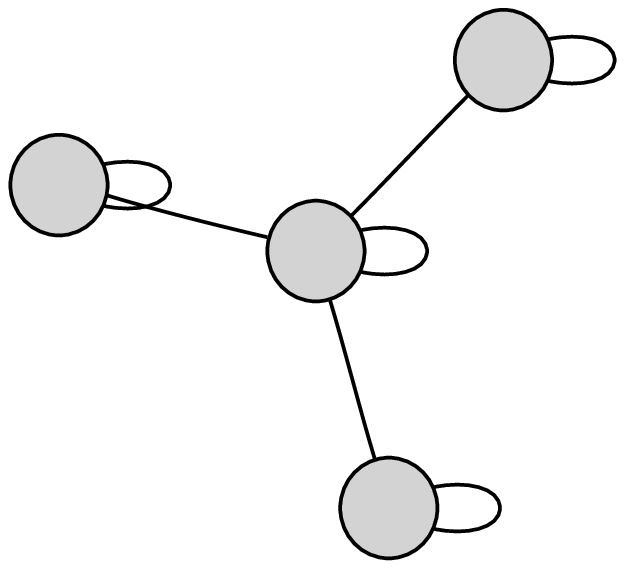}\hspace{1cm}}&
    \raisebox{1cm}{\includegraphics[width=3cm,height=2.8cm]{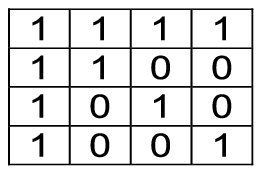}} \hspace{1cm} &
    \vspace{-1cm}\hspace{-1.5cm}{\includegraphics[width=0.4\textwidth]{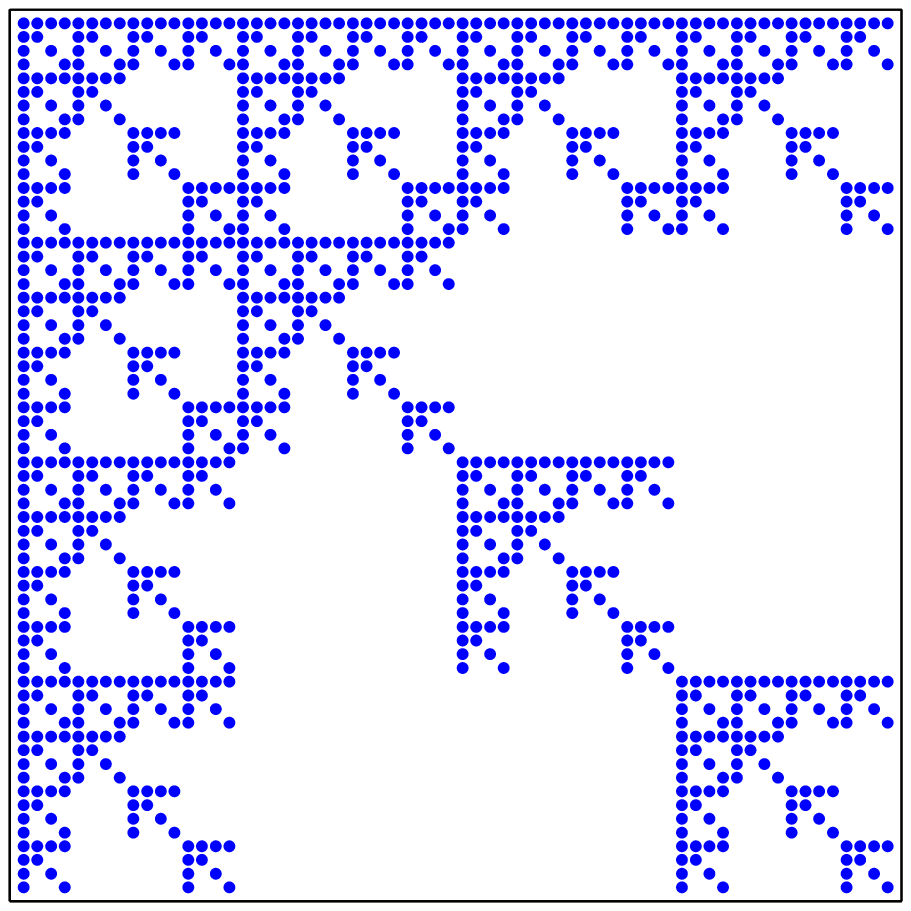}} \\
    & & \\
    \raisebox{1.5cm}{\includegraphics[width=0.2\textwidth]{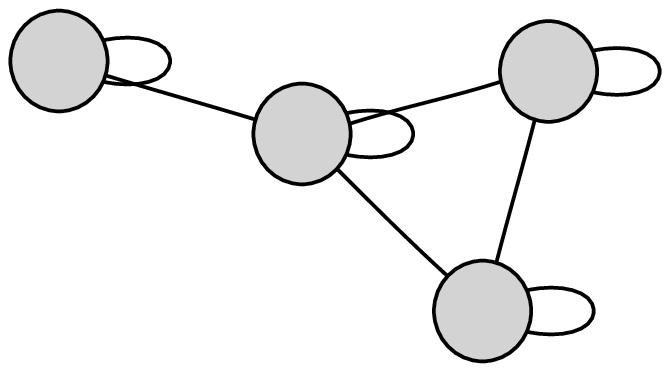}\hspace{1cm}}&
    \raisebox{1cm}{\includegraphics[width=3cm,height=2.8cm]{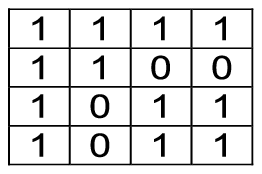}} \hspace{1cm} &
    \hspace{-1.5cm}{\includegraphics[width=0.4\textwidth]{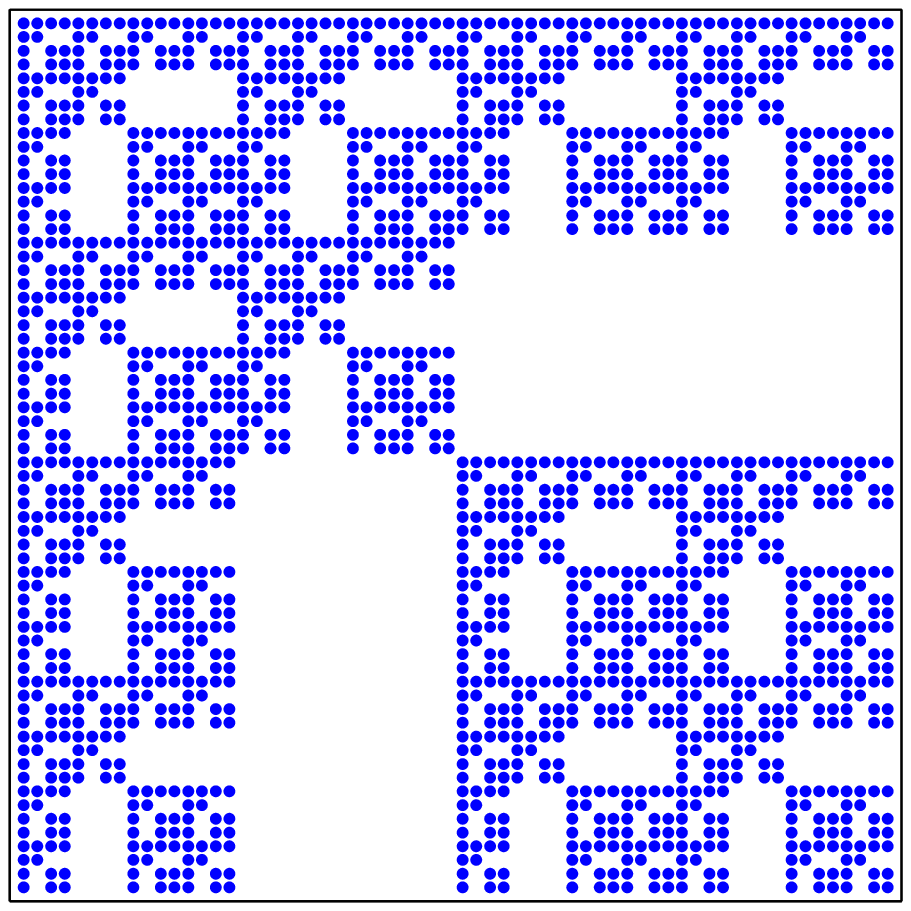}}\vspace{-0.5cm} \\
    Initiator $\kn{1}$ & $\kn{1}$ adjacency matrix & \hspace{-1.3cm}$\kn{3}$ adjacency matrix \\
  \end{tabular}
  \end{center}
  \caption{Two examples of Kronecker initiators on 4 nodes and
  the self-similar adjacency matrices they produce.}
  \label{fig:KronSpy4chain}
\end{figure}

\subsection{Analysis of Kronecker graphs}

We shall now discuss the properties of Kronecker graphs, specifically,
their degree distributions, diameters, eigenvalues, eigenvectors, and
time-evolution. Our ability to prove analytical results about all of these
properties is a major advantage of Kronecker graphs over other network
models.

\subsubsection{Degree distribution}

The next few theorems prove that several distributions of interest are
{\em multinomial} for our Kronecker graph model. This is important,
because a careful choice of the initial graph $\kn{1}$ makes the resulting
multinomial distribution to behave like a power law or Discrete Gaussian Exponential (DGX)
distribution~\cite{bi01dgx,clauset07powerlaw}.

\begin{theorem}[Multinomial degree distribution]
  Kronecker graphs have multinomial degree distributions, for both in-
  and out-degrees. \label{lem:KronDegdist}
\end{theorem}

\begin{proof}
  Let the initiator $\kn{1}$ have the degree sequence $d_1, d_2, \ldots,
  d_{N_1}$. Kronecker multiplication of a node with degree $d$ expands
  it into $N_1$ nodes, with the corresponding degrees being $d\times
  d_1, d\times d_2, \ldots, d\times d_{N_1}$. After Kronecker powering,
  the degree of each node in graph $\kn{k}$ is of the form $ d_{i_1}\times
  d_{i_2} \times \ldots d_{i_k}$, with $i_1, i_2,\ldots, i_k \in
  (1\ldots N_1)$, and there is one node for each ordered combination.
  This gives us the multinomial distribution on the degrees of $\kn{k}$.
  So, graph $\kn{k}$ will have multinomial degree distribution where
  the ``events'' (degrees) of the distribution will
  be combinations of degree products: $d_1^{i_1}d_2^{i_2}\dots
  d_{N_1}^{i_{N_1}}$ (where $\sum_{j=1}^{N_1} i_j = k$) and event (degree)
  probabilities will be proportional to ${k \choose i_1 i_2 \dots i_{N_1}}$.
  Note also that this is equivalent to noticing that
  the degrees of nodes in $\kn{k}$ can be expressed as the
  $k^{th}$ Kronecker power of the vector $(d_1, d_2, \dots, d_{N_1})$.
\end{proof}

\subsubsection{Spectral properties}

Next we analyze the spectral properties of adjacency matrix of a Kronecker
graph. We show that both the distribution of eigenvalues and the
distribution of component values of eigenvectors of the graph adjacency
matrix follow multinomial distributions.

\begin{theorem}[Multinomial eigenvalue distribution]
  The Kronecker graph $\kn{k}$ has a multinomial distribution for its
  eigenvalues.
\end{theorem}

\begin{proof}
  Let $\kn{1}$ have the eigenvalues $\lambda_{1}, \lambda_{2}, \ldots,
  \lambda_{N_1}$. By properties of the Kronecker
  multiplication~\cite{vanloan00ubiquitous,langville04kronecker}, the
  eigenvalues of $\kn{k}$ are the $k^{th}$ Kronecker power of the vector of
  eigenvalues of the initiator matrix,
  $(\lambda_{1}, \lambda_{2}, \ldots, \lambda_{N_1})^{[k]}$. As in
  Theorem~\ref{lem:KronDegdist}, the eigenvalue distribution is a
  multinomial.
\end{proof}

A similar argument using properties of Kronecker matrix multiplication
shows the following.

\begin{theorem}[Multinomial eigenvector distribution]
  The components of each eigenvector of the Kronecker graph $\kn{k}$
  follow a multinomial distribution.
\end{theorem}

\begin{proof}
  Let $\kn{1}$ have the eigenvectors $\vec{v}_{1}, \vec{v}_{2}, \ldots,
  \vec{v}_{N_1}$. By properties of the Kronecker
  multiplication~\cite{vanloan00ubiquitous,langville04kronecker}, the
  eigenvectors of $\kn{k}$ are given by the $k^{th}$ Kronecker power of
  the vector: $(\vec{v}_{1}, \vec{v}_{2}, \ldots, \vec{v}_{\nzero})$,
  which gives a multinomial distribution for the components of each
  eigenvector in $\kn{k}$.
\end{proof}

We have just covered several of the static graph patterns. Notice that the
proofs were a direct consequences of the Kronecker multiplication
properties.

\subsubsection{Connectivity of Kronecker graphs}

We now present a series of results on the connectivity of Kronecker
graphs. We show, maybe a bit surprisingly, that even if a Kronecker
initiator graph is connected its Kronecker power can in fact be
disconnected.

\begin{lemma}
  If at least one of $G$ and $H$ is a disconnected graph, then
  $G \otimes H$ is also disconnected.
\end{lemma}

\begin{proof}
Without loss of generality we can assume that $G$ has two connected
components, while $H$ is connected. Figure~\ref{fig:KronConnect}(a)
illustrates the corresponding adjacency matrix for $G$. Using the notation
from Observation~\ref{obs:KronEdges} let graph let $G$ have nodes $X_1,
\ldots, X_n$, where nodes $\{X_1, \ldots X_r\}$ and $\{X_{r+1}, \ldots,
X_n\}$ form the two connected components. Now, note that $(X_{ij}, X_{kl})
\notin G \otimes H$ for $i \in \{1, \ldots, r\}$, $k \in \{r+1, \ldots,
n\}$, and all $j$, $l$. This follows directly from
Observation~\ref{obs:KronEdges} as $(X_i, X_k)$ are not edges in $G$.
Thus, $G \otimes H$ must at least two connected components.
\end{proof}

Actually it turns out that both $G$ and $H$ can be connected while $G
\otimes H$ is disconnected. The following theorem analyzes this
case.

\begin{figure}[t]
  \begin{center}
  \begin{tabular}{ccccc}
    \includegraphics[width=0.24\textwidth]{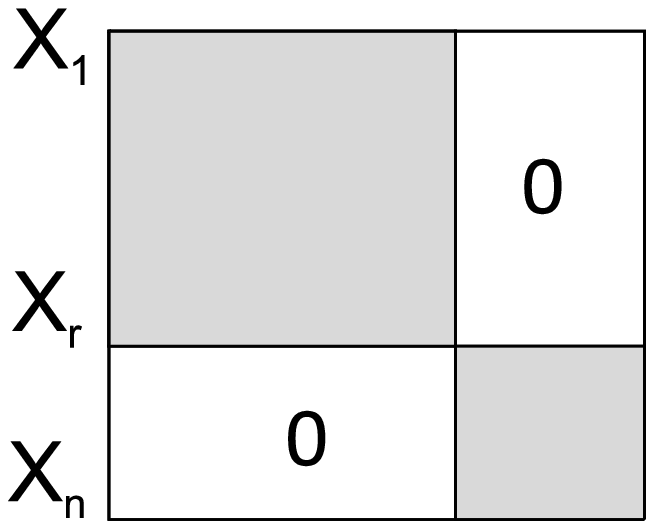} & \hspace{5mm} &
    \includegraphics[width=0.24\textwidth]{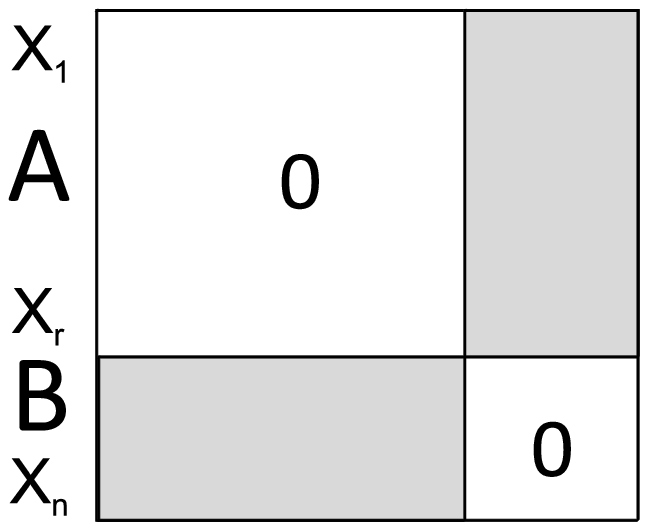} & \hspace{5mm} &
    \includegraphics[width=0.24\textwidth]{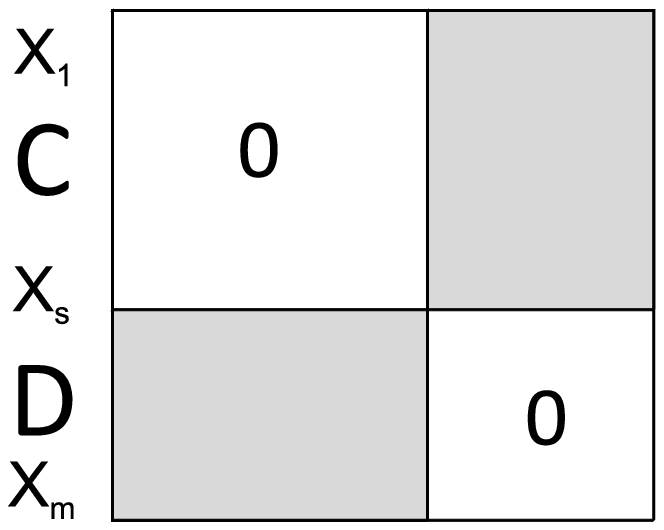}\\
    (a) Adjacency matrix  & & (b) Adjacency matrix  & & (c) Adjacency matrix \\
    when $G$ is disconnected & & when $G$ is bipartite & & when $H$ is bipartite\\
  \end{tabular}
  \begin{tabular}{ccc}
    & & \\
    \includegraphics[width=0.35\textwidth]{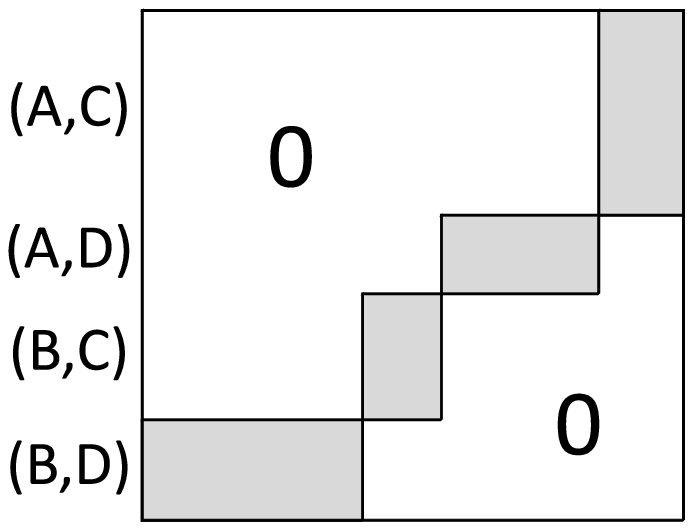} & \hspace{5mm} &
    \includegraphics[width=0.35\textwidth]{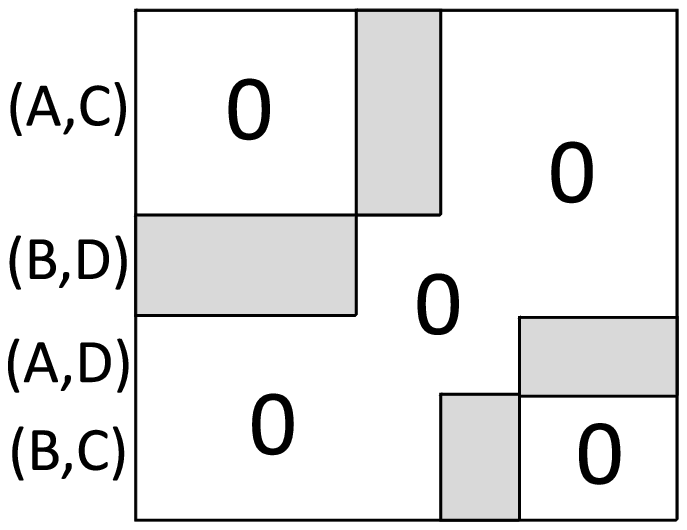} \\
    (d) Kronecker product of & & (e) Rearranged adjacency \\
    two bipartite graphs $G$ and $H$& &  matrix from panel (d)\\
  \end{tabular}
  \end{center}
  \caption{\new{Graph adjacency matrices. Dark parts represent connected (filled
  with ones) and white parts represent empty (filled with zeros) parts of the
  adjacency matrix. (a) When $G$ is disconnected, Kronecker multiplication
  with any matrix $H$ will result in $G \otimes H$ being disconnected. (b)
  Adjacency matrix of a connected bipartite graph $G$ with node partitions $A$
  and $B$. (c) Adjacency matrix of a connected bipartite graph $G$ with node partitions $C$
  and $D$. (e)  Kronecker product of two bipartite graphs $G$ and $H$.
  (d) After rearranging the adjacency matrix $G \otimes H$ we clearly see the
  resulting graph is disconnected.}}
  \label{fig:KronConnect}
\end{figure}

\begin{theorem}
  If both $G$ and $H$ are connected but bipartite, then $G \otimes H$ is
  disconnected, and each of the two connected components is again
  bipartite.
\end{theorem}

\begin{proof}
  Again without loss of generality let $G$ be bipartite with two partitions
  $A=\{X_1, \ldots X_r\}$ and $B=\{X_{r+1}, \ldots, X_n\}$, where edges
  exists only between the partitions, and no edges exist inside the
  partition: $(X_i, X_k) \notin G$ for $i,k \in A$ or $i,k \in B$.
  Similarly, let $H$ also be bipartite with two partitions $C=\{X_1, \ldots
  X_s\}$ and $D=\{X_{s+1}, \ldots, X_m\}$.
  Figures~\ref{fig:KronConnect}(b) and (c)
  illustrate the structure of the corresponding adjacency matrices.

  Now, there will be two connected components in $G \otimes H$: $1^{st}$
  component will be composed
  of nodes $\{X_{ij}\} \in G \otimes H$, where $(i \in A, j \in D)$ or $(i
  \in B, j \in C)$. And similarly, $2^{nd}$ component will be composed of
  nodes $\{X_{ij}\}$, where $(i \in A, j \in C)$ or $(i \in B, j \in D)$.
  Basically, there exist edges between node sets $(A,D)$ and $(B,C)$, and
  similarly between $(A,C)$ and $(B,D)$ but not across the sets. To see this
  we have to analyze the cases using Observation~\ref{obs:KronEdges}. For
  example, in $G \otimes H$ there exist edges between nodes $(A,C)$ and
  $(B,D)$ as there exist edges $(i, k) \in G$ for $i \in A, k \in B$, and
  $(j, l) \in H$ for $j \in C$ and $l \in D$. Similar is true for nodes
  $(A,C)$ and $(B,D)$. However, there are no edges cross the two sets, {\em}
  \emph{e.g.}, nodes from $(A,D)$ do not link to $(A,C)$, as there are no edges
  between nodes in $A$ (since $G$ is bipartite). See
  Figures~\ref{fig:KronConnect}(d) and~\ref{fig:KronConnect}(e) for a visual
  proof.
\end{proof}

\new{Note that bipartite graphs are triangle free and have no self-loops.
Stars, chains, trees and cycles of even length are all examples
of bipartite graphs. In order to ensure that $\kn{k}$ is connected,
for the remained of the paper we focus on initiator graphs $\kn{1}$
with self loops on all of the vertices.}

\subsubsection{Temporal properties of Kronecker graphs}

We continue with the analysis of temporal patterns of evolution of
Kronecker graphs: the densification power law, and shrinking/stabilizing
diameter~\cite{jure05dpl,jure07evolution}.

\begin{theorem}[Densification power law]
  Kronecker graphs follow the densification power law (DPL) with
  densification exponent  $a = \log(\ezero) / \log (\nzero)$.
\end{theorem}

\begin{proof}
  Since the $k^{\rm th}$ Kronecker power $\kn{k}$ has $\nn{k} = \nzero ^
  k$ nodes and $\en{k} = \ezero ^ k$ edges, it satisfies $\en{k} =
  \nn{k}^a$, where $a = \log(\ezero) / \log (\nzero)$. The crucial
  point is that this exponent $a$ is independent of $k$, and hence the
  sequence of Kronecker powers follows an exact version of the
  densification power law.
\end{proof}

We now show how the Kronecker product also preserves the property of
constant diameter, a crucial ingredient for matching the diameter
properties of many real-world network datasets. In order to establish
this, we will assume that the initiator graph $\kn{1}$ has a self-loop on
every node. Otherwise, its Kronecker powers may be disconnected.

\begin{lemma}
  If $G$ and $H$ each have diameter at most $\diam$
  and each has a self-loop on every node,
  then the Kronecker graph $G \otimes H$ also has diameter at most $\diam$.
  \label{thm:KronDia}
\end{lemma}

\begin{proof}
  Each node in $G \otimes H$ can be represented as an ordered pair
  $(v,w)$, with $v$ a node of $G$ and $w$ a node of $H$, and with an
  edge joining $(v,w)$ and $(x,y)$ precisely when $(v,x)$ is an edge
  of $G$ and $(w,y)$ is an edge of $H$. (Note this exactly the
  Observation \ref{obs:KronEdges}.) Now, for an arbitrary pair of
  nodes $(v,w)$ and $(v',w')$, we must show that there is a path of
  length at most $\diam$ connecting them. Since $G$ has diameter at most
  $\diam$, there is a path $v = v_1, v_2, \ldots, v_r = v'$, where $r \leq
  \diam$. If $r < \diam$, we can convert this into a path $v = v_1, v_2,
  \ldots, v_\diam = v'$ of length exactly $\diam$, by simply repeating $v'$ at
  the end for $\diam - r$ times. By an analogous argument, we have a path
  $w = w_1, w_2, \ldots, w_\diam = w'$. Now by the definition of the
  Kronecker product, there is an edge joining $(v_i,w_i)$ and
  $(v_{i+1},w_{i+1})$ for all $1 \leq i \leq \diam-1$, and so $(v,w) =
  (v_1,w_1), (v_2,w_2), \ldots, (v_\diam,w_\diam) = (v',w')$ is a path of
  length $\diam$ connecting $(v,w)$ to $(v',w')$, as required.
\end{proof}

\begin{theorem}
  If $\kn{1}$ has diameter $\diam$ and a self-loop on every node, then for
  every $k$, the graph $\kn{k}$ also has diameter $\diam$. \label{cor:KronDia}
\end{theorem}

\begin{proof}
  This follows directly from the previous lemma, combined with
  induction on $k$.
\end{proof}

As defined in section~\ref{sec:KronRelated} we also consider the {\em
effective diameter} $\ediam$. We defined the $q$-effective diameter as the
minimum $\ediam$ such that, for a $q$ fraction of the reachable
node pairs, the path length is at most $\ediam$. The $q$-effective
diameter is a more robust quantity than the diameter, the latter being
prone to the effects of degenerate structures in the graph (\emph{e.g.},
very long chains). However, the $q$-effective diameter and diameter tend
to exhibit qualitatively similar behavior. For reporting results in
subsequent sections, we will generally consider the $q$-effective diameter
with $q = 0.9$, and refer to this simply as the {\em effective diameter}.

\begin{theorem}[Effective Diameter]
  If $\kn{1}$ has diameter $\diam$ and a self-loop on every node, then for
  every $q$, the $q$-effective diameter of $\kn{k}$ converges to $\diam$
  (from \new{below}) as $k$ increases.
  \label{thm:KronEffDiam}
\end{theorem}

\begin{proof}
  To prove this, it is sufficient to show that for two randomly
  selected nodes of $\kn{k}$, the probability that their distance is $\diam$
  converges to $1$ as $k$ goes to infinity.

  We establish this as follows. Each node in $\kn{k}$ can be represented
  as an ordered sequence of $k$ nodes from $\kn{1}$, and we can view the
  random selection of a node in $\kn{k}$ as a sequence of $k$ independent
  random node selections from $\kn{1}$. Suppose that $v = (v_1, \ldots,
  v_k)$ and $w = (w_1, \ldots, w_k)$ are two such randomly selected
  nodes from $\kn{k}$. Now, if $x$ and $y$ are two nodes in $\kn{1}$ at
  distance $\diam$ (such a pair $(x,y)$ exists since $\kn{1}$ has diameter
  $\diam$), then with probability \new{$1 - (1 - \frac{1}{\nzero^2})^k$}, there is some
  index $j$ for which $\{v_j,w_j\} = \{x,y\}$. If there is such an
  index, then the distance between $v$ and $w$ is $\diam$. As the
  expression \new{$1 - (1 - \frac{1}{\nzero^2})^k$} converges to $1$ as $k$ increases,
  it follows that the $q$-effective diameter is converging to $\diam$.
\end{proof}

\subsection{Stochastic Kronecker graphs}
\label{sec:KronSkrg}

\begin{figure}[t]
  \begin{center}
  \begin{tabular}{ccc}
    \raisebox{10mm}{\includegraphics[width=0.20\textwidth]{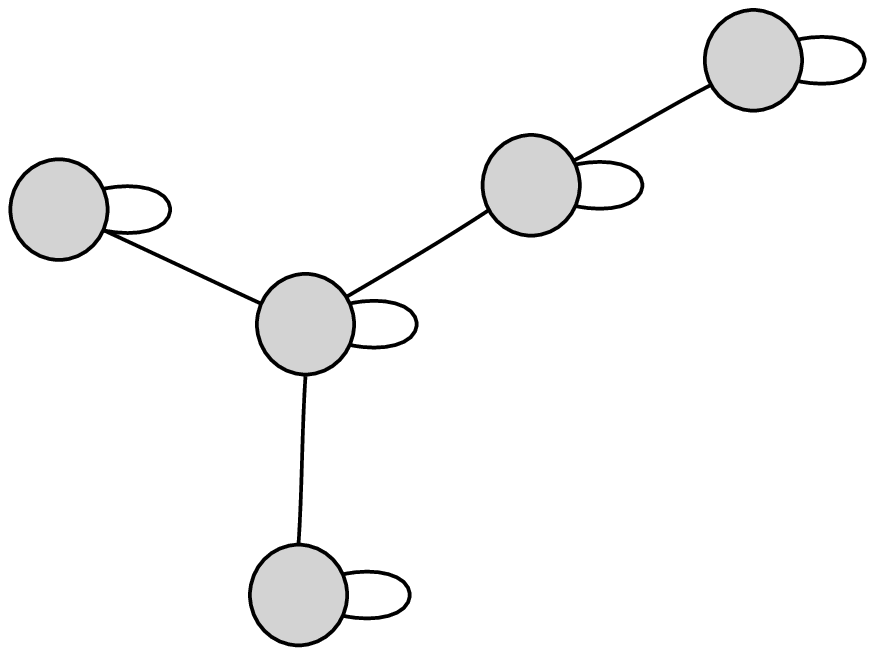}} &
    \includegraphics[width=0.37\textwidth]{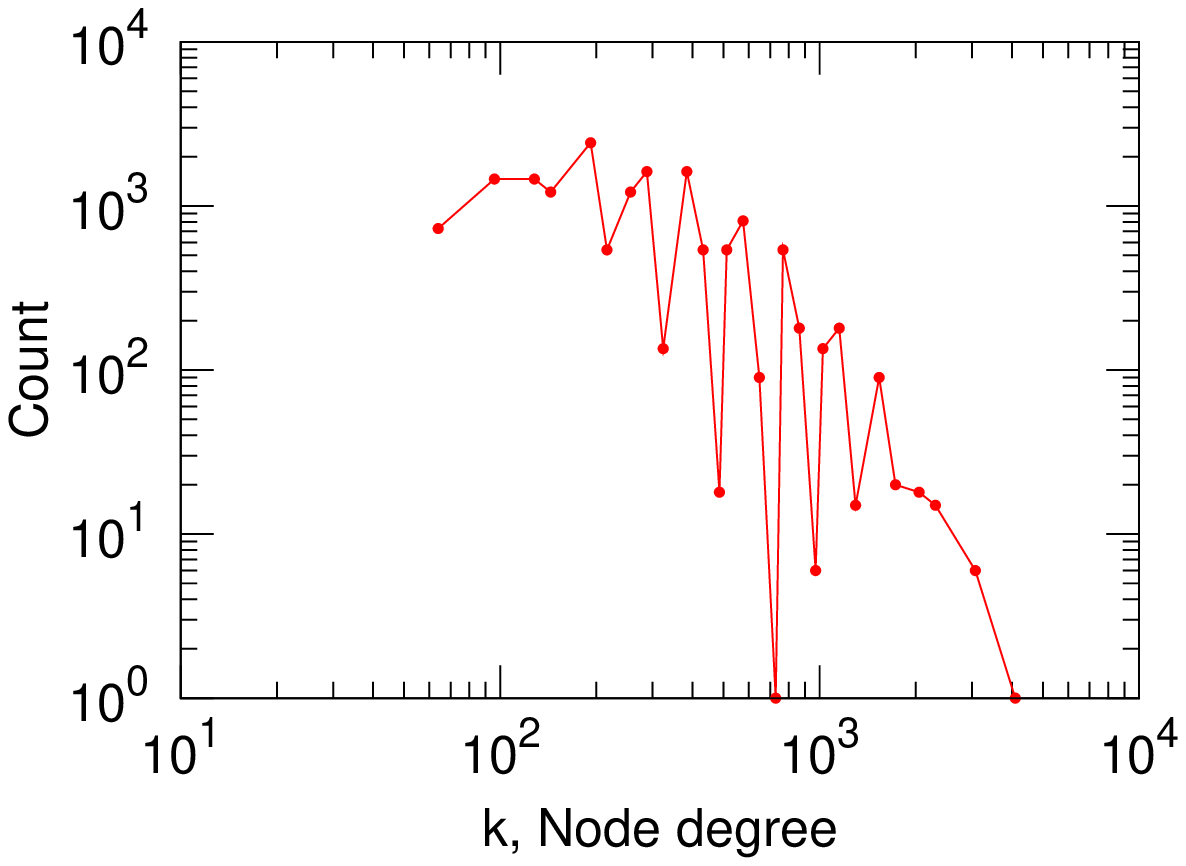} &
    \includegraphics[width=0.37\textwidth]{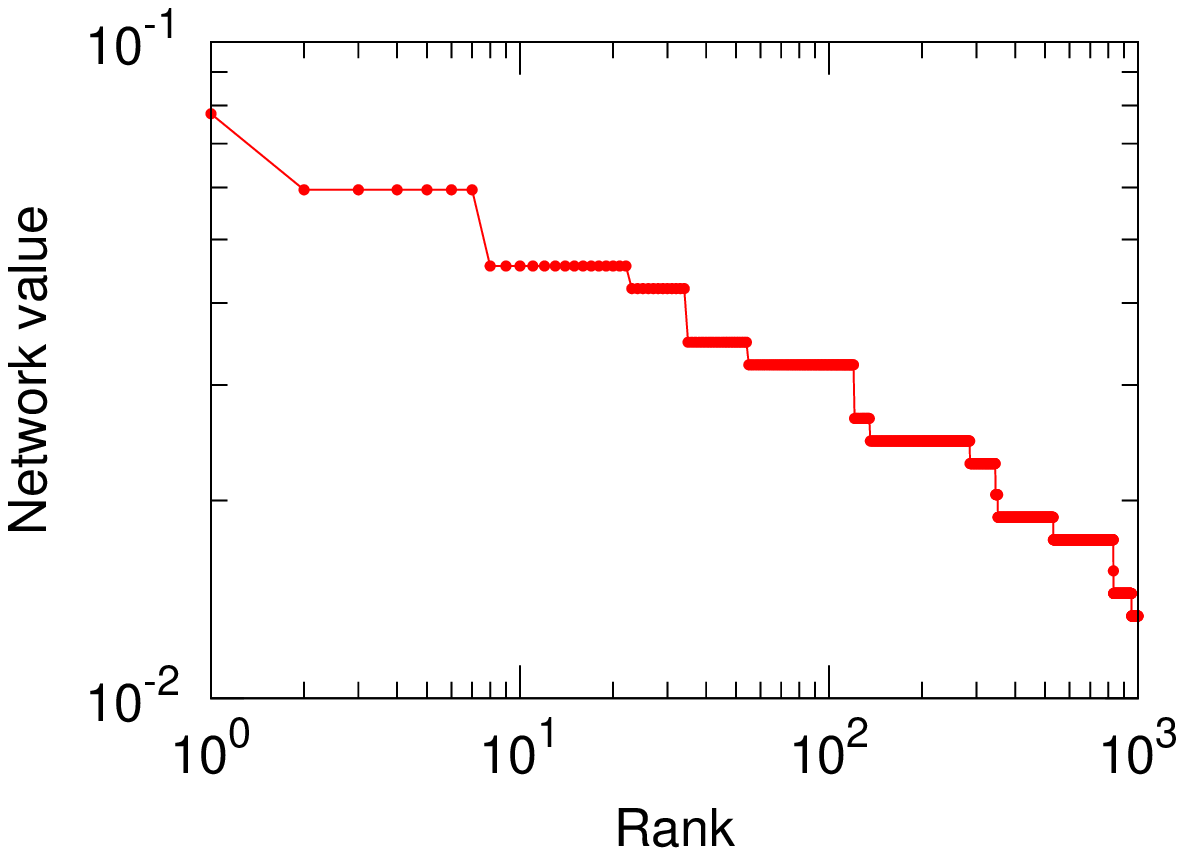}\\
    (a) Kronecker  & (b) Degree distribution of $\kn{6}$ & (c) Network value of $\kn{6}$ \\
    initiator $\kn{1}$  &  ($6^{th}$ Kronecker power of $\kn{1}$) & ($6^{th}$ Kronecker power of
$\kn{1}$)\\
  \end{tabular}
  \end{center}
  \caption{The ``staircase'' effect. Kronecker initiator and the degree
  distribution and network value plot for the $6^{th}$ Kronecker power of the
  initiator. Notice the non-smoothness of the curves.}
  \label{fig:Kron5burst}
\end{figure}

While the Kronecker power construction discussed so far yields graphs with
a range of desired properties, its discrete nature produces ``staircase
effects'' in the degrees and spectral quantities, simply because
individual values have large multiplicities. For example, degree
distribution and distribution of eigenvalues of graph adjacency matrix and
the distribution of the principal eigenvector components (\emph{i.e.}, the
``network'' value) are all impacted by this. These quantities are
multinomially distributed which leads to individual values with large
multiplicities. Figure~\ref{fig:Kron5burst} illustrates the staircase
effect.

Here we propose a stochastic version of Kronecker graphs that eliminates
this effect. There are many possible ways how one could introduce
stochasticity into Kronecker graph model. Before introducing the proposed
model, we introduce two simple ways of introducing randomness to Kronecker
graphs and describe why they do not work.

Probably the simplest (but wrong) idea is to generate a large
deterministic Kronecker graph $\kn{k}$, and then uniformly at random flip
some edges, \emph{i.e.}, uniformly at random select entries of the graph
adjacency matrix and flip them ($1 \rightarrow 0, 0 \rightarrow 1$).
However, this will not work, as it will essentially superimpose an
Erd\H{o}s-R\'{e}nyi random graph, which would, for example, corrupt the
degree distribution -- real networks usually have heavy tailed degree
distributions, while random graphs have Binomial degree distributions. A
second idea could be to allow a weighted initiator matrix, \emph{i.e.},
values of entries of $\kn{1}$ are not restricted to values $\{0, 1\}$ but
rather can be any non-negative real number. Using such $\kn{1}$ one would
generate $\kn{k}$ and then threshold the $\kn{k}$ matrix to obtain a
binary adjacency matrix $K$, \emph{i.e.}, for a chosen value of $\epsilon$
set $K[i,j] = 1$ if $\kn{k}[i,j] > \epsilon$ else $K[i,j] = 0$. This also
would not work as the mechanism would selectively remove edges and thus
the low degree nodes which would have low weight edges would get isolated
first.

\new{Now we define {\em Stochastic Kronecker graph model} which overcomes the
above issues. A more natural way to introduce stochasticity to Kronecker
graphs is to relax the assumption that entries of the initiator matrix
take only binary values. Instead we allow entries of the initiator to
take values on the interval $[0,1]$. This means now each entry of the
initiator matrix encodes the probability of that particular edge
appearing. We then Kronecker-power such initiator matrix to obtain a large
stochastic adjacency matrix, where again each entry of the large matrix
gives the probability of that particular edge appearing in a big graph.
Such a stochastic adjacency matrix defines a probability
distribution over all graphs. To obtain a graph we simply sample an
instance from this distribution by sampling individual edges, where each
edge appears independently with probability given by the entry of the
large stochastic adjacency matrix. More formally, we define:}

\begin{definition}[Stochastic Kronecker graph]
Let $\pn{1}$ be a $\nzero \times \nzero$ {\em probability matrix}: the
value $\thij{ij} \in \pn{1}$ denotes the probability that edge $(i,j)$ is
present, $\thij{ij} \in [0,1]$.

Then $k^{th}$ Kronecker power $\pn{1}^{[k]}=\pn{k}$, where each entry
$p_{uv} \in \pn{k}$ encodes the probability of an edge $(u,v)$.

To obtain a graph, an {\em instance} (or {\em realization}), $K =
R(\pn{k})$ we include edge $(u,v)$ in $K$ with probability $p_{uv}$,
$p_{uv} \in \pn{k}$.
\end{definition}

First, note that sum of the entries of $\pn{1}$, $\sum_{ij} \thij{ij}$,
can be greater than 1. Second, notice that in principle it takes
$O(\nzero^{2k})$ time to generate an instance $K$ of a Stochastic
Kronecker graph from the probability matrix $\pn{k}$. This means the time
to get a realization $K$ is quadratic in the size of $\pn{k}$ as one has
to flip a coin for each possible edge in the graph. Later we show how to
generate Stochastic Kronecker graphs much faster, in the time {\em linear}
in the expected number of edges in $\pn{k}$.

\subsubsection{Probability of an edge}
\label{sec:KronEdgeProb}

For the size of graphs we aim to model and generate here taking
$\pn{1}$ (or $\kzero$) and then explicitly performing the Kronecker
product of the initiator matrix is infeasible. The reason for this is that
$\pn{1}$ is usually dense, so $\pn{k}$ is also dense and one can not explicitly store
it in memory to directly iterate the Kronecker product.
However, due to the structure of Kronecker multiplication
one can easily compute the probability of an edge in $\pn{k}$.

The  probability $\pij{uv}$ of an edge $(u,v)$ occurring in $k$-th
Kronecker power $\pmat=\pn{k}$ can be calculated in $O(k)$ time as
follows:
\begin{equation}
  \pij{uv}=\prod_{i=0}^{k-1} \pn{1}\biggl[
    \Bigl\lfloor\frac{u-1}{\nzero^i}\Bigr\rfloor({\textrm{mod}}\nzero)+1,
    \Bigl\lfloor\frac{v-1}{\nzero^i}\Bigr\rfloor({\textrm{mod}}\nzero)+1\biggr]
  \label{eq:KronPij}
\end{equation}

The equation imitates recursive descent into the matrix $\pmat$, where at
every level $i$ the appropriate entry of $\pn{1}$ is chosen. Since $\pmat$
has $\nzero^k$ rows and columns it takes $O(k \log \nzero)$ to evaluate
the equation. Refer to Figure~\ref{fig:KronAttrView} for the illustration
of the recursive structure of $\pmat$.

\subsection{Additional properties of Kronecker graphs}

Stochastic Kronecker graphs with initiator matrix of size $\nzero=2$ were
studied by Mahdian and Xu \cite{mahdian07kronecker}. The authors showed a
phase transition for the emergence of the giant component and another
phase transition for connectivity, and proved that such graphs have
constant diameters beyond the connectivity threshold, but are not
searchable using a decentralized algorithm~\cite{kleinberg99navigation}.

General overview of Kronecker product is given in~\cite{imrich00product} and properties of Kronecker graphs related to graph minors, planarity, cut vertex and cut edge have been explored in~\cite{bottreau98kronecker}. Moreover, recently~\cite{tsourakakis08triangles} gave a closed form expression for the number of triangles in a Kronecker graph that depends on the eigenvalues of the initiator graph $\kzero$.

\subsection{Two interpretations of Kronecker graphs}
\label{sec:interpret}

Next, we present two natural interpretations of the generative process
behind the Kronecker graphs that go beyond the purely mathematical
construction of Kronecker graphs as introduced so far.

We already mentioned the first interpretation when we first defined
Kronecker graphs. One intuition is that networks are hierarchically organized
into communities (clusters). Communities then
grow recursively, creating miniature copies of themselves.
Figure~\ref{fig:KronSpyplots} depicts the process of the recursive
community expansion. In fact, several researchers have argued that real
networks are hierarchically
organized~\cite{ravasz02metabolic,ravasz03hierar} and algorithms to
extract the network hierarchical structure have also been
developed~\cite{sales07hierarchical,clauset08hierarchical}. Moreover,
especially web
graphs~\cite{dill02web,dorogovtsev02fractal,crovella97selfsimilarity} and
biological networks~\cite{ravasz03hierar} were found to be self-similar
and ``fractal''.

\begin{figure}[t]
  \begin{center}
  \begin{tabular}{ccccc}
    \raisebox{6mm}{\includegraphics[width=0.14\textwidth]{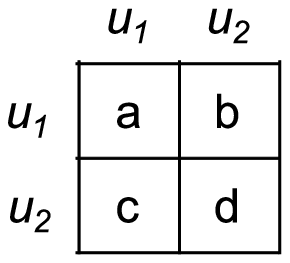}}& &
    \includegraphics[width=0.25\textwidth]{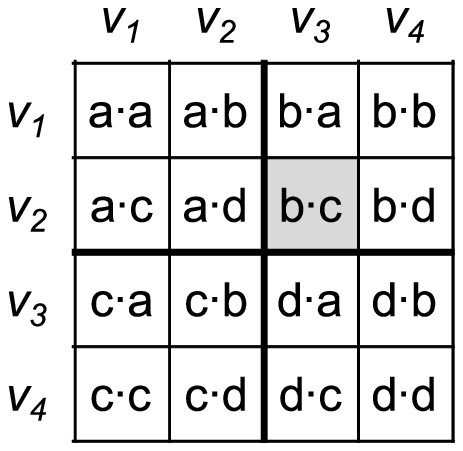}& \hspace*{1cm} &
    \includegraphics[width=0.25\textwidth]{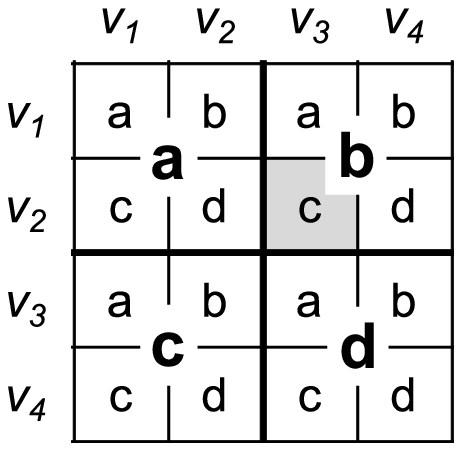}\\
    (a) $2\times2$ Stochastic  & & (b) Probability matrix & & (c) Alternative view \\
    Kronecker initiator $\pn{1}$ & & $\pn{2} = \pn{1}\otimes \pn{1}$ & & of $\pn{2} =
\pn{1}\otimes \pn{1}$ \\
  \end{tabular}
  \end{center}
  \caption{Stochastic Kronecker initiator $\pn{1}$ and the corresponding
  $2^{nd}$ Kronecker power $\pn{2}$. Notice the recursive nature of the
  Kronecker product, with edge probabilities in $\pn{2}$ simply being
  products of entries of $\pn{1}$.}
  \label{fig:KronAttrView}
\end{figure}

The second intuition comes from viewing every node of $\pn{k}$ as being
described with an ordered sequence of $k$ nodes from $\pn{1}$. (This is
similar to the Observation~\ref{obs:KronEdges} and the proof of
Theorem~\ref{thm:KronEffDiam}.)

Let's label nodes of the initiator matrix $\pn{1}$, $u_1, \ldots,
u_{\nzero}$, and nodes of $\pn{k}$ as $v_1, \ldots, v_{\nzero^k}$. Then
every node $v_i$ of $\pn{k}$ is described with a sequence $(v_i(1),
\ldots, v_i(k))$ of node labels of $\pn{1}$, where $v_i(l) \in \{u_1,
\ldots, u_k\}$. Similarly, consider also a second node $v_j$ with the
label sequence $(v_j(1), \ldots, v_j(k))$.
Then the probability $p_e$ of an edge $(v_i, v_j)$ in $\pn{k}$ is exactly:
$$
p_e(v_i, v_j) = \pn{k}[v_i, v_j] = \prod_{l=1}^k \pn{1}[v_i(l), v_j(l)]
$$
(Note this is exactly the Equation~\ref{eq:KronPij}.)

Now one can look at the description sequence of node $v_i$ as a $k$
dimensional vector of attribute values $(v_i(1), \ldots, v_i(k))$. Then
$p_e(v_i, v_j)$ is exactly the coordinate-wise product of appropriate
entries of $\pn{1}$, where the node description sequence selects which
entries of $\pn{1}$ to multiply. Thus, the $\pn{1}$ matrix can be thought of as the
attribute similarity matrix, \emph{i.e.}, it encodes the probability of
linking given that two nodes agree/disagree on the attribute value. Then
the probability of an edge is simply a product of individual attribute
similarities over the $k$ $\nzero$-valued attributes that describe each of
the two nodes.

This gives us a very natural interpretation of Stochastic Kronecker
graphs: Each node is described by a sequence of categorical attribute
values or features. And then the probability of two nodes linking depends
on the product of individual attribute similarities. This way Kronecker
graphs can effectively model homophily (nodes with similar attribute
values are more likely to link) by $\pn{1}$ having high value entries on
the diagonal. Or heterophily (nodes that differ are more likely to link)
by $\pn{1}$ having high entries off the diagonal.

Figure~\ref{fig:KronAttrView} shows an example. Let's label nodes of
$\pn{1}$ $u_1, u_2$ as in Figure~\ref{fig:KronAttrView}(a). Then every
node of $\pn{k}$ is described with an ordered sequence of $k$ binary
attributes. For example, Figure~\ref{fig:KronAttrView}(b) shows an
instance for $k=2$ where node $v_2$ of $\pn{2}$ is described by $(u_1,
u_2)$, and similarly $v_3$ by $(u_2, u_1)$. Then as shown in
Figure~\ref{fig:KronAttrView}(b), the probability of edge $p_e(v_2,v_3) =
b \cdot c$, which is exactly $\pn{1}[u_2, u_1] \cdot \pn{1}[u_1, u_2] = b
\cdot c$ --- the product of entries of $\pn{1}$, where the corresponding
elements of the description of nodes $v_2$ and $v_3$ act as selectors of
which entries of $\pn{1}$ to multiply.

Figure~\ref{fig:KronAttrView}(c) further illustrates the recursive nature
of Kronecker graphs. One can see Kronecker product as recursive descent
into the big adjacency matrix where at each stage one of the entries or
blocks is chosen. For example, to get to entry $(v_2, v_3)$ one first
needs to dive into quadrant $b$ following by the quadrant $c$. This
intuition will help us in section~\ref{sec:KronFastKron} to devise a fast
algorithm for generating Kronecker graphs.

However, there are also two notes to make here. First, using a single
initiator $\pn{1}$ we are implicitly assuming that there is one single and
universal attribute similarity matrix that holds across all $k$
$\nzero$-ary attributes. One can easily relax this assumption by taking a
different initiator matrix for each attribute (initiator matrices can even
be of different sizes as attributes are of different arity), and then
Kronecker multiplying them to obtain a large network. Here each initiator
matrix plays the role of attribute similarity matrix for that particular
attribute.

For simplicity and convenience we will work with a single initiator matrix
but all our methods can be trivially extended to handle multiple initiator
matrices. Moreover, as we will see later in
section~\ref{sec:KronExperiments} even a single $2 \times 2$ initiator
matrix seems to be enough to capture large scale statistical properties of
real-world networks.

The second assumption is harder to relax. When describing every node $v_i$
with a sequence of attribute values we are implicitly assuming that the values
of all attributes are uniformly distributed (have same proportions), and
that every node has a unique combination of attribute values. So, all
possible combinations of attribute values are taken. For example, node
$v_1$ in a large matrix $\pn{k}$ has attribute sequence $(u_1, u_1, \ldots,
u_1)$, and $v_{\nzero}$ has $(u_1, u_1, \ldots, u_1, u_{\nzero})$, while the
``last'' node $v_{\nzero^k}$ is has attribute values $(u_{\nzero},
u_{\nzero}, \ldots, u_{\nzero})$. One can think of this as counting in
$\nzero$-ary number system, where node attribute descriptions range from
$0$ (\emph{i.e.}, ``leftmost'' node with attribute description $(u_1, u_1,
\ldots, u_1)$) to $\nzero^k$ (\emph{i.e.}, ``rightmost'' node attribute
description $(u_{\nzero}, u_{\nzero}, \ldots, u_{\nzero})$).

A simple way to relax the above assumption is to take a larger initiator
matrix with a smaller number of parameters than the number of entries.
This means that multiple entries of $\pn{1}$ will share the same value
(parameter). For example, if attribute $u_1$ takes one value 66\% of the
times, and the other value 33\% of the times, then one can model this by
taking a $3 \times 3$ initiator matrix with only four parameters. Adopting
the naming convention of Figure~\ref{fig:KronAttrView} this means that
parameter $a$ now occupies a $2\times2$ block, which then also makes $b$
and $c$ occupy $2\times1$ and $1\times2$ blocks, and $d$ a single cell.
This way one gets a four parameter model with uneven feature value
distribution.

We note that the view of Kronecker graphs where every node is described
with a set of features and the initiator matrix encodes the probability of
linking given the attribute values of two nodes somewhat resembles the
Random dot product graph model~\cite{young07dotprod,nickel08dotprod}. The
important difference here is that we multiply individual linking
probabilities, while in Random dot product graphs one takes the sum of
individual probabilities which seems somewhat less natural.

\subsection{Fast generation of Stochastic Kronecker graphs}
\label{sec:KronFastKron}

The intuition for fast generation of Stochastic Kronecker graphs comes
from the recursive nature of the Kronecker product and is closely related
to the R-MAT graph generator~\cite{chakrabarti04rmat}. Generating a
Stochastic Kronecker graph $K$ on $N$ nodes naively takes $O(N^2)$ time.
Here we present a linear time $O(E)$ algorithm, where $E$ is the
(expected) number of edges in $K$.

Figure~\ref{fig:KronAttrView}(c) shows the recursive nature of the
Kronecker product. To ``arrive'' to a particular edge $(v_i, v_j)$ of
$\pn{k}$ one has to make a sequence of $k$ (in our case $k=2$) decisions
among the entries of $\pn{1}$, multiply the chosen entries of $\pn{1}$,
and then placing the edge $(v_i, v_j)$ with the obtained probability.

Instead of flipping $O(N^2) = O(\nzero^{2k})$ biased coins to determine
the edges, we can place $E$ edges by directly simulating the recursion of
the Kronecker product. Basically we recursively choose sub-regions of
matrix $K$ with probability proportional to $\thij{ij}$, $\thij{ij} \in
\pn{1}$ until in $k$ steps we descend to a single cell of the big adjacency matrix $K$ and place an edge. For example, for $(v_2, v_3)$ in
Figure~\ref{fig:KronAttrView}(c) we first have to choose $b$ following by
$c$.

The probability of each individual edge of $\pn{k}$ follows a Bernoulli
distribution, as the edge occurrences are independent. By the Central
Limit Theorem~\cite{petrov95limits} the number of edges in $\pn{k}$ tends to a normal
distribution with mean $(\sum_{i,j=1}^{N_1} \thij{ij})^k = \ezero^k$,
where $\thij{ij} \in \pn{1}$. So, given a stochastic initiator matrix
$\pn{1}$ we first sample the expected number of edges $E$ in $\pn{k}$.
Then we place $E$ edges in a graph $\kgraph$, by applying the recursive
descent for $k$ steps where at each step we choose entry $(i,j)$ with
probability $\thij{ij} / \ezero$ where $\ezero
= \sum_{ij} \thij{ij}$. Since we add $E = \ezero^k$ edges, the probability
that edge $(v_i,v_j)$ appears in $\kgraph$ is exactly $\pn{k}[v_i,v_j]$.
This basically means that in Stochastic Kronecker graphs the initiator
matrix encodes both the total number of edges in a graph and their
structure. $\sum \thij{ij}$ encodes the number of edges in the graph,
while the proportions (ratios) of values $\thij{ij}$ define how many edges
each part of graph adjacency matrix will contain.

In practice it can happen that more than one edge lands in the same
$(v_i,v_j)$ entry of big adjacency matrix $K$.
\new{If an edge lands in a already occupied cell
we insert it again. Even though values of $\pn{1}$ are usually
skewed, adjacency matrices of real networks are so sparse that this
is not really a problem in practice. Empirically we note that around 1\%
of edges collide.}

\subsection{Observations and connections}

Next, we describe several observations about the properties of Kronecker
graphs and make connections to other network models.

\begin{itemize}

\item {\em Bipartite graphs:} Kronecker graphs can naturally model
    bipartite graphs. Instead of starting with a square $\nzero \times
    \nzero$ initiator matrix, one can choose arbitrary $\nzero \times
    M_1$ initiator matrix, where rows define ``left'', and columns the
    ``right'' side of the bipartite graph. Kronecker multiplication
    will then generate bipartite graphs with partition sizes
    $\nzero^k$ and $M_1^k$.

\item {\em Graph distributions:} $\pn{k}$ defines a distribution over
    all graphs, as it encodes the probability of all possible
    $\nzero^{2k}$ edges appearing in a graph by using an exponentially
    smaller number of parameters (just $\nzero^2$). As we will later
    see, even a very small number of parameters, \emph{e.g.}, 4
    ($2\times2$ initiator matrix) or 9 ($3 \times 3$ initiator), is
    enough to accurately model the structure of large networks.

\item \new{{\em Extension of Erd\H{o}s-R\'{e}nyi random graph
    model:} Stochastic Kronecker graphs represent an extension
    of Erd\H{o}s-R\'{e}nyi~\cite{erdos60random} random graphs. If one
    takes $\pn{1} = [\thij{ij}]$, where every $\thij{ij} = p$ then we
    obtain exactly the Erd\H{o}s-R\'{e}nyi model of random graphs
    $G_{n,p}$, where every edge appears independently with probability
    $p$.}

\item {\em Relation to the R-MAT model:} The recursive nature of
    Stochastic Kronecker graphs makes them related to the R-MAT
    generator~\cite{chakrabarti04rmat}. The difference between the two
    models is that in R-MAT one needs to separately specify the number
    of edges, while in Stochastic Kronecker graphs initiator matrix
    $\pn{1}$ also encodes the number of edges in the graph.
    Section~\ref{sec:KronFastKron} built on this similarity to devise
    a fast algorithm for generating Stochastic Kronecker graphs.

\item {\em Densification:} Similarly as with deterministic Kronecker
    graphs the number of nodes in a Stochastic Kronecker graph grows
    as $\nzero^k$, and the expected number of edges grows as
    $(\sum_{ij} \thij{ij})^k$. This means one would want to choose
    values $\thij{ij}$ of the initiator matrix $\pn{1}$ so that
    $\sum_{ij} \thij{ij} > \nzero$ in order for the resulting network
    to densify.
\end{itemize}

\section{Simulations of Kronecker graphs}
\label{sec:KronSimulation}
\new{Next we perform a set of simulation experiments to demonstrate the ability
of Kronecker graphs to match the patterns of real-world networks.}
We will tackle the problem of estimating the Kronecker graph model from
real data, \emph{i.e.}, finding the most likely initiator $\pn{1}$, in the
next section. Instead here we present simulation experiments using
Kronecker graphs to explore the parameter space, and to compare properties
of Kronecker graphs to those found in large real networks.

\subsection{Comparison to real graphs}

We observe two kinds of graph patterns --- ``static'' and ``temporal.'' As
mentioned earlier, common static patterns include degree distribution,
scree plot (eigenvalues of graph adjacency matrix vs. rank) and
distribution of components of the principal eigenvector of graph adjacency
matrix. Temporal patterns include the diameter over time, and the
densification power law. For the diameter computation, we use the
effective diameter as defined in Section~\ref{sec:KronRelated}.

For the purpose of this section consider the following setting. Given a
real graph $\ggraph$ we want to find Kronecker initiator that produces
qualitatively similar graph. In principle one could try choosing each of
the $\nzero^2$ parameters for the matrix $\pn{1}$ separately. However, we
reduce the number of parameters from $\nzero^2$ to just two: $\alpha$ and
$\beta$. Let $\kn{1}$ be the initiator matrix (binary, deterministic). Then we
create the corresponding stochastic initiator matrix $\pn{1}$ by replacing
each ``1'' and ``0'' of $\kn{1}$ with $\alpha$ and $\beta$ respectively
($\beta \leq \alpha$). The resulting probability matrices maintain ---
with some random noise --- the self-similar structure of the Kronecker
graphs in the previous section (which, for clarity, we call {\em
deterministic Kronecker graphs}). We defer the discussion of how to
automatically estimate $\pn{1}$ from data $G$ to the next section.

The datasets we use here are:
\begin{itemize}
  \item {\dataset{Cit-hep-th}:} This is a citation graph for
      High-Energy Physics Theory research papers from pre-print
      archive ArXiv, with a total of $N=$29,555 papers and $E=$352,807
      citations~\cite{gehrke03kddcup}. We follow its evolution from
      January~1993 to April~2003, with one data-point per month.
  \item {\dataset{As-RouteViews}:} We also analyze a static dataset
      consisting of a single snapshot of connectivity among Internet
      Autonomous Systems~\cite{oregon97as} from January 2000, with
      $N=$6,474 and $E=$26,467.
\end{itemize}

Results are shown in Figure~\ref{fig:KronArXiv} for the
\dataset{Cit-hep-th} graph which evolves over time. We show the plots of
two static and two temporal patterns. We see that the deterministic
Kronecker model already to some degree captures the qualitative structure of the degree
and eigenvalue distributions, as well as the temporal patterns represented
by the densification power law and the stabilizing diameter. However, the
deterministic nature of this model results in ``staircase'' behavior, as
shown in scree plot for the deterministic Kronecker graph of
Figure~\ref{fig:KronArXiv} (column (b), second row). We see that the
Stochastic Kronecker graphs smooth out these distributions, further
matching the qualitative structure of the real data, and they also match the
shrinking-before-stabilization trend of the diameters of real graphs.

\begin{figure}[t]
\begin{center}
  \begin{tabular}{c@{\hspace{1em}}cc||cc}
  \raisebox{1.5em}{\includegraphics[width=0.03\textwidth]{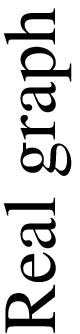}} &
  \includegraphics[width=0.2\textwidth]{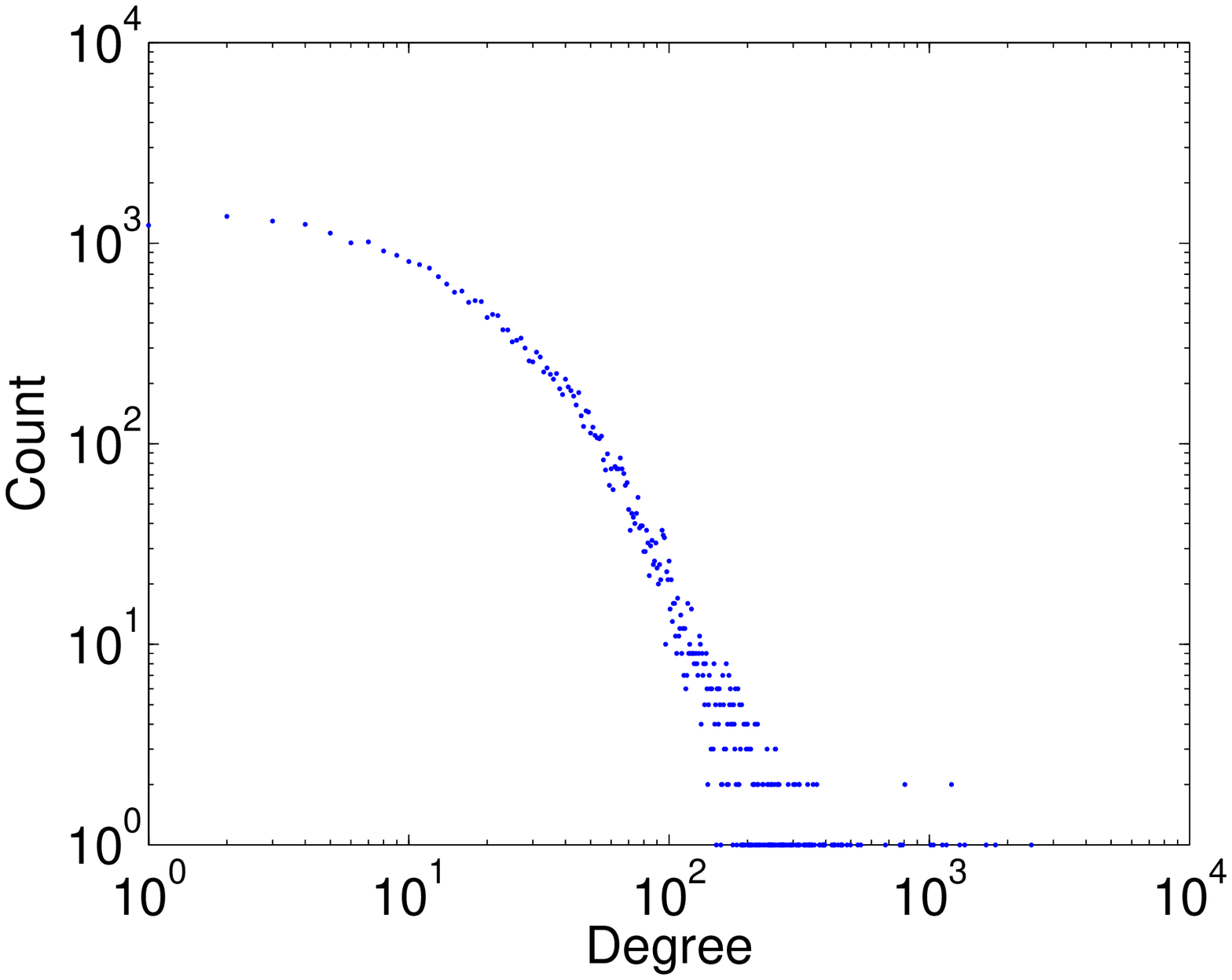} &
  \includegraphics[width=0.2\textwidth]{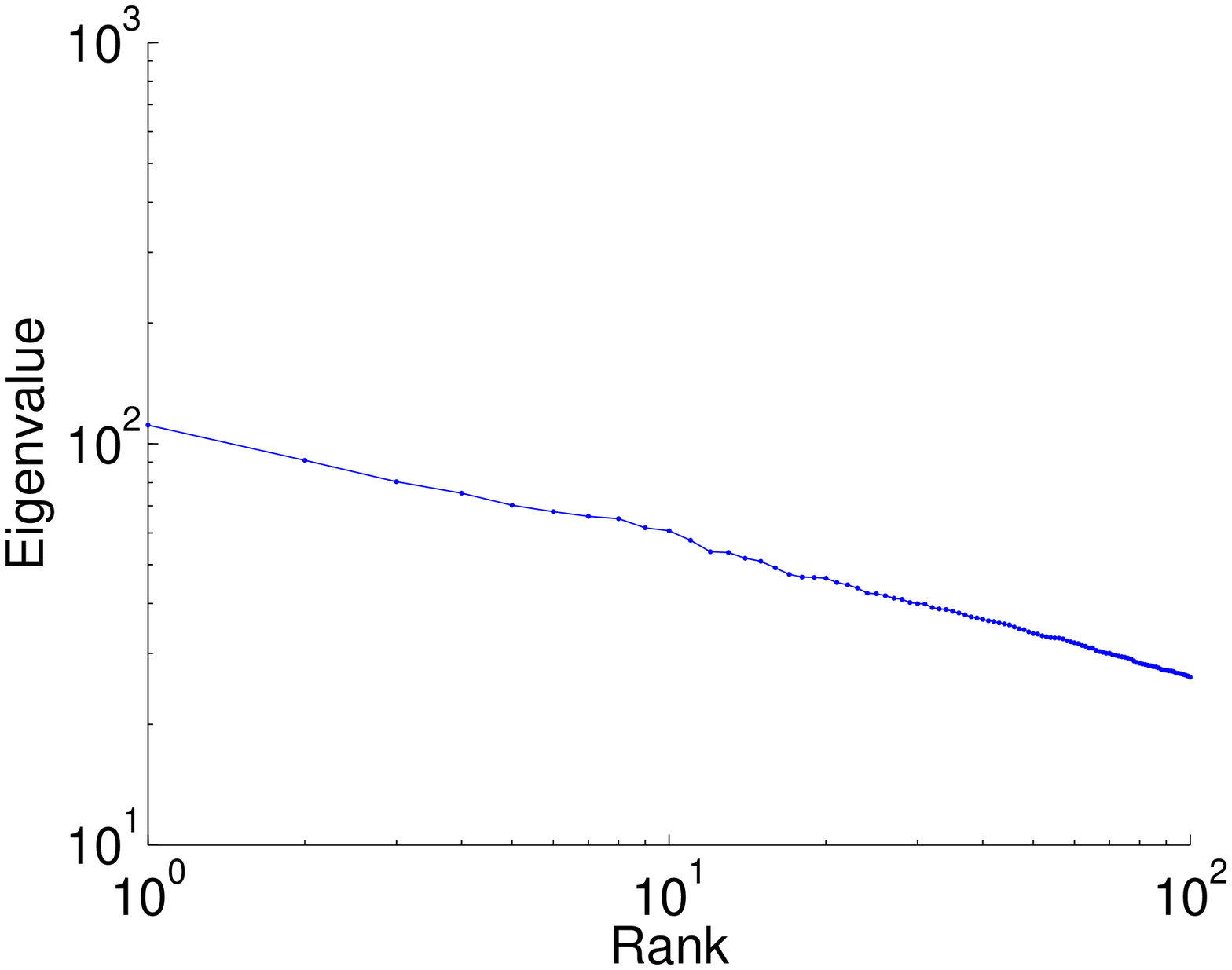} &
  \includegraphics[width=0.2\textwidth]{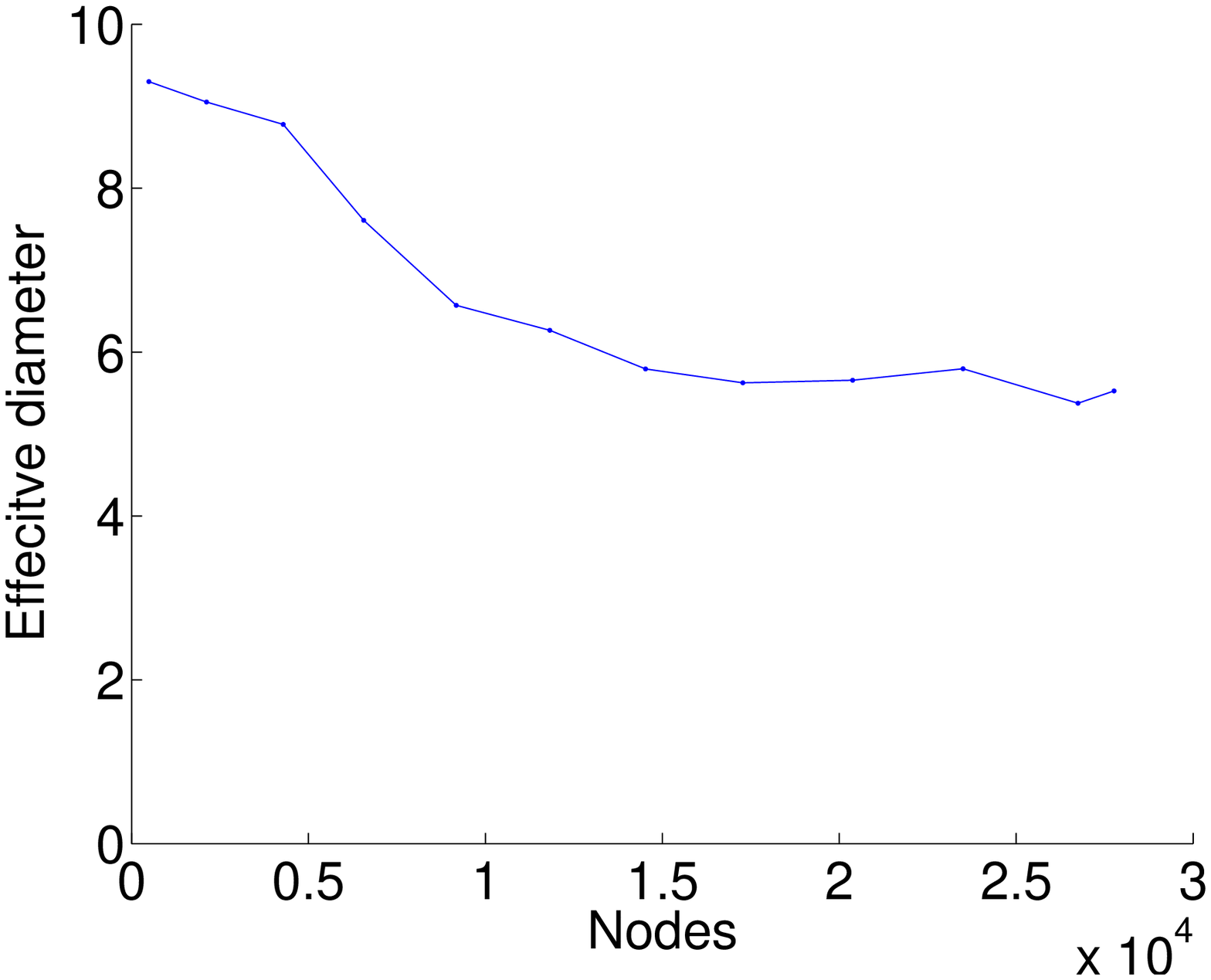} &
  \includegraphics[width=0.2\textwidth]{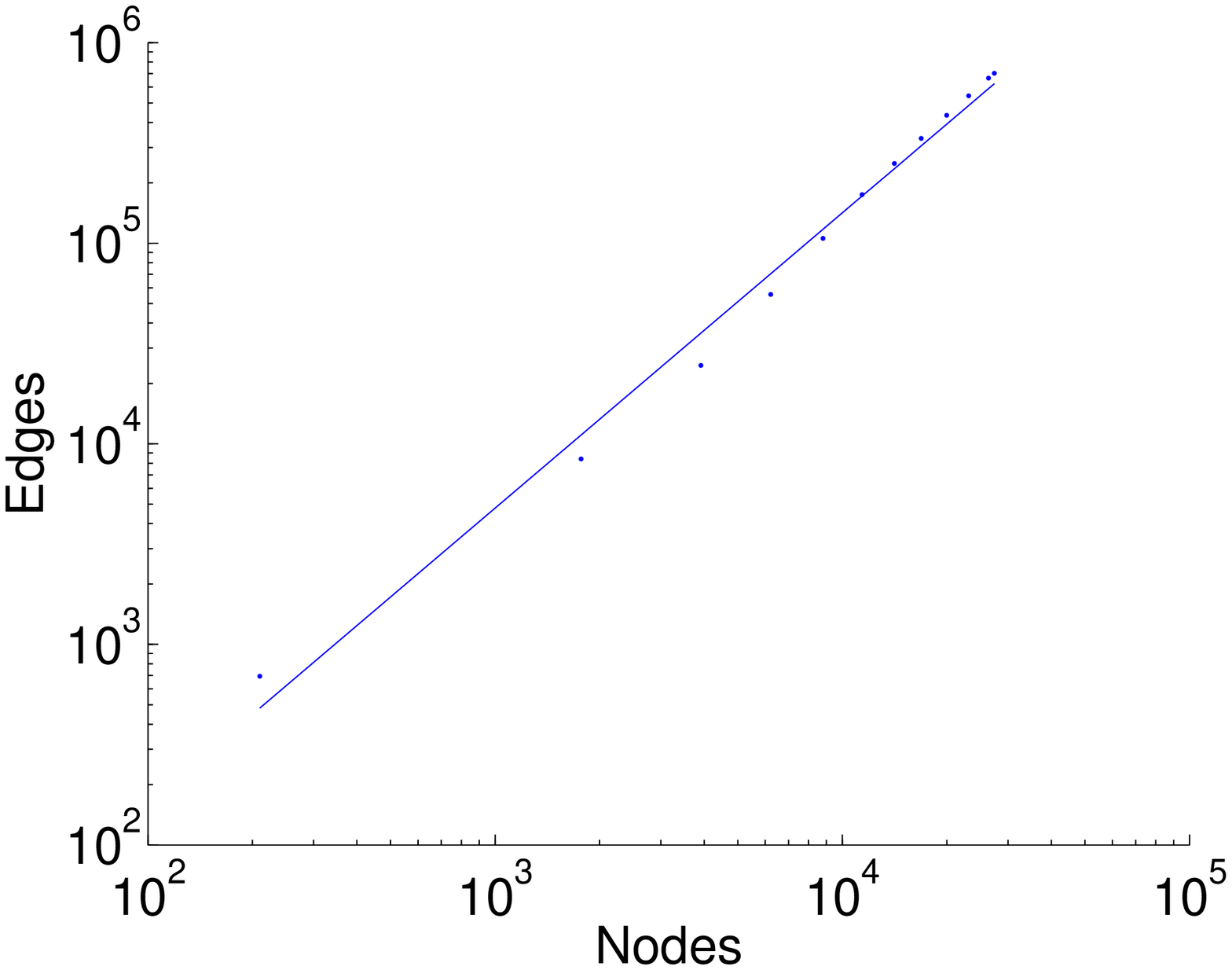} \\
  \raisebox{1.5em}{\includegraphics[width=0.05\textwidth]{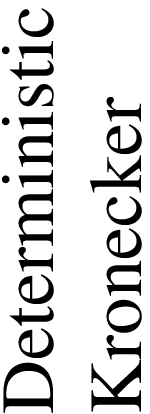}}&
  \includegraphics[width=0.2\textwidth]{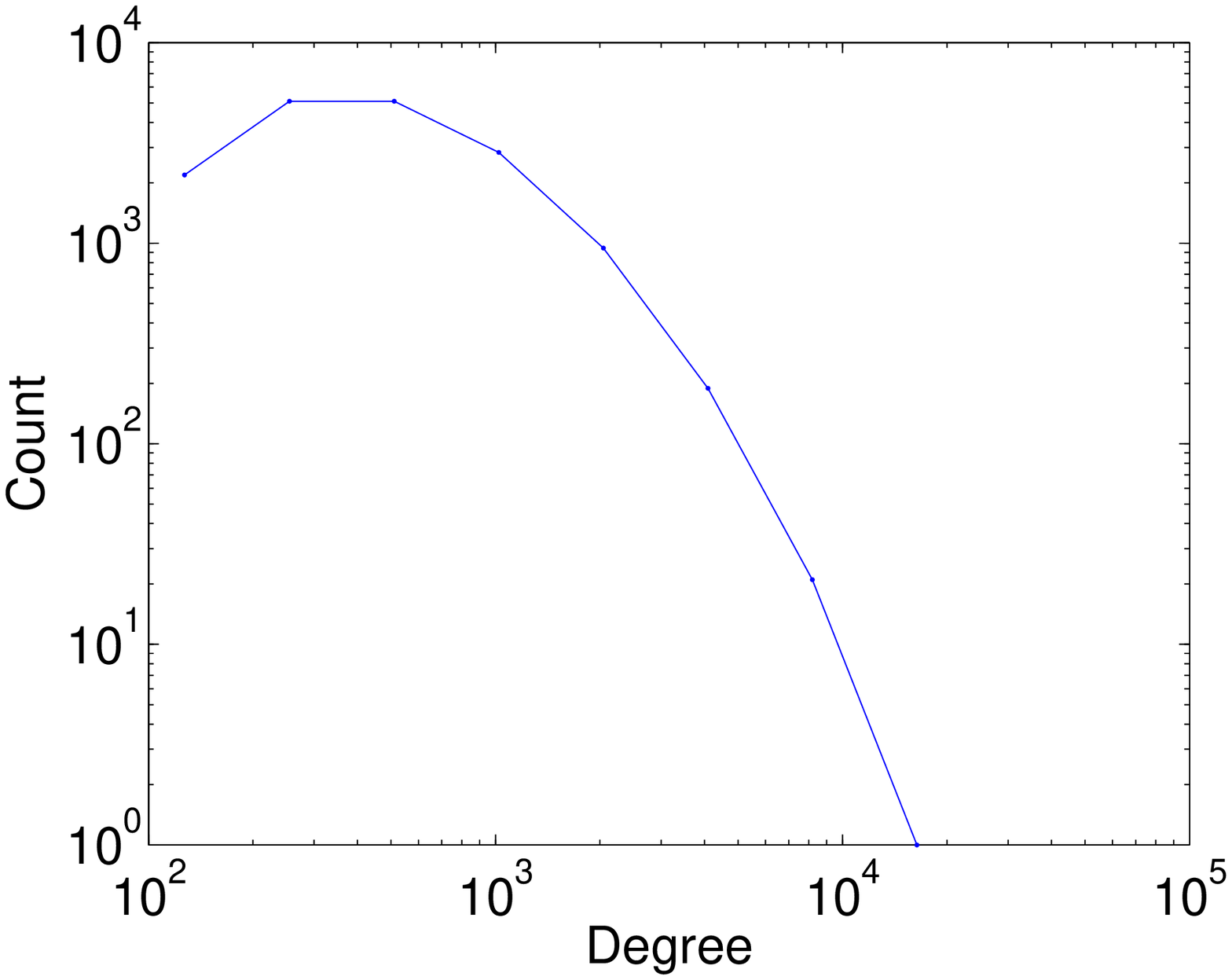} &
  \includegraphics[width=0.2\textwidth]{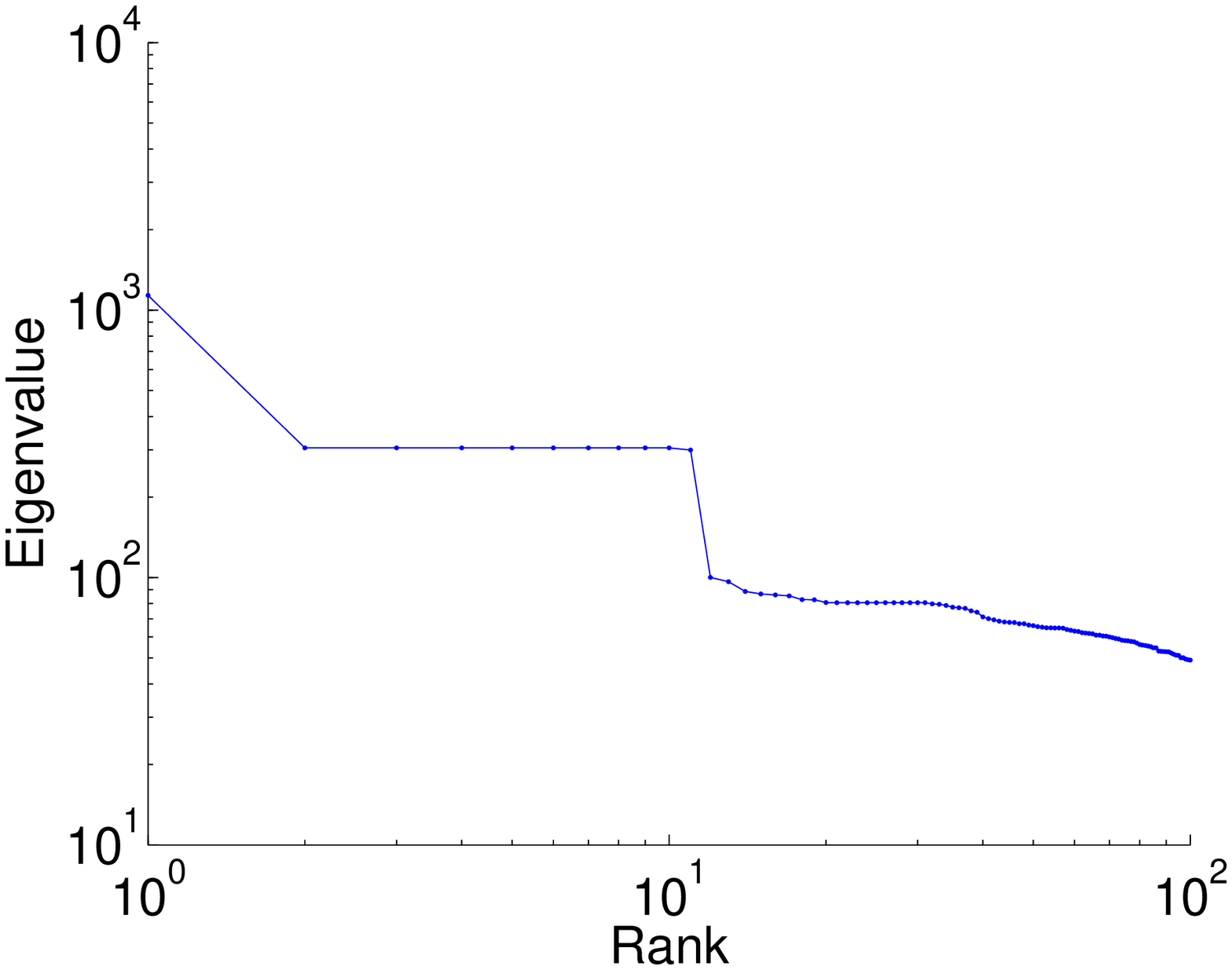} &
  \includegraphics[width=0.2\textwidth]{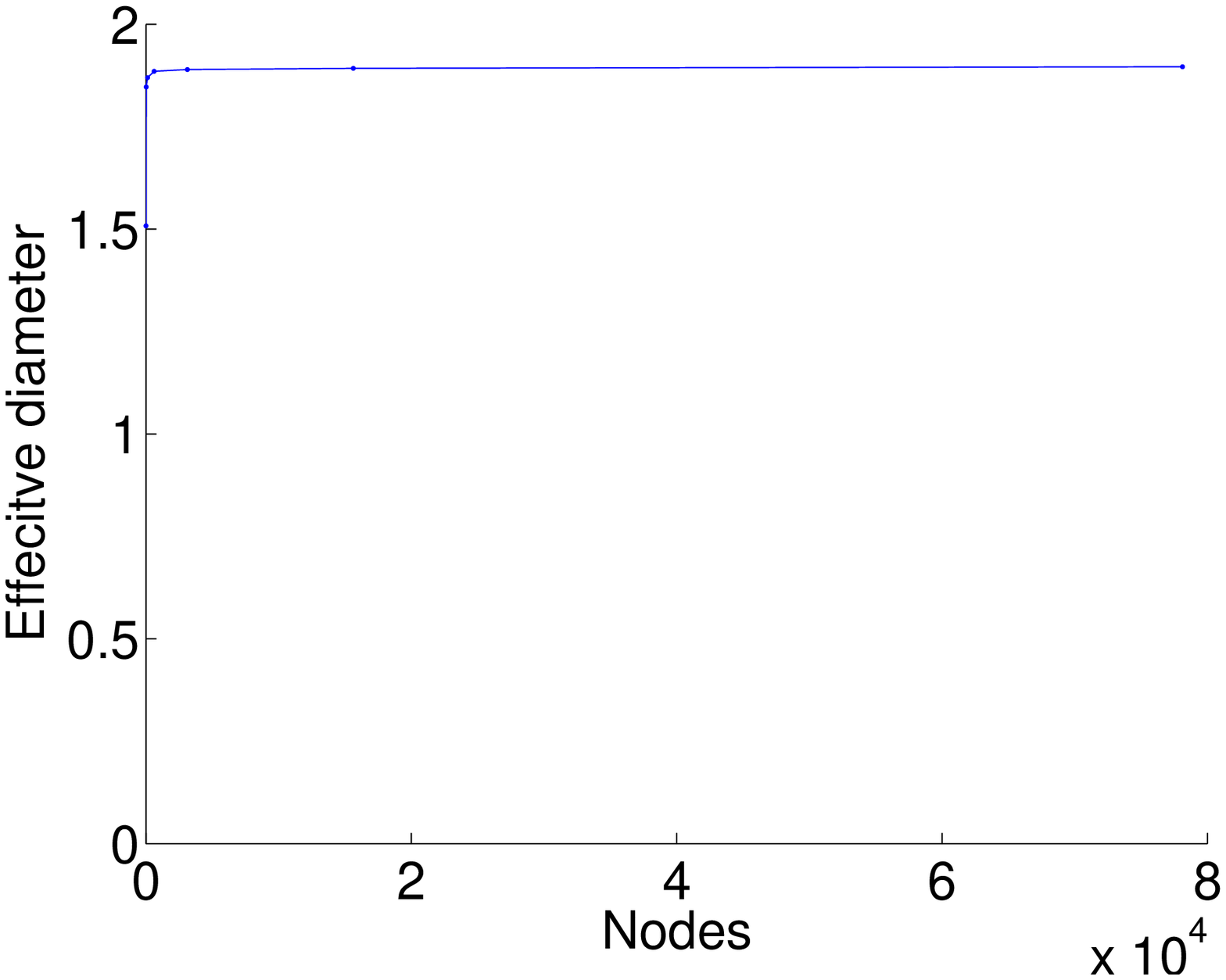} &
  \includegraphics[width=0.2\textwidth]{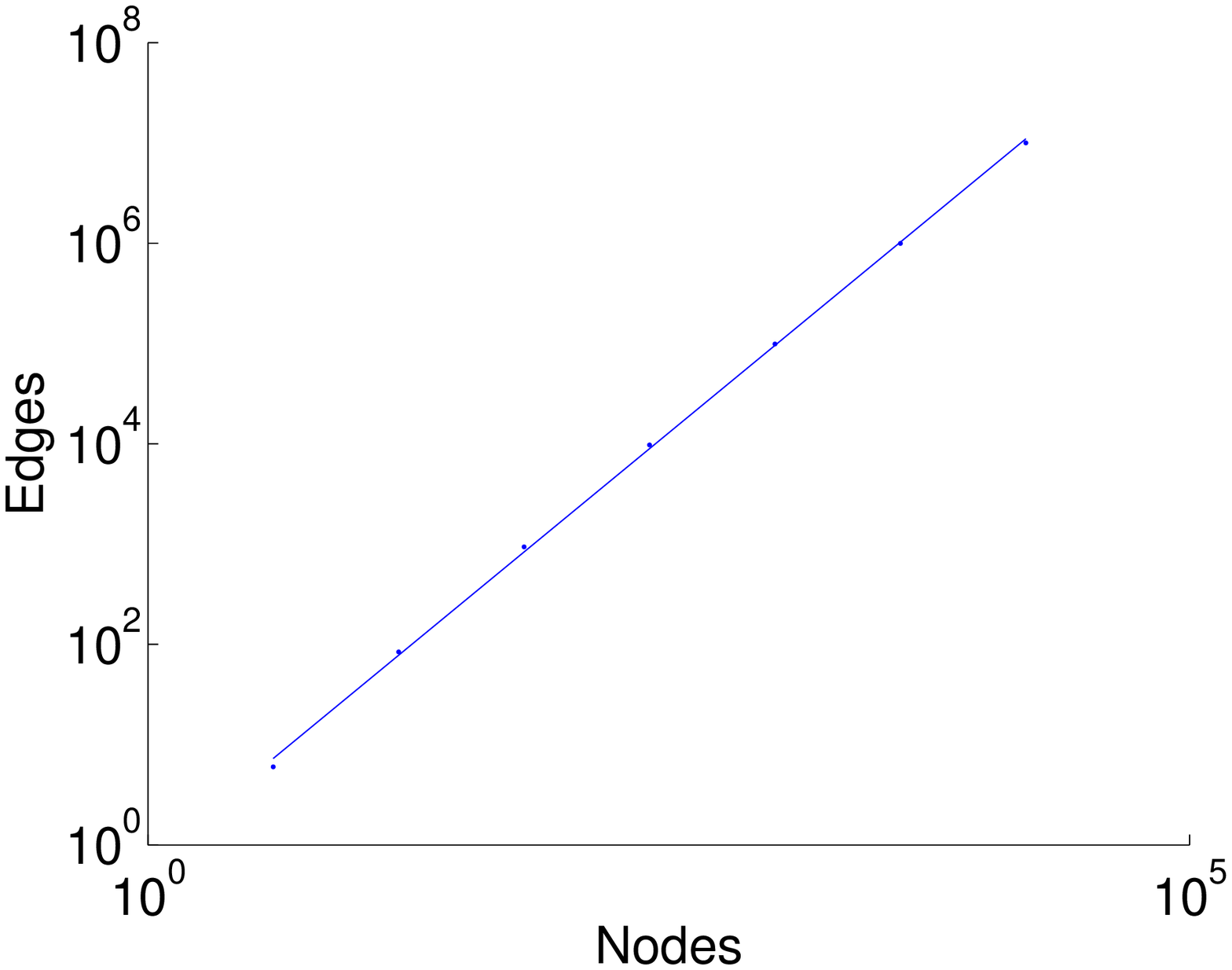} \\
  \raisebox{1.5em}{\includegraphics[width=0.05\textwidth]{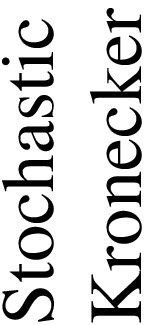}} &
  \includegraphics[width=0.2\textwidth]{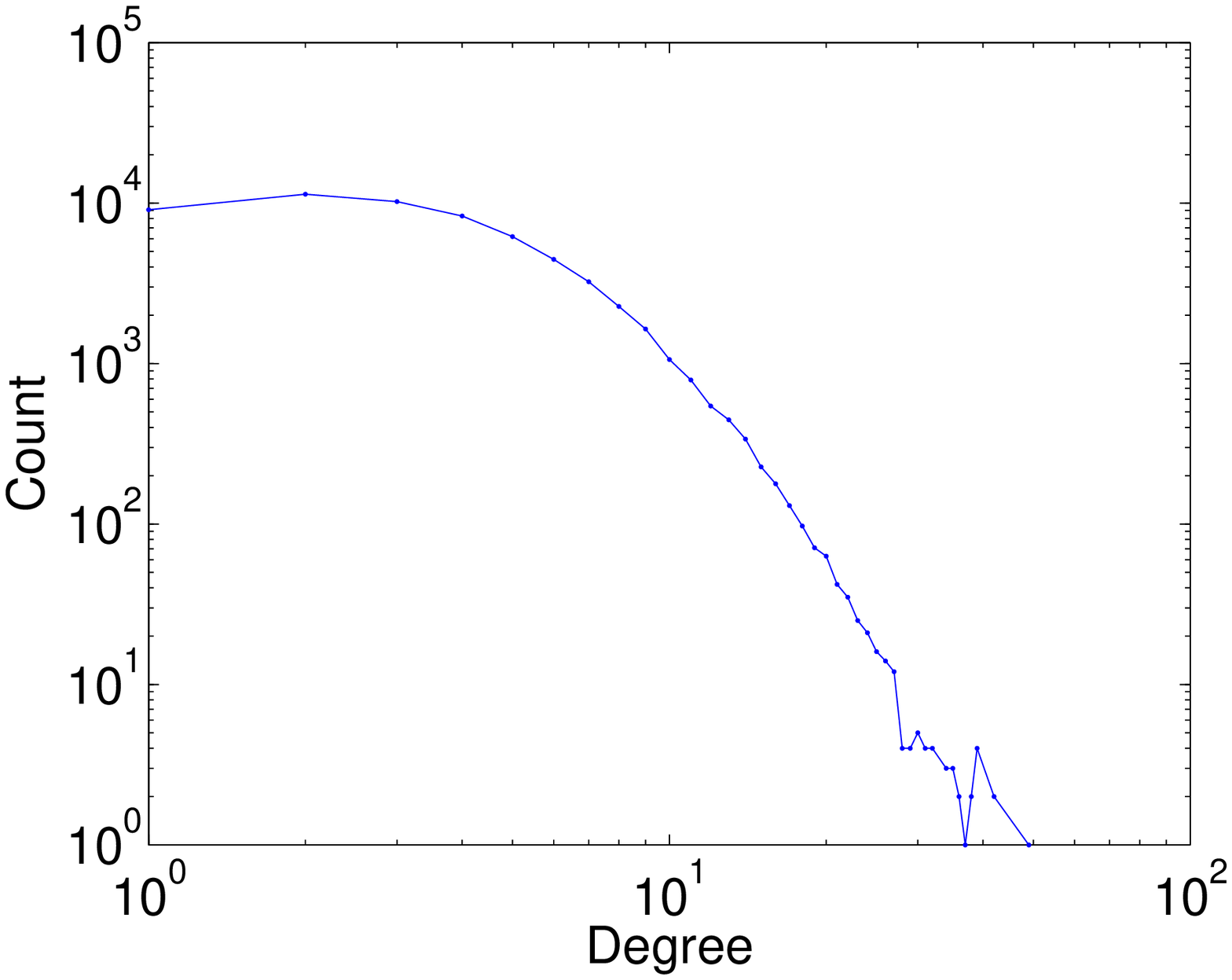} &
  \includegraphics[width=0.2\textwidth]{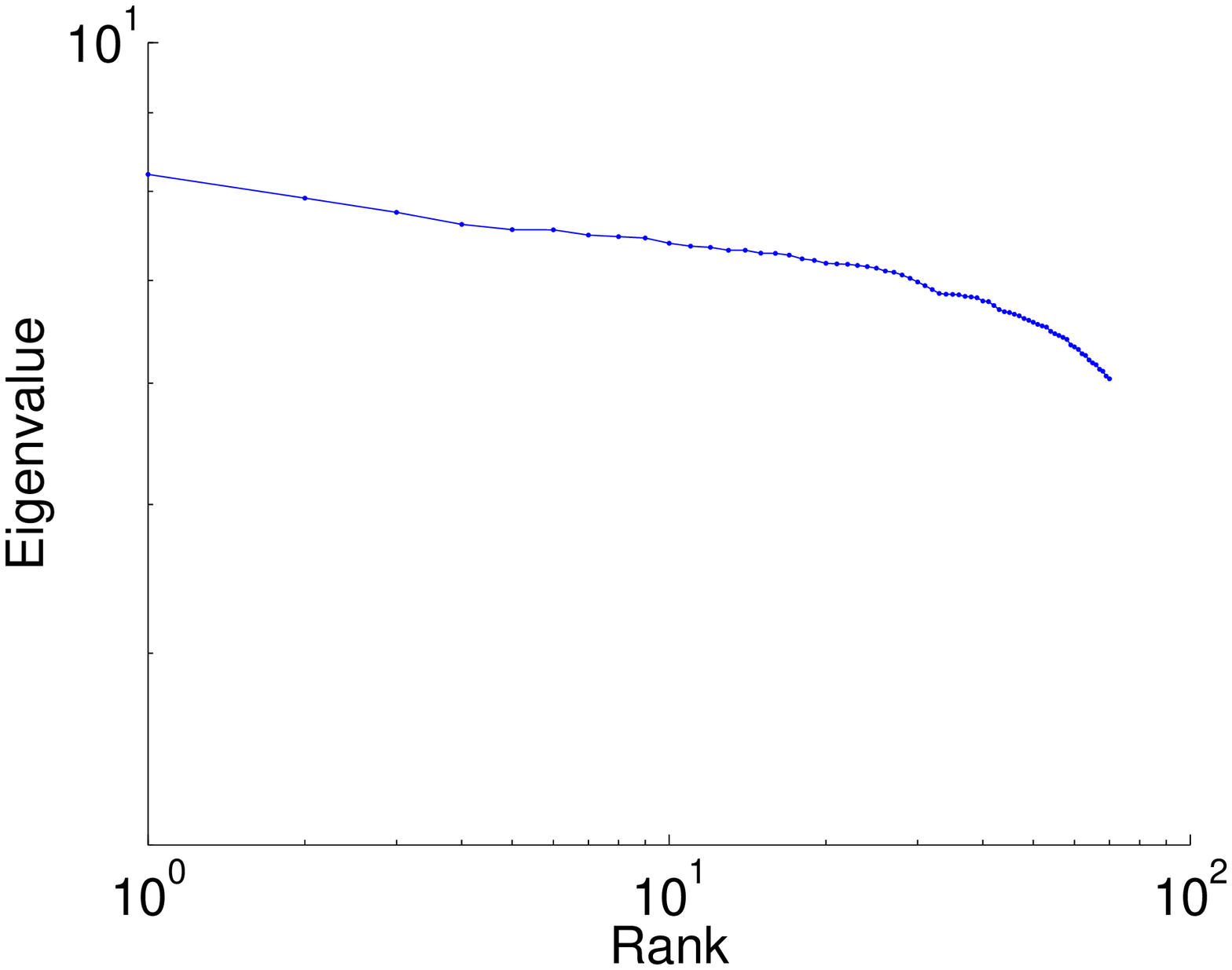} &
  \includegraphics[width=0.2\textwidth]{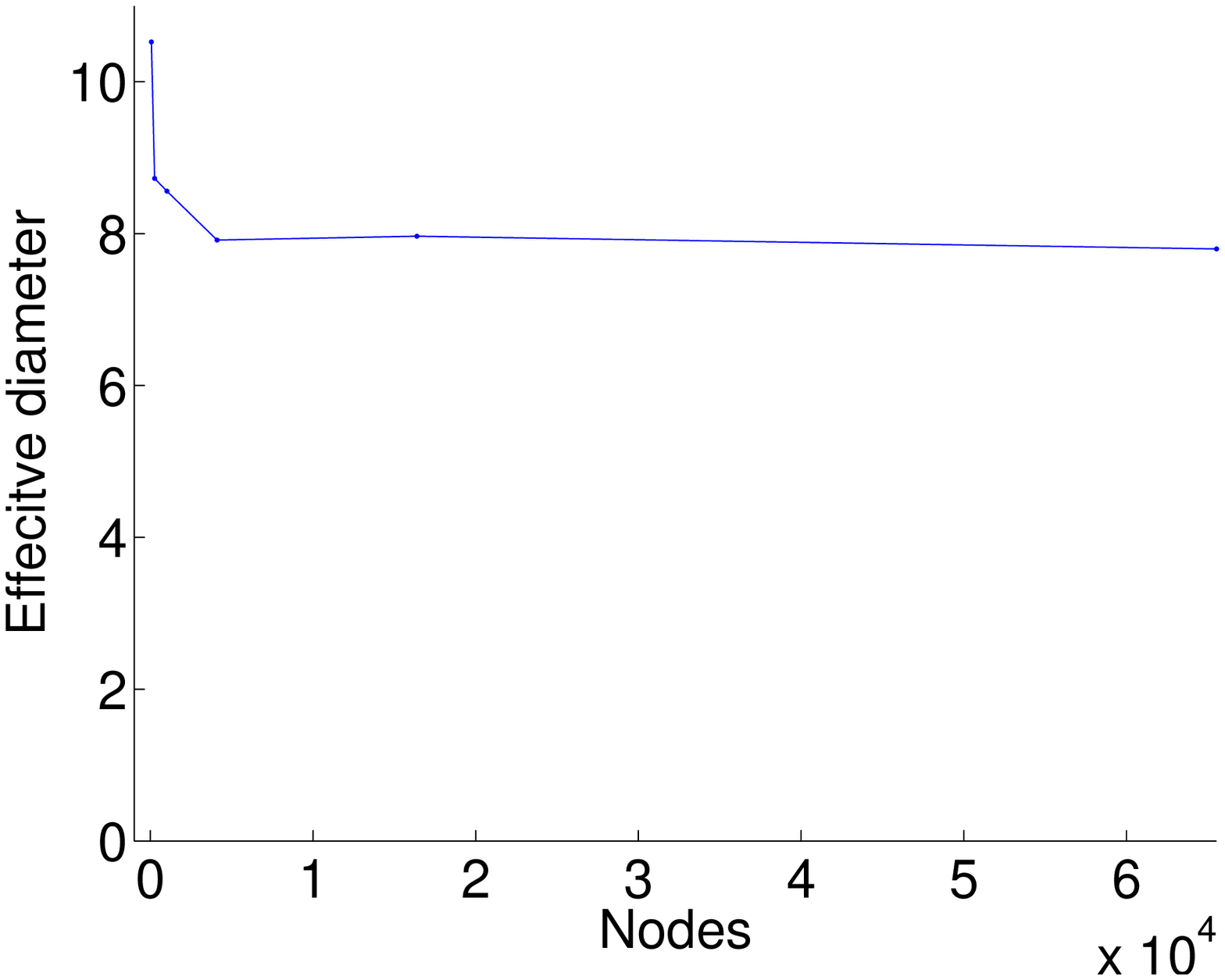} &
  \includegraphics[width=0.2\textwidth]{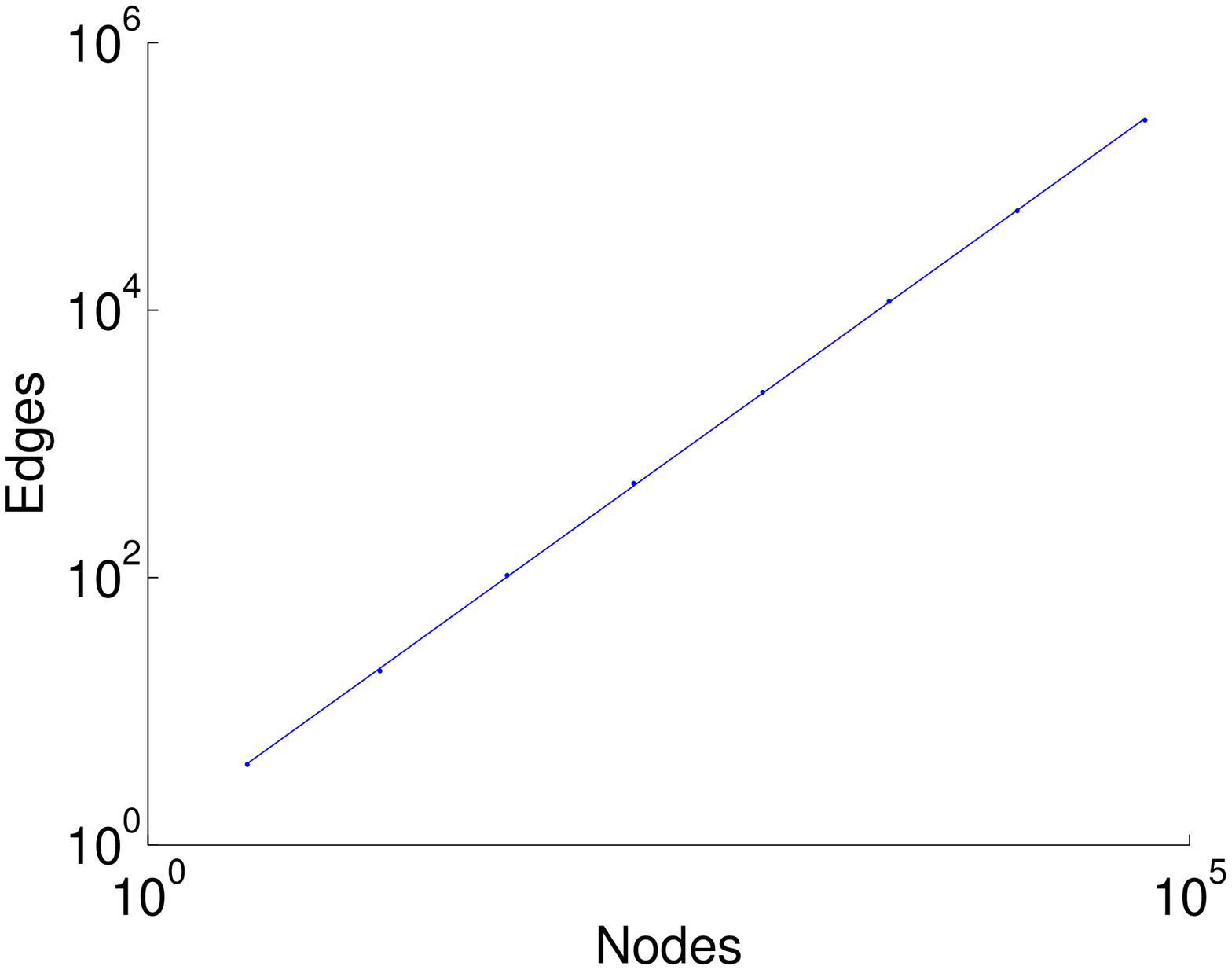} \\
  & (a) Degree & (b) Scree plot & (c) Diameter & (d) DPL\\
  & distribution & & over time &
  \end{tabular}
  \caption{{\em Citation network (\dataset{Cit-hep-th}):} Patterns from the
  real graph (top row), the deterministic Kronecker graph with $\kn{1}$
  being a star graph on 4 nodes (center + 3 satellites) (middle row), and
  the Stochastic Kronecker graph ($\alpha=0.41,~\beta=0.11$ -- bottom
  row). {\em Static} patterns: (a) is the PDF of degrees in the graph
  (log-log scale), and (b) the distribution of eigenvalues (log-log
  scale). {\em Temporal} patterns: (c) gives the effective diameter over
  time (linear-linear scale), and (d) is the number of edges versus number
  of nodes over time (log-log scale). Notice that the \SKRG\
  qualitatively matches all the patterns very well.}
\label{fig:KronArXiv}
\end{center}
\end{figure}

Similarly, Figure~\ref{fig:KronAS} shows plots for the static patterns in
the {\em Autonomous systems} (\dataset{As-RouteViews}) graph.  Recall that
we analyze a single, static network snapshot in this case.  In addition to
the degree distribution and scree plot, we also show two typical
plots~\cite{chakrabarti04rmat}: the distribution of {\em network values}
(principal eigenvector components, sorted, versus rank) and the {\em
hop-plot} (the number of reachable pairs $g(h)$ within $h$ hops or less,
as a function of the number of hops $h$). Notice that, again, the
Stochastic Kronecker graph matches well the properties of the real graph.

\begin{figure}[t]
\begin{center}
  \begin{tabular}{ccccc}
   \raisebox{1.5em}{\includegraphics[width=0.03\textwidth]{FIG/label-real}} &
    \includegraphics[width=0.2\textwidth]{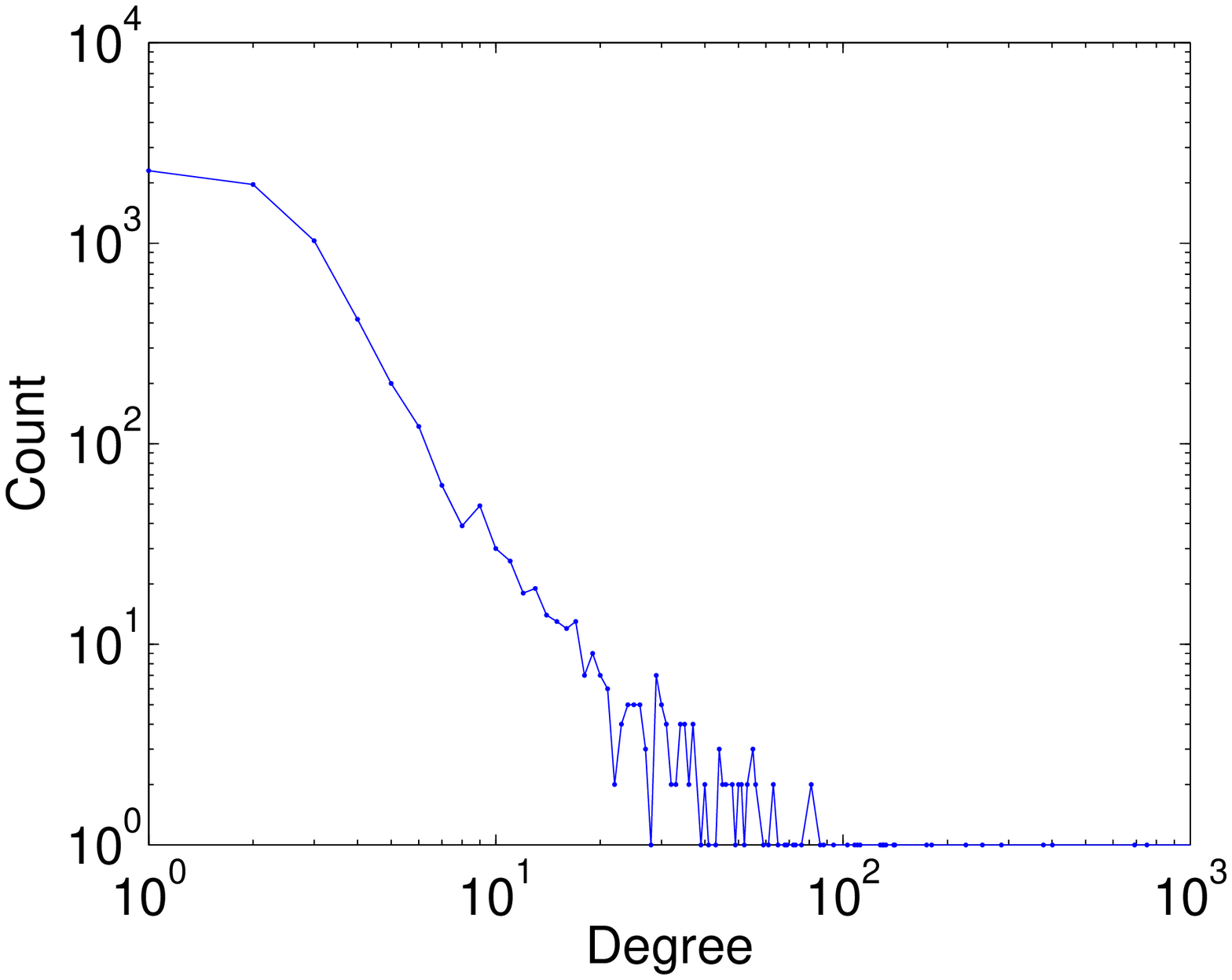} &
    \includegraphics[width=0.2\textwidth]{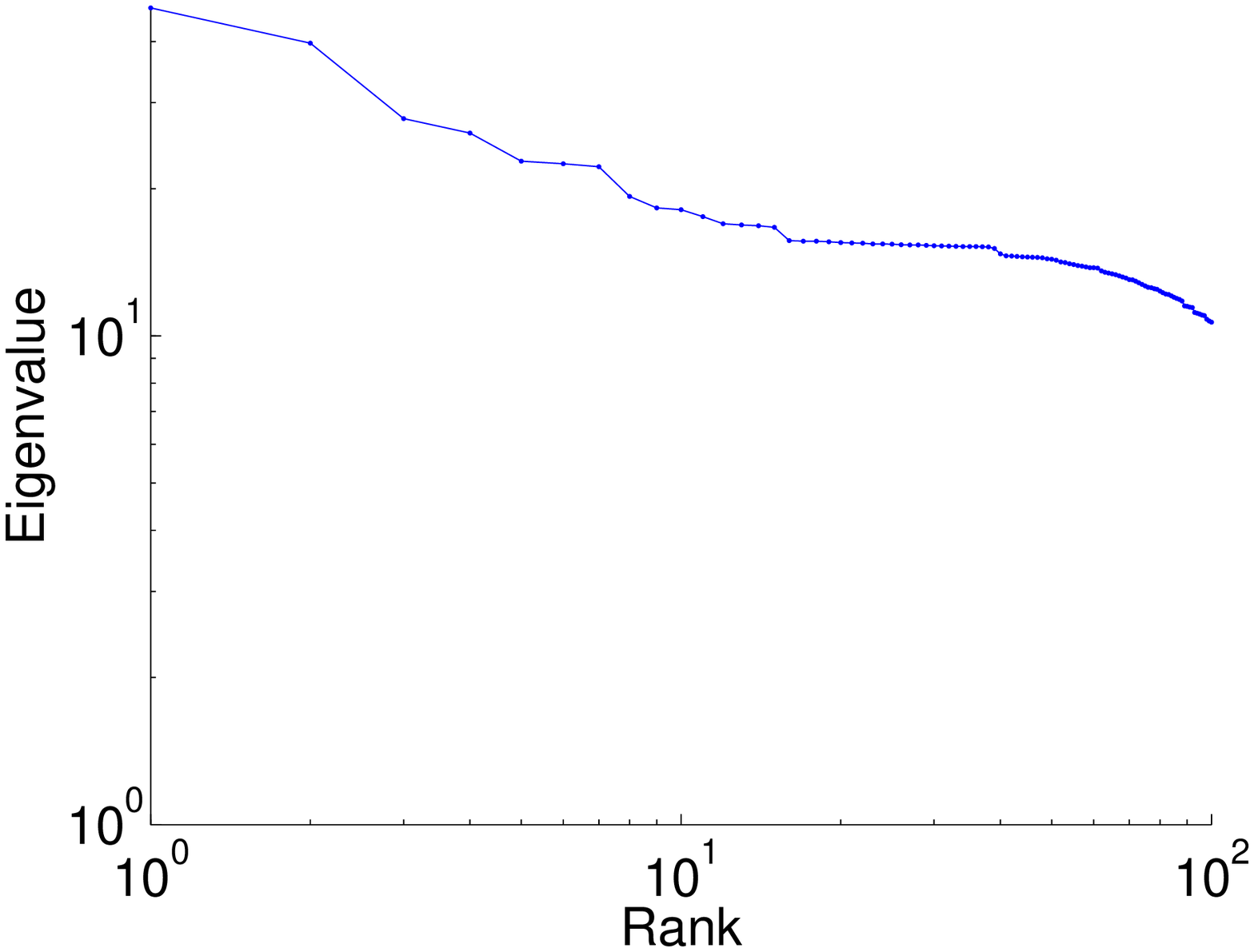} &
    \includegraphics[width=0.2\textwidth]{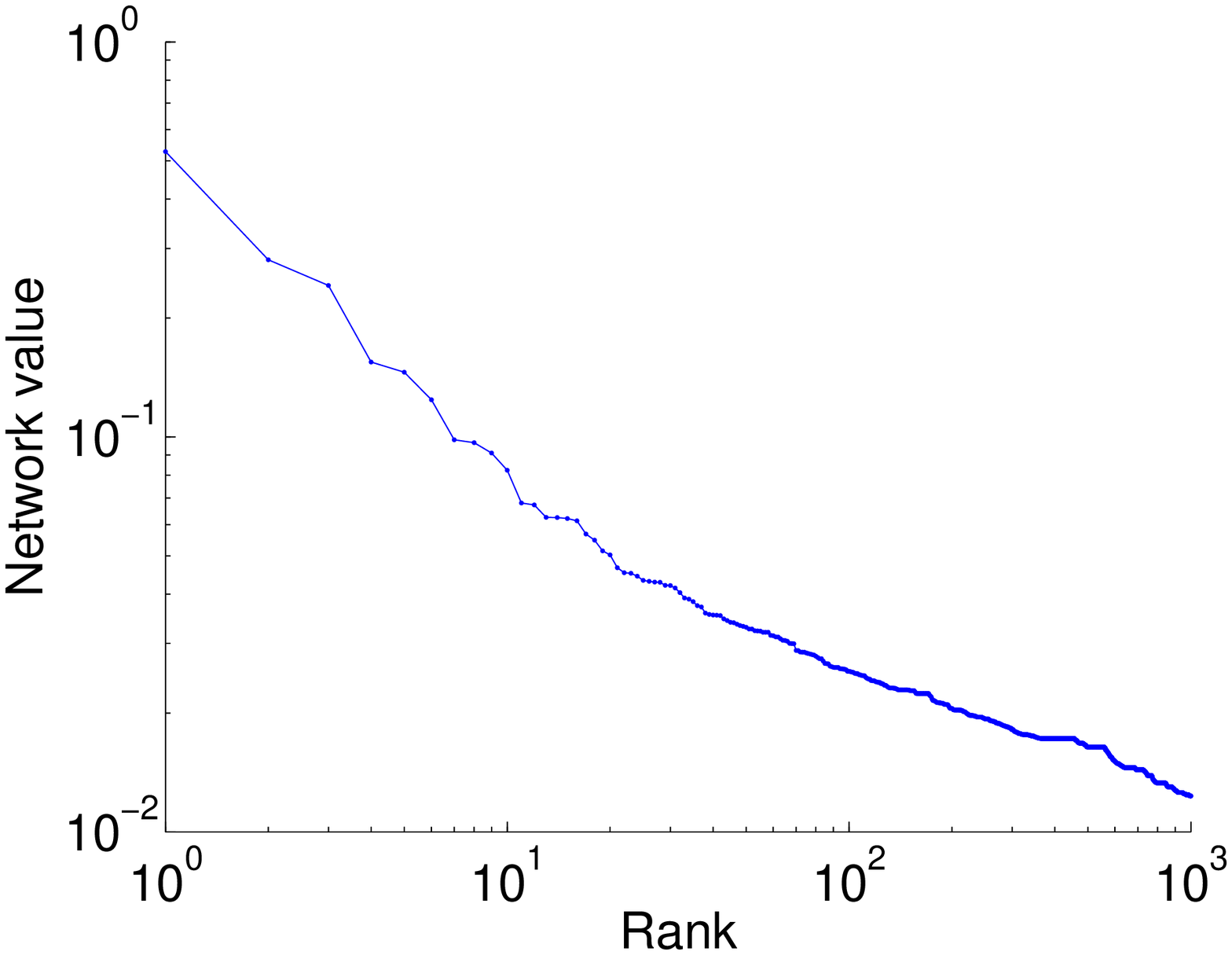} &
    \includegraphics[width=0.2\textwidth]{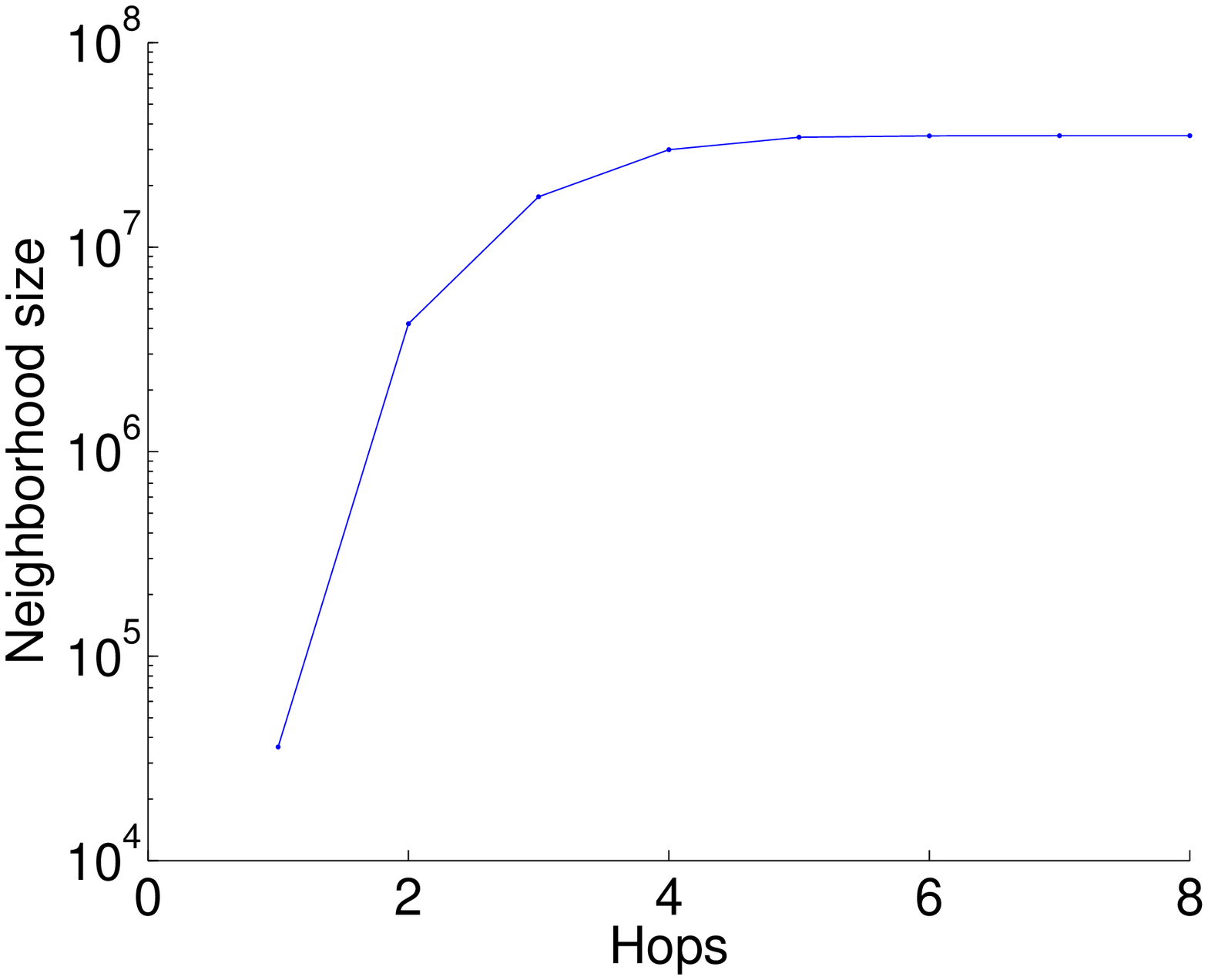} \\
    \raisebox{1.5em}{\includegraphics[width=0.05\textwidth]{FIG/label-sth}} &
    \includegraphics[width=0.2\textwidth]{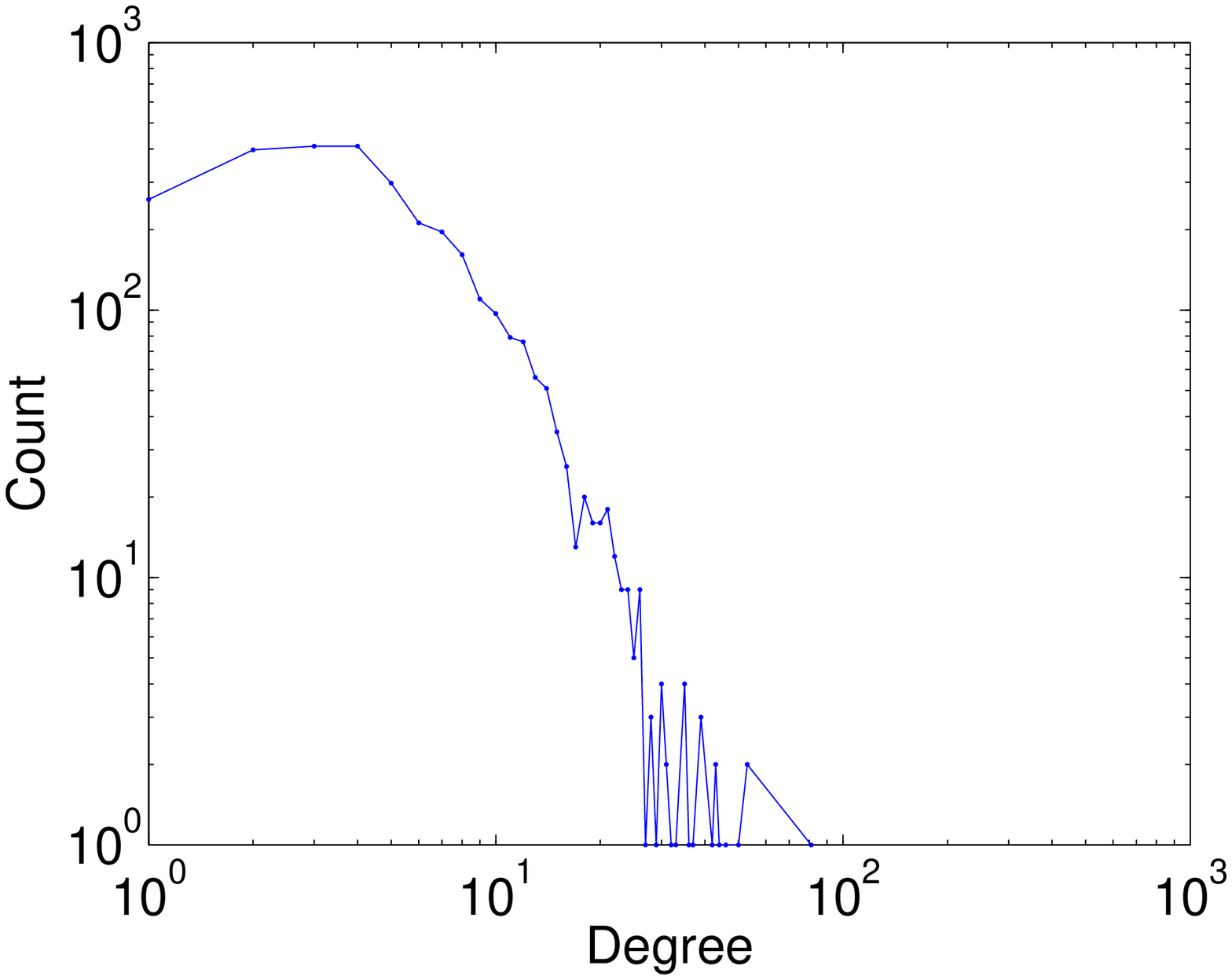} &
    \includegraphics[width=0.2\textwidth]{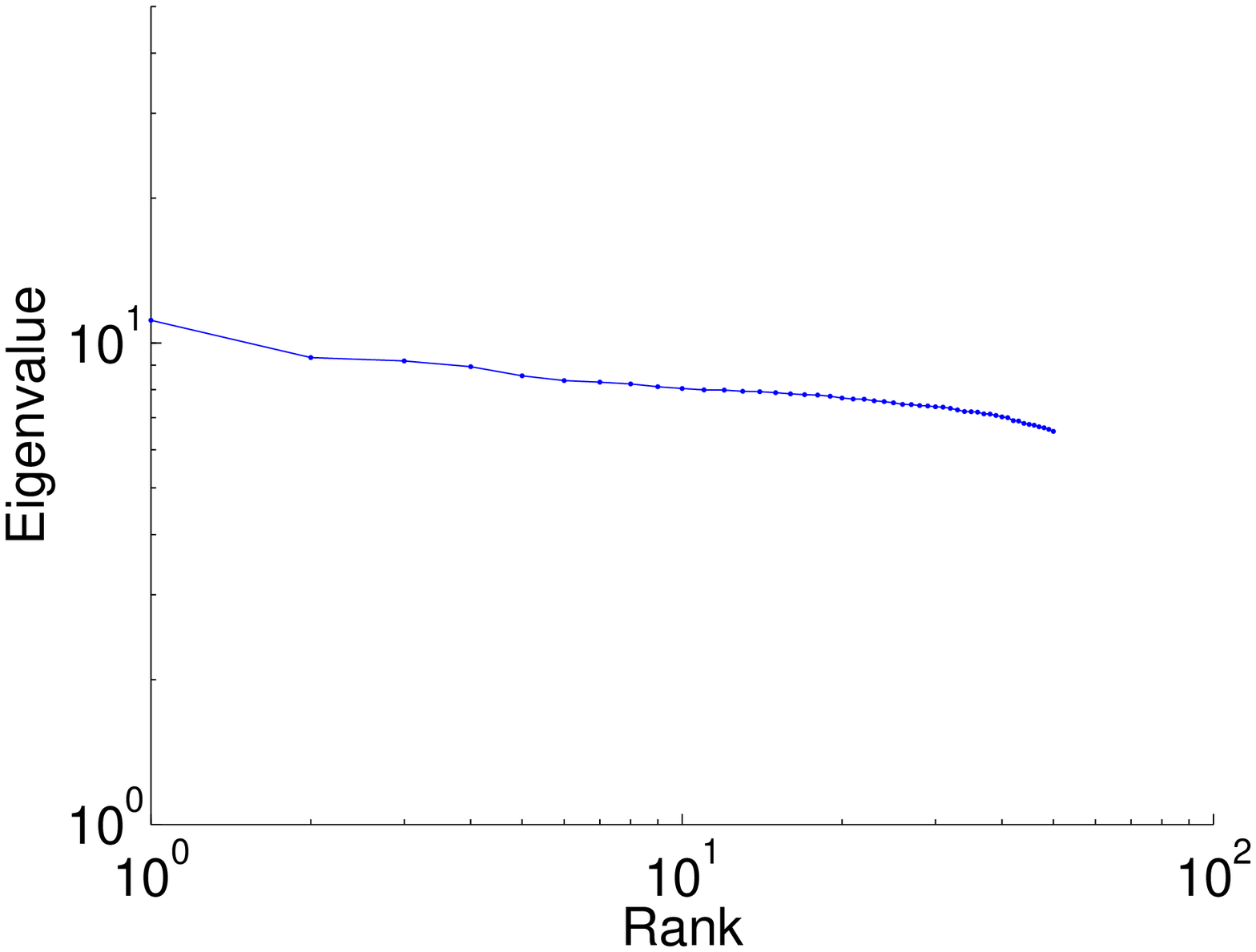} &
    \includegraphics[width=0.2\textwidth]{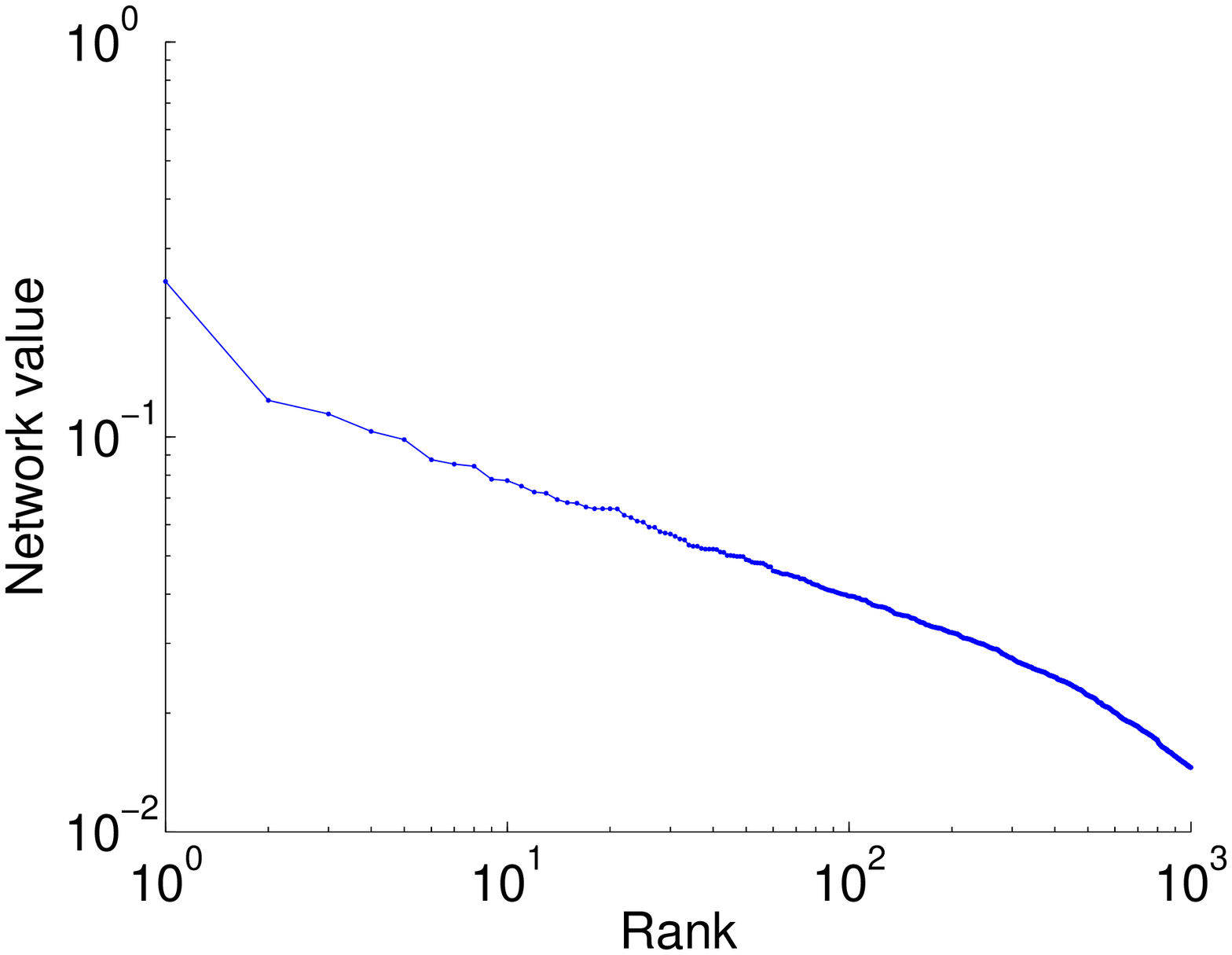} &
    \includegraphics[width=0.2\textwidth]{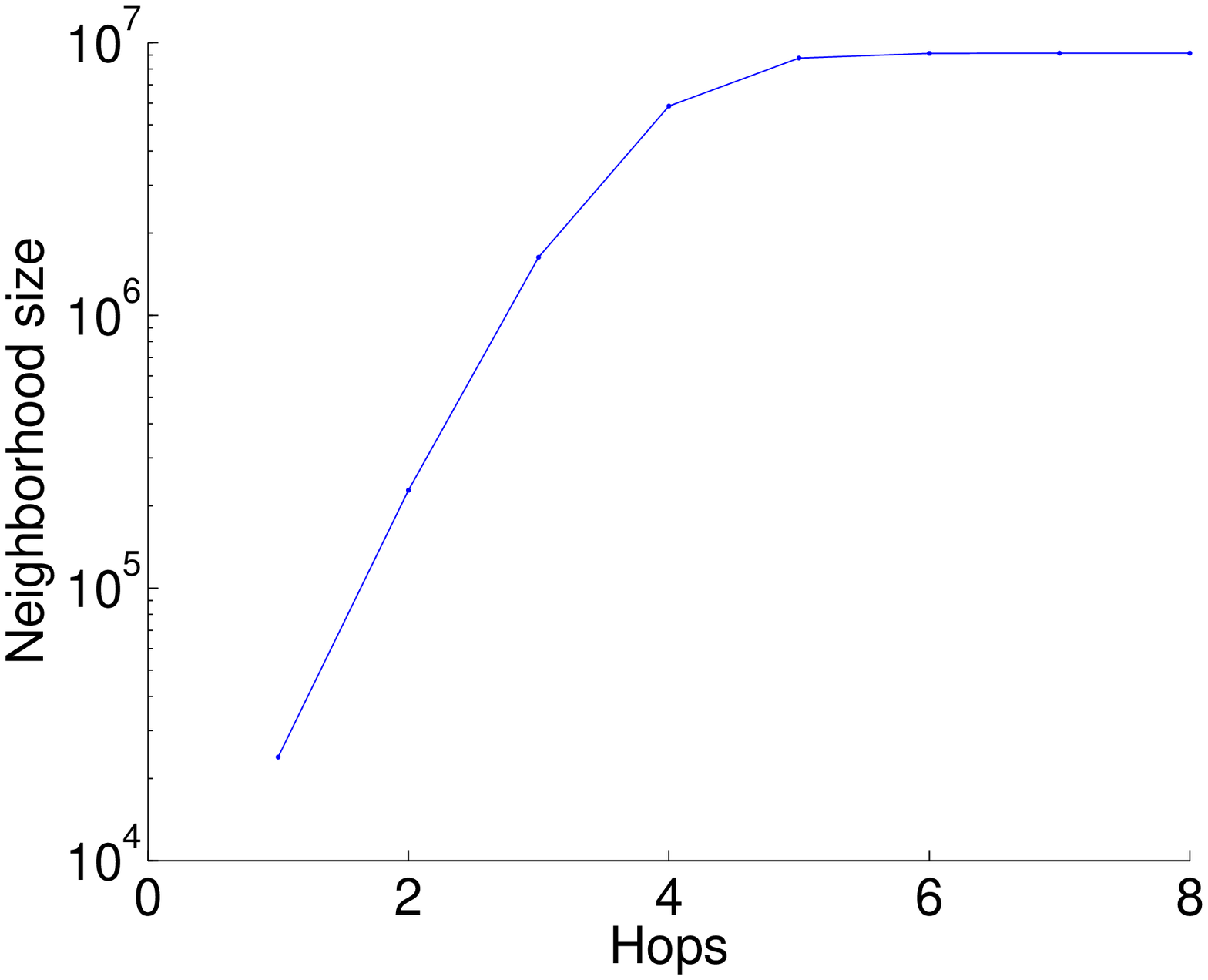} \\
    & (a) Degree & (b) Scree plot & (c) ``Network value'' & (d) ``Hop-plot''
    \\
    & distribution & & distribution &
  \end{tabular}
  \caption{{\em Autonomous systems (\dataset{As-RouteViews}):} Real (top)
  versus Kronecker (bottom). Columns (a) and (b) show the degree
  distribution and the scree plot, as before. Columns (c) and (d) show two
  more static patterns (see text). Notice that, again, the Stochastic Kronecker
  graph matches well the properties of the real graph.}
\label{fig:KronAS}
\end{center}
\end{figure}

\subsection{Parameter space of Kronecker graphs}

Last we present simulation experiments that investigate the parameter
space of Stochastic Kronecker graphs.

\begin{figure}[t]
\begin{center}
  \begin{tabular}{ccc}
    \includegraphics[width=0.3\textwidth]{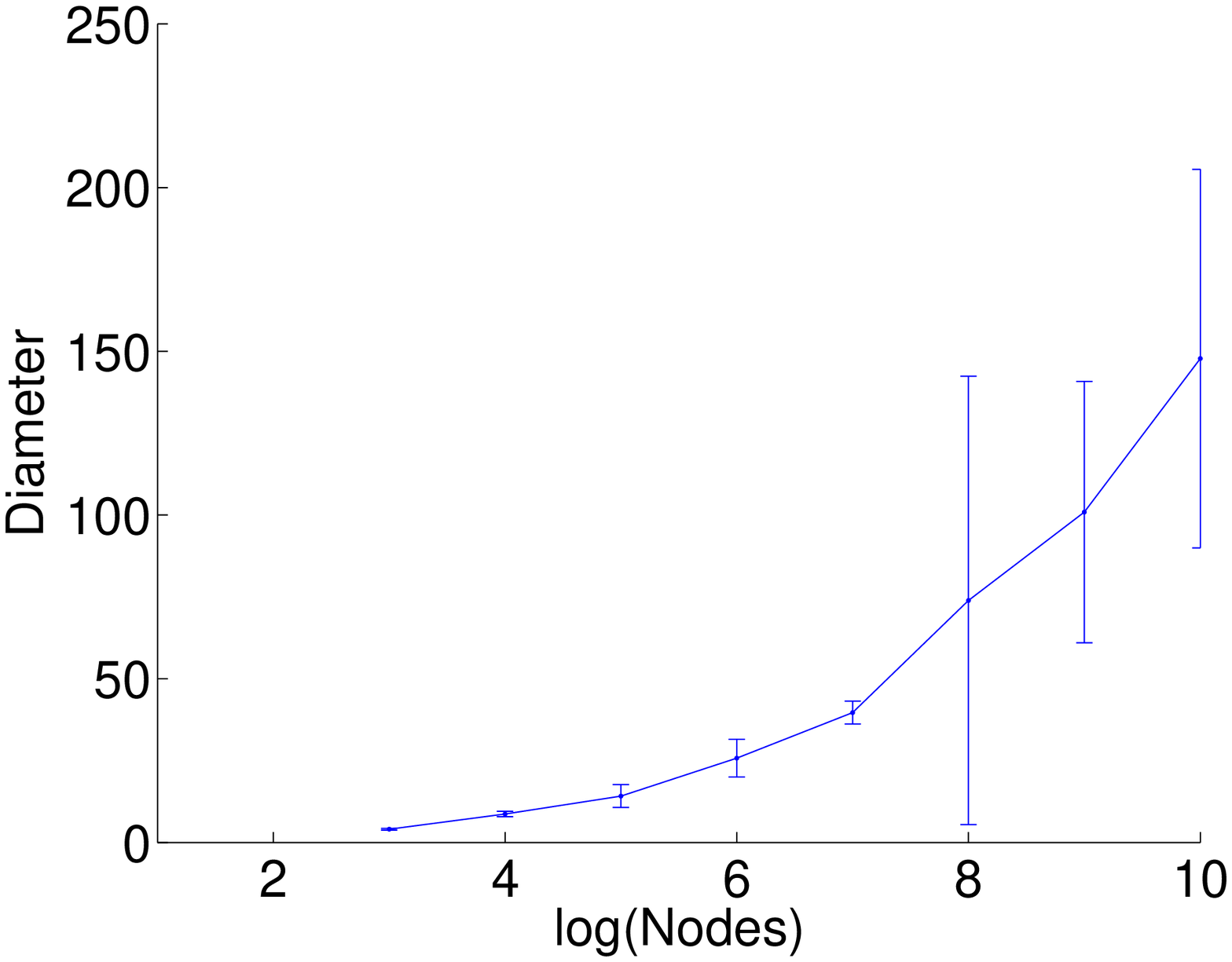} &
    \includegraphics[width=0.3\textwidth]{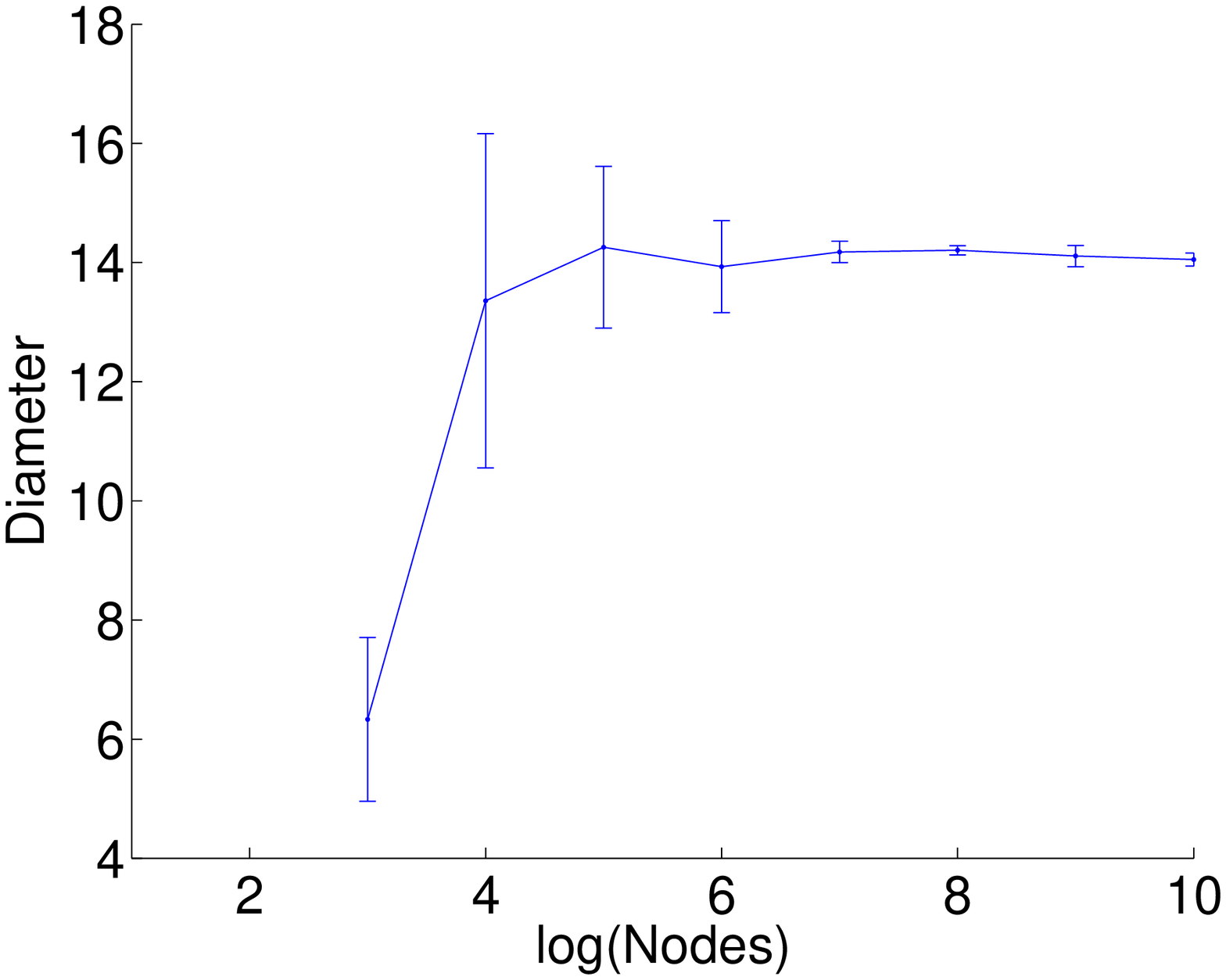} &
    \includegraphics[width=0.3\textwidth]{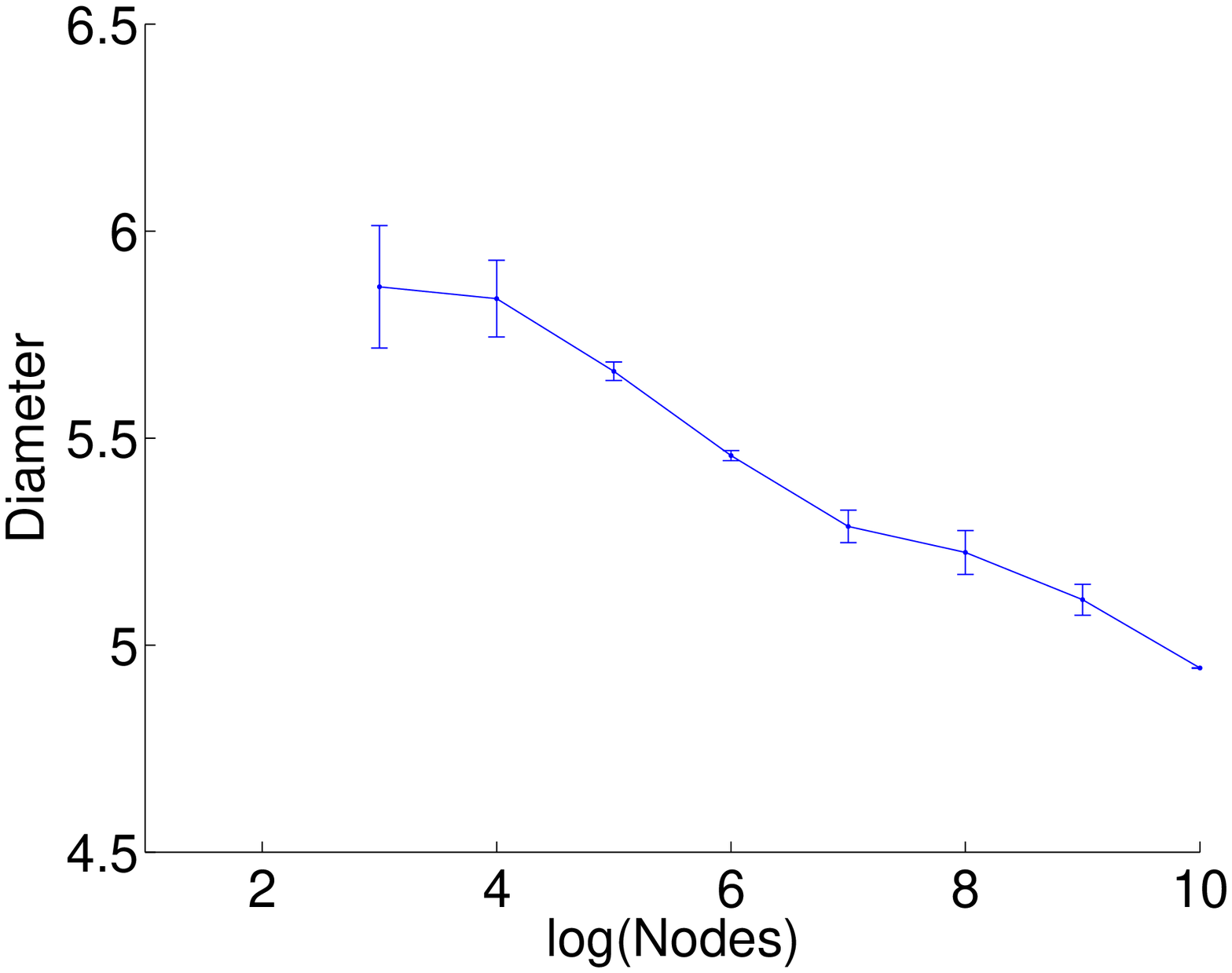} \\
    (a) Increasing diameter & (b) Constant diameter & (c) Decreasing diameter \\
    $\alpha = 0.38, \beta=0$ & $\alpha = 0.43, \beta=0$ & $\alpha = 0.54, \beta=0$ \\
  \end{tabular}
  \caption{\new{Effective diameter over time for a 4-node chain initiator graph. After
  each consecutive Kronecker power we measure the effective diameter. We
  use different settings of $\alpha$ parameter. $\alpha = 0.38, 0.43,
  0.54$ and $\beta=0$, respectively.}}
 \label{fig:Kron4star_diam}
\end{center}
\end{figure}

First, in Figure~\ref{fig:Kron4star_diam} we show the ability of Kronecker
Graphs to generate networks with increasing, constant and
decreasing/stabilizing effective diameter. We start with a 4-node chain
initiator graph (shown in top row of Figure~\ref{fig:KronSpy4chain}), setting each ``1'' of $\kn{1}$ to $\alpha$ and each ``0''
to $\beta=0$ to obtain $\pn{1}$ that we then use to generate a growing
sequence of graphs. We plot the effective diameter of each $R(\pn{k})$ as
we generate a sequence of growing graphs $R(\pn{2}), R(\pn{3}), \ldots,
R(\pn{10})$. $R(\pn{10})$ has exactly 1,048,576 nodes. Notice Stochastic
Kronecker graphs is a very flexible model. When the generated graph is
very sparse (low value of $\alpha$) we obtain graphs with slowly
increasing effective diameter (Figure~\ref{fig:Kron4star_diam}(a)). For
intermediate values of $\alpha$ we get graphs with constant diameter
(Figure~\ref{fig:Kron4star_diam}(b)) and that in our case also slowly
densify with densification exponent $a=1.05$. Last, we see an example
of a graph with shrinking/stabilizing effective diameter. Here we set the
$\alpha=0.54$ which results in a densification exponent of $a=1.2$. Note that
these observations are not contradicting Theorem~\ref{thm:KronDia}.
Actually, these simulations here agree well with the analysis
of~\cite{mahdian07kronecker}.

\begin{figure}[t]
\begin{center}
  \begin{tabular}{ccc}
    \includegraphics[width=0.32\textwidth]{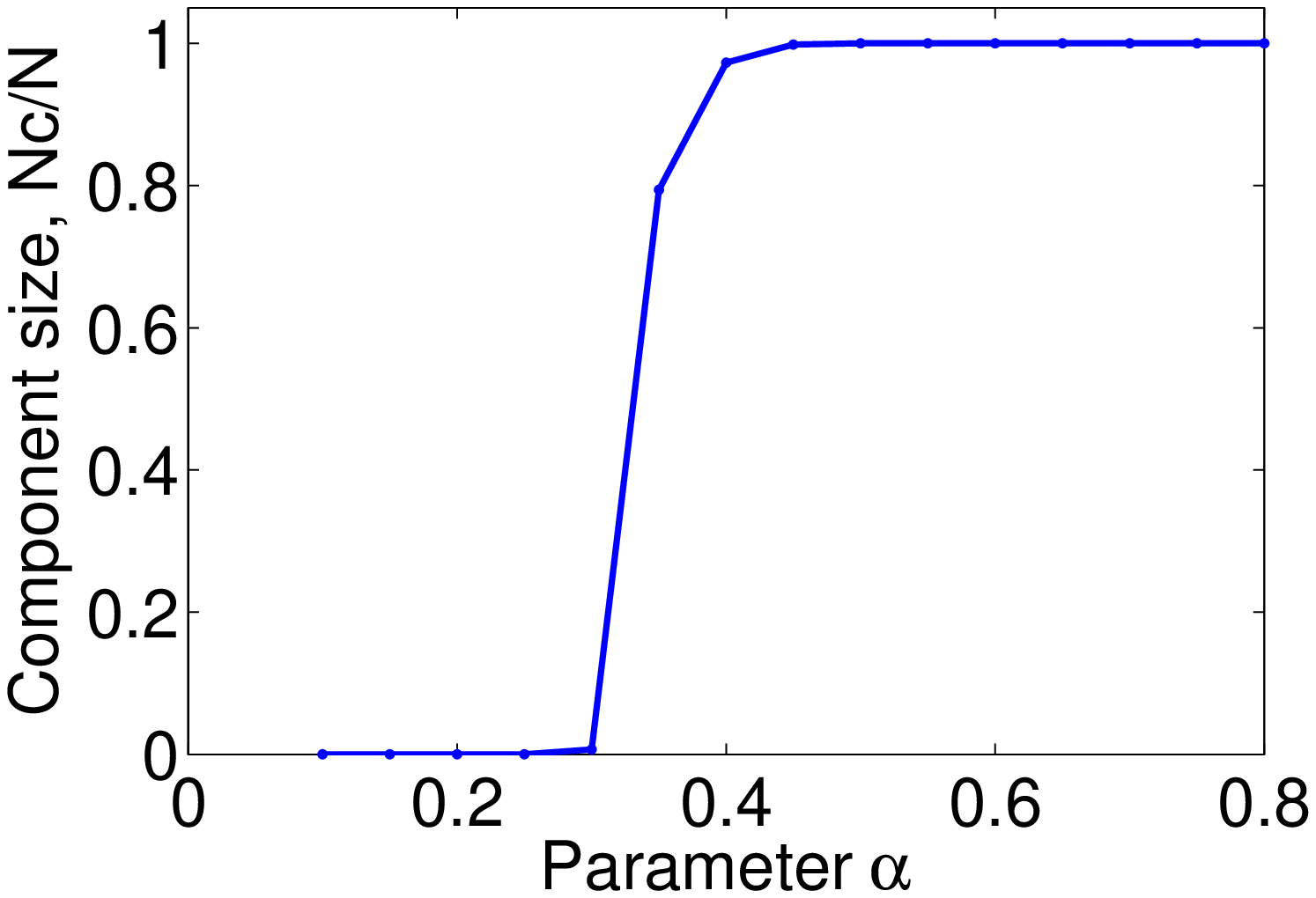} &
    \includegraphics[width=0.32\textwidth]{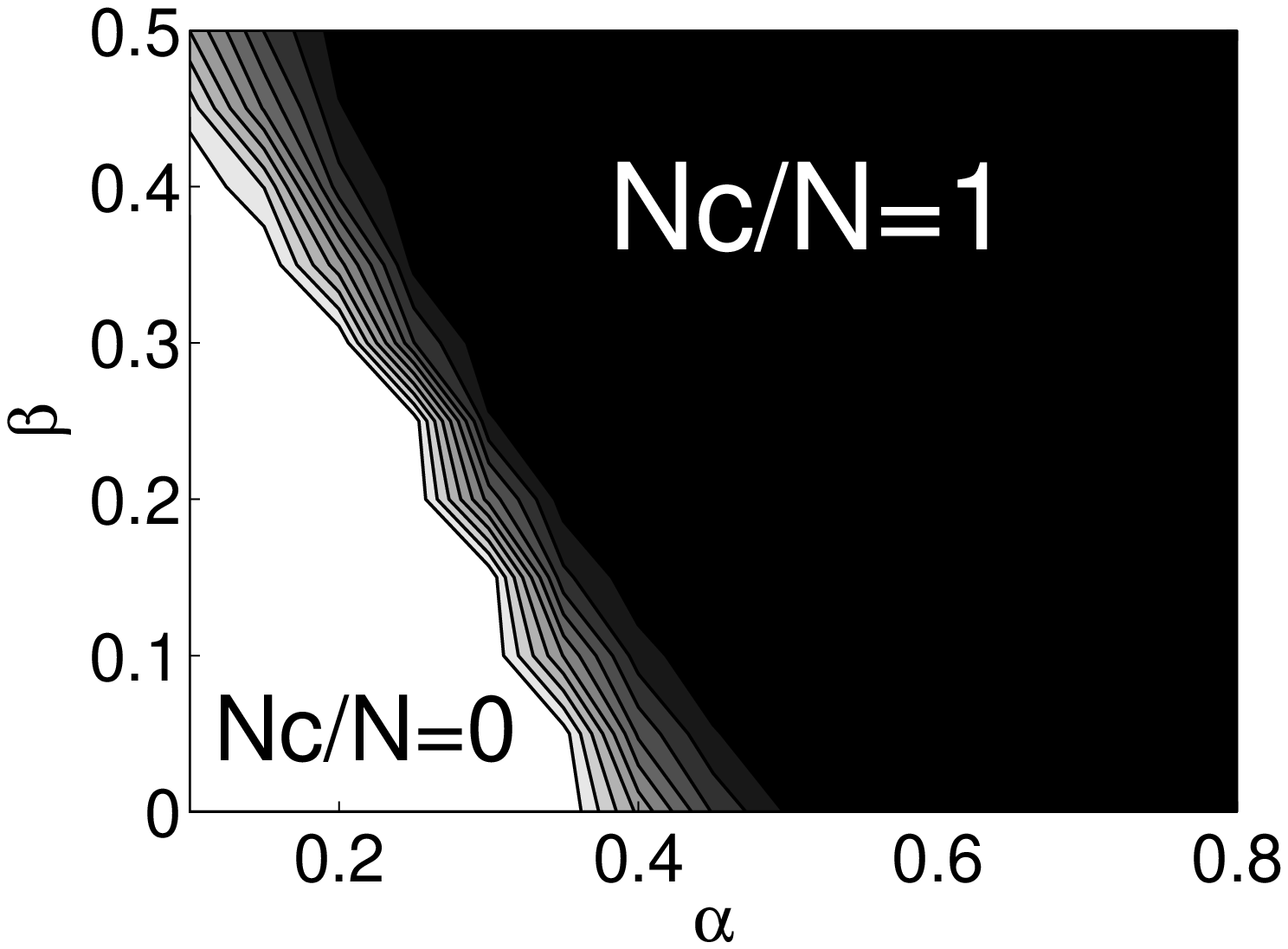} &
    \includegraphics[width=0.32\textwidth]{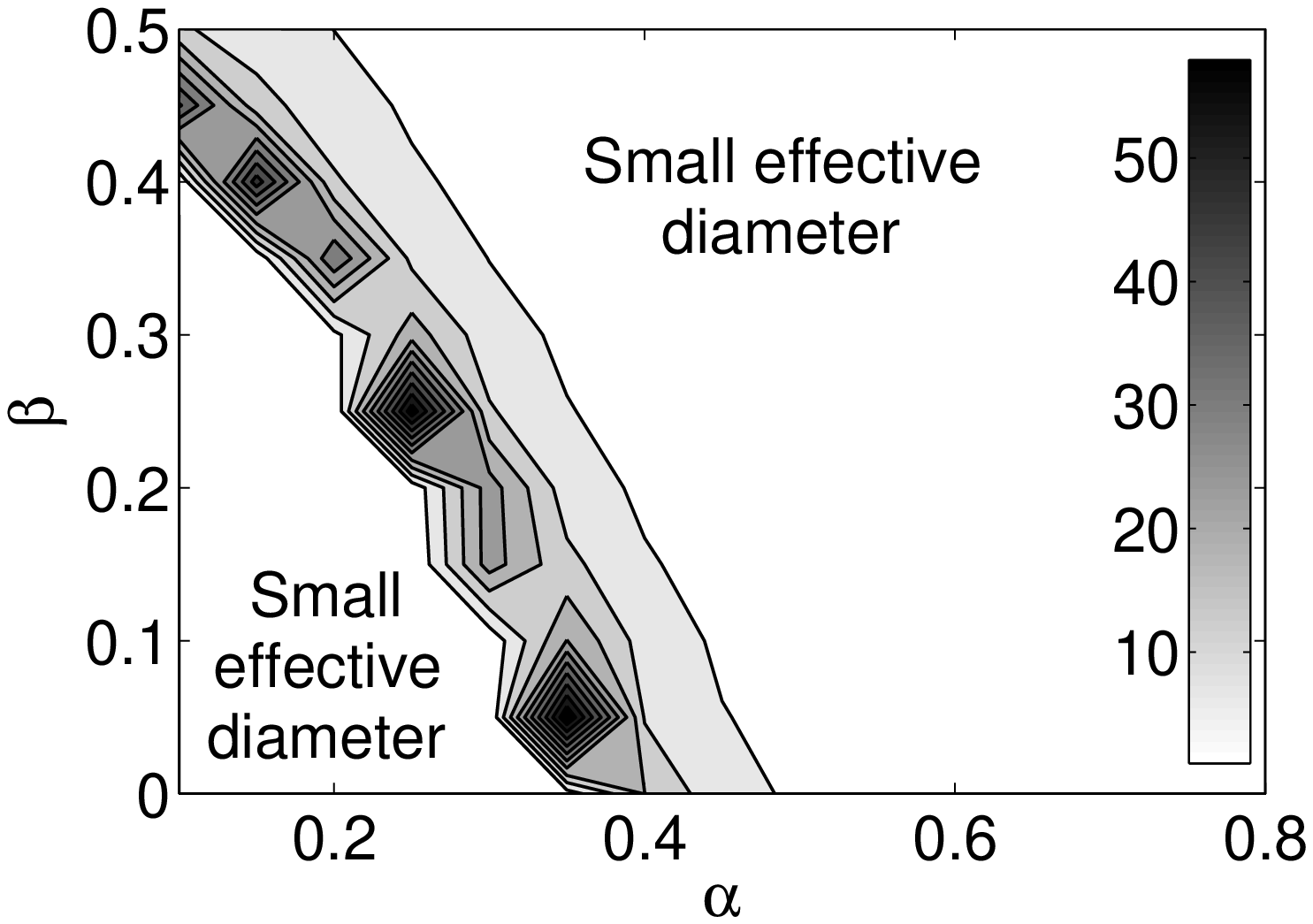} \\
    (a) Largest component size & (b) Largest component size & (c) Effective diameter \\
  \end{tabular}
  \caption{Fraction of nodes in the largest weakly connected component
  ($N_c/N$) and the effective diameter for 4-star initiator graph. (a) We
  fix $\beta=0.15$ and vary $\alpha$. (b) We vary both $\alpha$ and
  $\beta$. (c) Effective diameter of the network, if network is
  disconnected or very dense path lengths are short, the diameter is large when
  the network is barely connected.}
  \label{fig:Kron4star_phase}
  \vspace{2mm}
\end{center}
\end{figure}

Next, we examine the parameter space of a Stochastic Kronecker graphs where
we choose a star on 4 nodes as a initiator graph and parameterized
with $\alpha$ and $\beta$ as before. The initiator graph and the
structure of the corresponding (deterministic) Kronecker graph adjacency
matrix is shown in top row of Figure~\ref{fig:KronSpy4chain}.

Figure~\ref{fig:Kron4star_phase}(a) shows the sharp transition in the
fraction of the number of nodes that belong to the largest weakly
connected component as we fix $\beta=0.15$ and slowly increase $\alpha$.
Such phase transitions on the size of the largest connected component also
occur in Erd\H{o}s-R\'{e}nyi random graphs.
Figure~\ref{fig:Kron4star_phase}(b) further explores this by plotting the
fraction of nodes in the largest connected component ($N_c/N$) over the
full parameter space. Notice a sharp transition between disconnected
(white area) and connected graphs (dark).

Last, Figure~\ref{fig:Kron4star_phase}(c) shows the effective diameter
over the parameter space $(\alpha, \beta)$ for the 4-node star initiator
graph. Notice that when parameter values are small, the effective diameter
is small, since the graph is disconnected and not many pairs of nodes can
be reached. The shape of the transition between low-high diameter closely
follows the shape of the emergence of the connected component. Similarly,
when parameter values are large, the graph is very dense, and the diameter
is small. There is a narrow band in parameter space where we get graphs
with interesting diameters.

\section{Kronecker graph model estimation}
\label{sec:KronKronFit}
\new{In previous sections we investigated various properties of networks
generated by the (Stochastic) Kronecker graphs model. Many of these properties
were also observed in real networks. Moreover, we also gave closed form
expressions (parametric forms) for values of these statistical network
properties, which allow us to
calculate a property (\emph{e.g.}, diameter, eigenvalue spectrum) of a
network directly from just the initiator matrix. So in principle, one could invert
these equations and directly get from a property (\emph{e.g.}, shape of
degree distribution) to the values of initiator matrix.}

\new{However, in previous sections we did not say anything about how various
network properties of a Kronecker graph correlate and interdepend. For
example, it could be the case that two network properties are mutually exclusive.
For instance, perhaps only could only match the network diameter but not the degree
distribution or vice versa. However, as we show later this is not the
case.}

Now we turn our attention to automatically estimating the Kronecker
initiator graph. The setting is that we are given a real network $G$ and
would like to find a Stochastic Kronecker initiator $\pn{1}$ that produces
a synthetic Kronecker graph $K$ that is ``similar'' to $G$. One way to
measure similarity is to compare statistical network properties, like
diameter and degree distribution, of graphs $G$ and $K$.

Comparing statistical properties already suggests a very direct approach
to this problem: One could first identify the set of network properties (statistics) to match,
then define a quality of fit metric and somehow optimize over it. For example, one
could use the KL divergence~\cite{kullback51divergence}, or the sum of
squared differences between the degree distribution of the real network
$G$ and its synthetic counterpart $K$. Moreover, as we are interested in
matching several such statistics between the networks one would have to
meaningfully combine these individual error metrics into a global error
metric. So, one would have to specify what kind of properties he or she
cares about and then combine them accordingly. This would be a hard task
as the patterns of interest have very different magnitudes and scales.
Moreover, as new network patterns are discovered, the error functions
would have to be changed and models re-estimated. And even then it is not
clear how to define the optimization procedure to maximize the quality
of fit and how to perform optimization over the parameter space.

Our approach here is different. Instead of committing to a set of network
properties ahead of time, we try to directly match the adjacency
matrices of the real network $G$ and its synthetic counterpart $K$. The
idea is that if the adjacency matrices are similar then the global
statistical properties (statistics computed over $K$ and $G$) will also
match. Moreover, by directly working with the graph itself (and not
summary statistics), we do not commit to any particular set of network
statistics (network properties/patterns) and as new statistical properties
of networks are discovered our models and estimated parameters will
still hold.

\subsection{Preliminaries}

Stochastic graph models induce probability distributions over graphs. A
generative model assigns a probability $P(\ggraph)$ to every graph
$\ggraph$. $P(\ggraph)$ is the {\em likelihood} that a given model (with a
given set of parameters) generates the graph $\ggraph$. We concentrate on the
\SKRG\ model, and consider fitting it to a real graph $\ggraph$, our data.
We use the maximum likelihood approach, \emph{i.e.}, we aim to find
parameter values, the initiator $\pn{1}$, that maximize $P(\ggraph)$
under the Stochastic Kronecker graph model.

This presents several challenges:

\begin{itemize}
\item {\bf Model selection:} a graph is a single structure, and not a
    set of items drawn independently and identically-distributed (i.i.d.)
    from some distribution. So one cannot
    split it into independent training and test sets. The fitted
    parameters will thus be best to generate a {\em particular}
    instance of a graph. Also, overfitting could be an issue since a
    more complex model generally fits better.

\item {\bf Node correspondence:} The second challenge is the node
    correspondence or node labeling problem. The graph $\ggraph$ has a set
    of $\nnodes$ nodes, and each node has a unique label (index, ID).
    Labels do not carry any particular meaning, they just uniquely
    denote or identify the nodes. One can think of this as the graph
    is first generated and then the labels (node IDs) are randomly
    assigned. This means that two isomorphic graphs that have
    different node labels should have the same likelihood. A permutation
    $\perm$ is sufficient to describe the node correspondences as it
    maps labels (IDs) to nodes of the graph. To compute the likelihood
    $P(\ggraph)$ one has to consider all node correspondences
    $P(\ggraph) = \sum_\perm P(\ggraph|\perm)P(\perm)$, where the sum
    is over all $\nnodes!$ permutations $\perm$ of $\nnodes$ nodes.
    Calculating this {\em super-exponential} sum explicitly is
    infeasible for any graph with more than a handful of nodes.
    Intuitively, one can think of this summation as some kind of graph
    isomorphism test where we are searching for best correspondence
    (mapping) between nodes of $G$ and $\pmat$.

\item {\bf Likelihood estimation:} Even if we assume one can efficiently
solve the node correspondence problem, calculating $P(\ggraph|\perm)$
    naively takes $O(\nnodes^2)$ as one has to evaluate the
    probability of each of the $\nnodes^2$ possible edges in the graph
    adjacency matrix. Again, for graphs of size we want to model here,
    approaches with quadratic complexity are infeasible.
\end{itemize}

To develop our solution we use sampling to avoid the super-exponential sum
over the node correspondences. By exploiting the structure of the
Kronecker matrix multiplication we develop an algorithm to evaluate
$P(\ggraph|\perm)$ in {\em linear} time $O(\nedges)$. Since real graphs
are {\em sparse}, \emph{i.e.}, the number of edges is roughly of the same
order as the number of nodes, this makes fitting of Kronecker graphs to
large networks feasible.

\subsection{Problem formulation}

Suppose we are given a graph $\ggraph$ on $\nnodes = \nzero^k$ nodes (for
some positive integer $k$), and an $\nzero \times \nzero$ \SKRG\ initiator
matrix $\pn{1}$. Here $\pn{1}$ is a parameter matrix, a set of parameters
that we aim to estimate. For now also assume $\nzero$, the size of the
initiator matrix, is given. Later we will show how to automatically select
it. Next, using $\pn{1}$ we create a \SKRG\ probability matrix $\pn{k}$,
where every entry $\pij{uv}$ of $\pn{k}$ contains a probability that node
$u$ links to node $v$. We then evaluate the probability that $\ggraph$ is
a realization of $\pn{k}$. The task is to find such $\pn{1}$ that has the
highest probability of realizing (generating) $\ggraph$.

Formally, we are solving:

\begin{equation}
  \arg \max_{\pn{1}} P(\ggraph | \pn{1})
  \label{eq:KronMaxProbG}
\end{equation}

To keep the notation simpler we use standard symbol $\pzero$ to denote the
parameter matrix $\pn{1}$ that we are trying to estimate. We denote
entries of $\pzero = \pn{1} = [\thij{ij}]$, and similarly we denote $\pmat
= \pn{k} = [\pij{ij}]$. Note that here we slightly simplified the
notation: we use $\pzero$ to refer to $\pn{1}$, and $\thij{ij}$ are
elements of $\pzero$. Similarly, $\pij{ij}$ are elements of $\pmat$
($\equiv \pn{k}$). Moreover, we denote $\kgraph = R(\pmat)$, \emph{i.e.},
$K$ is a realization of the Stochastic Kronecker graph sampled from
probabilistic adjacency matrix $\pmat$.

As noted before, the node IDs are assigned arbitrarily and they carry no
significant information, which means that we have to consider all the
mappings of nodes from $\ggraph$ to rows and columns of stochastic
adjacency matrix $\pmat$. A priori all labelings are equally likely. A
permutation $\perm$ of the set $\{1, \dots, \nnodes\}$ defines this
mapping of nodes from $\ggraph$ to stochastic adjacency matrix $\pmat$. To
evaluate the likelihood of $\ggraph$ one needs to consider all possible
mappings of $\nnodes$ nodes of $\ggraph$ to rows (columns) of $\pmat$. For
convenience we work with {\em log-likelihood} $l(\pzero)$, and solve
$\hat{\pzero} = \arg \max_{\pzero} l(\pzero)$, where $l(\pzero)$ is
defined as:

\begin{eqnarray}
  l(\pzero) &=& \log P(\ggraph | \pzero)
  = \log \sum_{\perm} P(\ggraph | \pzero, \perm) P(\perm | \pzero) \nonumber \\
  &=& \log \sum_{\perm} P(\ggraph | \pzero, \perm) P(\perm)
  \label{eq:KronSumPerm}
\end{eqnarray}

The likelihood that a given initiator matrix $\pzero$ and permutation
$\perm$ gave rise to the real graph $\ggraph$, $P(\ggraph | \pzero,
\perm)$, is calculated naturally as follows. First, by using $\pzero$ we
create the Stochastic Kronecker graph adjacency matrix $\pmat = \pn{k} =
\pzero^{[k]}$. Permutation $\perm$ defines the mapping of nodes of
$\ggraph$ to the rows and columns of stochastic adjacency matrix $\pmat$.
(See Figure~\ref{fig:KronFitPGraph} for the illustration.)

We then model edges as independent Bernoulli random variables
parameterized by the parameter matrix $\pzero$. So, each entry $\pij{uv}$
of $\pmat$ gives exactly the probability of edge $(u,v)$ appearing.

We then define the likelihood:
\begin{equation}
  P(\ggraph | \pmat, \perm) =
  \prod_{(u,v) \in \ggraph}\pmat[\perm_u,\perm_v]
  \prod_{(u,v) \notin \ggraph}(1-\pmat[\perm_u,\perm_v]),
  \label{eq:KronProbGPS}
\end{equation}

where we denote $\perm_i$ as the $i^{th}$ element of the permutation
$\perm$, and $\pmat[i, j]$ is the element at row $i$, and column $j$ of
matrix $\pmat = \pzero^{[k]}$.

The likelihood is defined very naturally. We traverse the entries of
adjacency matrix $G$ and then based on whether a particular edge appeared
in $G$ or not we take the probability of edge occurring (or not) as given
by $\pmat$, and multiply these probabilities. As one has to touch all the
entries of the stochastic adjacency matrix $\pmat$ evaluating
Equation~\ref{eq:KronProbGPS} takes $O(\nnodes^2)$ time.

\begin{figure}[t]
  \begin{center}
    \includegraphics[width=0.9\textwidth]{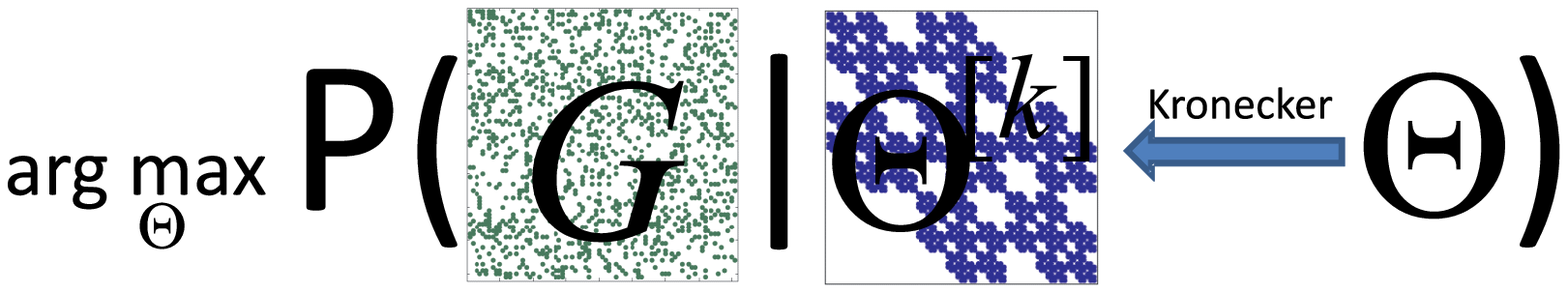}
  \end{center}
  \caption{Kronecker parameter estimation as an optimization problem. We
  search over the initiator matrices $\pzero$ ($\equiv \pn{1}$). Using
  Kronecker multiplication we create probabilistic adjacency matrix
  $\pzero^{[k]}$ that is of same size as real network $G$. Now, we
  evaluate the likelihood by simultaneously traversing and multiplying
  entries of $G$ and
  $\pzero^{[k]}$ (see Eq.~\ref{eq:KronProbGPS}). As shown by the figure
  permutation $\perm$ plays an important role, as permuting rows and
  columns of $G$ could make it look more similar to $\pzero^{[k]}$ and
  thus increase the likelihood.}
  \label{fig:KronFitPGraph}
\end{figure}

We further illustrate the process of estimating Stochastic Kronecker
initiator matrix $\pzero$ in Figure~\ref{fig:KronFitPGraph}. We search
over initiator matrices $\pzero$ to find the one that maximizes the
likelihood $P(\ggraph | \pzero)$. To estimate $P(\ggraph | \pzero)$ we are
given a concrete $\pzero$ and now we use Kronecker multiplication to
create probabilistic adjacency matrix $\pzero^{[k]}$ that is of same size
as real network $G$. Now, we evaluate the likelihood by traversing the
corresponding entries of $G$ and $\pzero^{[k]}$.
Equation~\ref{eq:KronProbGPS} basically traverses the adjacency matrix of
$\ggraph$, and maps every entry $(u,v)$ of $G$ to a corresponding entry
$(\perm_u,\perm_v)$ of $\pmat$. Then in case that edge $(u,v)$ exists in
$G$ (\emph{i.e.}, $\ggraph[u,v] = 1$) the likelihood that particular edge
existing is $\pmat[\perm_u,\perm_v]$, and similarly, in case the edge
$(u,v)$ does not exist the likelihood is simply $1 -
\pmat[\perm_u,\perm_v]$. This also demonstrates the importance of
permutation $\perm$, as permuting rows and columns of $G$ could make the
adjacency matrix look more ``similar'' to $\pzero^{[k]}$, and would
increase the likelihood.


So far we showed how to assess the quality (likelihood) of a particular
$\pzero$. So, naively one could perform some kind of grid
search to find best $\pzero$. However, this is very inefficient. A better
way of doing it is to compute the gradient of the log-likelihood
$\frac{\partial}{\partial \pzero} l(\pzero)$, and then use the
gradient to update the current estimate of $\pzero$ and move towards a
solution of higher likelihood. Algorithm~\ref{alg:KronGradDescent} gives
an outline of the optimization procedure.

However, there are several difficulties with this algorithm. First, we are
assuming gradient descent type optimization will find a good solution, \emph{i.e.}, the
problem does not have (too many) local minima. Second, we are summing over
exponentially many permutations in equation~\ref{eq:KronSumPerm}. Third,
the evaluation of equation~\ref{eq:KronProbGPS} as it is written now takes
$O(\nnodes^2)$ time and needs to be evaluated $\nnodes!$ times. So, given a concrete $\pzero$ just naively calculating the likelihood takes $O(\nnodes! \nnodes^2)$ time, and then one also has to optimize over $\pzero$.

\begin{observation}
  The complexity of calculating the likelihood $P(\ggraph|\pzero)$ of the
  graph $\ggraph$ naively is $O(\nnodes! \nnodes^2)$, where $\nnodes$ is
  the number of nodes in $\ggraph$.
  \label{obs:KronFitComplex}
\end{observation}

Next, we show that all this can be done in {\em linear time}.

\subsection{Summing over the node labelings}

To maximize equation~\ref{eq:KronMaxProbG} using
algorithm~\ref{alg:KronGradDescent} we need to obtain the gradient of the
log-likelihood $\D{\pzero}{} l(\pzero)$. We can write:
\begin{eqnarray}
  \D{\pzero}{} l(\pzero)& =&
  \frac{\sum_\perm \frac{\partial}{\partial \pzero} P(G|\perm, \pzero) P(\perm)}
    {\sum_{\perm'} P(G | \perm', \pzero) P(\perm')} \nonumber\\
  &=&
  \frac{\sum_\perm \dfrac{\partial \log P(G|\perm,
  \pzero)}{\partial \pzero}
    P(\ggraph|\perm,\pzero) P(\perm)} {P(G | \pzero)} \nonumber\\
  &=& \sum_\perm \frac{\partial \log
  P(\ggraph|\perm,\pzero)}{\partial \pzero}
    P(\perm | \ggraph,\pzero)
    \label{eq:KronSampleGrad}
\end{eqnarray}

\begin{algorithm}[t]
  \dontprintsemicolon
  \SetKwInOut{Input}{input}
  \SetKwInOut{Output}{output}
  \Input{size of parameter matrix $\nzero$, graph $\ggraph$ on $\nnodes=\nzero^k$ nodes, and
learning rate $\lambda$}
  \Output{MLE parameters $\hat\pzero$ ($\nzero\times\nzero$ probability matrix)}
  \BlankLine
  initialize $\hat\pzero_1$\;
  \While{not converged}{
    evaluate gradient: $\frac{\partial}{\partial \hat\pzero_t}
    l(\hat\pzero_t)$\;
    update parameter estimates: $\hat\pzero_{t+1} = \hat\pzero_t +
      \lambda \frac{\partial}{\partial \hat\pzero_t}
      l(\hat\pzero_t)$\;
  }
  \Return $\hat\pzero = \hat\pzero_{t}$
  \BlankLine
  \caption{\KronFit\ algorithm.}
  \label{alg:KronGradDescent}
\end{algorithm}

Note we are still summing over all $N!$ permutations $\perm$, so
calculating Eq.~\ref{eq:KronSampleGrad} is computationally intractable for
graphs with more than a handful of nodes. However, the equation has a nice
form which allows for use of simulation techniques to avoid the summation
over super-exponentially many node correspondences. Thus, we simulate
draws from the permutation distribution $P(\perm | \ggraph,\pzero)$, and
then evaluate the quantities at the sampled permutations to obtain the
expected values of log-likelihood and gradient.
Algorithm~\ref{alg:KronCalcGrad} gives the details.

Note that we can also permute the rows and columns of the parameter matrix
$\pzero$ to obtain equivalent estimates. Therefore $\pzero$ is not
strictly identifiable exactly because of these permutations. Since the
space of permutations on $N$ nodes is very large (grows as $N!$) the
{\new permutation sampling algorithm} will explore only a small fraction
of the space of all
permutations and may converge to one of the global maxima (but may not
explore all $\nzero!$ of them) of the parameter space. As we empirically
show later our results are not sensitive to this and multiple restarts
result in equivalent (but often permuted) parameter estimates.

\begin{algorithm}[t]
  \dontprintsemicolon
  \SetKwFunction{SamplePermutation}{SamplePermutation}
  \SetKwData{grad}{grad}
  \SetKwInOut{Input}{input}
  \SetKwInOut{Output}{output}
  \Input{Parameter matrix $\pzero$, and graph $\ggraph$}
  \Output{Log-likelihood $l(\pzero)$, and gradient $\frac{\partial}{\partial \pzero} l(\pzero)$}
  \BlankLine
  \For{t := 1 \KwTo T}{
    $\perm_t$ := \SamplePermutation($\ggraph, \pzero$)\;
    $ l_t = \log P(\ggraph|\perm^{(t)},\pzero)$\;
    ${\textrm grad}_t$ := $\frac{\partial}{\partial\pzero} \log P(\ggraph|\perm^{(t)},\pzero)$\;
  }
  \Return $l(\pzero)=\frac{1}{T} \sum_t l_t$, and
    $\frac{\partial}{\partial\pzero} l(\pzero)=\frac{1}{T} \sum_t {\textrm grad}_t$
  \BlankLine
  \caption{Calculating log-likelihood and gradient}
  \label{alg:KronCalcGrad}
\end{algorithm}

\subsubsection{Sampling permutations}
\label{sec:KronSamplePerm}

Next, we describe the Metropolis algorithm to simulate draws from the
permutation distribution $P(\perm | \ggraph, \pzero)$, which is given by
$$
P(\perm | \ggraph, \pzero) = \frac{P(\perm, \ggraph,\pzero)}{\sum_\tau
P(\tau, \ggraph,\pzero)} = \frac{P(\perm, \ggraph,\pzero)} {Z}
$$

\new{where $Z$ is the normalizing constant that is hard to compute since
it involves the sum over $\nnodes!$ elements.} However, if we compute the
likelihood ratio between permutations $\perm$ and $\perm'$
(Equation~\ref{eq:KronLikeRatio1}) the normalizing constants nicely cancel
out:
\begin{eqnarray}
  \frac{P(\perm'|\ggraph,\pzero)}{P(\perm|\ggraph,\pzero)} &=&
    \prod_{(u,v)\in \ggraph}
      \frac{\pmat[\perm_u,\perm_v]}{\pmat[\perm'_u,\perm'_v]}
    \prod_{(u,v) \notin \ggraph}
      \frac{(1-\pmat[\perm_u,\perm_v])}{(1-\pmat[\perm'_u,\perm'_v])}
  \label{eq:KronLikeRatio1}\\
  &=& \prod_{\substack{(u,v)\in \ggraph \\ (\perm_u,\perm_v) \neq
  (\perm'_u,\perm'_v)}}
    \frac{\pmat[\perm_u,\perm_v]}{\pmat[\perm'_u,\perm'_v]}
  \prod_{\substack{(u,v) \notin \ggraph \\ (\perm_u,\perm_v) \neq
  (\perm'_u,\perm'_v)}}
    \frac{(1-\pmat[\perm_u,\perm_v])}{(1-\pmat[\perm'_u,\perm'_v])}
  \label{eq:KronLikeRatio2}
\end{eqnarray}

This immediately suggests the use of a Metropolis sampling
algorithm~\cite{gamerman97mcmc} to simulate draws from the permutation
distribution since Metropolis is solely based on such ratios (where
normalizing constants cancel out). In particular, suppose that in the
Metropolis algorithm (Algorithm~\ref{alg:KronSamplePerm}) we consider a
move from permutation $\perm$ to a new permutation $\perm'$. Probability
of accepting the move to $\perm'$ is given by
Equation~\ref{eq:KronLikeRatio1} (if
$\frac{P(\perm'|\ggraph,\pzero)}{P(\perm|\ggraph,\pzero)} \le 1$) or 1
otherwise.

\begin{algorithm}[t]
    \dontprintsemicolon
    \SetKwFunction{SwapNodes}{Swap}
    \SetKwInOut{Input}{input}
    \SetKwInOut{Output}{output}
    \Input{Kronecker initiator matrix $\pzero$ and a graph
    $\ggraph$  on $\nnodes$ nodes}
    \Output{Permutation $\perm^{(i)} \sim P(\perm | \ggraph,\pzero)$}
    \BlankLine
    $\perm^{(0)} := (1, \dots, \nnodes)$\;
    $i = 1$\;
    \Repeat{$\perm^{(i)} \sim P(\perm | \ggraph, \pzero)$}{
      Draw $j$ and $k$ uniformly from $(1,\dots,\nnodes)$\;
      $\perm^{(i)}$ := {\tt SwapNodes}($\perm^{(i-1)}$, $j$, $k$)\;
      Draw $u$ from $U(0,1)$\;
      \If{$u > \frac{P(\perm^{(i)}|\ggraph,\pzero)}{P(\perm^{(i-1)}|\ggraph,\pzero)}$}
        { $\perm^{(i)} := \perm^{(i-1)}$}
      i = i + 1\; }
    \Return $\perm^{(i)}$\;
    \BlankLine
    {Where $U(0,1)$ is a uniform distribution on $[0,1]$, and
    $\perm'$ := {\tt SwapNodes($\perm, j, k$)} is the permutation
    $\perm'$ obtained from $\perm$ by swapping elements
    at positions $j$ and $k$.} \;
  \caption{{\tt SamplePermutation($\ggraph, \pzero$)}: Metropolis sampling of the
  node permutation.}
  \label{alg:KronSamplePerm}
\end{algorithm}

Now we have to devise a way to sample permutations $\perm$ from the
proposal distribution. One way to do this would be to simply generate a
random permutation $\perm'$ and then check the acceptance condition. This
would be very inefficient as we expect the distribution
$P(\perm|\ggraph,\pzero)$ to be heavily skewed, \emph{i.e.}, there will be
a relatively small number of good permutations (node mappings). Even more so as the
degree distributions in real networks are skewed there will be many bad
permutations with low likelihood, and few good ones that do a good job in
matching nodes of high degree.

To make the sampling process ``smoother'', \emph{i.e.}, sample
permutations that are not that different (and thus are not randomly
jumping across the permutation space) we design a Markov chain. The idea
is to stay in high likelihood part of the permutation space longer. We do this
by making samples dependent, \emph{i.e.}, given $\perm'$ we want to
generate next candidate permutation $\perm''$ to then evaluate the
likelihood ratio. When designing the Markov chain step one has to be
careful so that the proposal distribution satisfies the detailed balance
condition: $\pi(\perm')P(\perm' |\perm'') =
\pi(\perm'')P(\perm''|\perm')$, where $P(\perm' |\perm'')$ is the
transition probability of obtaining permutation $\perm'$ from $\perm''$
and, $\pi(\perm')$ is the stationary distribution.

In Algorithm~\ref{alg:KronSamplePerm} we use a simple proposal where given
permutation $\perm'$ we generate $\perm''$ by swapping elements at two
uniformly at random chosen positions of $\perm'$. We refer to this
proposal as {\tt SwapNodes}. While this is simple and clearly satisfies
the detailed balance condition it is also inefficient in a way that most
of the times low degree nodes will get swapped (a direct consequence of
heavy tailed degree distributions). This has two consequences, (a) we will
slowly converge to good permutations (accurate mappings of high degree
nodes), and (b) once we reach a good permutation, very few permutations
will get accepted as most proposed permutations $\perm'$ will swap low
degree nodes (as they form the majority of nodes).

A possibly more efficient way would be to swap elements of $\perm$ biased
based on corresponding node degree, so that high degree nodes would get swapped
more often. However, doing this directly does not
satisfy the detailed balance condition. A way of sampling labels biased by
node degrees that at the same time satisfies the detailed balance
condition is the following: we pick an edge in $\ggraph$ uniformly at
random and swap the labels of the nodes at the edge endpoints. Notice this is biased towards
swapping labels of nodes with high degrees simply as they have more edges.
The detailed balance condition holds as edges are sampled uniformly at
random. We refer to this proposal as {\tt SwapEdgeEndpoints}.

However, the issue with this proposal is that if the graph $\ggraph$ is
disconnected, we will only be swapping labels of nodes that belong to the
same connected component. This means that some parts of the permutation
space will never get visited. To overcome this problem we execute {\tt
SwapNodes} with some probability $\omega$ and {\tt SwapEdgeEndpoints} with
probability $1-\omega$.

To summarize we consider the following two permutation proposal
distributions:
\begin{itemize}
 \item $\perm'' =$ {\tt SwapNodes}$(\perm')$: we obtain $\perm''$ by
     taking $\perm'$, uniformly at random selecting a pair of elements
     and swapping their positions.
 \item $\perm'' =$ {\tt SwapEdgeEndpoints}$(\perm')$: we obtain
     $\perm''$ from $\perm'$ by first sampling an edge $(j,k)$ from
     $\ggraph$ uniformly at random, then we take $\perm'$ and swap the
     labels at positions $j$ and $k$.
\end{itemize}

\subsubsection{Speeding up the likelihood ratio calculation}

We further speed up the algorithm by using the following observation. As
written the equation~\ref{eq:KronLikeRatio1} takes $O(N^2)$ to evaluate
since we have to consider $\nnodes^2$ possible edges. However, notice that
permutations $\perm$ and $\perm'$ differ only at two positions,
\emph{i.e.} elements at position $j$ and $k$ are swapped, \emph{i.e.},
$\perm$ and $\perm'$ map all nodes except the two to the same locations.
This means those elements of equation~\ref{eq:KronLikeRatio1} cancel out.
Thus to update the likelihood we only need to traverse two rows and
columns of matrix $\pmat$, namely rows and columns $j$ and $k$, since
everywhere else the mapping of the nodes to the adjacency matrix is the
same for both permutations. This gives equation~\ref{eq:KronLikeRatio2}
where the products now range only over the two rows/columns of $\pmat$
where $\perm$ and $\perm'$ differ.

Graphs we are working with here are too large to allow us to explicitly
create and store the stochastic adjacency matrix $\pmat$ by Kronecker
powering the initiator matrix $\pzero$. Every time probability
$\pmat[i,j]$ of edge $(i,j)$ is needed the equation~\ref{eq:KronPij} is
evaluated, which takes $O(k)$. So a single iteration of
Algorithm~\ref{alg:KronSamplePerm} takes $O(k\nnodes)$.

\begin{observation} Sampling a permutation $\perm$ from $P(\perm |
\ggraph,\pzero)$ takes $O(k\nnodes)$.
\end{observation}

This is gives us an improvement over the $O(\nnodes!)$ complexity of
summing over all the permutations. So far we have shown how to obtain a
permutation but we still need to evaluate the likelihood and find the
gradients that will guide us in finding good initiator matrix. The problem
here is that naively evaluating the network likelihood (gradient) as
written in equation~\ref{eq:KronSampleGrad} takes time $O(\nnodes^2)$.
This is exactly what we investigate next and how to calculate the
likelihood in {\em linear time}.

\subsection{Efficiently approximating likelihood and gradient}

We just showed how to efficiently sample node permutations. Now, given a
permutation we show how to efficiently evaluate the likelihood and it's
gradient. Similarly as evaluating the likelihood ratio, naively
calculating the log-likelihood $l(\pzero)$ or its gradient
$\frac{\partial}{\partial \pzero} l(\pzero)$ takes time quadratic in the
number of nodes. Next, we show how to compute this in linear time
$O(\nedges)$.

We begin with the observation that real graphs are sparse, which means
that the number of edges is not quadratic but rather almost linear in the
number of nodes, $\nedges \ll \nnodes^2$. This means that majority of
entries of graph adjacency matrix are zero, \emph{i.e.}, most of the edges
are not present. We exploit this fact. The idea is to first calculate the
likelihood (gradient) of an empty graph, \emph{i.e.}, a graph with zero
edges, and then correct for the edges that actually appear in $G$.

To naively calculate the likelihood for an empty graph one needs to
evaluate every cell of graph adjacency matrix. We consider Taylor
approximation to the likelihood, and exploit the structure of matrix
$\pmat$ to devise a constant time algorithm.

First, consider the second order Taylor approximation to log-likelihood of
an edge that succeeds with probability $x$ but does not appear in the
graph:
$$
\log(1 - x) \approx - x - \frac{1}{2}x^2
$$

Calculating $l_{e}(\pzero)$, the log-likelihood of an empty graph,
becomes:
\begin{eqnarray}
  l_{e}(\pzero) =
    \sum_{i=1}^\nnodes \sum_{j=1}^\nnodes \log(1-\pij{ij}) \approx 
    - \bigg(\sum_{i=1}^{\nzero} \sum_{j=1}^{\nzero} \thij{ij}\bigg)^k
    - \frac{1}{2}\bigg(\sum_{i=1}^{\nzero} \sum_{j=1}^{\nzero} \thij{ij}^2\bigg)^k
  \label{eq:KronApproxGrad1}
\end{eqnarray}

Notice that while the first pair of sums ranges over $\nnodes$ elements,
the last pair only ranges over $\nzero$ elements ($\nzero = \log_k
\nnodes$). Equation~\ref{eq:KronApproxGrad1} holds due to the recursive
structure of matrix $\pmat$ generated by the Kronecker product. We
substitute the $\log(1-\pij{ij})$ with its Taylor approximation, which
gives a sum over elements of $\pmat$ and their squares. Next, we notice
the sum of elements of $\pmat$ forms a multinomial series, and thus
$\sum_{i,j} \pij{ij} = (\sum_{i,j} \thij{ij})^k$, where $\thij{ij}$
denotes an element of $\pzero$, and $\pij{ij}$ element of $\pzero^{[k]}$.

Calculating log-likelihood of $\ggraph$ now takes $O(\nedges)$: First, we
approximate the likelihood of an empty graph in constant time, and then
account for the edges that are actually present in $\ggraph$, \emph{i.e.},
we subtract ``no-edge'' likelihood and add the ``edge'' likelihoods:
\begin{eqnarray}
  l(\pzero) =  l_{e}(\pzero) +
  \sum_{(u,v) \in \ggraph} 
  - \log(1-\pmat[\perm_u, \perm_v]) +  \log(\pmat[\perm_u,
  \perm_v]) 
  \nonumber
  \label{eq:KronApproxGrad2}
\end{eqnarray}

\new{We note that by using the second order Taylor approximation to the log-likelihood of an empty graph, the error term of the approximation is $\frac{1}{3}(\sum_i \thij{ij}^3)^k$, which can diverge for large $k$. For typical values of initiator matrix $\pn{1}$ (that we present in Section~\ref{sec:bigtable}) we note that one needs about fourth or fifth order Taylor approximation for the error of the approximation actually go to zero as $k$ approaches infinity, {\em i.e.}, $\sum_{ij} \thij{ij}^{n+1}<1$, where $n$ is the order of Taylor approximation employed.}

\subsection{Calculating the gradient}

Calculation of the gradient of the log-likelihood follows exactly the same
pattern as described above. First by using the Taylor approximation we
calculate the gradient as if graph
$\ggraph$ would have no edges. Then we correct the gradient for the edges
that are present in $\ggraph$. As in previous section we speed up the
calculations of the gradient by exploiting the fact that two consecutive
permutations $\perm$ and $\perm'$ differ only at two positions, and thus
given the gradient from previous step one only needs to account for the
swap of the two rows and columns of the gradient matrix $\partial \pmat /
\partial \pzero$ to update to the gradients of individual parameters.

\subsection{Determining the size of initiator matrix}

The question we answer next is how to determine the right number of
parameters, \emph{i.e.}, what is the right size of matrix $\pzero$? This
is a classical question of model selection where there is a tradeoff
between the complexity of the model, and the quality of the fit. Bigger
model with more parameters usually fits better, however it is also more
likely to overfit the data.

For model selection to find the appropriate value of $\nzero$, the size of
matrix $\pzero$, and choose the right tradeoff between the complexity of
the model and the quality of the fit, we propose to use the Bayes
Information Criterion (BIC)~\cite{schwarz78bic}. Stochastic Kronecker
graph model the presence of edges with independent Bernoulli random
variables, where the canonical number of parameters is $\nzero^{2k}$,
which is a function of a lower-dimensional parameter $\pzero$. This is
then a {\em curved exponential family}~\cite{efron75curvature}, and BIC
naturally applies:
$$\textrm{BIC}(\nzero) = -l(\hat\pzero_{\nzero})+\frac{1}{2}\nzero^2\log(N^2)$$

\new{where $\hat\pzero_{\nzero}$ are the maximum likelihood parameters of the model with
 $\nzero\times\nzero$ parameter matrix, and $N$ is the number of nodes
in $G$.} Note that one could also additional term to the above formula to
account for multiple global maxima of the likelihood space but as $\nzero$
is small the additional term would make no real difference.

As an alternative to BIC one could also consider the Minimum Description Length
(MDL)~\cite{rissanen78mdl} principle where the model is scored by the
quality of the fit plus the size of the description that encodes the model
and the parameters.

\section{Experiments on real and synthetic data}
\label{sec:KronExperiments}
Next we described our experiments on a range of real and synthetic networks.
We divide the experiments into several subsections. First we examine the
convergence and mixing of the Markov chain of our permutation sampling
scheme. Then we consider estimating the parameters of synthetic
Kronecker graphs to see whether \KronFit is able to recover the parameters
used to generate the network. Last, we consider fitting \SKRG to large
real world networks.

\subsection{Permutation sampling}

In our experiments we considered both synthetic and real graphs. Unless
mentioned otherwise all synthetic Kronecker graphs were generated using
$\pn{1}^* = [0.8, 0.6; 0.5, 0.3]$, and $k=14$ which gives us a graph
$\ggraph$ on $\nnodes=$16,384 nodes and $\nedges=$115,741 edges. We chose
this particular $\pn{1}^*$ as it resembles the typical initiator for
real networks analyzed later in this section.

\subsubsection{Convergence of the log-likelihood and the gradient}

\begin{figure}[t]
  \begin{center}
  \begin{tabular}{cc}
    \includegraphics[width=0.45\textwidth]{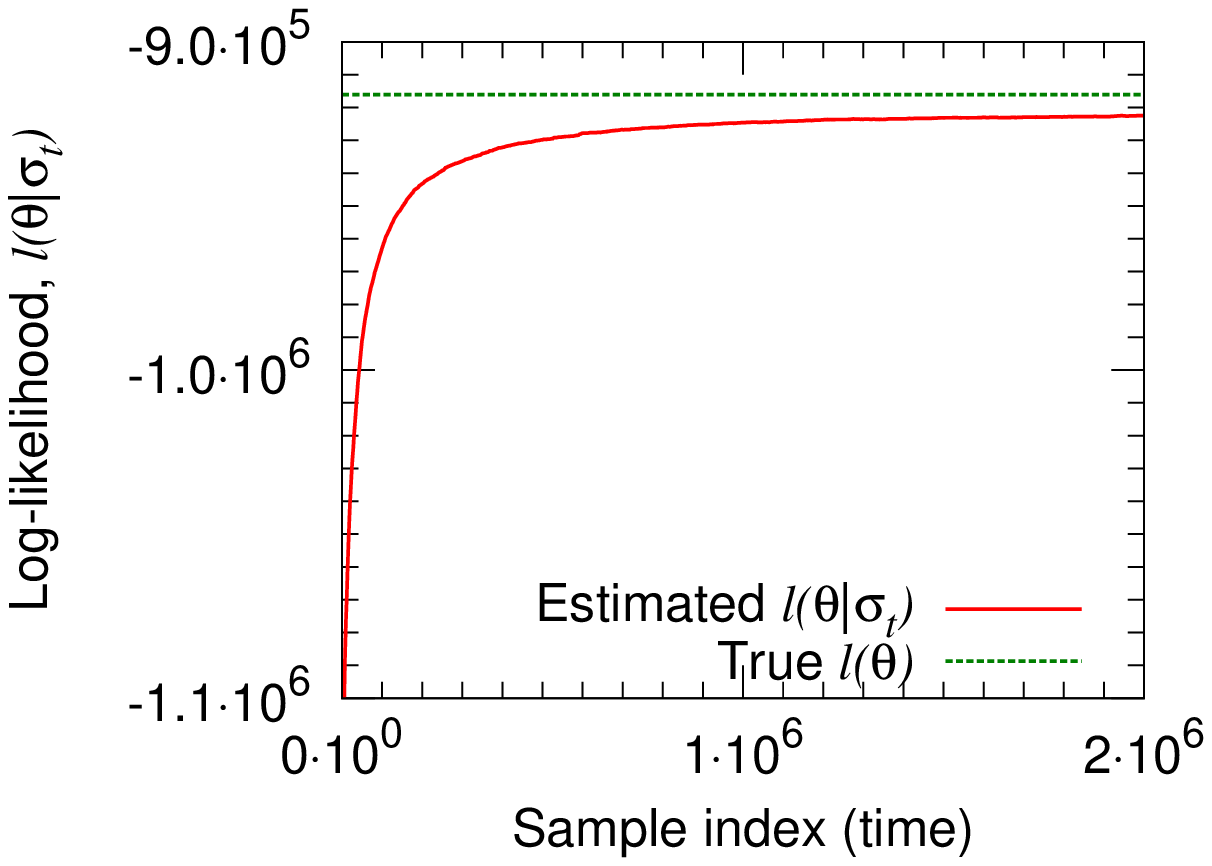} &
    \includegraphics[width=0.45\textwidth]{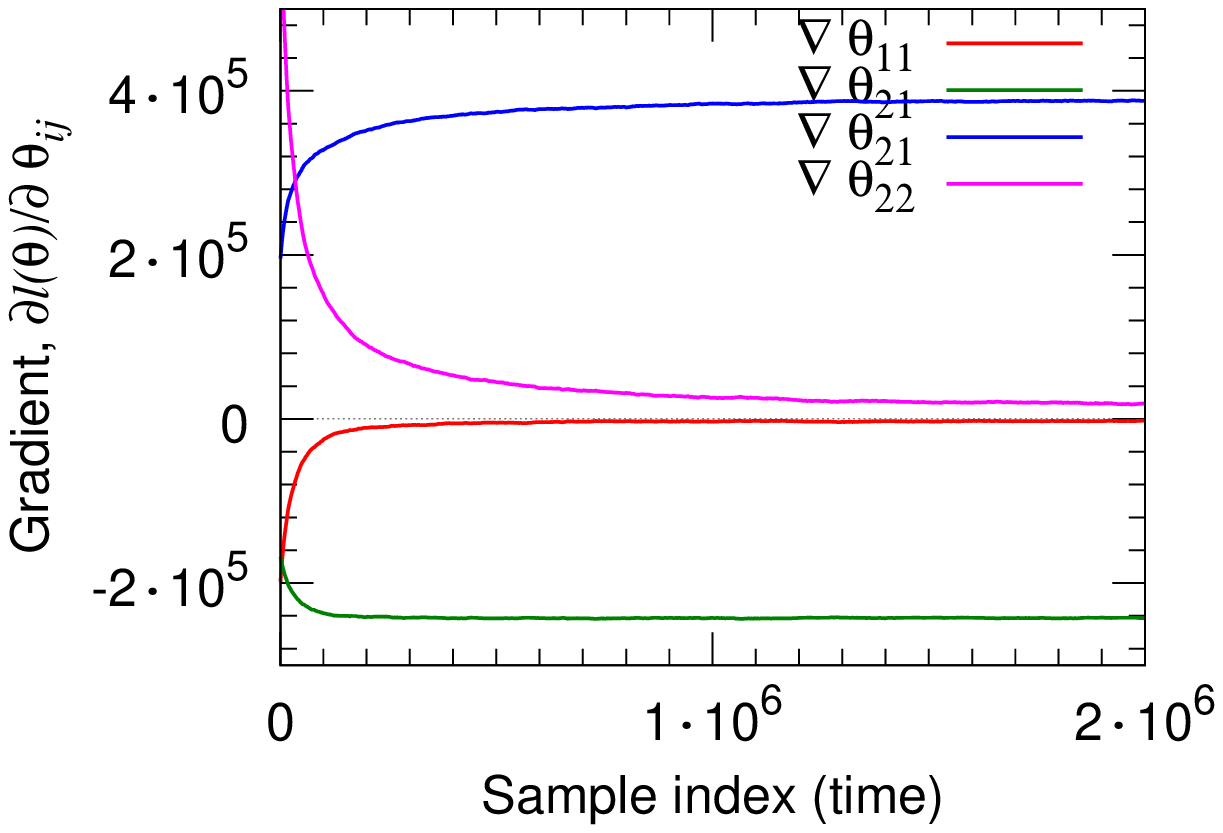} \\
    (a) Log-likelihood & (b) Gradient\\
    \includegraphics[width=0.45\textwidth]{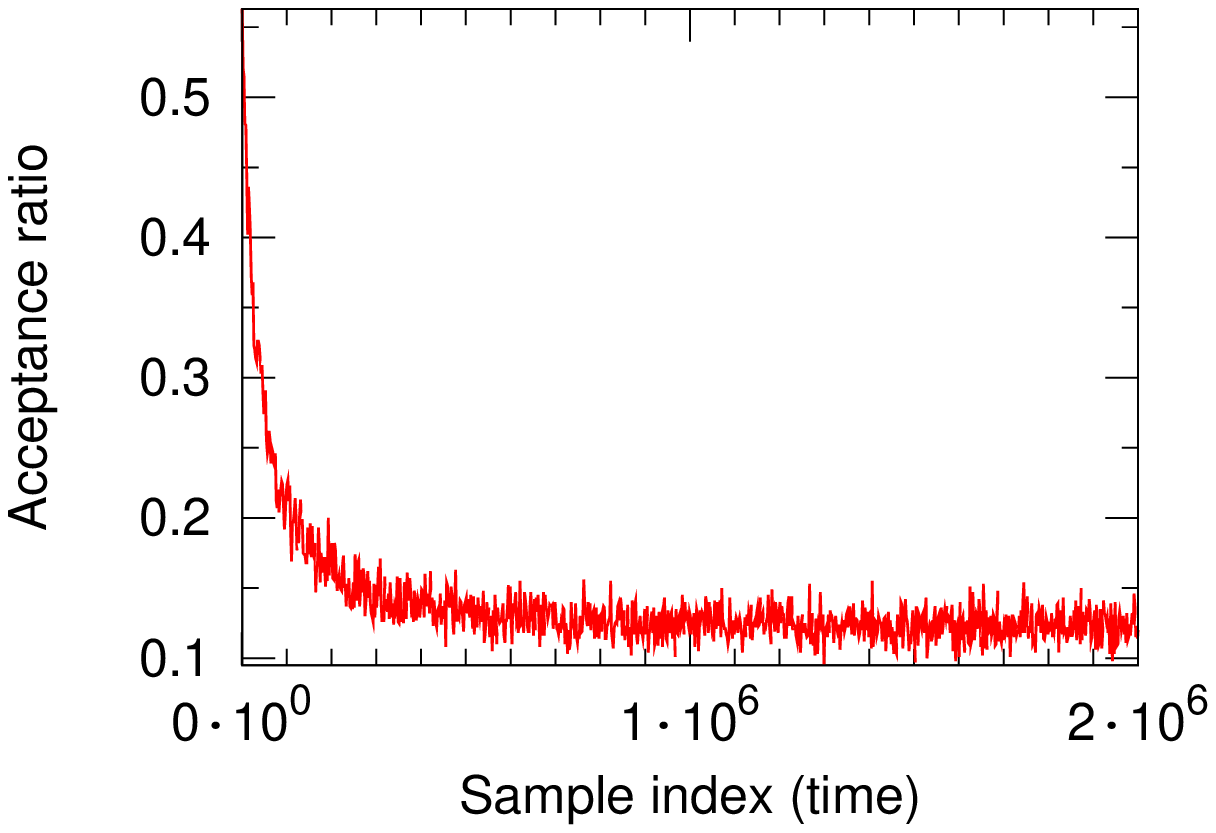} &
    \includegraphics[width=0.45\textwidth]{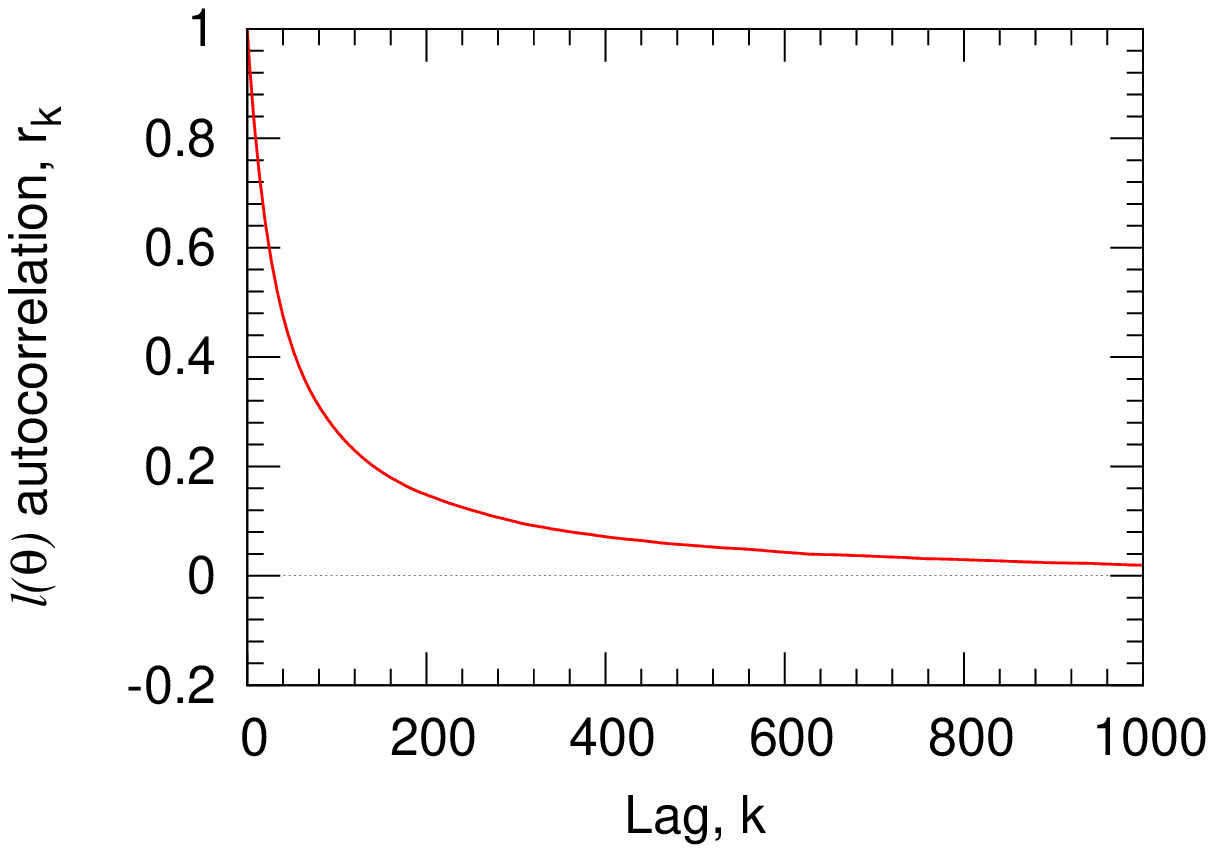} \\
    (c) Acceptance ratio & (d) Autocorrelation\\
  \end{tabular}
  \end{center}
  \caption{Convergence of the log-likelihood and components of the gradient towards their
  true values for Metropolis permutation sampling
  (Algorithm~\ref{alg:KronSamplePerm}) with the number of samples.}
  \label{fig:KronSamplePerm}
\end{figure}

\new{First, we examine the convergence of Metropolis permutation sampling, where
permutations are sampled sequentially. A new permutation is obtained
by modifying the previous one which creates a Markov chain. We want to assess the
convergence and mixing of the chain. We aim to determine how many
permutations one needs to draw to reliably estimate the likelihood and the
gradient, and also how long does it take until the samples converge to the
stationary distribution. For the experiment we generated a synthetic \SKRG
using $\pn{1}^*$ as defined above. Then, starting with a random
permutation we run Algorithm~\ref{alg:KronSamplePerm}, and measure how the
likelihood and the gradients converge to their true values.}

In this particular case, we first generated a \SKRG $\ggraph$ as described
above, but then calculated the likelihood and the parameter gradients for
$\pzero' = [0.8, 0.75; 0.45, 0.3]$. We average the likelihoods and
gradients over buckets of 1,000 consecutive samples, and plot how the
log-likelihood calculated over the sampled permutations approaches the
true log-likelihood (that we can compute since $\ggraph$ is a \SKRG).

First, we present experiments that aim to answer how many samples
(\emph{i.e.}, permutations) does one need to draw to obtain a reliable
estimate of the gradient (see Equation~\ref{eq:KronSampleGrad}).
Figure~\ref{fig:KronSamplePerm}(a) shows how the estimated log-likelihood
approaches the true likelihood. Notice that estimated values quickly
converge to their true values, \emph{i.e.}, Metropolis sampling quickly
moves towards ``good'' permutations. Similarly,
Figure~\ref{fig:KronSamplePerm}(b) plots the convergence of the gradients.
Notice that $\thij{11}$ and $\thij{22}$ of $\pzero'$ and $\pn{1}^*$ match,
so gradients of these two parameters should converge to zero and indeed
they do. On the other hand, $\thij{12}$ and $\thij{21}$ differ between
$\pzero'$ and $\pn{1}^*$. Notice, the gradient for one is positive as the
parameter $\thij{12}$ of $\pzero'$ should be decreased, and similarly for
$\thij{21}$ the gradient is negative as the parameter value should be
increased to match the $\pzero'$. In summary, this  shows that
log-likelihood and gradients rather quickly converge to their true values.

In Figures~\ref{fig:KronSamplePerm}(c) and (d) we also investigate
the properties of the Markov Chain Monte Carlo sampling procedure, and
assess convergence and mixing criteria. First, we plot the fraction of
accepted proposals. It stabilizes at around 15\%, which is quite close to
the rule-of-a-thumb of 25\%. Second, Figure~\ref{fig:KronSamplePerm}(d)
plots the autocorrelation of the log-likelihood as a function of the lag.
Autocorrelation $r_k$ of a signal $X$ is a function of the lag $k$ where
$r_k$ is defined as the correlation of signal $X$ at time $t$ with $X$ at
$t+k$, {\em i.e.}, correlation of the signal with itself at lag $k$. High
autocorrelations within chains indicate slow mixing and, usually, slow
convergence. On the other hand fast decay of autocorrelation implies better
the mixing and thus one needs less samples to accurately estimate the
gradient or the likelihood. Notice the rather fast autocorrelation decay.

All in all, these experiments show that one needs to sample on the order of
tens of thousands of permutations for the estimates to converge. We also
verified that the variance of the estimates is sufficiently small. In our
experiments we start with a random permutation and use long burn-in time.
Then when performing optimization we use the permutation from the previous
step to initialize the permutation at the current step of the gradient descent.
Intuitively, small changes in parameter space $\pzero$ also mean small
changes in $P(\perm|G,\pzero)$ .

\subsubsection{Different proposal distributions}

\begin{figure}[t]
  \begin{center}
    \includegraphics[width=0.7\textwidth]{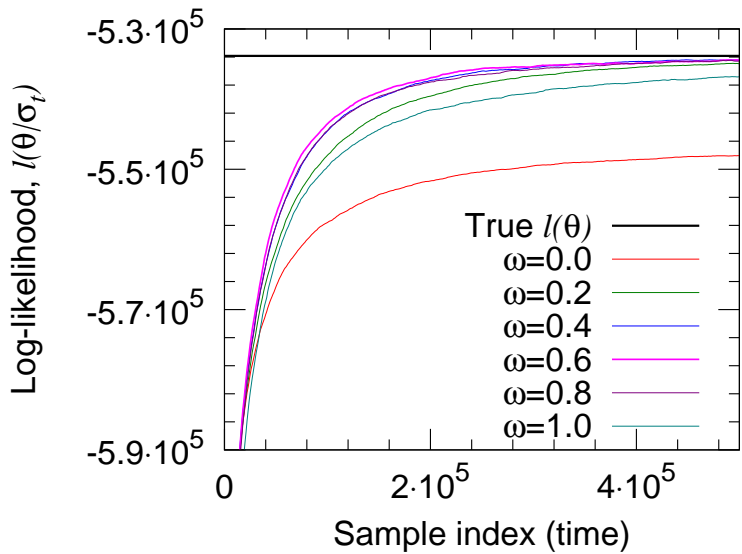}
  \end{center}
  \caption{Convergence of the log-likelihood and gradients for
  Metropolis permutation sampling (Algorithm~\ref{alg:KronSamplePerm}) for
  different choices of $\omega$ that interpolates between the
  {\tt SwapNodes} ($\omega=1$) and {\tt SwapEdgeEndpoints} ($\omega=0$)
  permutation proposal distributions. Notice fastest convergence of
  log-likelihood for $\omega = 0.6$.}
  \label{fig:KronSamplePermLL2}
\end{figure}

In section~\ref{sec:KronSamplePerm} we defined two permutation sampling
strategies: {\tt SwapNodes} where we pick two nodes uniformly
at random and swap their labels (node ids), and {\tt SwapEdgeEndpoints}
where we pick a random edge in a graph and then swap the labels of the
edge endpoints. We also discussed that one can interpolate between the two
strategies by executing {\tt SwapNodes} with probability $\omega$, and
{\tt SwapEdgeEndpoints} with probability $1-\omega$.

So, given a \SKRG $\ggraph$ on $\nnodes=$16,384 and $\nedges=$115,741
generated from $\pn{1}^*=[0.8, 0.7; 0.5, 0.3]$ we evaluate the likelihood
of $\pzero' = [0.8, 0.75; 0.45, 0.3]$. As we sample permutations we
observe how the estimated likelihood converges to the true likelihood.
Moreover we also vary parameter $\omega$ which interpolates between the two
permutation proposal distributions. The quicker the convergence towards the
true log-likelihood the better the proposal distribution.

Figure~\ref{fig:KronSamplePermLL2} plots the convergence of the
log-likelihood with the number of sampled permutations. We plot the
average over non-overlapping buckets of 1,000 consecutive permutations.
Faster convergence implies better permutation proposal distribution. When we
use only {\tt SwapNodes} ($\omega=1$) or {\tt SwapEdgeEndpoints}
($\omega=0$) convergence is rather slow.  We obtain best convergence for
$\omega$ around $0.6$.

\begin{figure}[t]
  \begin{center}
  \begin{tabular}{cc}
    \includegraphics[width=0.48\textwidth]{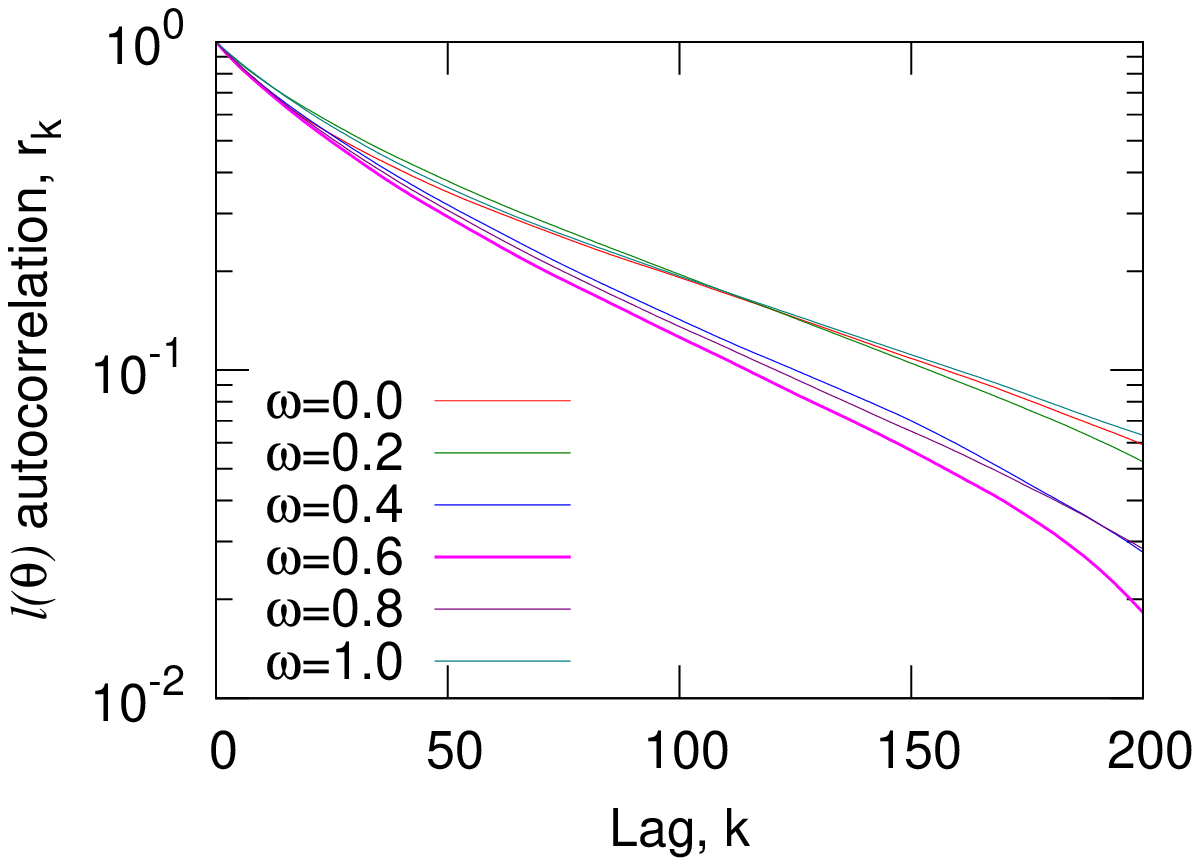} &
    \includegraphics[width=0.48\textwidth]{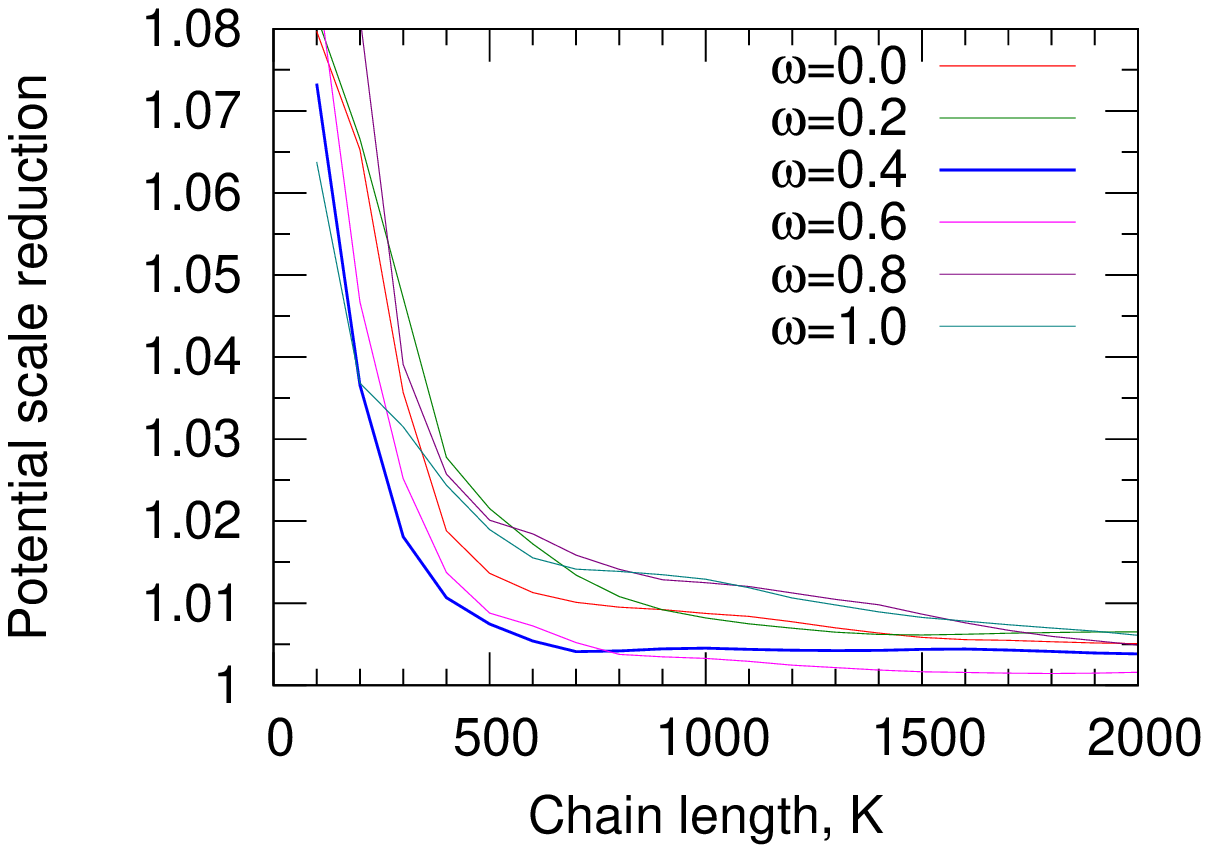}  \\
    (a) Autocorrelation & (b) Potential scale reduction\\
  \end{tabular}
  \end{center}
  \caption{(a) Autocorrelation plot of the log-likelihood for the
  different choices of parameter $\omega$. Notice we get best mixing with
  $\omega \approx 0.6$. (b) The potential scale reduction that compares the
  variance inside- and across- independent Markov chains for different
  values of parameter $\omega$.}
  \label{fig:KronSamplePermAutoCor}
  \label{fig:KronSamplePermGelman}
\end{figure}

Similarly, Figure~\ref{fig:KronSamplePermAutoCor}(a) plots the
autocorrelation as a function of the lag $k$ for different choices of
$\omega$. Faster autocorrelation decay means better mixing of the Markov
chain. Again, notice that we get best mixing for $\omega \approx 0.6$. (Notice
logarithmic y-axis.)

Last, we diagnose how long the sampling procedure must be run before the
generated samples can be considered to be drawn (approximately) from the
stationary distribution. We call this the burn-in time of the chain. There
are various procedures for assessing convergence. Here we adopt the
approach of Gelman {\em et al.}~\cite{gelman03bayesian}, that is based on
running multiple Markov chains each from a different starting point, and
then comparing the variance within the chain and between the chains. The
sooner the within- and between-chain variances become equal the shorter the
burn-in time, {\em i.e.}, the sooner the samples are drawn from the
stationary distribution.

Let $l$ be the parameter that is being simulated with $J$ different
chains, and then let $l^{(k)}_{j}$ denote the $k^{th}$ sample of the
$j^{th}$ chain, where $j=1, \ldots, J$ and $k=1, \ldots, K$. More
specifically, in our case we run separate permutation sampling chains. So,
we first sample permutation $\sigma^{(k)}_{j}$ and then calculate the
corresponding log-likelihood $l^{(k)}_{j}$.

First, we compute between and within chain variances $\hat\sigma^{2}_B$
and $\hat\sigma^{2}_W$, where between-chain variance is obtained by
$$
  \hat\sigma^{2}_B = \frac{K}{J-1} \sum^{J}_{j=1}
    (\bar{l}_{\cdot j} - \bar{l}_{\cdot \cdot})^2
$$
where $
  \bar{l}_{\cdot j} = \frac{1}{K} \sum^{K}_{k=1} l^{(k)}_{j}
$ and $
  \bar{l}_{\cdot \cdot} = \frac{1}{J} \sum^{J}_{j=1} \bar{l}_{\cdot j}
$

Similarly the within-chain variance is defined by
$$
  \hat\sigma^{2}_W = \frac{1}{J(K-1)} \sum^{J}_{j=1} \sum^{K}_{k=1}
    (l_{j}^{(k)} - \bar{l}_{\cdot j})^2
$$

Then, the marginal posterior variance of $\hat{l}$ is calculated using
$$
  \hat\sigma^2 = \frac{K-1}{K} \hat\sigma^{2}_W + \frac{1}{K} \hat\sigma^{2}_B
$$

And, finally, we estimate the {\em potential scale
reduction}~\cite{gelman03bayesian} of $l$ by
$$
\sqrt{\hat{R}} = \sqrt{\frac{\hat\sigma^2}{\hat\sigma_W^2}}
$$

Note that as the length of the chain $K \rightarrow \infty$,
$\sqrt{\hat{R}}$ converges to 1 from above. The recommendation for
convergence assessment from~\cite{gelman03bayesian} is that the potential
scale reduction is below $1.2$.

Figure~\ref{fig:KronSamplePermGelman}(b) gives the Gelman-Rubin-Brooks
plot, where we plot the potential scale reduction $\sqrt{\hat{R}}$ over
the increasing chain length $K$ for different choices of parameter
$\omega$. Notice that the potential scale reduction quickly decays
towards 1. Similarly as in Figure~\ref{fig:KronSamplePermAutoCor} the
extreme values of $\omega$ give slow decay, while we obtain the fastest
potential scale reduction when $\omega \approx 0.6$.

\subsubsection{Properties of the permutation space}

\begin{figure}[t]
  \begin{center}
  \begin{tabular}{ccc}
    \includegraphics[width=0.31\textwidth]{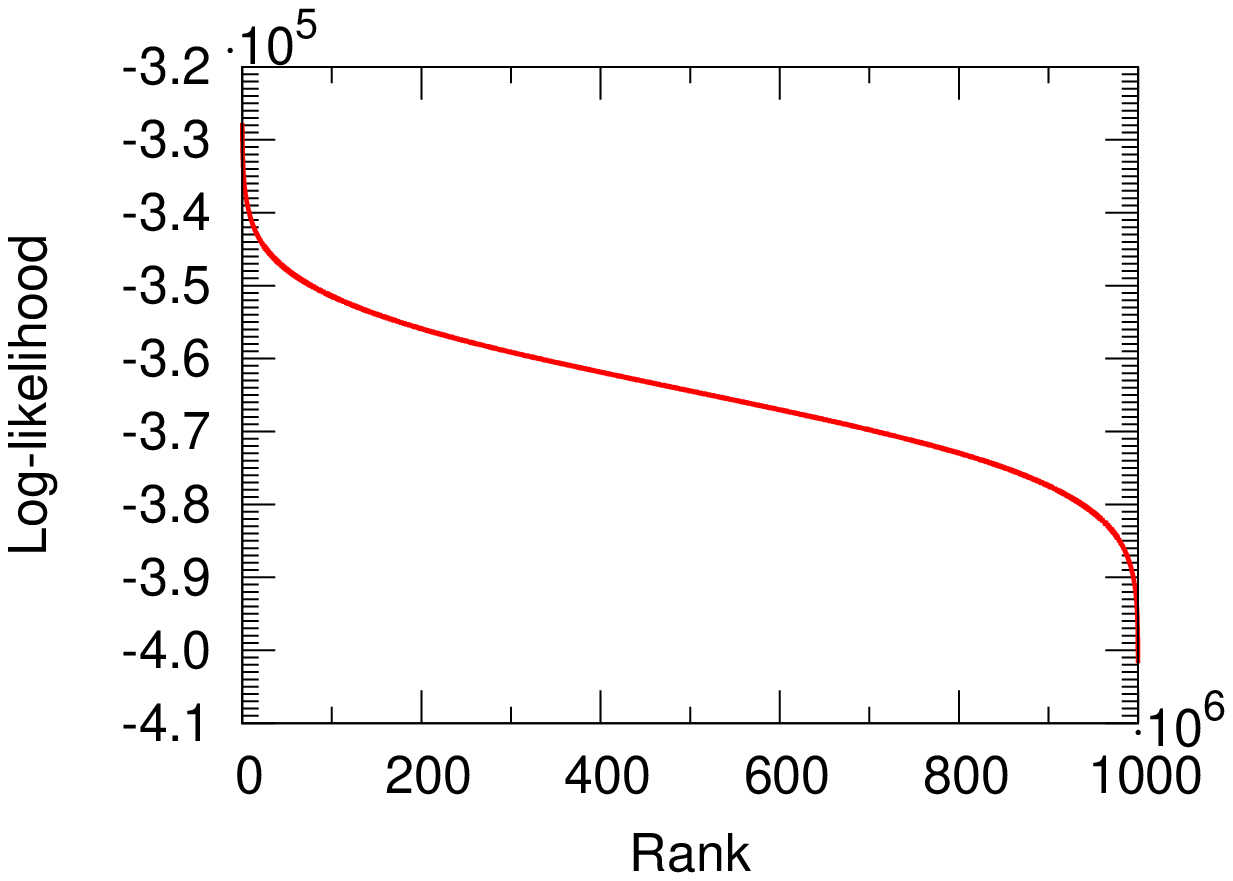} &
    \includegraphics[width=0.31\textwidth]{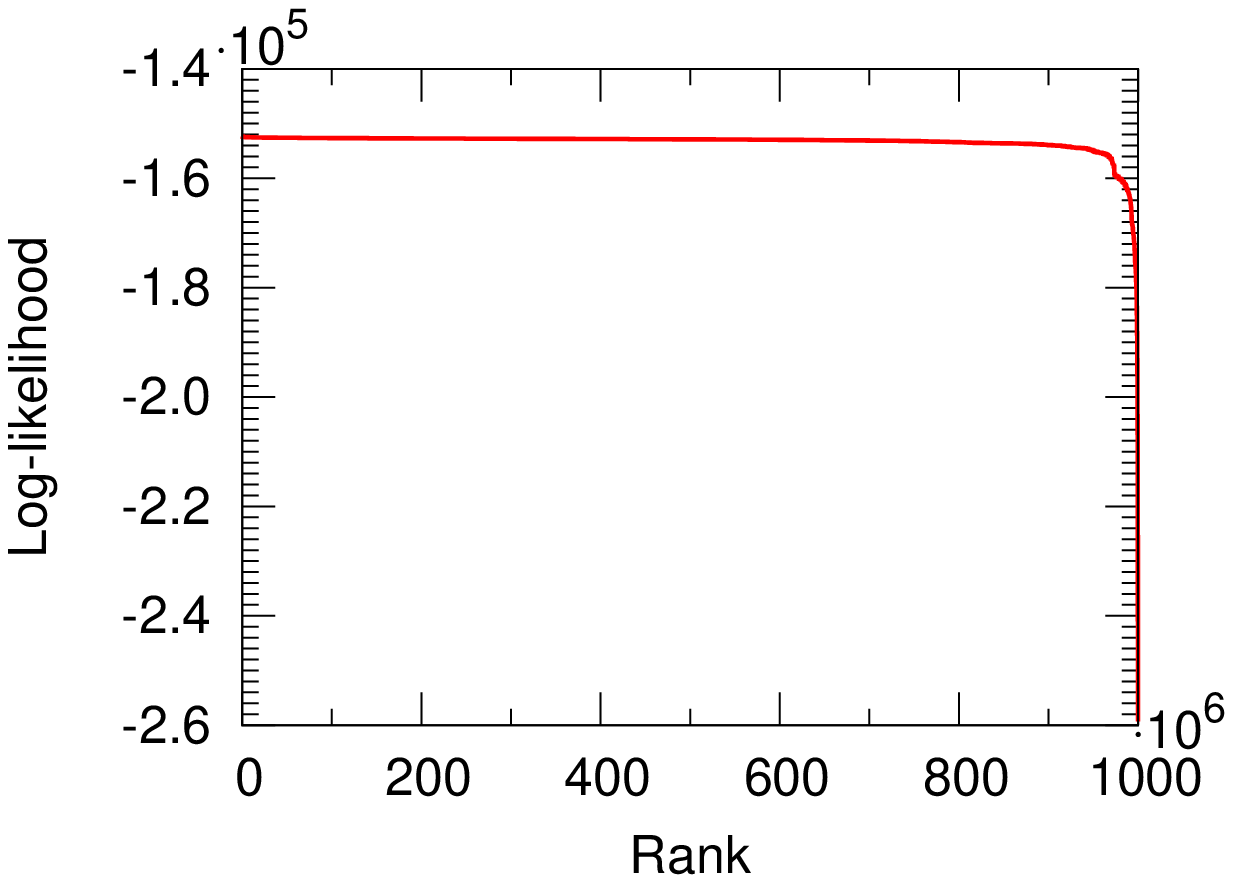} &
    \includegraphics[width=0.31\textwidth]{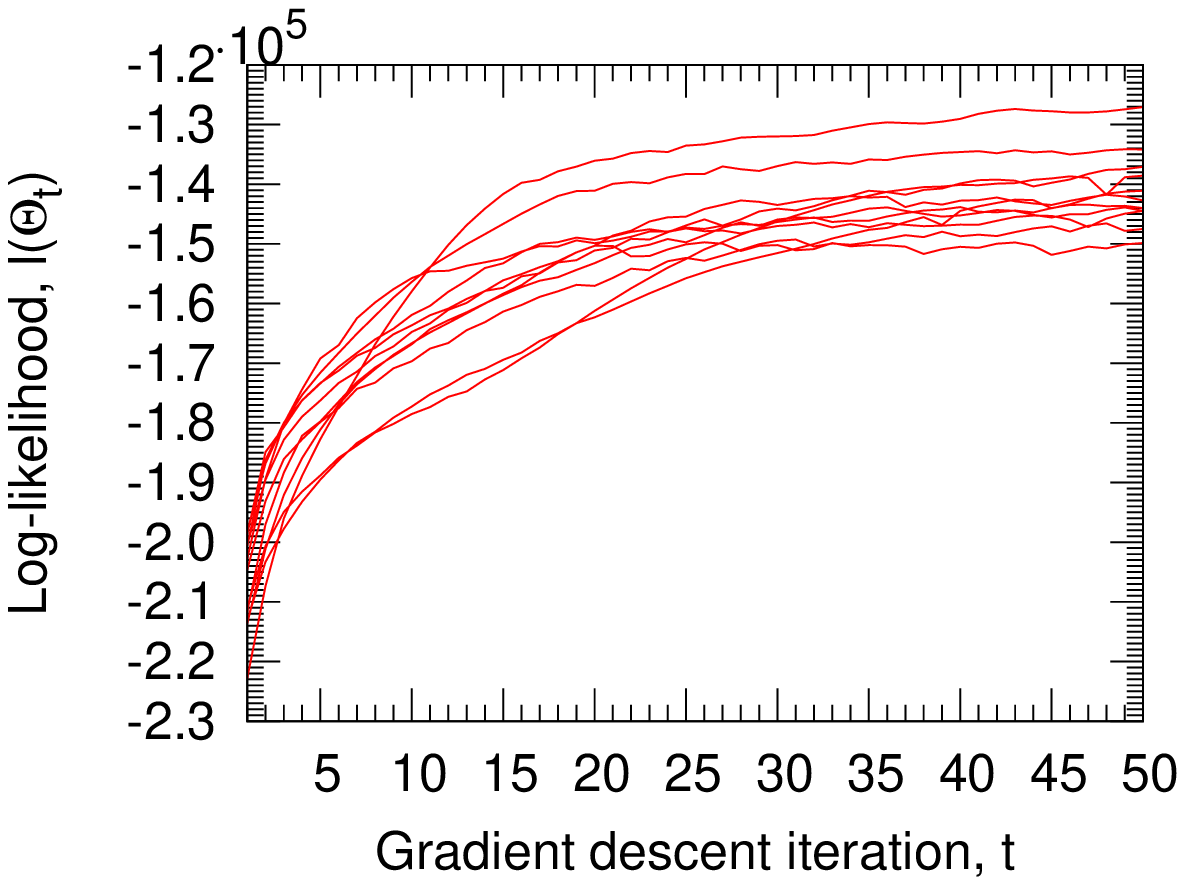} \\
    (a) $l(\pzero | \perm_i)$ where & (b) $l(\pzero | \perm_i)$ where & (c) $l(\pzero_t)$ for 10
random\\
    $\perm_i \sim P(\perm)$ & $\perm_i \sim P(\perm | \pzero, G)$ & gradient descent runs \\
  \end{tabular}
  \end{center}
  \caption{(a) Distribution of log-likelihood of permutations sampled
  uniformly at random, and (b) when sampled from $P(\perm | \pzero, G)$. Notice
  the space of good permutations is rather small but our sampling quickly finds
  permutations of high likelihood. (c) Convergence of log-likelihood for 10
  runs of gradient descent, each from a different random starting point.}
  \label{fig:KronFitPermSpace}
\end{figure}

Next we explore the properties of the permutation space. We would like to
quantify what fraction of permutations are ``good'' (have high
likelihood), and how quickly are they discovered. For the experiment we
took a real network $G$ (\dataset{As-RouteViews} network) and the MLE
parameters $\hat\pzero$ for it that we estimated before hand
($l(\hat\pzero) \approx -$150,000). The network $G$ has 6,474 nodes
which means the space of all permutations has $\approx 10^{22,000}$
elements.

First, we sampled 1 billion ($10^9$) permutations $\perm_i$ uniformly at
random, \emph{i.e.}, $P(\perm_i) =$ 1/(6,474!) and for each evaluated its
log-likelihood $l(\sigma_i | \pzero) = \log P(\pzero | G, \sigma_i)$. We
ordered the permutations in deceasing log-likelihood and plotted
$l(\sigma_i | \pzero)$ vs. rank. Figure~\ref{fig:KronFitPermSpace}(a) gives
the plot. Notice that very few random permutations are very bad
(\emph{i.e.}, they give low likelihood), similarly few permutations are
very good, while most of them are somewhere in between. Notice that best
``random'' permutation has log-likelihood of $\approx -$320,000, which is
far below true likelihood $l(\hat\pzero) \approx -$150,000. This suggests
that only a very small fraction of all permutations gives good node
labelings.

On the other hand, we also repeated the same experiment but now using
permutations sampled from the permutation distribution $\perm_i \sim P(\perm |
\pzero, G)$ via our Metropolis sampling scheme.
Figure~\ref{fig:KronFitPermSpace}(b) gives the plot. Notice the radical
difference. Now the $l(\sigma| \pzero_i)$ very quickly converges to the
true likelihood of $\approx -150,000$. This suggests that while the number
of ``good'' permutations (accurate node mappings) is rather small, our
sampling procedure quickly converges to the ``good'' part of the
permutation space where node mappings are accurate, and spends the most time there.

\subsection{Properties of the optimization space}

In maximizing the likelihood we use a stochastic approximation to the
gradient. This adds variance to the gradient and makes efficient
optimization techniques, like conjugate gradient, highly unstable. Thus we
use gradient descent, which is slower but easier to control. First, we
make the following observation:

\medskip
\begin{observation} Given a real graph $\ggraph$ then finding the
maximum likelihood Stochastic Kronecker initiator matrix $\hat\pzero$
$$
  \hat\pzero = \arg \max_{\pzero} P(\ggraph | \pzero)
$$
is a non-convex optimization problem.
\end{observation}

\begin{proof}
By definition permutations of the Kronecker graphs initiator matrix
$\pzero$ all have the same log-likelihood. This means that we have
several global minima that correspond to permutations of parameter matrix
$\pzero$, and then between them the log-likelihood drops. This means
that the optimization problem is non-convex.
\end{proof}

\new{The above observation does not seem promising for estimating
$\hat\pzero$ using gradient descent as it is prone to finding local minima. To
test for this behavir we run the following experiment: we generated 100 synthetic
Kronecker graphs on 16,384 ($2^{14}$) nodes and 1.4 million edges on the
average, each with a randomly chosen $2\times2$ parameter matrix $\pzero^*$.
For each of the 100 graphs we run a single trial of gradient descent starting from a
random parameter matrix $\pzero'$, and try to recover
$\pzero^*$. In 98\% of the cases the gradient descent converged to the
true parameters. Many times the algorithm converged to a different global
minima, \emph{i.e.}, $\hat\pzero$ is a permuted version of original
parameter matrix $\pzero^*$. Moreover, the median number of gradient
descent iterations was only 52.}

This suggests surprisingly nice structure of our optimization space: it
seems to behave like a convex optimization problem with many equivalent
global minima. Moreover, this experiment is also a good sanity check as it
shows that given a Kronecker graph we can recover and identify the
parameters that were used to generate it.

Moreover, Figure~\ref{fig:KronFitPermSpace}(c) plots the log-likelihood
$l(\pzero_t)$ of the current parameter estimate $\pzero_t$ over the
iterations $t$ of the stochastic gradient descent. We plot the
log-likelihood for 10 different runs of gradient descent, each time
starting from a different random set of parameters $\pzero_0$. Notice that
in all runs gradient descent always converges towards the optimum, and
none of the runs gets stuck in some local maxima.

\subsection{Convergence of the graph properties}

\begin{figure}[t]
  \begin{center}
  \begin{tabular}{cc}
    \includegraphics[width=0.45\textwidth]{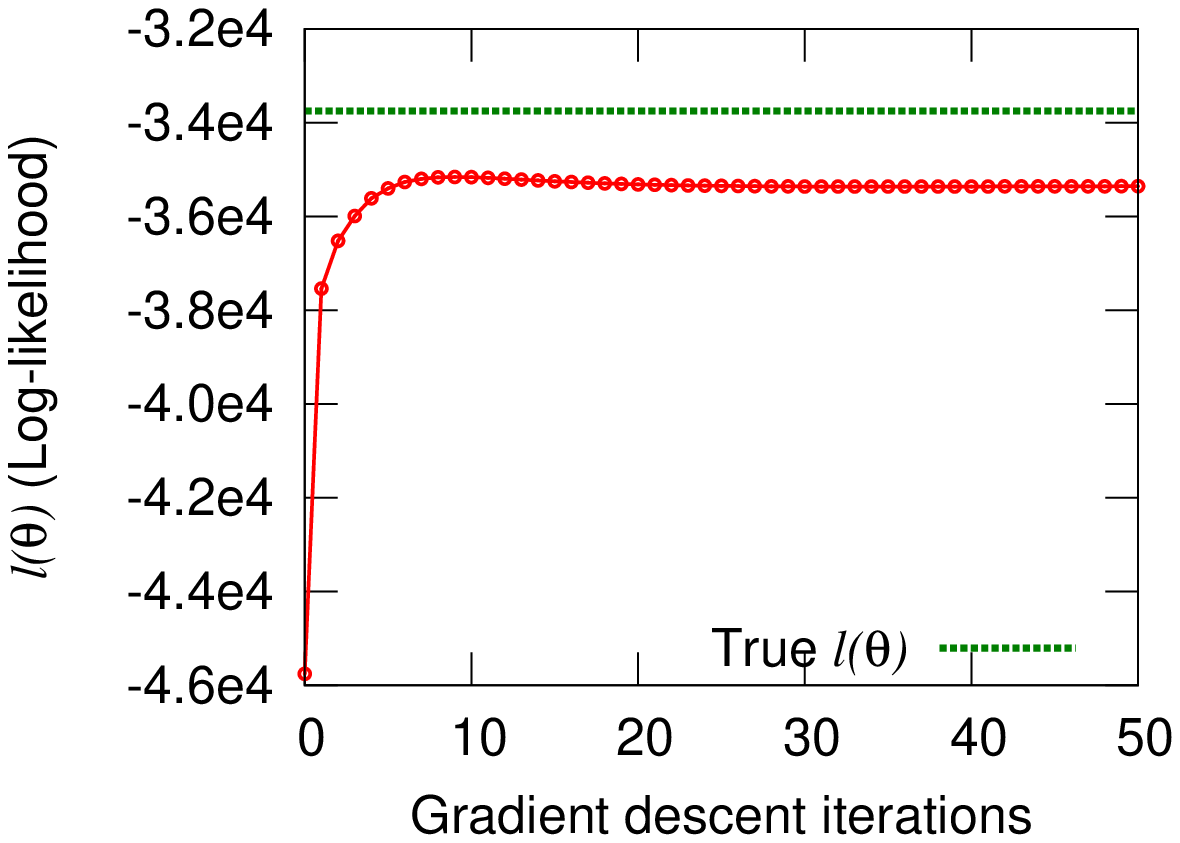} &
    \includegraphics[width=0.45\textwidth]{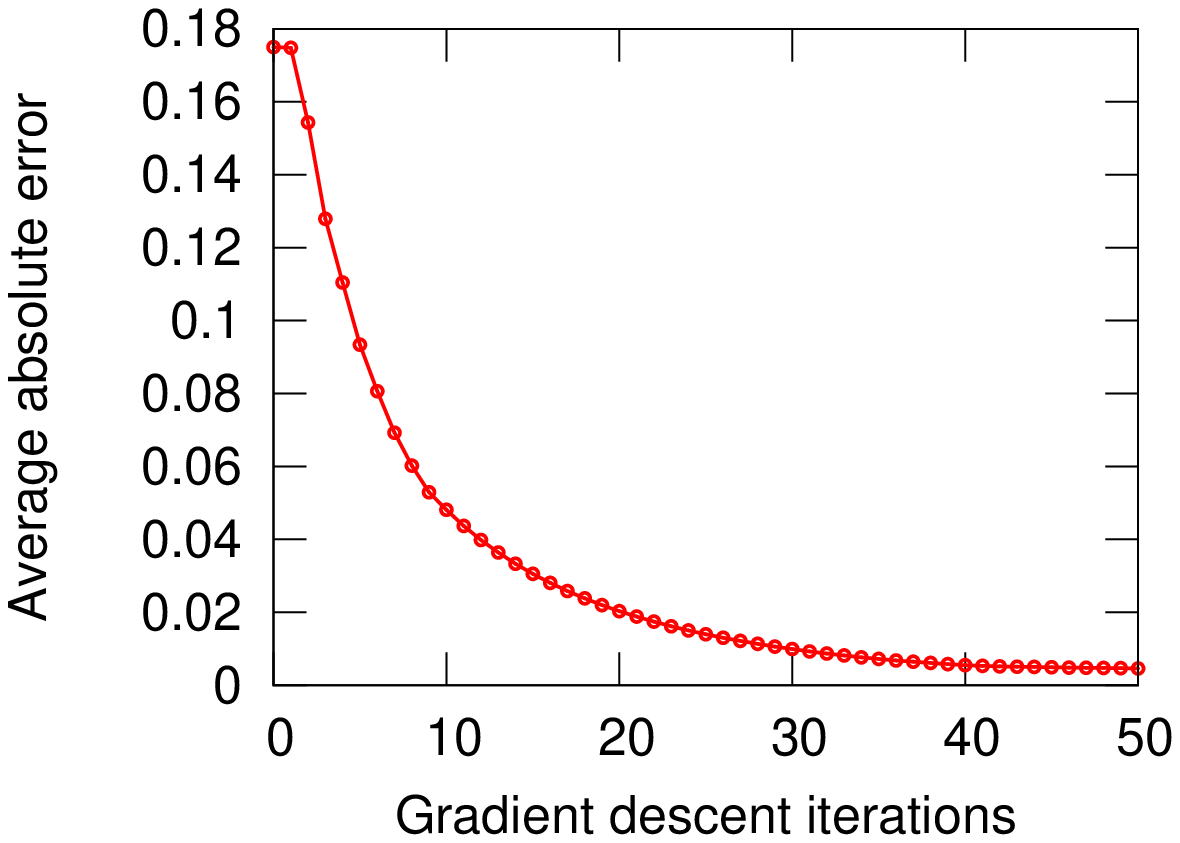} \\
    (a) Log-likelihood & (b) Average error \\
    \includegraphics[width=0.45\textwidth]{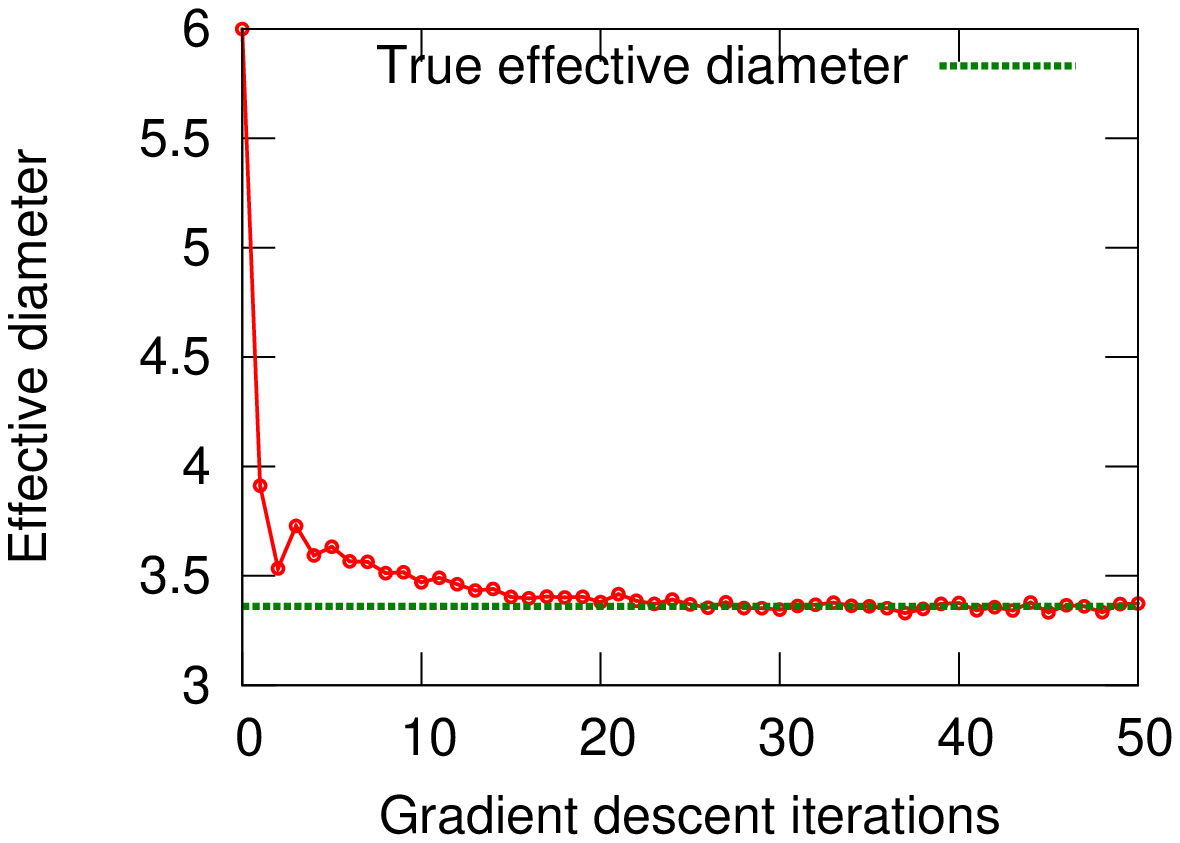} &
    \includegraphics[width=0.45\textwidth]{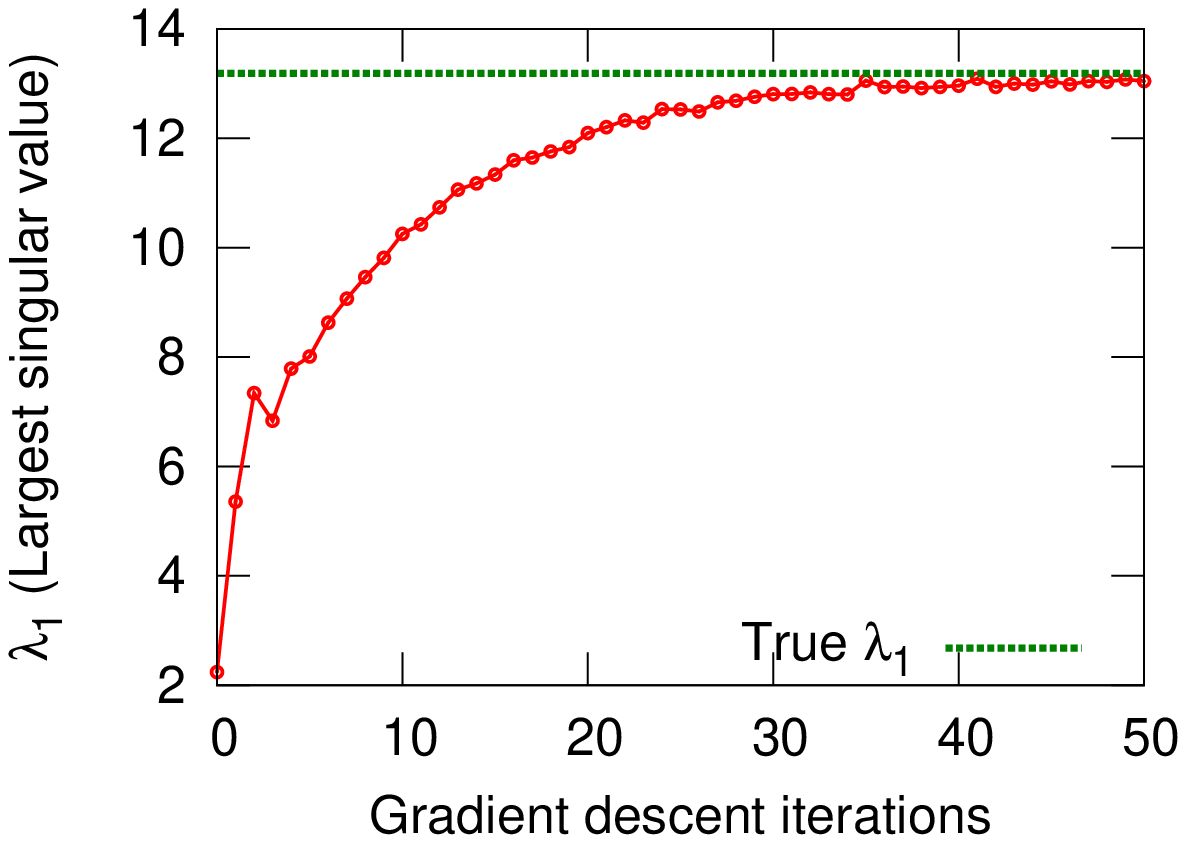} \\
    (c) Effective diameter & (d) Largest singular-value\\
  \end{tabular}
  \end{center}
  \caption{Convergence of graph properties with the number of
  iterations of gradient descent using the synthetic dataset. We start
  with a random choice of parameters and with steps of gradient
  descent the Kronecker graph better and better matches network
  properties of the target graph.}
  \label{fig:KronFitGradDescent}
\end{figure}

We approached the problem of estimating Stochastic Kronecker initiator
matrix $\pzero$ by defining the likelihood over the individual entries of
the graph adjacency matrix. However, what we would really like is to be
given a real graph $\ggraph$ and then generate a synthetic graph $\kgraph$
that has similar properties as the real $\ggraph$. By properties we mean
network statistics that can be computed from the graph, \emph{e.g.},
diameter, degree distribution, clustering coefficient, etc. A priori it is
not clear that our approach which tries to match individual entries of
graph adjacency matrix will also be able to reproduce these global network
statistics. However, as show next this is not the case.

To get some understanding of the convergence of the gradient descent in
terms of the network properties we performed the following experiment.
After every step $t$ of stochastic gradient descent, we compare the true
graph $\ggraph$ with the synthetic Kronecker graph $\kgraph_t$ generated
using the current parameter estimates $\hat\pzero_t$.
Figure~\ref{fig:KronFitGradDescent}(a) gives the convergence of
log-likelihood, and (b) gives absolute error in parameter values ($\sum
|\hat\theta_{ij} - \theta_{ij}^*|$, where $\hat\theta_{ij} \in
\hat\pzero_t$, and $\theta_{ij}^* \in \pzero^*$). Similarly,
Figure~\ref{fig:KronFitGradDescent}(c) plots the effective diameter, and
(d) gives the largest singular value of graph adjacency matrix $K$ as it
converges to largest singular value of $G$.

Note how with progressing iterations of gradient descent properties of
graph $\kgraph_t$ quickly converge to those of $\ggraph$ even though we
are not directly optimizing the similarity in network properties:
log-likelihood increases, absolute error of parameters decreases, diameter
and largest singular value of $\kgraph_t$ both converge to $\ggraph$. This
is a nice result as it shows that through maximizing the likelihood the
resulting graphs become more and more similar also in their structural
properties (even though we are not directly optimizing over them).

\subsection{Fitting to real-world networks}

\begin{figure}[t]
  \begin{center}
  \begin{tabular}{cc}
    \includegraphics[width=0.45\textwidth]{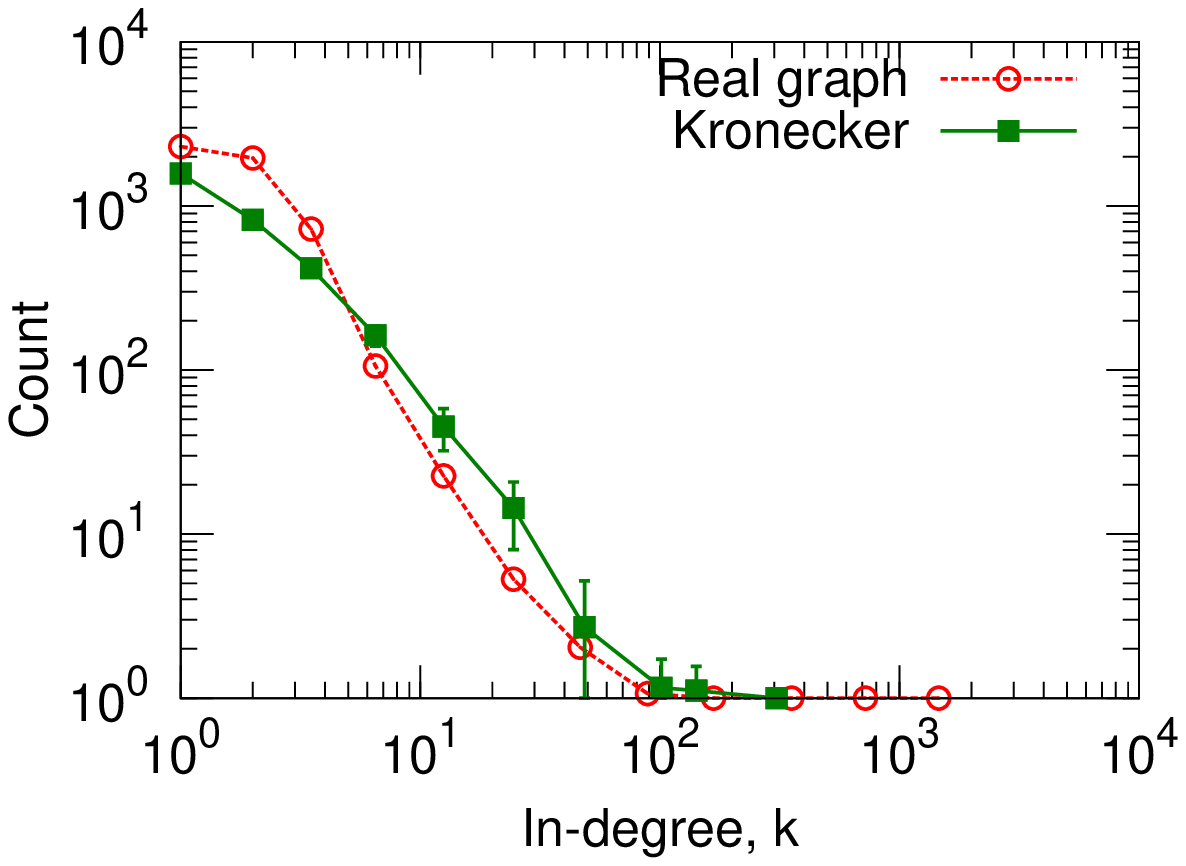} &
    \includegraphics[width=0.45\textwidth]{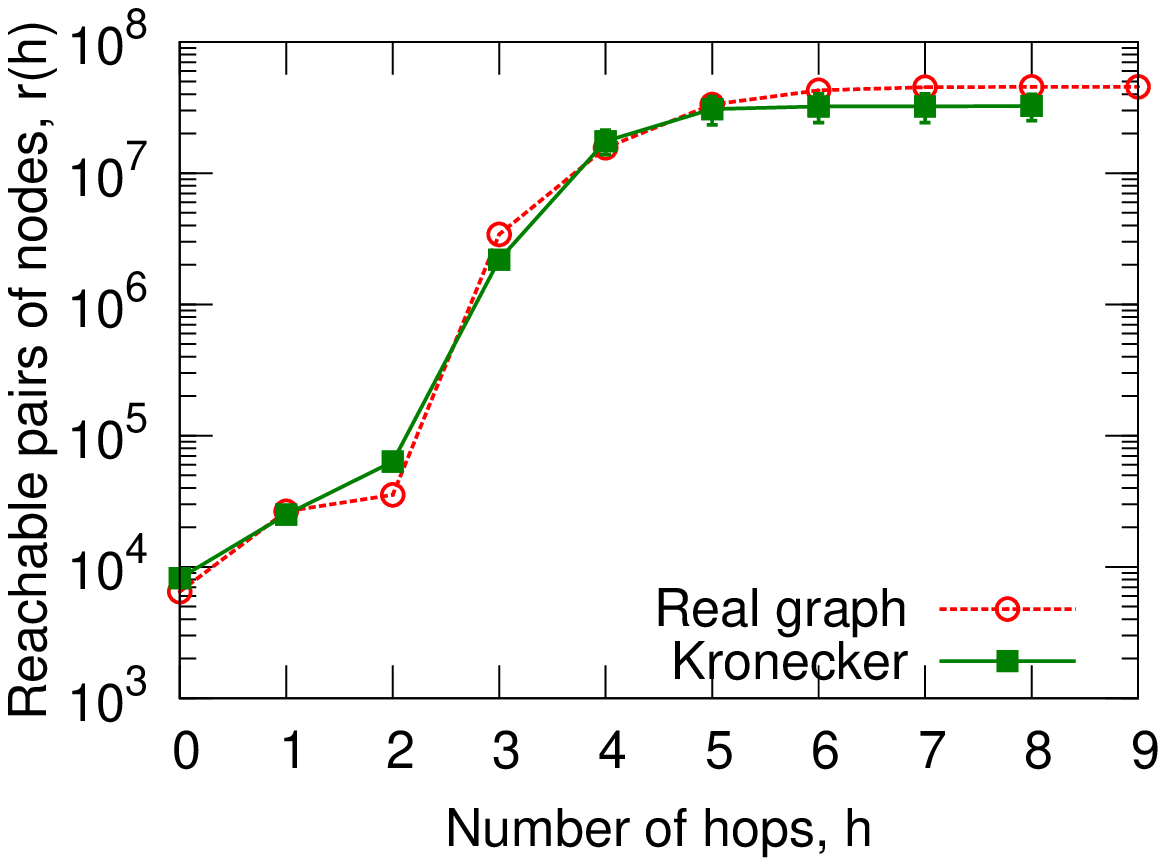} \\
    (a) Degree distribution & (b) Hop plot \\
    \includegraphics[width=0.45\textwidth]{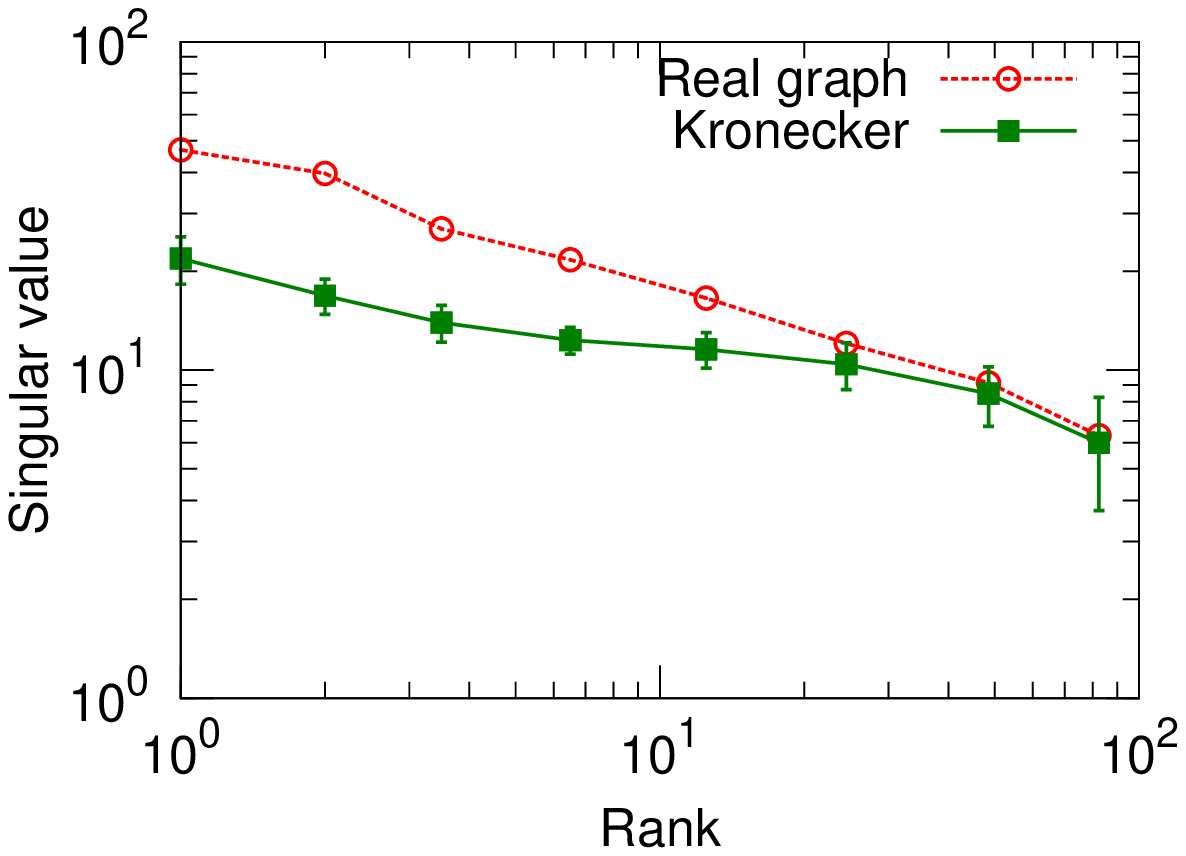} &
    \includegraphics[width=0.45\textwidth]{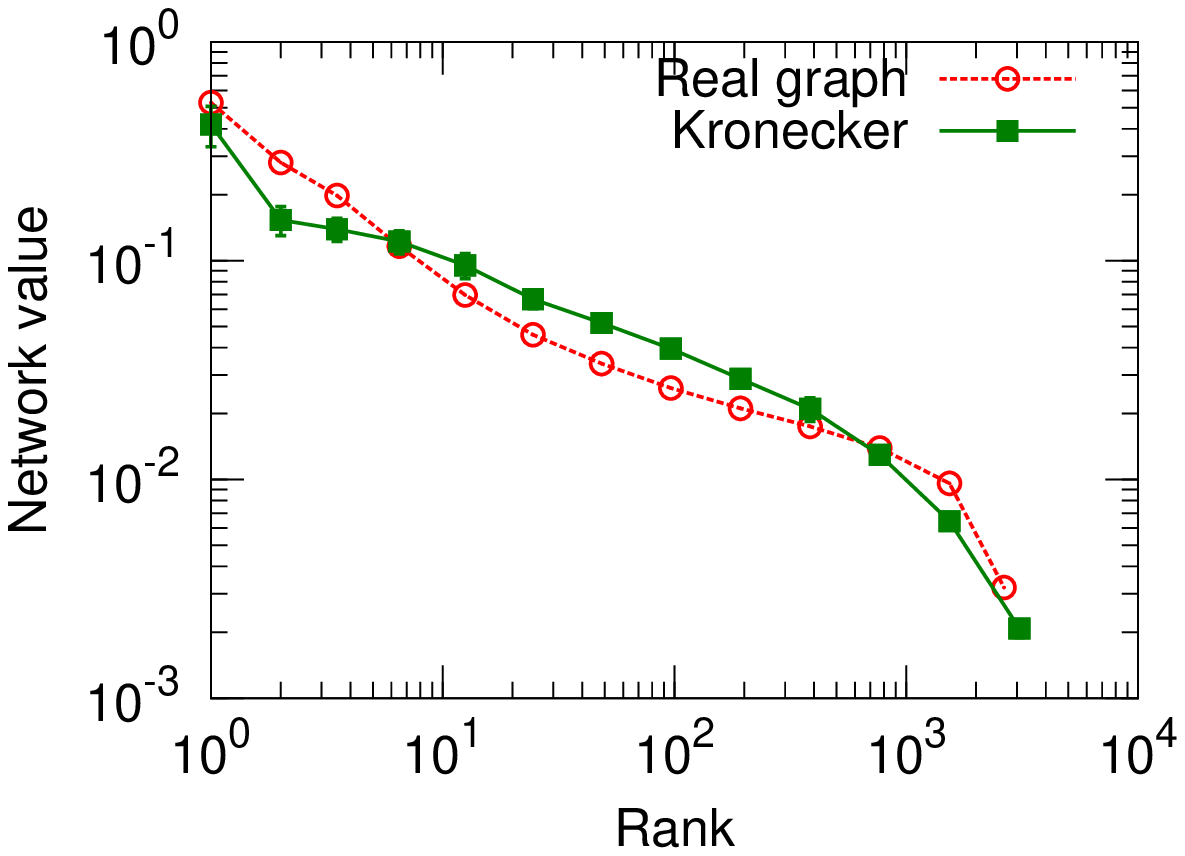} \\
    (c) Scree plot & (d) ``Network'' value \\
  \end{tabular}
  \end{center}
  \caption{{\em Autonomous Systems (\dataset{As-RouteViews}):} Overlayed
  patterns of real graph and the fitted Kronecker graph. Notice that the
  fitted Kronecker graph matches patterns of the real graph while using
  only four parameters ($2\times2$ initiator matrix).}
  \label{fig:KronFitAS}
\end{figure}

Next, we present experiments of fitting Kronecker graph model to
real-world networks. Given a real network $G$ we aim to discover the most
likely parameters $\hat\pzero$ that ideally would generate a synthetic
graph $\kgraph$ having similar properties as real $G$. This assumes that
Kronecker graphs are a good model of the network structure, and that
\KronFit\ is able to find good parameters. In previous section we showed
that \KronFit\ can efficiently recover the parameters. Now we examine how
well can Kronecker graph model the structure of real networks.

We consider several different networks, like a graph of connectivity among
Internet Autonomous systems (\dataset{As-RouteViews}) with $\nnodes=$6,474
and $\nedges=$26,467 a who-trusts-whom type social network from
Epinions~\cite{richardson03trust} (\dataset{Epinions}) with
$\nnodes=$75,879 and $\nedges=$508,960 and many others. The largest
network we consider for fitting is \dataset{Flickr} photo-sharing online
social network with 584,207 nodes and 3,555,115 edges.

For the purpose of this section we take a real network $G$, find
parameters $\hat\pzero$ using \KronFit, generate a synthetic graph
$\kgraph$ using $\hat\pzero$, and then compare $G$ and $\kgraph$ by
comparing their properties that we introduced in
section~\ref{sec:KronRelated}. In all experiments we started from a random
point (random initiator matrix) and run gradient descent for 100 steps. At
each step we estimate the likelihood and the gradient based on 510,000
sampled permutations where we discard first 10,000 samples to allow the
chain to burn-in.

\subsubsection{Fitting to Autonomous Systems network}

First, we focus on the Autonomous Systems network obtained from the
University of Oregon Route Views project~\cite{oregon97as}. Given the AS
network $G$ we run \KronFit\ to obtain parameter estimates $\hat\pzero$.
Using the $\hat\pzero$ we then generate a synthetic Kronecker graph $K$,
and compare the properties of $G$ and $K$.

Figure~\ref{fig:KronFitAS} shows properties of \dataset{As-RouteViews},
and compares them with the properties of a synthetic Kronecker graph
generated using the fitted parameters $\hat\pzero$ of size $2 \times 2$.
Notice that properties of both graphs match really well. The estimated
parameters are $\hat\pzero=[0.987, 0.571; 0.571, 0.049]$.

Figure~\ref{fig:KronFitAS}(a) compares the degree distributions of the
\dataset{As-RouteViews} network and its synthetic Kronecker estimate. In
this and all other plots we use the exponential binning which is a
standard procedure to de-noise the data when plotting on log--log scales.
Notice a very close match in degree distribution between the real graph
and its synthetic counterpart.

Figure~\ref{fig:KronFitAS}(b) plots the cumulative number of pairs of
nodes $g(h)$ that can be reached in $\le h$ hops. The hop plot gives a
sense about the distribution of the shortest path lengths in the network
and about the network diameter. Last, Figures~\ref{fig:KronFitAS}(c) and
(d) plot the spectral properties of the graph adjacency matrix.
Figure~\ref{fig:KronFitAS}(c) plots largest singular values vs. rank, and
(d) plots the components of left singular vector (the network value) vs.
the rank. Again notice the good agreement with the real graph while using only
four parameters.

Moreover, on all plots the error bars of two standard deviations show the
variance of the graph properties for different realizations
$R(\hat\pzero^{[k]})$. To obtain the error bars we took the same
$\hat\pzero$, and generated 50 realizations of a Kronecker graph. As for
the most of the plots the error bars are so small to be practically
invisible, this shows that the variance of network properties when
generating a Stochastic Kronecker graph is indeed very small.

Also notice that the \dataset{As-RouteViews} is an undirected graph, and
that the fitted parameter matrix $\hat\pzero$ is in fact symmetric. This
means that without a priori biasing the fitting towards undirected graphs,
the recovered parameters obey this aspect of the network. Fitting
\dataset{As-RouteViews} graph from a random set of parameters, performing
gradient descent for 100 iterations and at each iteration sampling half a
million permutations, took less than 10 minutes on a standard desktop PC.
This is a significant speedup over~\cite{bezakova06mle}, where by using a
similar permutation sampling approach for calculating the likelihood of a
preferential attachment model on similar \dataset{As-RouteViews} graph
took about two days on a cluster of 50 machines.

\subsubsection{Choice of the initiator matrix size $\nzero$}

\begin{table}[t]
  \begin{center}
    \begin{tabular}{c||c|c|c|c|c}
    $\nzero$ & $l(\hat\pzero)$ & $\nzero^k$ & $\ezero^k$ & $|\{\textrm{deg}(u)>0\}|$ & BIC
score\\ \hline \hline
    2 & $-$152,499 & 8,192  & 25,023 & 5,675 & 152,506 \\ \hline
    3 & $-$127,066 & 6,561 & 28,790 & 5,683 & 127,083  \\ \hline
    4 & $-$153,260 & 16,384 & 24,925 & 8,222 & 153,290 \\ \hline
    5 & $-$149,949 & 15,625 & 29,111 & 9,822 & 149,996 \\ \hline
    6 & $-$128,241 & 7,776 & 26,557 & 6,623 & 128,309  \\ \hline \hline
    \multicolumn{3}{l|}{\dataset{As-RouteViews}} & 26,467 & 6,474 & \\
    \end{tabular}
    \caption{Log-likelihood at MLE for different choices of the size of
    the initiator matrix $\nzero$ for the \dataset{As-RouteViews} graph.
    Notice the log-likelihood $l(\hat\theta)$ generally increases with
    the model complexity $\nzero$. Also notice the effect of zero-padding,
    \emph{i.e.}, for $\nzero=4$ and $\nzero=5$ the constraint of the
    number of nodes being an integer power of $\nzero$ decreases the
    log-likelihood. However, the column $|\{\textrm{deg}(u)>0\}|$ gives
    the number of non-isolated nodes in the network which is much less
    than $\nzero^k$ and is in fact very close to the true number of nodes in
    the \dataset{As-RouteViews}. Using the BIC scores we see that
    $\nzero=3$ or $\nzero=6$ are best choices for the size of the initiator
    matrix.}
    \label{tab:KronFitN0}
  \end{center}
\end{table}

As mentioned earlier for finding the optimal number of parameters, {\em
i.e.}, selecting the size of initiator matrix, BIC criterion naturally
applies to the case of Kronecker graphs. Figure~\ref{fig:KronTimeBic}(b)
shows BIC scores for the following experiment: We generated Kronecker
graph with $\nnodes=$2,187 and $\nedges=$8,736 using $\nzero = 3$ (9
parameters) and $k=7$. For $1 \le \nzero \le 9$ we find the MLE parameters
using gradient descent, and calculate the BIC scores. The model with the
lowest score is chosen. As figure~\ref{fig:KronTimeBic}(b) shows we
recovered the true model, \emph{i.e.}, BIC score is the lowest for the
model with the true number of parameters, $\nzero = 3$.

Intuitively we expect a more complex model with more parameters to fit the
data better. Thus we expect larger $\nzero$ to generally give better
likelihood. On the other hand the fit will also depend on the size of the
graph $\ggraph$. Kronecker graphs can only generate graphs on $\nzero^k$
nodes, while real graphs do not necessarily have $\nzero^k$ nodes (for
some, preferably small, integers $\nzero$ and $k$). To solve this problem
we choose $k$ so that $\nzero^{k-1} < \nnodes(G) \le \nzero^k$, and then
augment $G$ by adding $\nzero^k - \nnodes$ isolated nodes. Or
equivalently, we pad the adjacency matrix of $G$ with zeros until it is of
the appropriate size, $\nzero^k \times \nzero^k$. While this solves the
problem of requiring the integer power of the number of nodes it also
makes the fitting problem harder as when $\nnodes \ll \nzero^k$ we are
basically fitting $G$ plus a large number of isolated nodes.

Table~\ref{tab:KronFitN0} shows the results of fitting Kronecker graphs to
\dataset{As-RouteViews} while varying the size of the initiator matrix
$\nzero$. First, notice that in general larger $\nzero$ results in higher
log-likelihood $l(\hat\pzero)$ at MLE. Similarly, notice (column
$\nzero^k$) that while \dataset{As-RouteViews} has 6,474 nodes,
Kronecker estimates have up to 16,384 nodes (16,384$=4^7$, which is the
first integer power of 4 greater than 6,474). However, we also show the
number of non-zero degree (non-isolated) nodes in the Kronecker graph
(column $|\{\textrm{deg}(u)>0\}|$). Notice that the number of non-isolated
nodes well corresponds to the number of nodes in \dataset{As-RouteViews}
network. This shows that \KronFit is actually fitting the graph well and it
successfully fits the structure of the graph plus a number of isolated
nodes. Last, column $\ezero^k$ gives the number of edges in the
corresponding Kronecker graph which is close to the true number of edges
of the \dataset{As-RouteViews} graph.

Last, comparing the log-likelihood at the MLE and the BIC score in
Table~\ref{tab:KronFitN0} we notice that the log-likelihood heavily
dominates the BIC score. This means that the size of the initiator matrix
(number of parameters) is so small that overfitting is not a concern.
Thus we can just choose the initiator matrix that
maximizes the likelihood. A simple calculation shows that one would need
to take initiator matrices with thousands of entries before the model
complexity part of BIC score would start to play a significant role.

\begin{figure}[t]
  \begin{center}
  \begin{tabular}{cc}
    \includegraphics[width=0.45\textwidth]{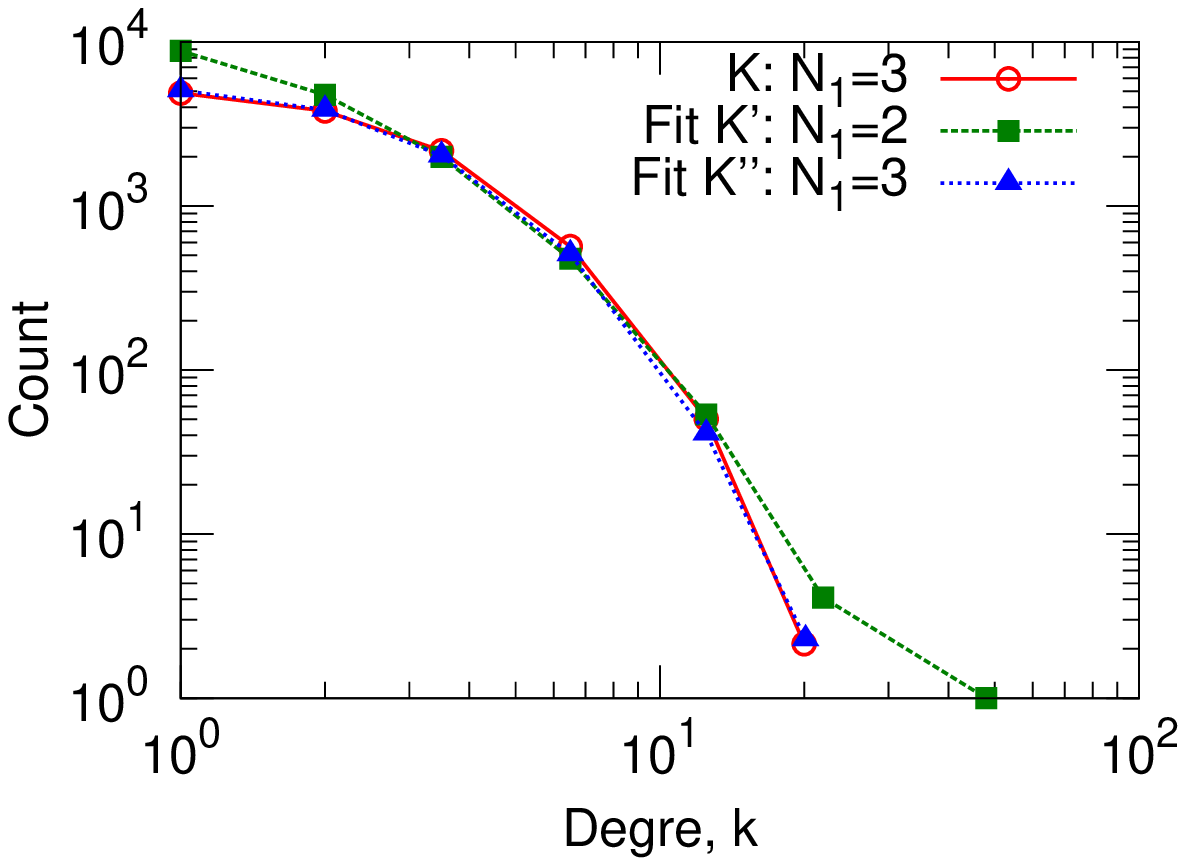} &
    \includegraphics[width=0.45\textwidth]{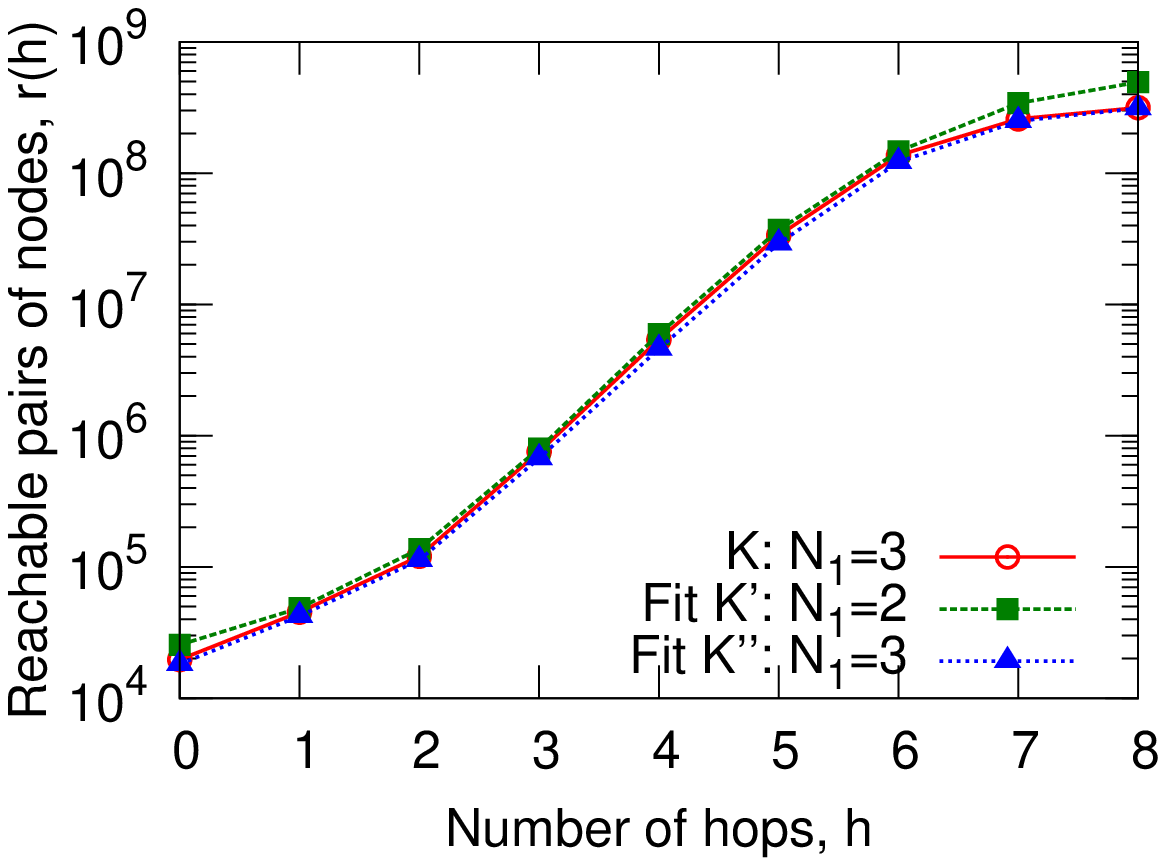} \\
    (a) Degree distribution & (b) Hop plot \\
    \includegraphics[width=0.45\textwidth]{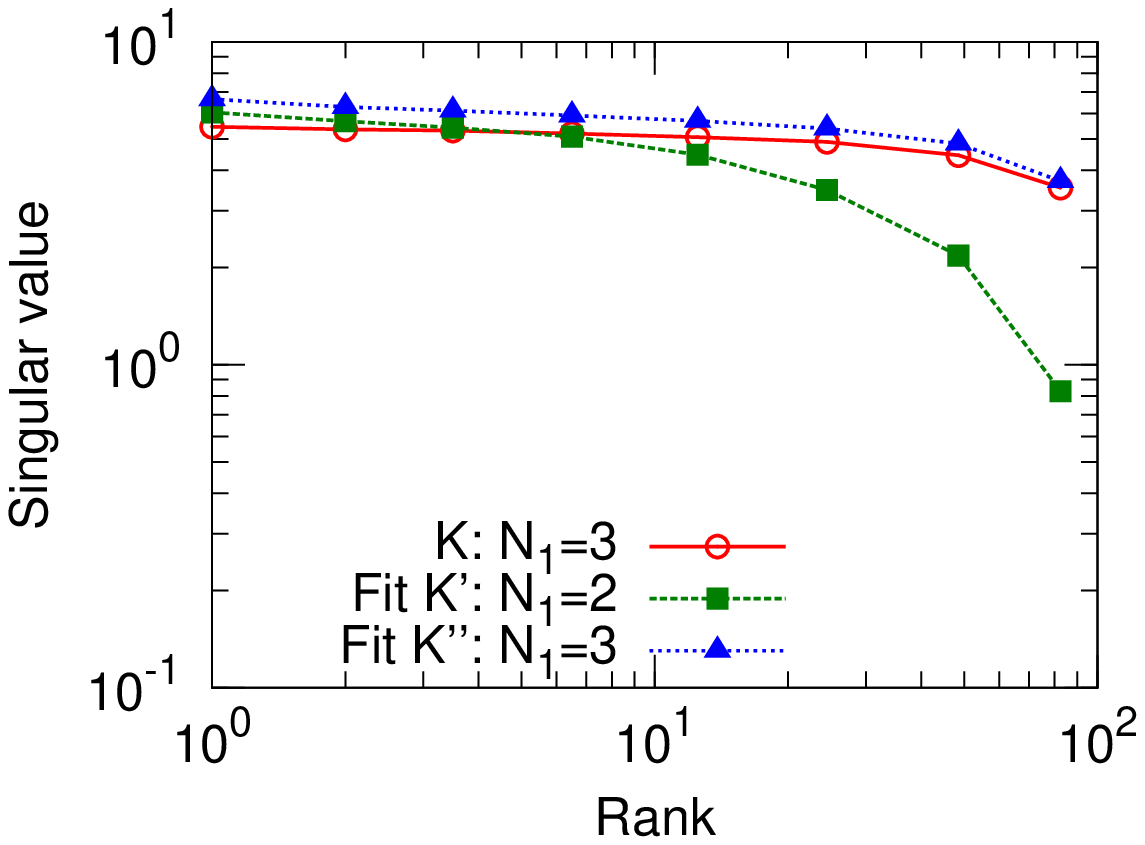} &
    \includegraphics[width=0.45\textwidth]{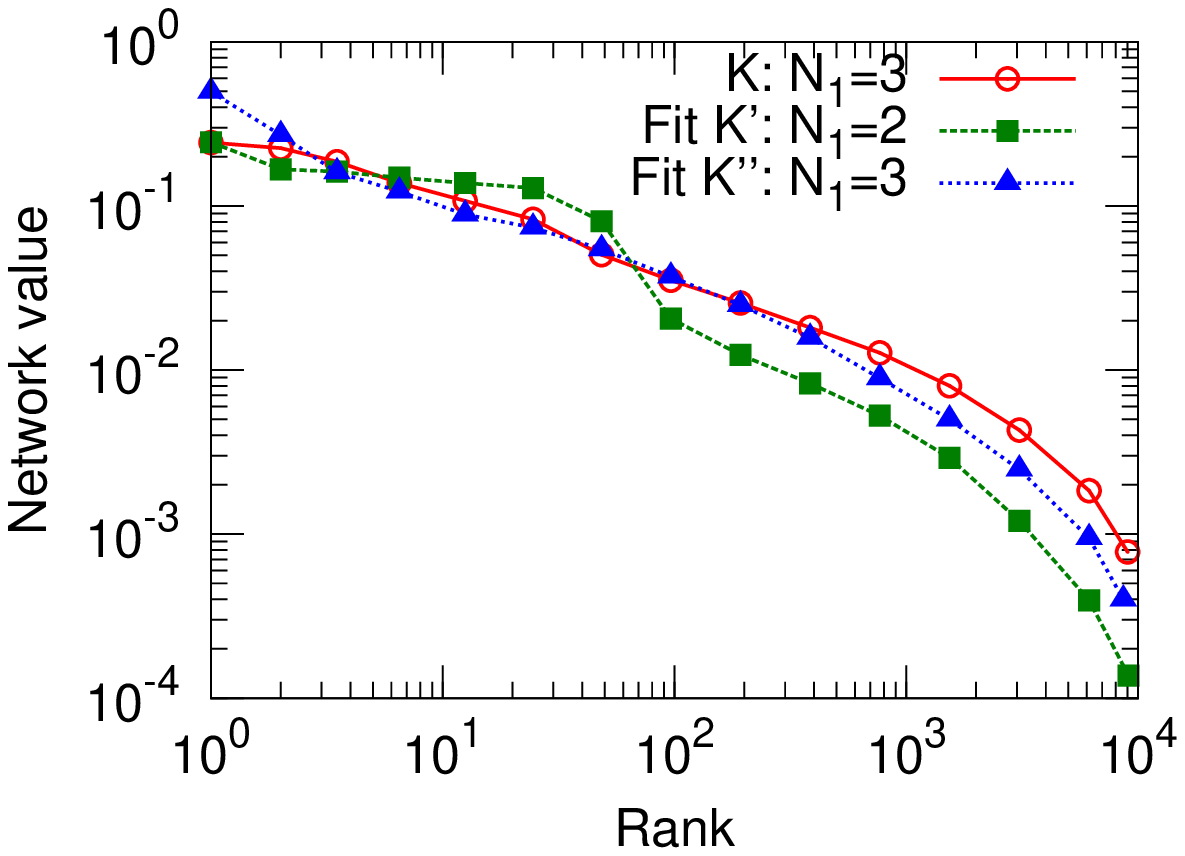} \\
    (c) Scree plot & (d) ``Network'' value \\
  \end{tabular}
  \end{center}
  \caption{{\em 3 by 3 \SKRG:} Given a \SKRG $G$ generated from $\nzero=3$
  (red curve), we fit a Kronecker graph $K'$ with $\nzero'=2$ (green)
  and $K''$ with $\nzero''=3$ (blue). Not surprisingly $K''$ fits the
  properties of $K$ perfectly as the model is the of same complexity.
  On the other hand $K'$ has only 4 parameters (instead of 9 as in $K$ and
  $K''$) and still fits well.}
  \label{fig:KronFitKron3by3}
\end{figure}

We further examine the sensitivity of the choice of the initiator size by
the following experiment. We generate a \SKRG $K$ on 9 parameters
($\nzero=3$), and then fit a Kronecker graph $K'$ with a smaller number of
parameters (4 instead of 9, $\nzero'=2$). And also a Kronecker graph $K''$
of the same complexity as $K$ ($\nzero''=3$).

Figure~\ref{fig:KronFitKron3by3} plots the properties of all three graphs.
Not surprisingly $K''$ (blue) fits the properties of $K$ (red) perfectly
as the initiator is of the same size. On the other hand $K'$ (green) is a
simpler model with only 4 parameters (instead of 9 as in $K$ and $K''$)
and still generally fits well: hop plot and degree distribution match
well, while spectral properties of graph adjacency matrix, especially
scree plot, are not matched that well. This shows that nothing drastic
happens and that even a bit too simple model still fits the data well. In
general we observe empirically that by increasing the size of the initiator
matrix one does not gain radically better fits for degree distribution and
hop plot. On the other hand there is usually an improvement in the scree
plot and the plot of network values when one increases the initiator size.

\subsubsection{Network parameters over time}

\begin{table}[t]
  \begin{center}
    \begin{tabular}{c||c|c|c|c}
    Snapshot at time & $\nnodes$ & $\nedges$ & $l(\hat\pzero)$ & Estimates at MLE, $\hat\pzero$
\\ \hline \hline
    $T_1$ & 2,048 & 8,794 & $-$40,535  & $[0.981, 0.633; 0.633, 0.048]$ \\ \hline
    $T_2$ & 4,088 & 15,711 & $-$82,675 & $[0.934, 0.623; 0.622, 0.044]$ \\ \hline
    $T_3$ & 6,474 & 26,467 & $-$152,499 & $[0.987, 0.571; 0.571, 0.049]$ \\
    \end{tabular}
    \caption{Parameter estimates of the three temporal snapshots of the
    \dataset{As-RouteViews} network. Notice that estimates stay remarkably
    stable over time.}
    \label{tab:KronFitAsOverTm}
    \vspace{-5mm}
  \end{center}
\end{table}

\begin{figure}[t]
  \begin{center}
  \begin{tabular}{cc}
    \includegraphics[width=0.45\textwidth]{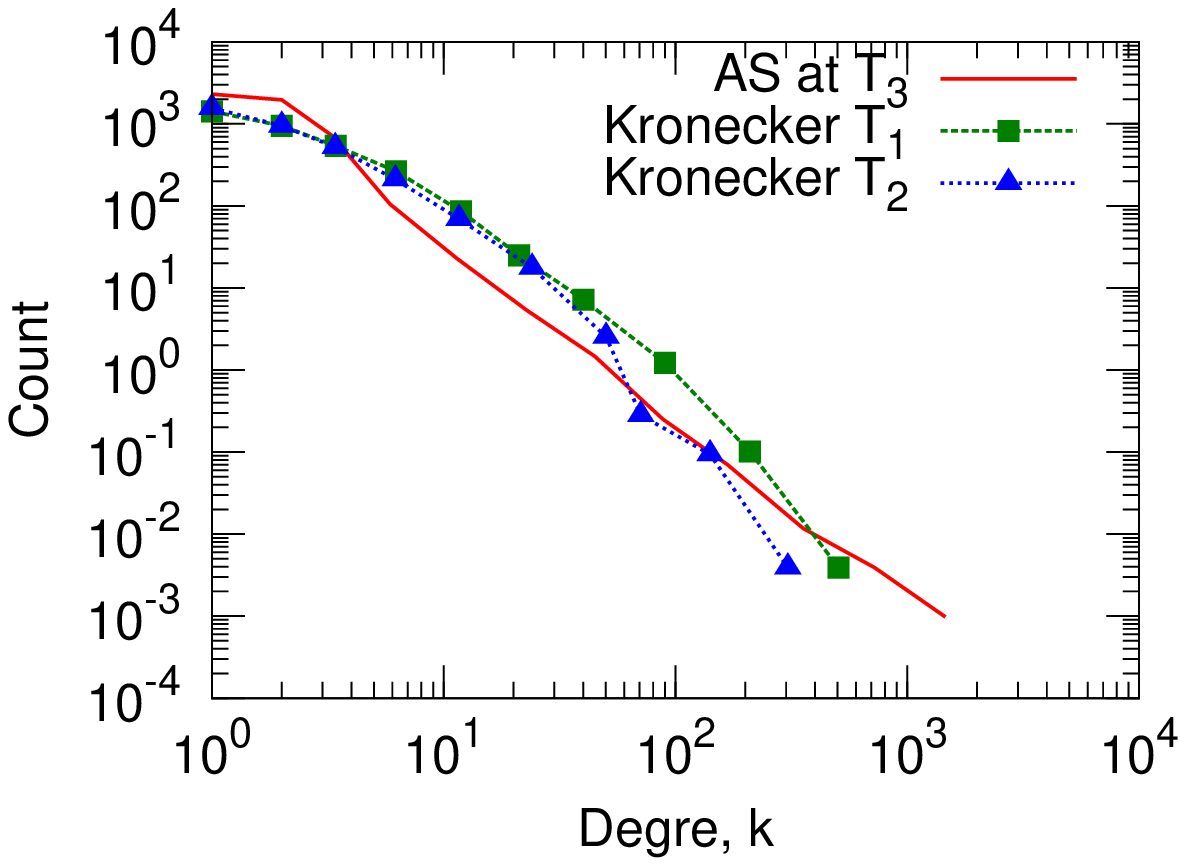} &
    \includegraphics[width=0.45\textwidth]{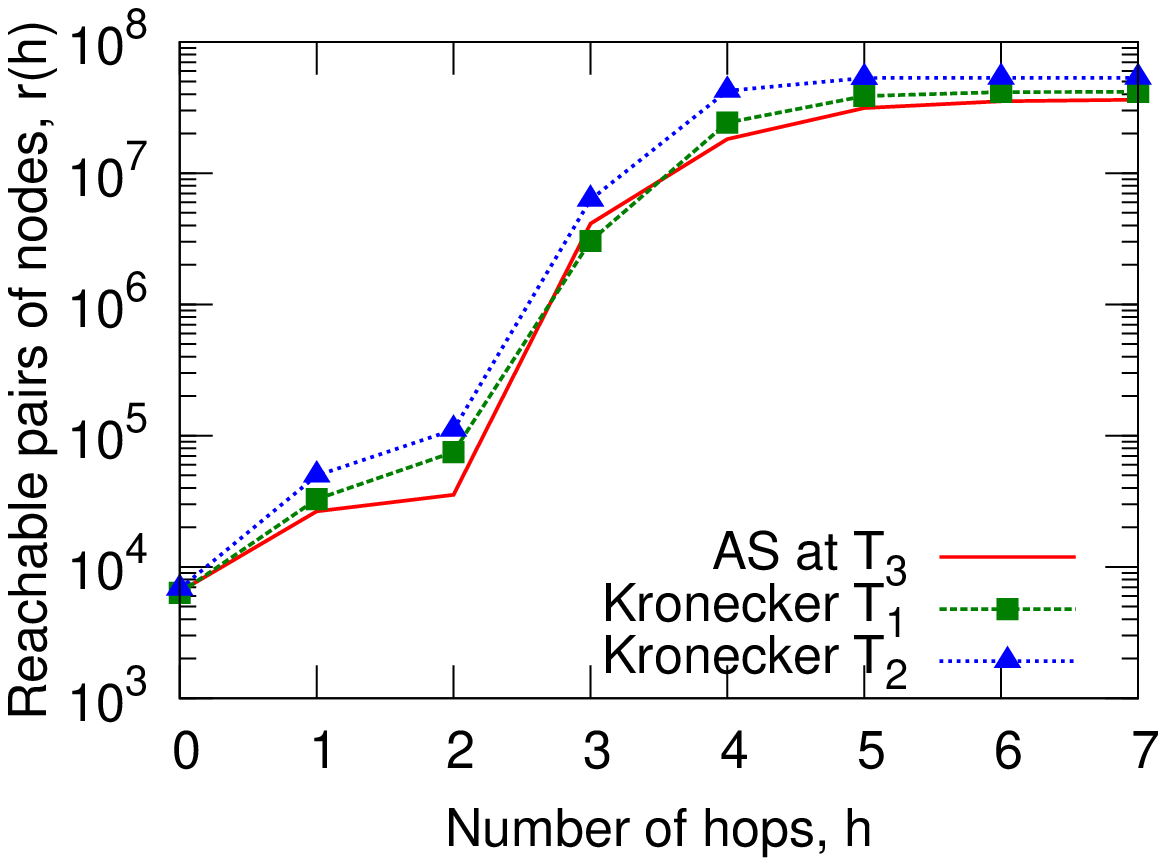} \\
    (a) Degree & (b) Hop plot \\
    \includegraphics[width=0.45\textwidth]{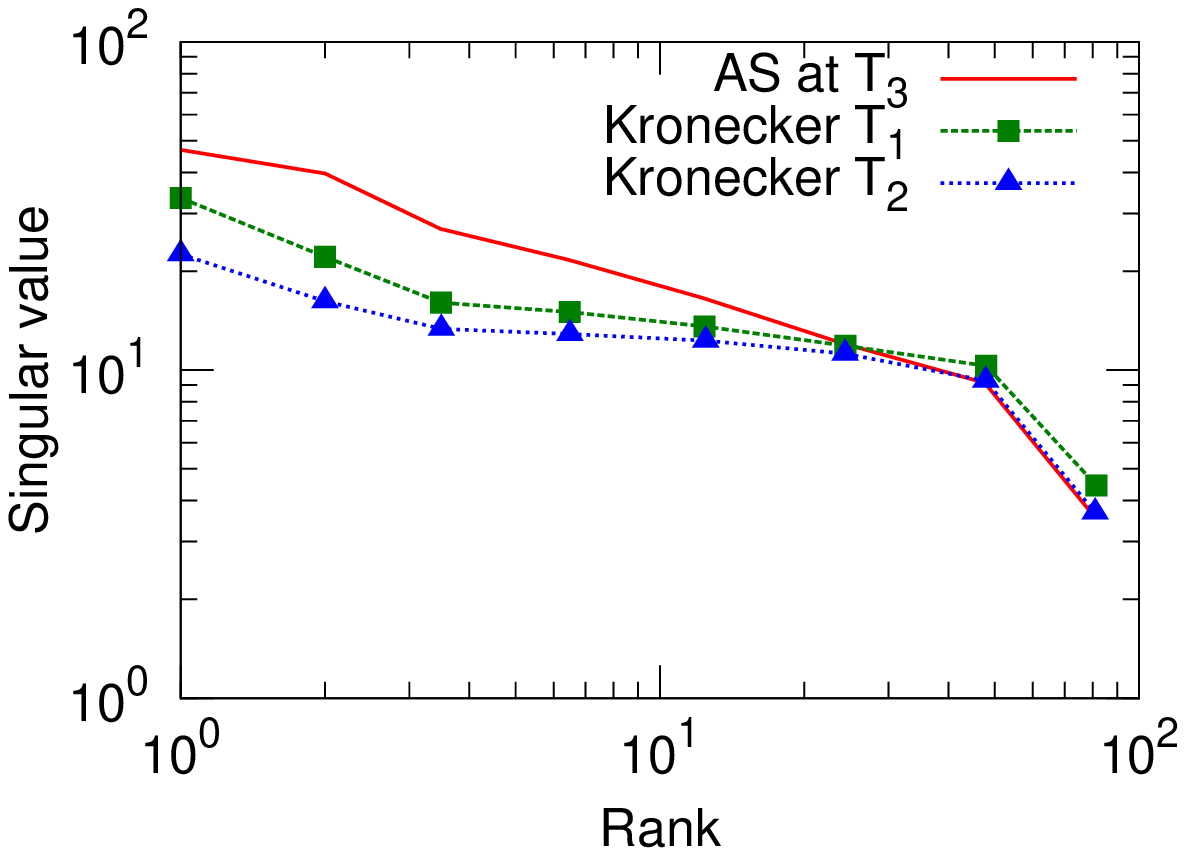} &
    \includegraphics[width=0.45\textwidth]{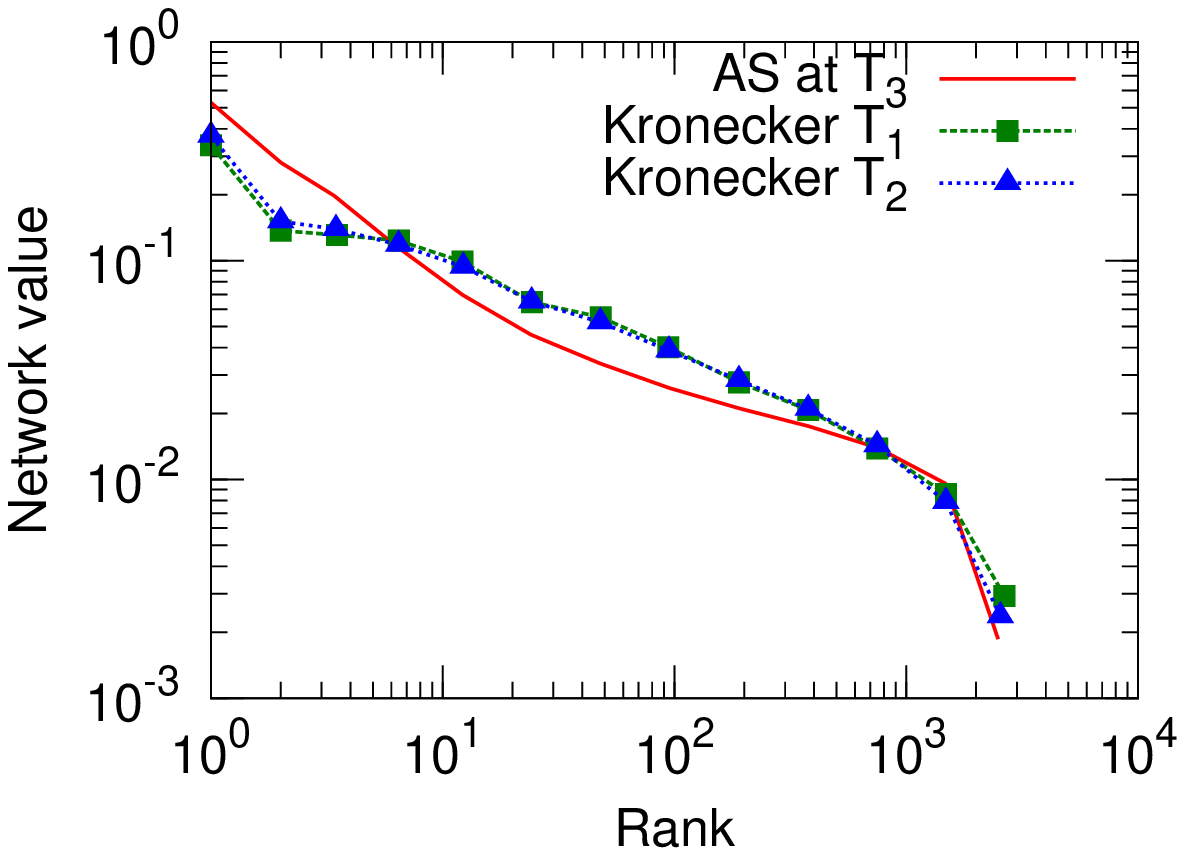} \\
    (c) Scree plot & (d) ``Network'' value \\
  \end{tabular}
  \end{center}
  \caption{{\em Autonomous systems network over time}
  (\dataset{As-RouteViews}): Overlayed patterns of real
  \dataset{As-RouteViews} network at time $T_3$ and the Kronecker graphs
  with parameters estimated from \dataset{As-RouteViews} at time $T_1$ and
  $T_2$. Notice good fits which means that parameters estimated on
  historic snapshots can be used to estimate the graph in the future.}
  \label{fig:KronFitAsOverTm}
  \vspace{-4mm}
\end{figure}

Next we briefly examine the evolution of the Kronecker initiator for a
temporally evolving graph. The idea is that given parameter estimates of a
real-graph $G_t$ at time $t$, we can forecast the future structure of the
graph $G_{t+x}$ at time $t+x$, \emph{i.e.}, using parameters obtained from
$G_t$ we can generate a larger synthetic graph $K$ that will be similar to
$G_{t+x}$.

As we have the information about the evolution of the
\dataset{As-RouteViews} network, we estimated parameters for three
snapshots of the network when it had about $2^k$ nodes.
Table~\ref{tab:KronFitAsOverTm} gives the results of the fitting for the
three temporal snapshots of the \dataset{As-RouteViews} network. Notice
the parameter estimates $\hat\pzero$ remain remarkably stable over time.
This means that Kronecker graphs can be used to estimate the structure of
the networks in the future, \emph{i.e.}, parameters estimated from the
historic data can extrapolate the graph structure in the future.

Figure~\ref{fig:KronFitAsOverTm} further explores this. It overlays the
graph properties of the real \dataset{As-RouteViews} network at time $T_3$
and the synthetic graphs for which we used the parameters obtained on
historic snapshots of \dataset{As-RouteViews} at times $T_1$ and $T_2$.
The agreements are good which demonstrates that Kronecker graphs can
forecast the structure of the network in the future.

Moreover, this experiments also shows that parameter estimates do not
suffer much from the zero padding of graph adjacency matrix (\emph{i.e.},
adding isolated nodes to make $G$ have $\nzero^k$ nodes). Snapshots of
\dataset{As-RouteViews} at $T_1$ and $T_2$ have close to $2^k$ nodes,
while we had to add 26\% (1,718) isolated nodes to the network at $T_3$ to
make the number of nodes be $2^k$. Regardless of this we see the parameter
estimates $\hat\pzero$ remain basically constant over time, which seems to
be independent of the number of isolated nodes added. This means that the
estimated parameters are not biased too much from zero padding the
adjacency matrix of $G$.

\subsection{Fitting to other large real-world networks}
\label{sec:bigtable}

\begin{table}[t]
  \begin{center}
    {\small
    \begin{tabular}{l||r|r|c|r|r}
      Network & $\nnodes$ & $\nedges$ & Estimated MLE parameters $\hat\pzero$  & $l(\hat\pzero)$ & Time \\
      \hline \hline
      \dataset{As-RouteViews} & 6,474 & 26,467 & $[0.987, 0.571; 0.571, 0.049]$& $-$152,499& 8m15s \\ \hline
      \dataset{AtP-gr-qc} & 19,177 & 26,169 & $[0.902, 0.253; 0.221, 0.582]$& $-$242,493& 7m40s \\ \hline
      \dataset{Bio-Proteins} & 4,626 & 29,602 & $[0.847, 0.641; 0.641, 0.072]$& $-$185,130& 43m41s   \\ \hline
      \dataset{Email-Inside} & 986 & 32,128 & $[0.999, 0.772; 0.772, 0.257]$& $-$107,283& 1h07m \\ \hline
      \dataset{CA-gr-qc} & 5,242 & 28,980 & $[0.999, 0.245; 0.245, 0.691]$ & $-$160,902& 14m02s \\ \hline
      \dataset{As-Newman}  &  22,963  &  96,872  &  $[0.954, 0.594; 0.594, 0.019]$ &  $-$593,747&  28m48s \\ \hline
      \dataset{Blog-nat05-6m}  &  31,600  &  271,377  &  $[0.999, 0.569; 0.502, 0.221]$ & $-$1,994,943 &  47m20s \\ \hline
      \dataset{Blog-nat06all}  &  32,443  &  318,815  &  $[0.999, 0.578; 0.517, 0.221]$ & $-$2,289,009 &  52m31s \\ \hline
      \dataset{CA-hep-ph}  &  12,008  &  237,010  &  $[0.999, 0.437; 0.437, 0.484]$ & $-$1,272,629 &  1h22m \\ \hline
      \dataset{CA-hep-th}  &  9,877  &  51,971  &  $[0.999, 0.271; 0.271, 0.587]$ &  $-$343,614&  21m17s \\ \hline
      \dataset{Cit-hep-ph}  &  30,567  &  348,721  &  $[0.994, 0.439; 0.355, 0.526]$ & $-$2,607,159 &  51m26s \\ \hline
      \dataset{Cit-hep-th}  &  27,770  &  352,807  &  $[0.990, 0.440; 0.347, 0.538]$ & $-$2,507,167 &  15m23s \\ \hline
      \dataset{Epinions}  &  75,879  &  508,837  &  $[0.999, 0.532; 0.480, 0.129]$ & $-$3,817,121 &  45m39s \\ \hline
      \dataset{Gnutella-25}  &  22,687  &  54,705  &  $[0.746, 0.496; 0.654, 0.183]$ & $-$530,199 &  16m22s \\ \hline
      \dataset{Gnutella-30}  &  36,682  &  88,328  &  $[0.753, 0.489; 0.632, 0.178]$ & $-$919,235 &  14m20s \\ \hline
      \dataset{Delicious}  &  205,282  &  436,735  &  $[0.999, 0.327; 0.348, 0.391]$ & $-$4,579,001 &  27m51s \\ \hline
      \dataset{Answers}  &  598,314 & 1,834,200 & $[0.994, 0.384; 0.414, 0.249]$ & $-$20,508,982 & 2h35m  \\ \hline
      \dataset{CA-DBLP} & 425,957 & 2,696,489 & $[0.999, 0.307; 0.307, 0.574]$ &  $-$26,813,878 &  3h01m\\ \hline
      \dataset{Flickr} & 584,207 & 3,555,115 & $[0.999, 0.474; 0.485, 0.144]$ &  $-$32,043,787 &  4h26m \\ \hline
      \dataset{Web-Notredame} & 325,729 & 1,497,134 &  $[0.999, 0.414; 0.453, 0.229]$ & $-$14,588,217 &  2h59m
      \end{tabular} }
    \caption{Results of parameter estimation for 20 different networks.
    Table~\ref{tab:data_StatsDesc} gives the description and
    basic properties of the above network datasets.
    Networks are available for download at {\tt http://snap.stanford.edu}.
    }
    \label{tab:KronFitEst}
  \end{center}
\end{table}

\begin{figure}[t]
  \begin{center}
  \begin{tabular}{ccc}
    \includegraphics[width=0.31\textwidth]{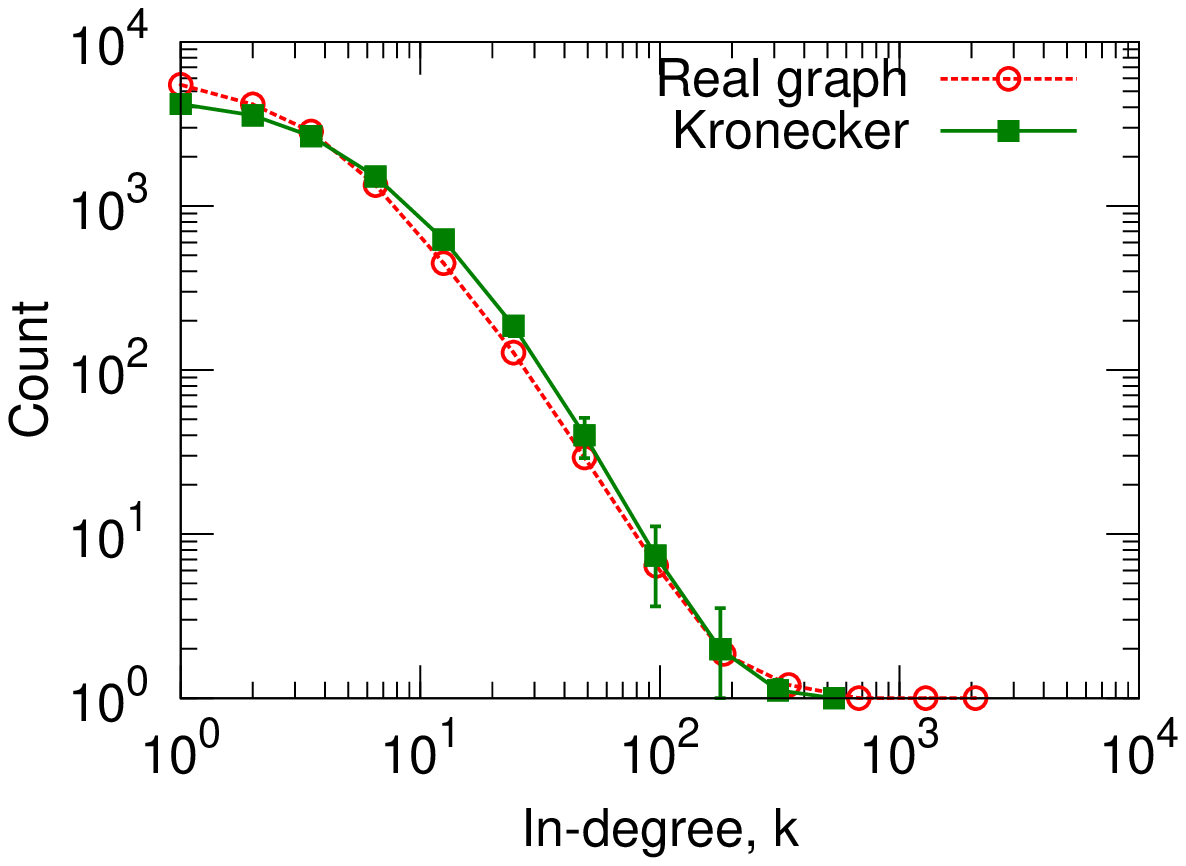} &
    \includegraphics[width=0.31\textwidth]{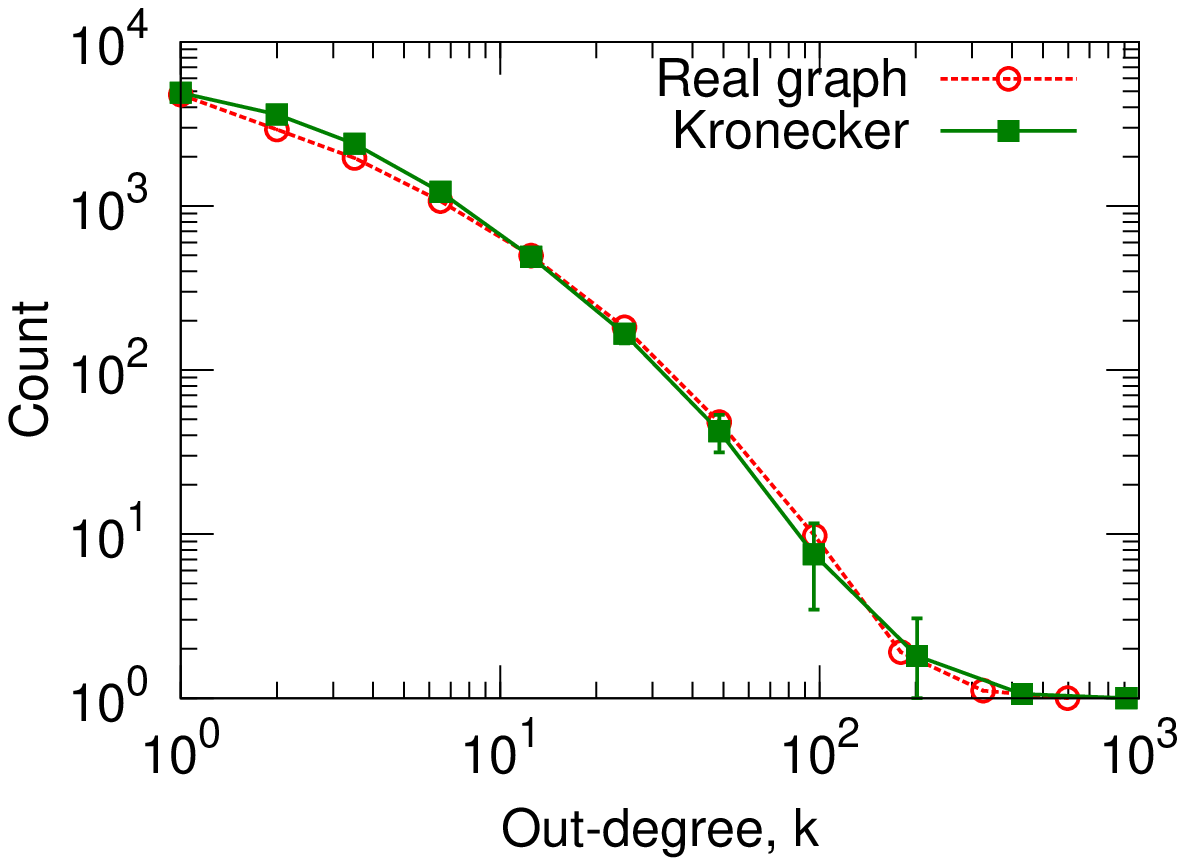} &
    \includegraphics[width=0.31\textwidth]{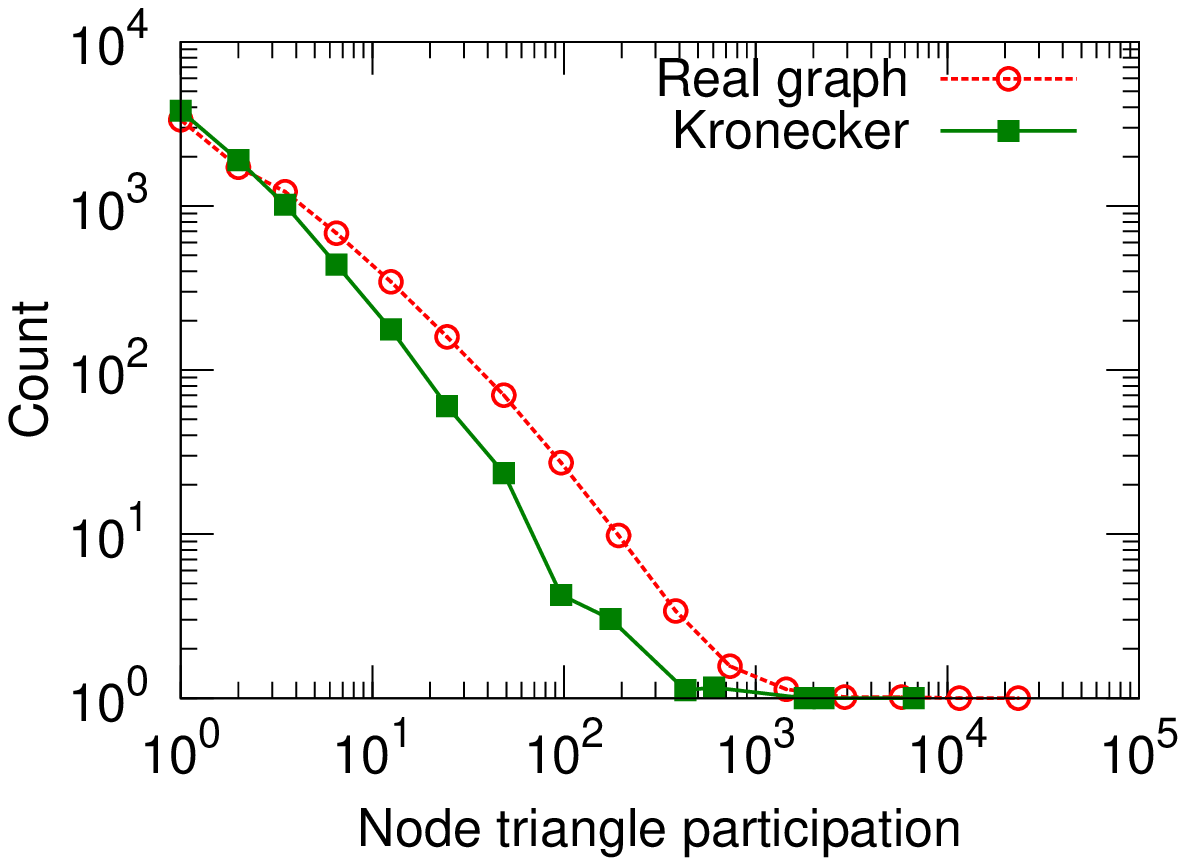} \\
    (a) In-Degree & (b) Out-degree & (c) Triangle participation \\
    \includegraphics[width=0.31\textwidth]{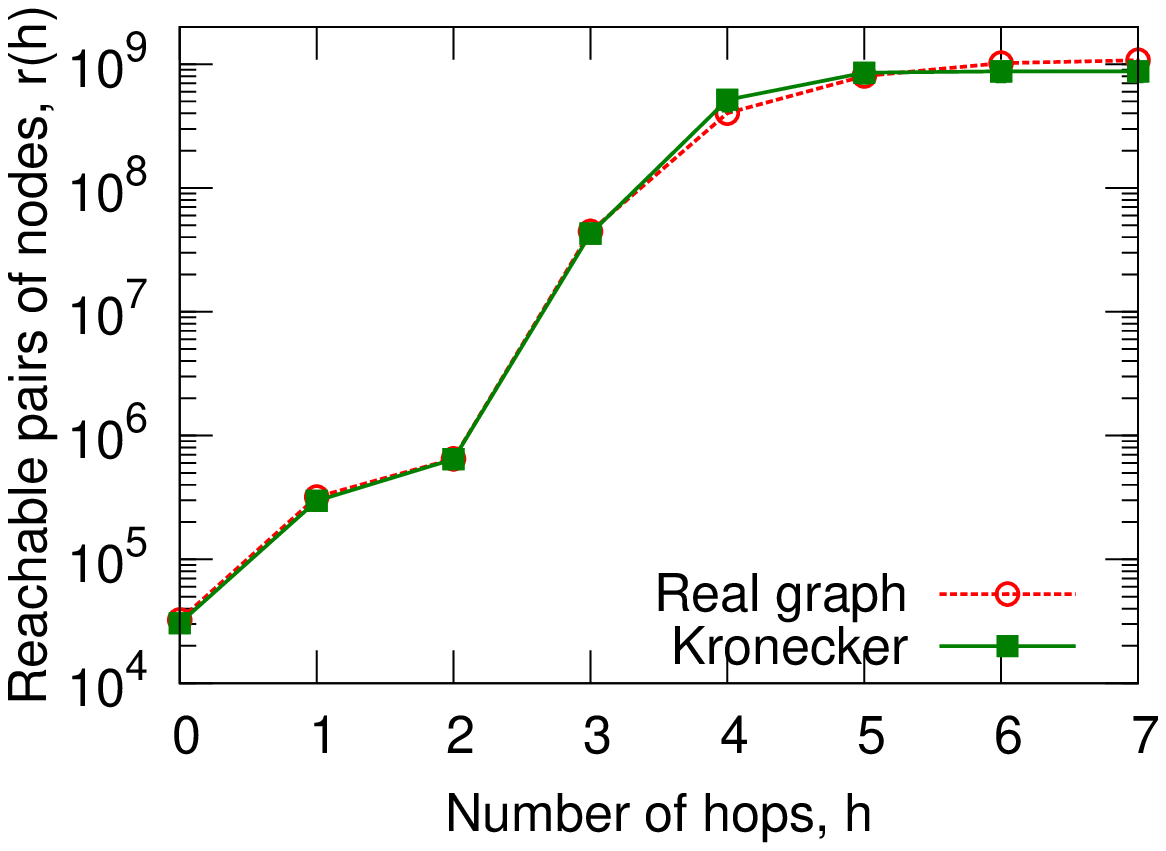} &
    \includegraphics[width=0.31\textwidth]{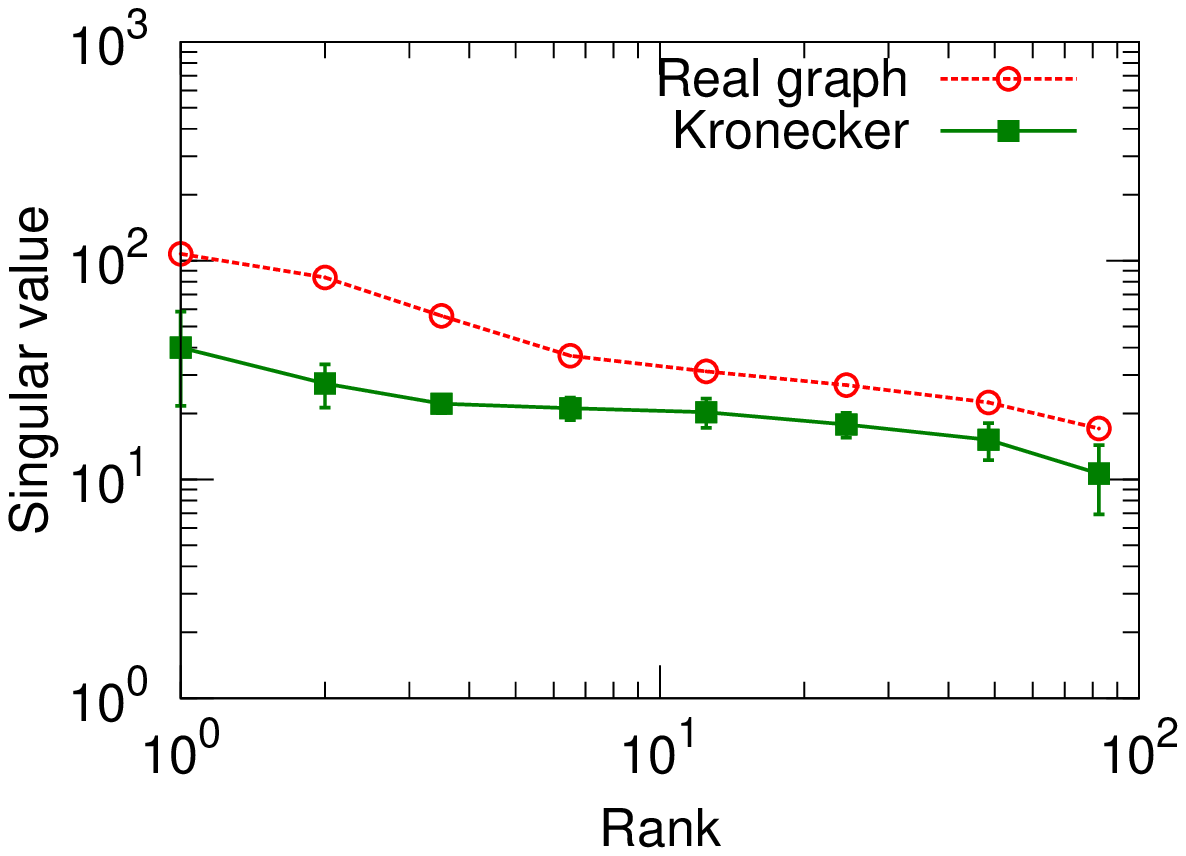} &
    \includegraphics[width=0.31\textwidth]{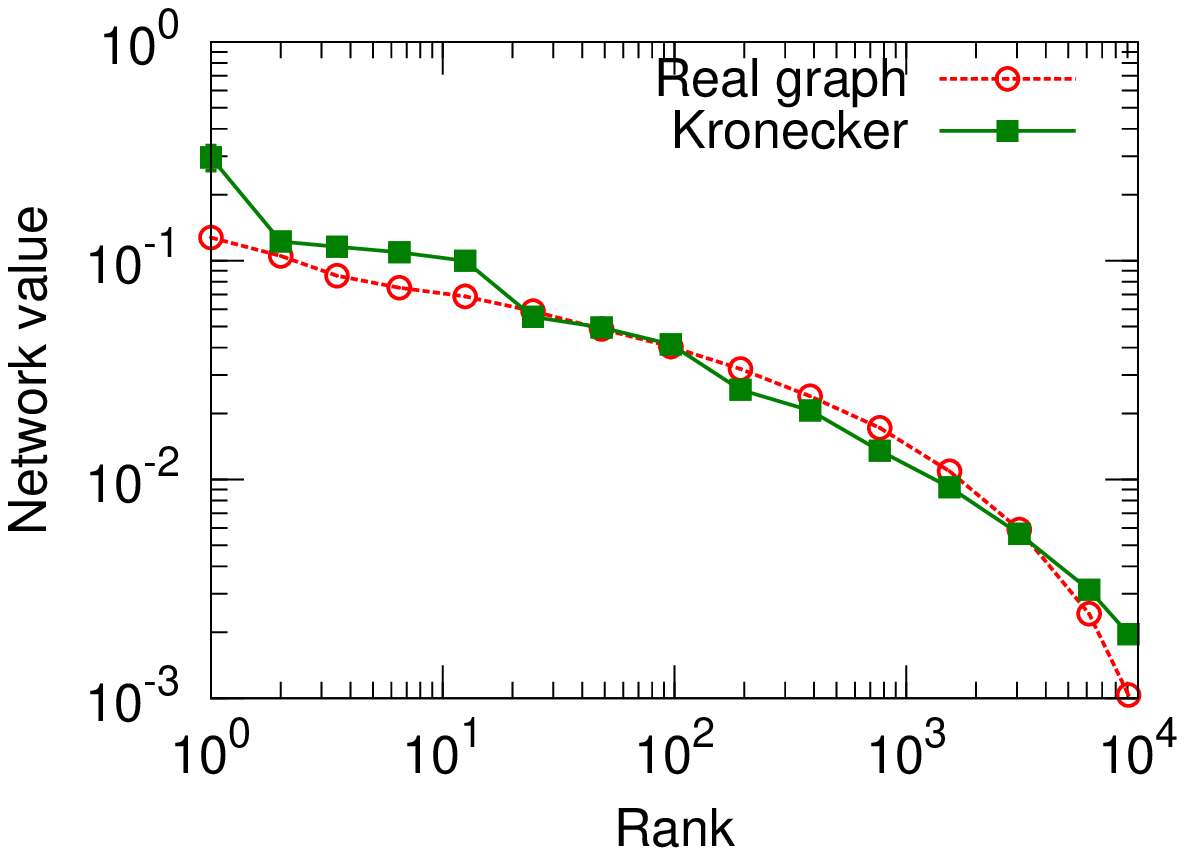} \\
    (d) Hop plot & (e) Scree plot & (f) ``Network'' value \\
  \end{tabular}
  \end{center} \caption{{\em Blog network} (\dataset{Blog-nat06all}):
  Overlayed patterns of real network and the estimated Kronecker graph
  using 4 parameters ($2\times2$ initiator matrix). Notice
  that the Kronecker graph matches all properties of the real network.}
  \label{fig:KronFitBlog}
\end{figure}

\begin{figure}[t]
  \begin{center}
  \begin{tabular}{ccc}
    \includegraphics[width=0.31\textwidth]{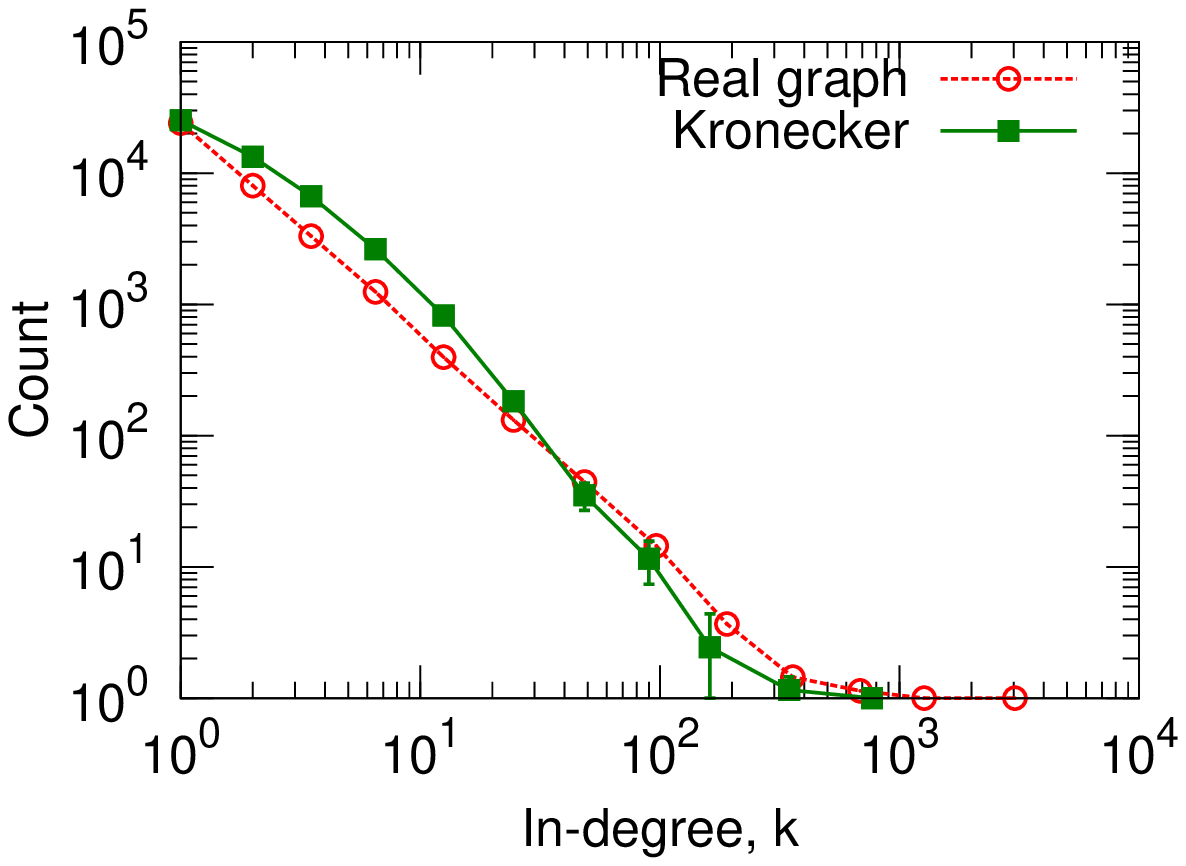} &
    \includegraphics[width=0.31\textwidth]{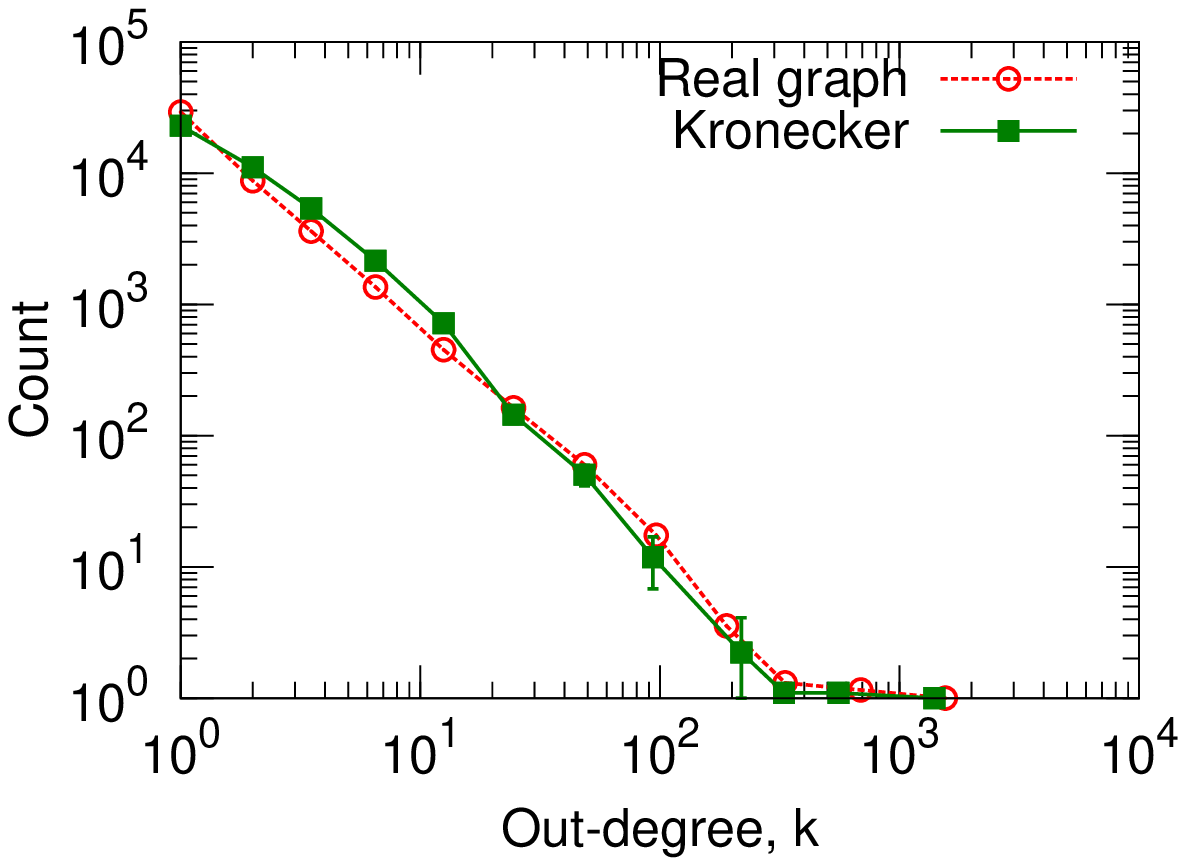} &
    \includegraphics[width=0.31\textwidth]{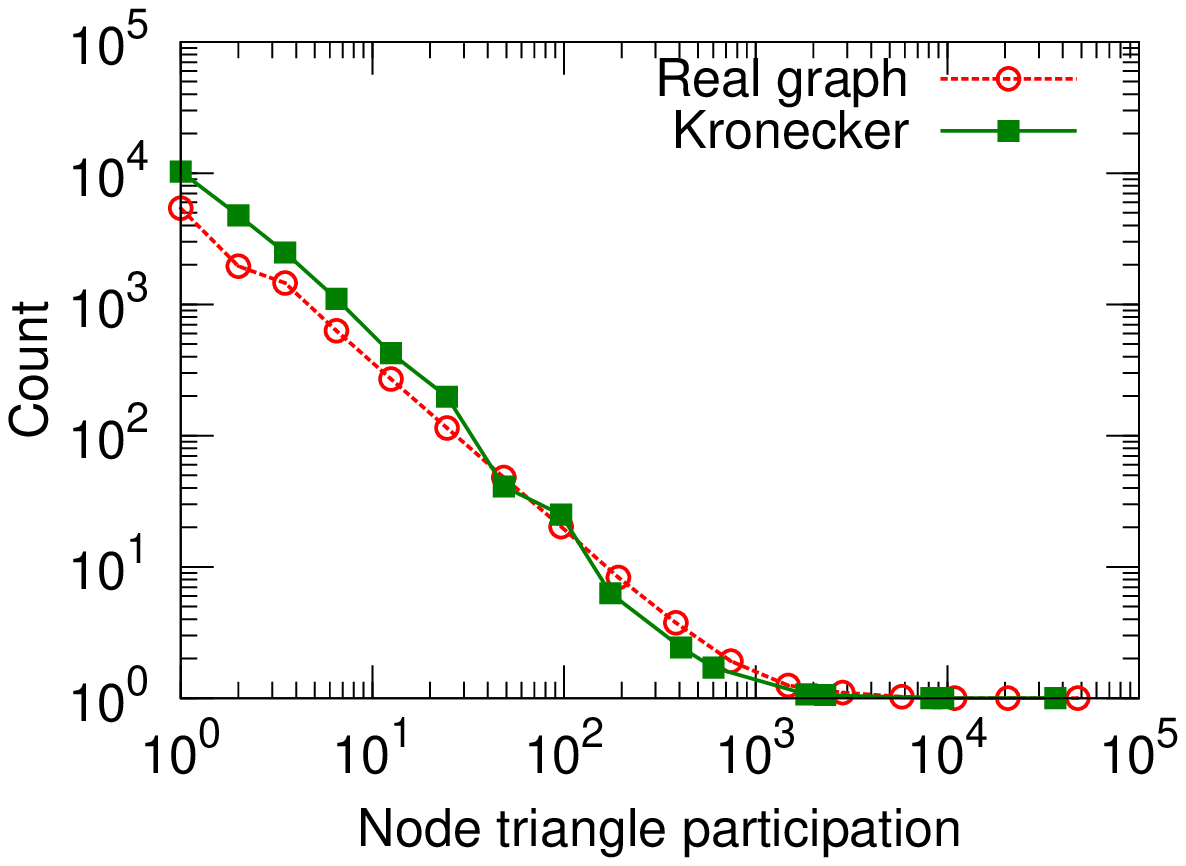} \\
    (a) In-Degree & (b) Out-degree & (c) Triangle participation \\
    \includegraphics[width=0.31\textwidth]{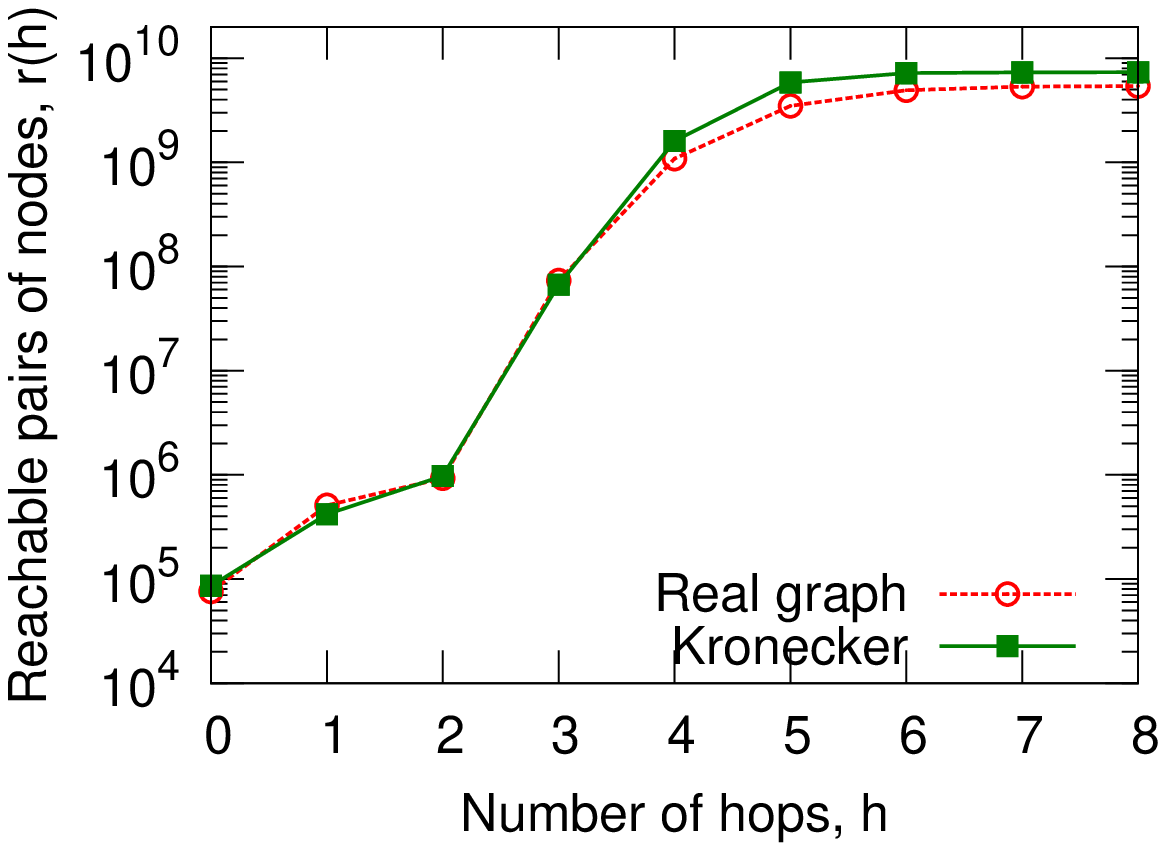} &
    \includegraphics[width=0.31\textwidth]{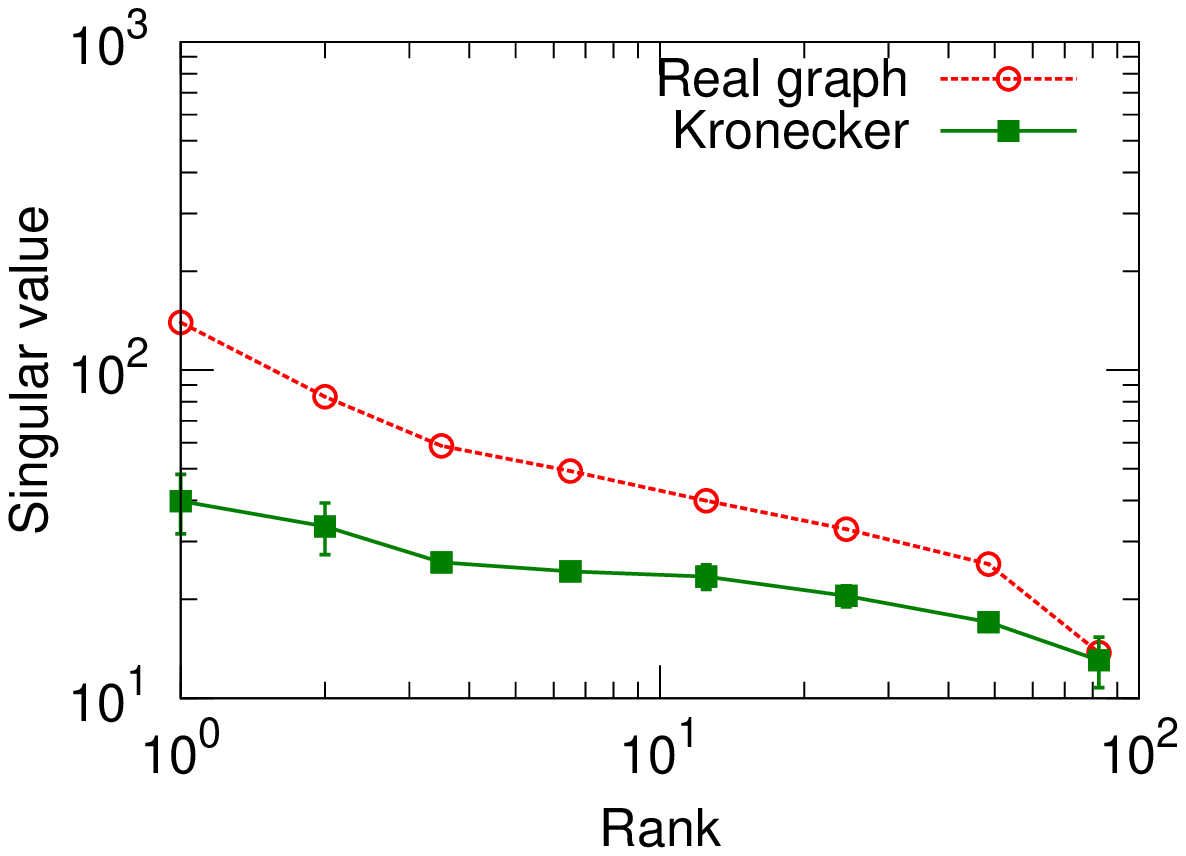} &
    \includegraphics[width=0.31\textwidth]{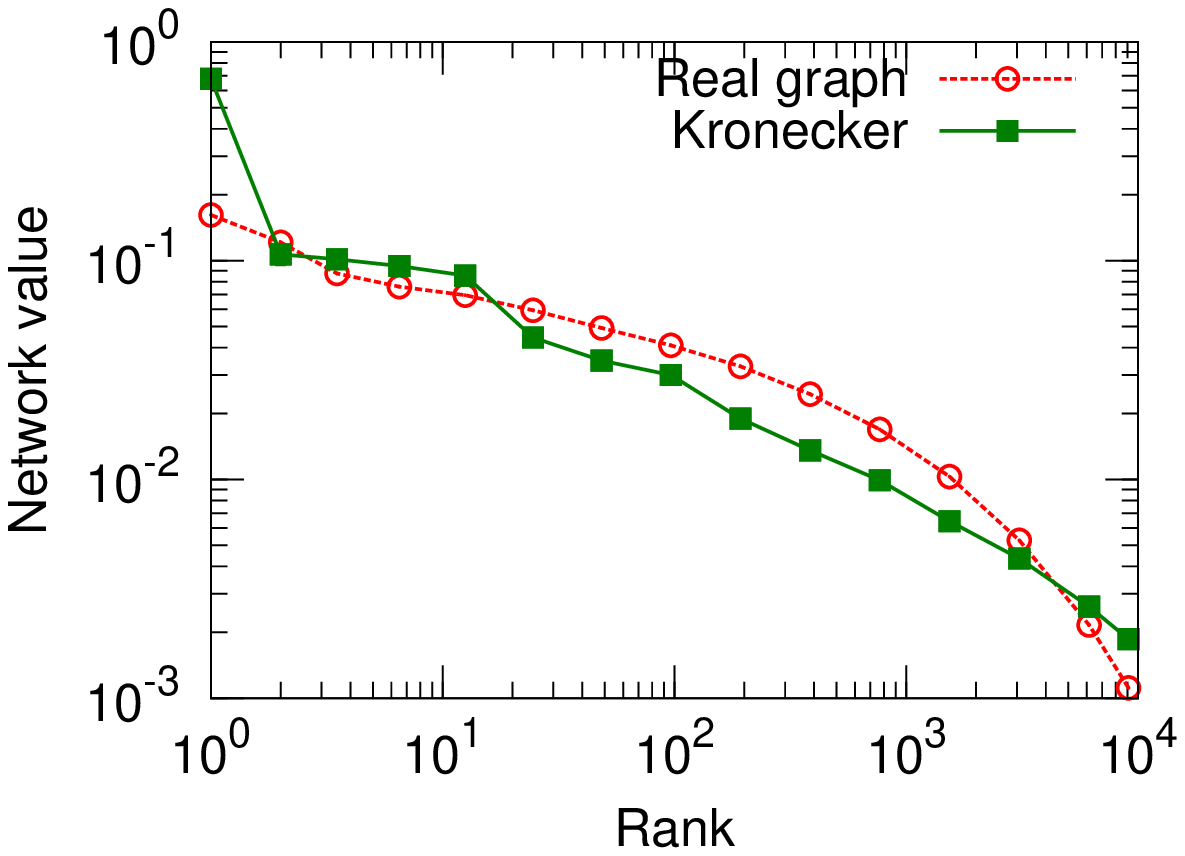} \\
    (d) Hop plot & (e) Scree plot & (f) ``Network'' value \\
  \end{tabular}
  \end{center}
  \caption{{\em \dataset{Epinions} who-trusts-whom social network:}
  Overlayed patterns of real network and the fitted Kronecker graph
  using only 4 parameters ($2\times2$ initiator matrix). Again,
  the synthetic Kronecker graph matches all the properties of the
  real network.}
  \label{fig:KronFitEpinions}
\end{figure}

Last, we present results of fitting \SKRG to 20 large real-world networks:
large online social networks, like \dataset{Epinions}, \dataset{Flickr}
and \dataset{Delicious}, web and blog graphs (\dataset{Web-Notredame},
\dataset{Blog-nat05-6m}, \dataset{Blog-nat06all}), internet and
peer-to-peer networks (\dataset{As-Newman}, \dataset{Gnutella-25},
\dataset{Gnutella-30}), collaboration networks of co-authorships from DBLP
(\dataset{CA-DBLP}) and various areas of physics (\dataset{CA-hep-th},
\dataset{CA-hep-ph}, \dataset{CA-gr-qc}), physics citation networks
(\dataset{Cit-hep-ph}, \dataset{Cit-hep-th}), an email network
(\dataset{Email-Inside}), a protein interaction network
\dataset{Bio-Proteins}, and a bipartite affiliation network
(authors-to-papers, \dataset{AtP-gr-qc}). Refer to
table~\ref{tab:data_StatsDesc} in the appendix for the description and
basic properties of these networks.
They are available for download at \url{http://snap.stanford.edu}.

For each dataset we started gradient descent from a random point (random
initiator matrix) and ran it for 100 steps. At each step we estimate the
likelihood and the gradient based on 510,000 sampled permutations where we
discard first 10,000 samples to allow the chain to burn-in.

Table~\ref{tab:KronFitEst} gives the estimated parameters, the
corresponding log-likelihoods and the wall clock times. All experiments
were carried out on standard desktop computer. Notice that the estimated
initiator matrices $\hat\Theta$ seem to have almost universal structure
with a large value in the top left entry, a very small value at the bottom
right corner and intermediate values in the other two corners. We further
discuss the implications of such structure of Kronecker initiator matrix
on the global network structure in next section.

Last, Figures~\ref{fig:KronFitBlog} and~\ref{fig:KronFitEpinions} show
overlays of various network properties of real and the estimated synthetic
networks. In addition to the network properties we plotted in
Figure~\ref{fig:KronFitKron3by3}, we also separately plot in- and
out-degree distributions (as both networks are directed) and plot the node
triangle participation in panel (c), where we plot the number of triangles
a node participates in versus the number of such nodes. (Again the error
bars show the variance of network properties over different realizations
$R(\hat\Theta^{[k]})$ of a Stochastic Kronecker graph.)

Notice that for both networks and in all cases the properties of the real
network and the synthetic Kronecker coincide really well. Using \SKRG with
just 4 parameters we match the scree plot, degree distributions, triangle
participation, hop plot and network values.

Given the previous experiments from the Autonomous systems graph we only present the
results for the simplest model with initiator size $\nzero=2$. Empirically
we also observe that $\nzero=2$ gives surprisingly good fits and the
estimation procedure is the most robust and converges the fastest. Using
larger initiator matrices $\nzero > 2$ generally helps improve the
likelihood but not dramatically. In terms of matching the network
properties we also gent a slight improvement by making the model more
complex. Figure~\ref{fig:KronLLImprovement} gives the percent improvement
in log-likelihood as we make the model more complex. We use the
log-likelihood of a $2\times2$ model as a baseline and estimate the
log-likelihood at the MLE for larger initiator matrices. Again, models with
more parameters tend to fit better. However, sometimes due to zero-padding
of a graph adjacency matrix they actually have lower log-likelihood (as for
example seen in Table~\ref{tab:KronFitN0}).

\begin{figure}[h]
  \begin{center}
    \includegraphics[width=0.8\textwidth]{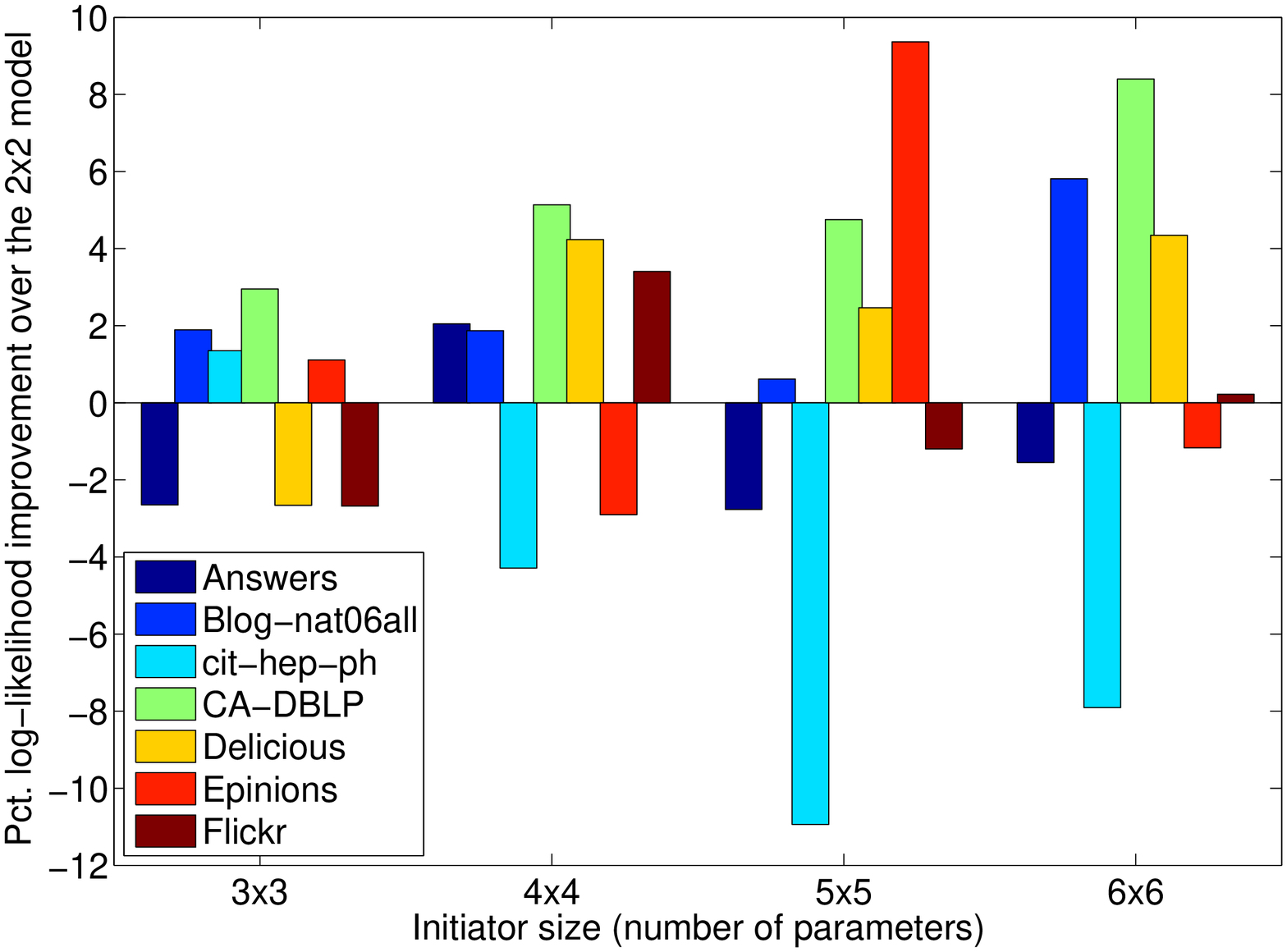}
  \end{center}
  \caption{Percent improvement in log-likelihood over the $2\times 2$ model
  as we increase the model complexity (size of initiator matrix). In general
  larger initiator matrices that have more degrees of freedom help improving
  the fit of the model.}
  \label{fig:KronLLImprovement}
\end{figure}

\subsection{Scalability}

\begin{figure}[t]
  \begin{center}
  \begin{tabular}{cc}
    \includegraphics[width=0.45\textwidth]{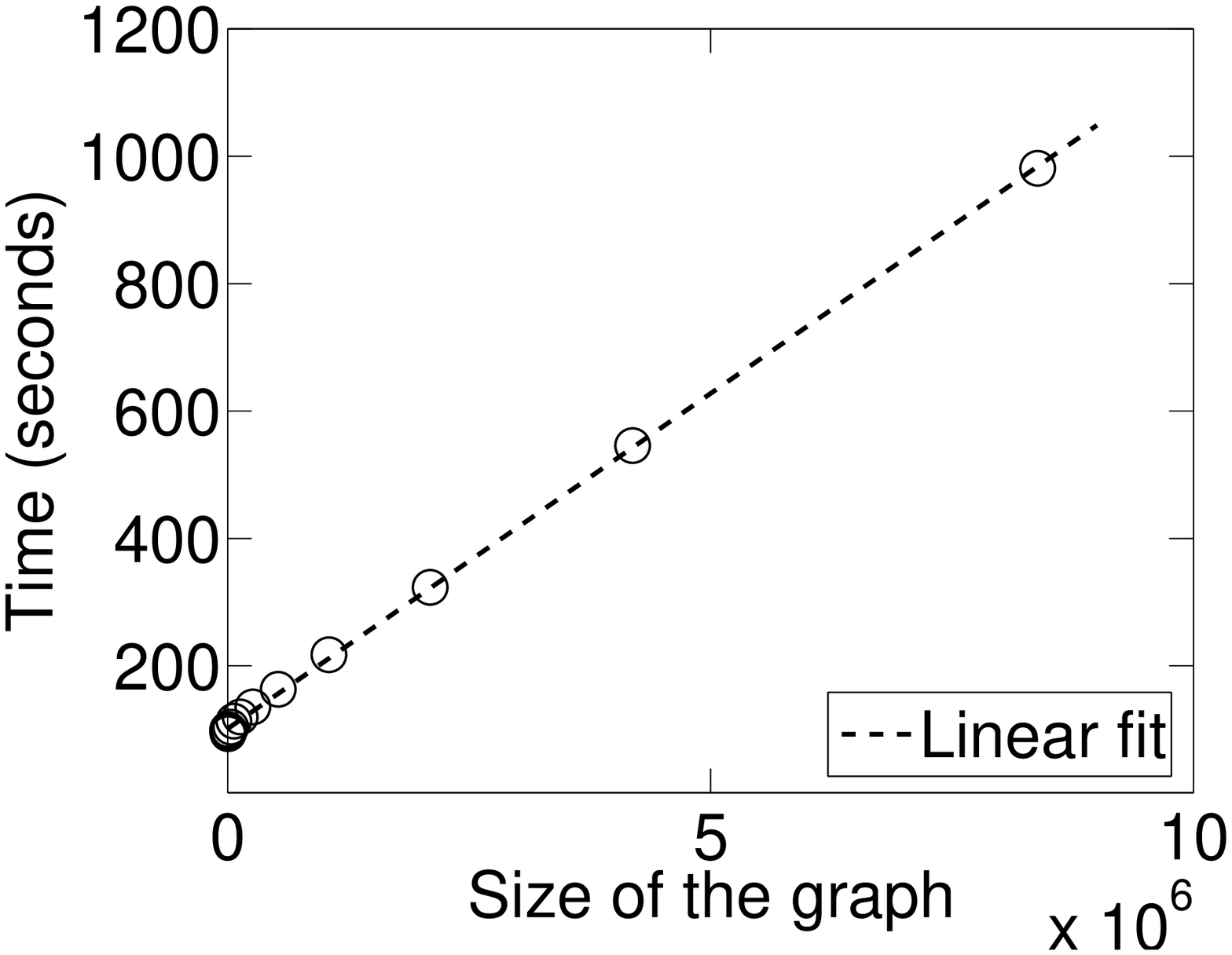}&
    \includegraphics[width=0.45\textwidth]{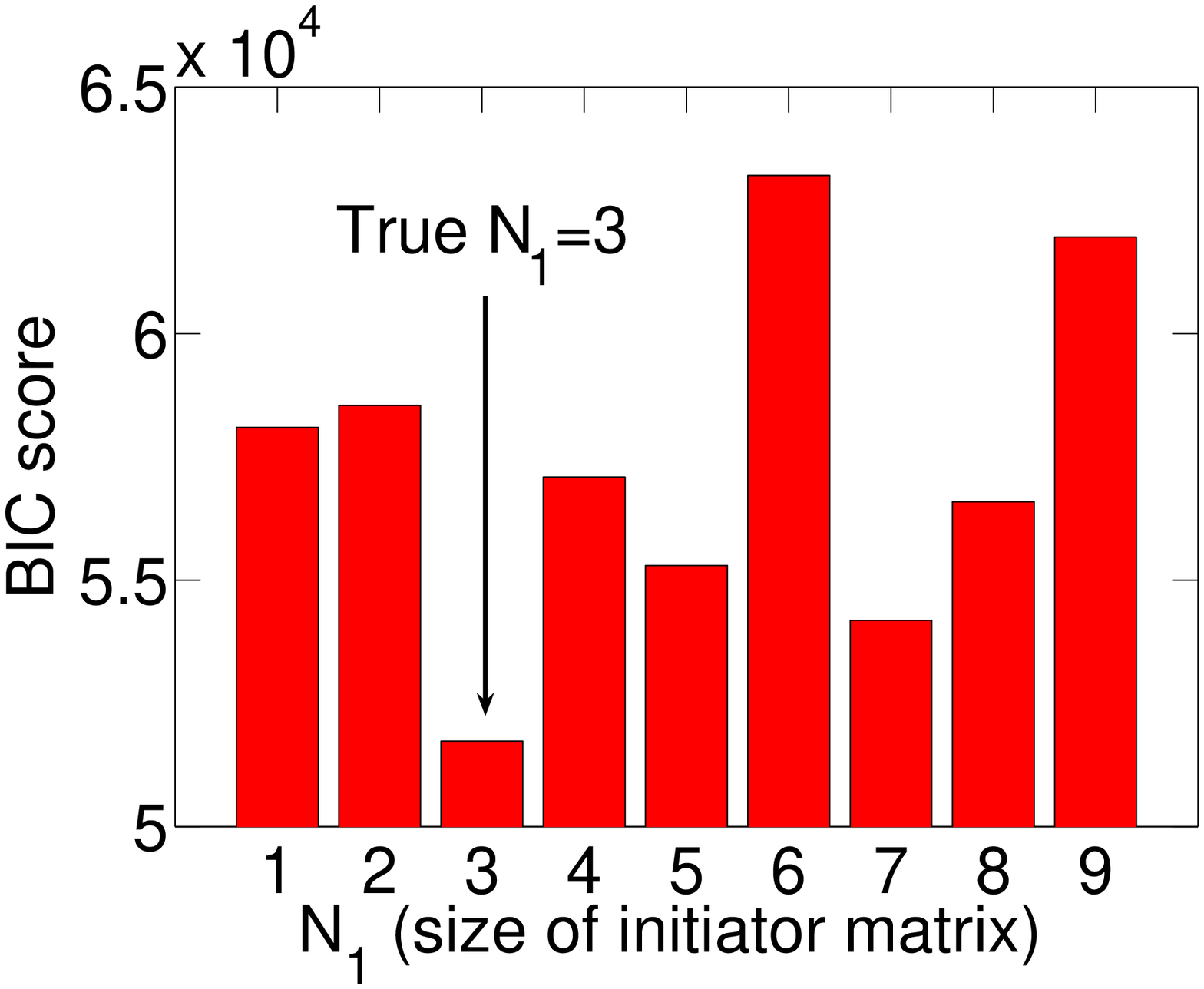} \\
    (a) Scalability & (b) Model selection \\
  \end{tabular}
  \end{center}
  \caption{(a) Processor time to sample 1 million gradients as the graph
  grows. Notice the algorithm scales linearly with the graph size.
  (b) BIC score for model selection.}
  \label{fig:KronTimeBic}
\end{figure}

Last we also empirically evaluate the scalability of the \KronFit. The
experiment confirms that \KronFit\ runtime scales linearly with the number
of edges $\nedges$ in a graph $G$. More precisely, we performed the
following experiment.

We generated a sequence of increasingly larger synthetic graphs on
$\nnodes$ nodes and $8\nnodes$ edges, and measured the time of one
iteration of gradient descent, \emph{i.e.}, sample 1 million permutations
and evaluate the gradients. We started with a graph on 1,000 nodes, and
finished with a graph on 8 million nodes, and 64 million edges.
Figure~\ref{fig:KronTimeBic}(a) shows \KronFit\ scales {\em linearly} with
the size of the network. We plot wall-clock time vs. size of the graph.
The dashed line gives a linear fit to the data points.

\section{Discussion}
\label{sec:KronDiscussion}
Here we discuss several of the desirable properties of the proposed
Kronecker graphs.

{\bf Generality:} Stochastic Kronecker graphs include several other
generators as special cases: For $\thij{ij}=c$, we obtain the classical
Erd\H{o}s-R\'{e}nyi random graph model. For $\thij{i,j} \in \{0,1\}$, we
obtain a deterministic Kronecker graph. Setting the $\kzero$ matrix to a
$2 \times 2$ matrix, we obtain the R-MAT
generator~\cite{chakrabarti04rmat}. In contrast to Kronecker graphs, the
RMAT cannot extrapolate into the future, since it needs to know the number
of edges to insert. Thus, it is incapable of obeying the densification
power law.

{\bf Phase transition phenomena:} The Erd\H{o}s-R\'{e}nyi graphs exhibit
phase transitions~\cite{erdos60random}. Several researchers argue that
real systems are ``at the edge of chaos'' or phase transition~\cite{bak96how,sole00signs}.
Stochastic Kronecker graphs also exhibit phase
transitions~\cite{mahdian07kronecker} for the emergence of the giant
component and another phase transition for connectivity.

{\bf Implications to the structure of the large-real networks:}
Empirically we found that $2\times2$ initiator ($\nzero=2$) fits well the
properties of real-world networks. Moreover, given a $2 \times 2$
initiator matrix, one can look at it as a recursive expansion of two
groups into sub-groups. We introduced this recursive view of Kronecker
graphs back in section~\ref{sec:KronProposed}. So, one can then interpret
the diagonal values of $\Theta$ as the proportion of edges inside each of
the groups, and the off-diagonal values give the fraction of edges
connecting the groups. Figure~\ref{fig:Kron2Cmty} illustrates the setting
for two groups.

For example, as shown in Figure~\ref{fig:Kron2Cmty}, large $a, d$ and
small $b, c$ would imply that the network is composed of hierarchically
nested communities, where there are many edges inside each community and
few edges crossing them~\cite{jure09coreper}.
One could think of this structure as some kind of
organizational or university hierarchy, where one expects the most
friendships between people within same lab, a bit less between people in
the same department, less across different departments, and the least
friendships to be formed across people from different schools of the
university.

\begin{figure}[t]
  \begin{center}
  \begin{tabular}{ccc}
    \raisebox{5mm}{\includegraphics[width=0.15\textwidth]{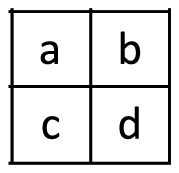}} &
    \includegraphics[width=0.22\textwidth]{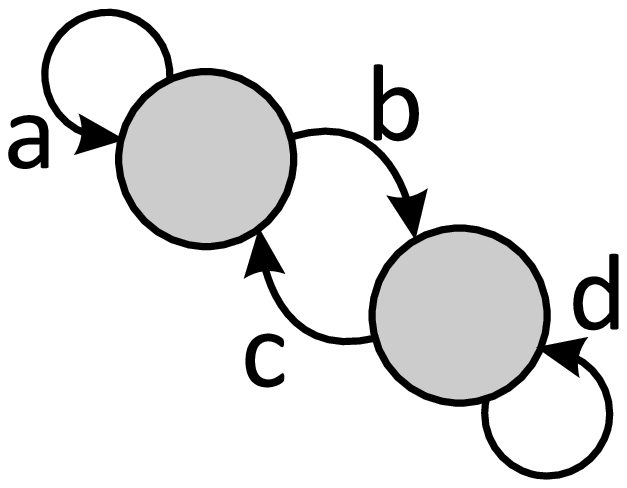} &
    \includegraphics[width=0.35\textwidth]{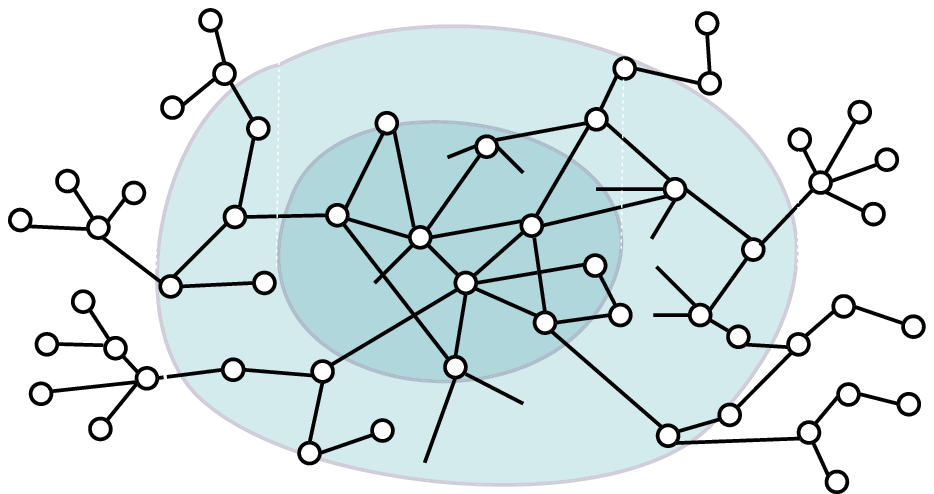} \\
    (a) $2 \times 2$ initiator matrix & (b) Two recursive communities & (c) Core-periphery\\
  \end{tabular}
  \end{center}
  \caption{$2 \times 2$ Kronecker initiator matrix (a) can be thought of
  as two communities where there are $a$ and $d$ edges inside each of the
  communities and $b$ and $c$ edges crossing the two communities as illustrated
  in (b). Each community can then be recursively divided using the same pattern.
  (c) The onion like core-periphery structure where the network gets denser and
  denser as we move towards the center of the network.}
  \label{fig:Kron2Cmty}
\end{figure}

However, parameter estimates for a wide range of networks presented in
Table~\ref{tab:KronFitEst} suggests a very different picture of the
network structure. Notice that for most networks $a \gg b > c \gg d$.
Moreover, $a \approx 1$, $b \approx c \approx 0.6$ and $d \approx 0.2$. We
empirically observed that the same structure of initiator matrix
$\hat\pzero$ also holds when fitting $3 \times 3$ or $4 \times 4$ models.
Always the top left element is the largest and then the values on the
diagonal decay faster than off the diagonal~\cite{jure09coreper}.

This suggests a network structure which is also known as {\em
core-periphery}~\cite{borgatti00core,holme05core}, the {\em
jellyfish}~\cite{tauro01topology,siganos06jellyfish}, or the {\em
octopus}~\cite{chung06octopus} structure of the network as illustrated in
Figure \ref{fig:Kron2Cmty}(c).

All of the above basically say that the network is composed of a densely
linked network core and the periphery. In our case this would imply the
following structure of the initiator matrix. The core is modeled by parameter
$a$ and the periphery by $d$. Most edges are inside the core (large
$a$), and very few between the nodes of periphery (small $d$). Then
there are many more edges between the core and the periphery than inside
the periphery ($b,c > d$)~\cite{jure09coreper}. This is exactly what we see as well.
And in spirit of
Kronecker graphs the structure repeats recursively --- the core again has the
dense core and the periphery, and so on. And similarly the periphery
itself has the core and the periphery.

This suggest an ``onion'' like {\em nested core-periphery}~\cite{jure08ncp,jure08ncp2} network structure as illustrated in Figure
\ref{fig:Kron2Cmty}(c), where the network is composed of denser and denser
layers as one moves towards the center of the network. We also observe
similar structure of the Kronecker initiator when fitting $3 \times 3$ or
$4\times4$ initiator matrix. The diagonal elements have large but
decreasing values with off diagonal elements following same decreasing
pattern.

One of the implications of this is that networks do not break nicely into
hierarchically organized sets of communities that lend themselves
to graph partitioning and community detection algorithms. On contrary, this
suggests that large networks can be decomposed into a densely linked core
with many small periphery pieces hanging off the core. This is in
accordance with our recent results~\cite{jure08ncp,jure08ncp2}, that make
similar observation (but based on a completely different methodology based on
graph partitioning) about the clustering and community structure of large real-world networks.

\section{Conclusion}
\label{sec:KronConclusion}
In conclusion, the main contribution of this work is a family of models of
network structure that uses a non-traditional matrix operation, the {\em
Kronecker product}. The resulting graphs (a) have all the static
properties (heavy-tailed degree distribution, small diameter, etc.), (b)
all the temporal properties (densification, shrinking diameter) that are
found in real networks. And in addition, (c) we can formally prove all of
these properties.

Several of the proofs are extremely simple, thanks to the rich theory of
Kronecker multiplication. We also provide proofs about the diameter and
effective diameter, and we show that Stochastic Kronecker graphs can mimic
real graphs well.

Moreover, we also presented \KronFit, a fast, scalable algorithm to
estimate Stochastic Kronecker initiator, which can be then used to create
a synthetic graph that mimics the properties of a given real network.

In contrast to earlier work, our work has the following novelties: (a) it
is among the few that estimates the parameters of  the chosen generator in
a principled way, (b) it is among the few that has a concrete measure of
goodness of the fit (namely, the likelihood), (c) it avoids the quadratic
complexity of computing the likelihood by exploiting the properties of the
Kronecker graphs, and (d) it avoids the factorial explosion of the node
correspondence problem, by using the Metropolis sampling.

The resulting algorithm matches well all the known properties of real
graphs, as we show with the Epinions graph and the AS graph, it scales
linearly on the number of edges, and it is orders of magnitudes faster
than earlier graph-fitting attempts: 20 minutes on a commodity PC, versus
2 days on a cluster of 50 workstations~\cite{bezakova06mle}.

The benefits of fitting a Kronecker graph model into a real graph are
several:

\begin{itemize}
  \item {\em Extrapolation}: Once we have the Kronecker generator
      $\pzero$ for a given real matrix $\ggraph$ (such that $\ggraph$
      is mimicked by $\pzero^{[k]}$), a larger version of $\ggraph$
      can be generated by $\pzero^{[k+1]}$.
  \item {\em Null-model}: When analyzing a real network $G$ one often
      needs to asses the significance of the observation.
      $\pzero^{[k]}$ that mimics $G$ can be used as an accurate model
      of $G$.
  \item {\em Network structure}: Estimated parameters give insight into
      the global network and community structure of the network.
  \item {\em Forecasting}: As we demonstrated one can obtain $\pzero$
      from a graph $G_t$ at time $t$ such that $\ggraph$ is mimicked
      by $\pzero^{[k]}$. Then $\pzero$ can be used to model the
      structure of $G_{t+x}$ in the future.
  \item {\em Sampling}: Similarly, if we want a realistic sample
      of the real graph, we could use a smaller exponent in the
      Kronecker exponentiation, like $\pzero^{[k-1]}$.
  \item {\em Anonymization}: Since $\pzero^{[k]}$ mimics $\ggraph$, we
      can publish $\pzero^{[k]}$, without revealing information about
      the nodes of the real graph $\ggraph$.
\end{itemize}

\new{Future work could include extensions of Kronecker graphs to
evolving networks. We envision formulating a dynamic Bayesian network with
first order Markov dependencies, where parameter matrix at time $t$ depends
on the graph $G_t$ at current time $t$ and the parameter matrix at time $t-1$.
Given a series of network snapshots one would then aim to estimate
initiator matrices at individual time steps and the parameters of the model
governing the evolution of the initiator matrix. We expect that based on
the evolution of initiator matrix one would gain greater insight in the
evolution of large networks.}

\new{Second direction for future work is to explore connections between Kronecker graphs and Random Dot Product graphs~\cite{young07dotprod,nickel08dotprod}. This also nicely connects with the ``attribute view'' of Kronecker graphs as described in Section~\ref{sec:interpret}. It would be interesting to design methods to estimate the individual node attribute values as well as the attribute-attribute similarity matrix (i.e., the initiator matrix). As for some networks node attributes are already given one could then try to infer ``hidden'' or missing node attribute values and this way gain insight into individual nodes as well as individual edge formations. Moreover, this would be interesting as one could further evaluate how realistic is the ``attribute view'' of Kronecker graphs.}

\new{Last, we also mention possible extensions of Kronecker graphs for modeling weighted and labeled networks. Currently Stochastic Kronecker graphs use a Bernoulli edge generation model, {\em i.e.}, an entry of big matrix $\pmat$ encodes the parameter of a Bernoulli coin. In similar spirit one could consider entries of $\pmat$ to encode parameters of different edge generative processes. For example, to generate networks with weights on edges an entry of $\pmat$ could encode the parameter of an exponential distribution, or in case of labeled networks one could use several initiator matrices in parallel and this way encode parameters of a multinomial distribution over different node
attribute values.}

\vskip 0.2in

\appendix
\section{Table of networks}
\label{sec:appendix}
Table~\ref{tab:data_StatsDesc} lists all the network datasets that were
used in this paper. We also computed some of the structural network
properties. Most of the networks are available for download from
\url{http://snap.stanford.edu}.

\begin{sidewaystable}
\begin{center}
{\footnotesize
\begin{tabular}{l|r|r|r|r|r|r|r|l}
Network & $N$ & $E$ & $N_c$ & $N_c/N$ & $\bar{C}$ & $D$ & $\bar{D}$  & Description \\
\hline \hline
\multicolumn{9}{l}{Social networks} \\
\hline \hline
\dataset{Answers}   & 598,314 & 1,834,200 & 488,484 & 0.82 & 0.11 & 22 & 5.72 & Yahoo! Answers social network~\cite{jure08ncp}  \\
\dataset{Delicious}   & 205,282 & 436,735 & 147,567 & 0.72 & 0.3 & 24 & 6.28 & \url{del.icio.us} social network~\cite{jure08ncp}  \\
\dataset{Email-Inside}  & 986 & 32,128 & 986 & 1.00 & 0.45 & 7 & 2.6 &   European research organization email network~\cite{jure07evolution}  \\
\dataset{Epinions}   & 75,879 & 508,837 & 75,877 & 1.00 & 0.26 & 15 & 4.27 & Who-trusts-whom graph of epinions.com~\cite{richardson03trust}  \\
\dataset{Flickr}  & 584,207 & 3,555,115 & 404,733 & 0.69 & 0.4 & 18 & 5.42 & Flickr photo sharing social network~\cite{kumar06evolution}  \\
\hline \hline
\multicolumn{9}{l}{Information (citation) networks} \\
\hline \hline
\dataset{Blog-nat05-6m}   & 31,600 & 271,377 & 29,150 & 0.92 & 0.24 & 10 & 3.4 &   Blog-to-blog citation network (6 months of data)~\cite{jure07cascades}  \\
\dataset{Blog-nat06all}   & 32,443 & 318,815 & 32,384 & 1.00 & 0.2 & 18 & 3.94 &   Blog-to-blog citation network (1 year of data)~\cite{jure07cascades}  \\
\dataset{Cit-hep-ph}   & 30,567 & 348,721 & 34,401 & 1.13 & 0.3 & 14 & 4.33 &   Citation network of ArXiv {\tt hep-th} papers~\cite{gehrke03kddcup}  \\
\dataset{Cit-hep-th}   & 27,770 & 352,807 & 27,400 & 0.99 & 0.33 & 15 & 4.2 &   Citations network of ArXiv {\tt hep-ph} papers~\cite{gehrke03kddcup}  \\
\hline \hline
\multicolumn{9}{l}{Collaboration networks} \\
\hline \hline
\dataset{CA-DBLP}  & 425,957 & 2,696,489 & 317,080 & 0.74 & 0.73 & 23 & 6.75 &   DBLP co-authorship network~\cite{backstrom06groups}  \\
\dataset{CA-gr-qc}  & 5,242 & 28,980 & 4,158 & 0.79 & 0.66 & 17 & 6.1 &   Co-authorship network in {\tt gr-qc} category  of ArXiv~\cite{jure05dpl}  \\
\dataset{CA-hep-ph}   & 12,008 & 237,010 & 11,204 & 0.93 & 0.69 & 13 & 4.71 &   Co-authorship network in {\tt hep-ph} category of ArXiv~\cite{jure05dpl}  \\
\dataset{CA-hep-th}   & 9,877 & 51,971 & 8,638 & 0.87 & 0.58 & 18 & 5.96 &   Co-authorship network in {\tt hep-th} category of ArXiv~\cite{jure05dpl}  \\
\hline \hline
\multicolumn{9}{l}{Web graphs} \\
\hline \hline
\dataset{Web-Notredame}  & 325,729 & 1,497,134 & 325,729 & 1.00 & 0.47 & 46 & 7.22 &   Web graph of University of Notre Dame~\cite{barabasi99diameter}  \\
\hline \hline
\multicolumn{9}{l}{Internet networks} \\
\hline \hline
\dataset{As-Newman}   & 22,963 & 96,872 & 22,963 & 1.00 & 0.35 & 11 & 3.83 &   AS graph from Newman~\cite{newman07netdata}  \\
\dataset{As-RouteViews}  & 6,474 & 26,467 & 6,474 & 1.00 & 0.4 & 9 & 3.72 &   AS from Oregon Route View~\cite{jure05dpl}  \\
\dataset{Gnutella-25}   & 22,687 & 54,705 & 22,663 & 1.00 & 0.01 & 11 & 5.57 &   Gnutella P2P network on 3/25 2000~\cite{ripeanu02gnutella}  \\
\dataset{Gnutella-30}   & 36,682 & 88,328 & 36,646 & 1.00 & 0.01 & 11 & 5.75 &   Gnutella P2P network on 3/30 2000~\cite{ripeanu02gnutella}  \\
\hline \hline
\multicolumn{9}{l}{Bi-partite networks} \\
\hline \hline
\dataset{AtP-gr-qc}  & 19,177 & 26,169 & 14,832 & 0.77 & 0 & 35 & 11.08 &   Affiliation network of {\tt gr-qc} category in ArXiv~\cite{jure07cascades}  \\
\hline \hline
\multicolumn{9}{l}{Biological networks} \\
\hline \hline
\dataset{Bio-Proteins}  & 4,626 & 29,602 & 4,626 & 1.00 & 0.12 & 12 & 4.24 &   Yeast protein interaction network~\cite{coliza05protein}  \\
\end{tabular}
}
\end{center}
\caption{ Network datasets we analyzed. Statistics of networks we
consider: number of nodes $N$; number of edges $E$, number of nodes in
largest connected component $N_c$, fraction of nodes in largest connected
component $N_c/N$, average clustering coefficient $\bar{C}$; diameter $D$,
and average path length $\bar{D}$. Networks are available for download at {\tt http://snap.stanford.edu}.
} \label{tab:data_StatsDesc}
\end{sidewaystable}

%

\end{document}